\documentclass{cup_ims}

\usepackage{fix-cm}
\usepackage{pgfplots}
\pgfplotsset{compat=1.18}

\makeatletter
\providecommand*{\toclevel@schapter}{0}
\makeatother

\usepackage{makeidx}
\makeindex

\usepackage{subcaption}
\usepackage{natbib}

\usepackage{amsmath,amssymb,bm,graphicx,tikz,colonequals,mathtools}
\usepackage{booktabs,multirow,multicol}
\usepackage{microtype}
\usepackage[utf8]{inputenc}
\usepackage[hidelinks]{hyperref}
\usepackage{comment}

\renewcommand{\Re}{\mathbb{R}}
\newcommand{\N}{\mathcal{N}}

\newcommand{\Trans}{^{\intercal}}

\newcommand{\argmin}{\operatorname*{arg\:min}}

\newcommand{\q}{\quad}
\newcommand{\qq}{\qquad}

\newcommand{\GP}{\mathcal{GP}}

\newcommand{\Id}{I}

\makeatletter
\newcommand{\superimpose}[2]{
 {\ooalign{$#1\@firstoftwo#2$\cr\hfil$#1\@secondoftwo#2$\hfil\cr}}}
\makeatother

\newcommand{\ssf}{F}

\newcommand{\ssg}{G}
\newcommand{\ff}{f^*}
\newcommand{\ssy}{Y}

\newcommand{\cH}{\mathcal{H}}
\newcommand{\cF}{\mathcal{F}}
\newcommand{\cG}{\mathcal{G}}
\newcommand{\cX}{\mathcal{X}}
\newcommand{\cY}{\mathcal{Y}}
\newcommand{\cS}{\mathcal{S}}
\newcommand{\cU}{\mathcal{U}}

\newcommand{\bE}{\mathbb{E}}

\newcommand{\iid}{\overset{i.i.d.}{\sim}}

\newcommand{\hrarrow}{\hookrightarrow}

\newcommand{\iu}{\mathrm{i}\mkern1mu}

\usetikzlibrary{arrows,shapes,plotmarks,decorations.pathmorphing}
\usetikzlibrary{backgrounds,calc,positioning,fadings}

\tikzset{>=stealth'}
\tikzstyle{graphnode} =
  [circle,draw=black,minimum size=22pt,text centered,text
    width=22pt,inner sep=0pt]
\tikzstyle{var}   =[graphnode,fill=white]
\tikzstyle{obs}   =[graphnode,fill=black,text=white]
\tikzstyle{act}   =[rectangle,draw=black,text=white,minimum
size=22pt,text centered, text width=22pt,inner sep=0pt]
\tikzstyle{fac}   =[rectangle,draw=black,fill=black!25,minimum size=5pt]
\tikzstyle{facprior} =[rectangle,draw=black,fill=black,text=white,minimum size=5pt]
\tikzstyle{edge}  =[draw=white,double=black,thick,-]
\tikzstyle{prior} =[rectangle, draw=black, fill=black, minimum size=
5pt, inner sep=0pt]
\tikzstyle{dirprior} = [circle, draw=black, fill=black, minimum
size=5pt, inner sep=0pt]

\usepackage{amsthm}
\usepackage{newtxtext}
\usepackage{newtxmath}

\theoremstyle{plain}
\newtheorem{theorem}{Theorem}[chapter]
\newtheorem{lemma}[theorem]{Lemma}
\newtheorem{proposition}[theorem]{Proposition}
\newtheorem{corollary}[theorem]{Corollary}

\theoremstyle{definition}
\newtheorem{definition}[theorem]{Definition}
\newtheorem{example}[theorem]{Example}

\theoremstyle{remark}
\newtheorem{remark}[theorem]{Remark}

\newcommand{\commentAlt}[2]{}
\newcommand{\commentLongAlt}[2]{}

\title{Gaussian Processes and\\ Reproducing Kernel Hilbert Spaces:\\ Connections and Equivalences}

\author{
Motonobu Kanagawa\\
EURECOM\\
\texttt{motonobu.kanagawa@eurecom.fr}\\[1em]
Philipp Hennig\\
University of Tübingen, Tübingen AI Center\\
Max Planck Institute for Intelligent Systems\\
\texttt{philipp.hennig@uni-tuebingen.de}\\[1em]
Dino Sejdinovic\\
Adelaide University\\
Australian Institute for Machine Learning\\
\texttt{dino.sejdinovic@adelaide.edu.au}\\[1em]
Bharath K. Sriperumbudur\\
Pennsylvania State University\\
\texttt{bks18@psu.edu}
}

\begin{document}

\frontmatter
\maketitle

\begin{abstract}
This monograph studies the relations between two approaches using positive definite kernels: probabilistic methods using Gaussian processes, and non-probabilistic methods using reproducing kernel Hilbert spaces (RKHS).
They are widely studied and used in machine learning, statistics, and numerical analysis.
We study connections and equivalences for fundamental topics such as regression, interpolation, numerical integration, distributional discrepancies, and statistical dependence, as well as sample path properties of Gaussian processes.
A unifying perspective for these equivalences is established, based on the equivalence between the Gaussian Hilbert space and the RKHS.
The monograph serves as a basis to bridge many other methods based on Gaussian processes and reproducing kernels, which are developed in parallel by the two research communities.
\end{abstract}

\tableofcontents

\mainmatter

\chapter*{Acknowledgments}

We thank the many friends and colleagues whose comments and discussions helped improve this monograph. We also thank the anonymous reviewer of an earlier, shorter version, whose suggestion to establish the generic equivalence result in Theorem~\ref{theo:master-theorem} motivated the substantial revision that led to the present monograph.

We thank Natalie Tomlinson and Alice Murray at Cambridge University Press for their guidance and support throughout the preparation of this monograph.

The original idea for this monograph arose during Workshop~16481 at the Leibniz Centre for Informatics, Schloss Dagstuhl. We thank the Centre for its hospitality and support.

MK and PH acknowledge support from the European Research Council through StG Project PANAMA. PH also acknowledges support through ERC CoG Project ANUBIS (101123955).

PH is a member of the Machine Learning Cluster of Excellence, funded by the Deutsche Forschungsgemeinschaft (DFG, German Research Foundation) under Germany's Excellence Strategy---EXC number 2064/1---Project number 390727645. He also gratefully acknowledges support from the Ministry of Science, Research and Arts of the State of Baden-Württemberg.

BKS acknowledges support from the National Science Foundation (NSF) through NSF-DMS-1713011 and NSF CAREER Award DMS-1945396. DS acknowledges support from The Alan Turing Institute (EP/N510129/1) and from the Responsible AI Research Centre (RAIR) at Adelaide University.

MK thanks his family, especially his wife, Fangyuan Zhang, for her support throughout the writing of this monograph. He dedicates this monograph to his mother, Hiroko Kanagawa, and to the memory of his late father, Hiroshi Kanagawa.

All authors except the first are listed in alphabetical order.

\chapter{Introduction}
\label{sec:introduction}

Positive definite kernels, or simply {\em kernels}, appear in two widely used approaches in machine learning, statistics, and numerical analysis.
One is the probabilistic approach using a {\em Gaussian process} (GP), whose {\em covariance function} is a positive definite kernel.
An unknown function under investigation is inferred from data by computing the conditional distribution of the function given the data, assuming that the function is a sample of a specified GP prior~\citep[e.g.,][]{RasmussenWilliams}.
Such methods will be referred to as GP-based methods.
The other is the non-probabilistic approach based on a {\em reproducing kernel Hilbert space} (RKHS), whose {\em reproducing kernel} is a positive definite kernel.
An unknown function is estimated by searching within a specified RKHS for a smooth function that fits the data well~\citep[e.g.,][]{scholkopf2002learning}.
Such methods are called {\em kernel methods} in machine learning, and will be referred to here as RKHS-based methods.
These two approaches offer a principled way of incorporating a prior assumption or knowledge about an unknown function through the ``hypothesis space'' defined by a kernel, i.e., the GP or RKHS.

Since both approaches use a positive definite kernel, it is natural to ask whether there is any connection between them.
Equivalences between the two approaches are known for particular problems.
In regression, in which an unknown function is estimated from its noisy observations, a well-known equivalence holds between the two approaches~\citep{kimeldorf1970correspondence}: the posterior mean function in {\em Gaussian process regression} (aka Kriging or Wiener--Kolmogorov prediction) is identical to the estimator of {\em kernel ridge regression} (aka regularized least-squares or spline smoothing).
Another example is numerical integration, where the task is to compute an integral of a function by deciding where to evaluate the function and using the evaluated function values as data.
An equivalence holds for the GP-based and RKHS-based methods for this problem~\citep{Rit00,huszar2012optimally}.

This monograph presents the connections and equivalences between GP-based and RKHS-based methods in a unified manner.
The motivations and main contents can be summarized as follows.
\begin{itemize}
    \item To our knowledge, the literature lacks a unifying account of why equivalences hold for the GP-based and RKHS-based methods. We show that an equivalence holds because both approaches estimate a quantity of interest (e.g., the function value at a test point) as the {\em projection} of that quantity onto the data subspace in a Hilbert space, which is the {\em Gaussian Hilbert space} (GHS) for a GP-based method and the RKHS for an RKHS-based method.
    Since there is an isometry between GHS and RKHS, as described in Chapter~\ref{sec:definition}, the equivalence between the two methods holds.
    This is shown in Chapters \ref{sec:interpolation}, \ref{sec:regression}, and \ref{sec:integral_transforms} on various tasks, such as interpolation, regression, and integration.

    \item
    A more fundamental reason for the equivalences is that a linear functional of a GP, such as its integrals and function values, is {\em defined} as the GHS element corresponding to the linear functional's Riesz representation in the RKHS, as described in Chapter~\ref{sec:fund-equiv-gp-RKHS}.

    \item A GP-based method performs the {\em uncertainty quantification} of the quantity of interest by computing its {\em posterior standard deviation} given data. This is possible because the unknown function, on which the quantity of interest is defined, is modeled as a stochastic process.
    Whether similar uncertainty quantification is possible with RKHS-based methods, where the unknown function is modeled deterministically, is not well documented in the literature.
    We describe that, from the above Hilbert space viewpoint, the posterior standard deviation is identical to {\em the distance between the quantity of interest and its projection onto the data subspace of the RKHS}.
    Moreover, this identity leads to the interpretation of the posterior standard deviation as the {\em worst-case error} in estimating the quantity of interest when the unknown function has the unit RKHS norm.
    These interpretations equip RKHS-based methods with the added functionality of uncertainty quantification, as shown in Chapters~\ref{sec:interpolation}, \ref {sec:regression}, and \ref{sec:integral_transforms}.

    \item The properties of GP samples and RKHS functions are closely related but have subtle differences.
    The most well-known fact would be that a GP sample does {\em not} belong to the RKHS of the GP covariance kernel with {\em probability one}, if the RKHS is infinite-dimensional \citep[e.g.,][Corollary 7.1]{LukBed01}.
    This fact may provide the impression that the modeling assumptions of the GP-based and RKHS-based methods are fundamentally different and their equivalences are ``coincidental.''
    To resolve this potential misunderstanding, we review a result showing that, under appropriate conditions, a GP sample path belongs almost surely to an RKHS ``slightly larger'' than the RKHS of its covariance kernel \citep{steinwart2019convergence}.
    Thus, assuming that an unknown function is a GP sample corresponds to assuming that the unknown function can be well approximated by functions in the RKHS of the covariance kernel in an appropriate sense.
    We demonstrate this quantitatively by showing that the optimal convergence rate for the posterior mean function in GP regression~\citep{VarZan11} can be recovered by considering the corresponding setting for kernel ridge regression~\citep{fischer2020sobolev}.
    Therefore, there is no contradiction in the modeling assumptions of the two approaches.
    These points are discussed in Chapters~\ref{sec:theory} and \ref{sec:regression}.

\end{itemize}

The monograph has two main purposes.
One is to help researchers working on either GP-based or RKHS-based methods understand methods in the other field and seamlessly transfer the results between the two fields.
In fact, there are concepts or methods studied almost exclusively in one community but not in the other, such as uncertainty quantification in the GP community and distance metrics on probability distributions in the RKHS community.
Bridging the two communities would help explore new research directions.
The second purpose is to offer a pedagogical introduction to researchers and students interested in either or both of the two fields.

\section{Organization of the Monograph}

The rest of this monograph is organized as follows.

\paragraph{Chapter~\ref{sec:definition}: Gaussian Processes and RKHSs: Preliminaries}
Definitions and basic properties of GPs, RKHSs, and Gaussian Hilbert spaces are provided with examples.

\paragraph{Chapter~\ref{sec:theory}: Sample Paths of Gaussian Processes and RKHSs}
We review the properties of a GP sample characterized by the RKHS of the GP covariance kernel.
We describe the necessary and sufficient condition for a GP sample to belong to a given RKHS (which can be different from the RKHS of the GP kernel) \citep{driscoll1973reproducing,LukBed01}.
This condition essentially states that, to contain a GP sample, a given RKHS should be ``slightly larger'' than the RKHS of the GP kernel, so that the former contains slightly less smooth functions than the latter \citep{steinwart2019convergence}.
Consequences for GPs of the Mat\'ern and Gaussian (squared-exponential) kernels are provided.

\paragraph{Chapter~\ref{sec:fund-equiv-gp-RKHS}: Linear Functionals of Gaussian Processes and RKHS Representations}

Linear functionals of a GP, such as its function values, integrals, and derivatives, provide information about the GP and are thus fundamental in defining GP-based algorithms.
We describe that a linear functional of a GP is {\em defined} as the element in the Gaussian Hilbert space corresponding to the Riesz representation of that functional in the RKHS.
Therefore, the necessary and sufficient condition for a linear functional to be well-defined for a GP is that the linear functional is continuous on the RKHS of the GP kernel.
Examples are described for function values, integrals, derivatives, and (Paley--Wiener and It{\=o}) stochastic integrals.

A fundamental equivalence between the GP-based and RKHS-based methods is here established: {\em the root mean square of a linear functional of a GP is identical to the maximum absolute value of the linear functional over unit-norm RKHS functions}. This result underlies many specific equivalences between GP- and RKHS-based methods.

\paragraph{Chapter~\ref{sec:interpolation}: Interpolation}
This chapter discusses {\em interpolation}, the problem of estimating an unknown function using its observed values at finite locations.
This problem is important in its own right and also is a foundation for the {\em regression} problem in Chapter~\ref{sec:regression} where observations are noisy.

The GP-based method for interpolation, here referred to as {\em GP interpolation}, models the unknown function as a GP sample and computes the conditional distribution of the function given its observations.
This conditional, or {\em posterior}, distribution is another GP whose mean and covariance functions are respectively used for the estimation and uncertainty quantification of the unknown function.
On the other hand, the RKHS-based method for interpolation, referred to as {\em RKHS interpolation}, estimates the unknown function by the smoothest RKHS function interpolating the observations.
This chapter discusses how these two methods are related.

As is well known, the posterior mean function in GP interpolation is identical to the RKHS interpolant.
This identity holds because either method estimates the unknown function value at a test location by {\em projecting} that unknown function value onto the {\em data subspace in a Hilbert space}: this projection is the {\em best approximation} of the unknown function value based on a linear combination of the function values at training locations.
This Hilbert space is the Gaussian Hilbert space for GP interpolation, and the RKHS for RKHS interpolation, and the identity of the two estimators follows from the isometry between the two Hilbert spaces.

The uncertainty of the unknown function value at a test location is quantified using the {\em posterior standard deviation} in GP interpolation.
It is shown to be identical to the {\em distance} in the Gaussian Hilbert space between the unknown function value and its projection onto the data subspace.
Therefore, the posterior standard deviation is identical to the corresponding RKHS distance.
This identity enables interpreting the posterior standard deviation as the {\em worst-case error} of RKHS interpolation when the unknown function has the unit RKHS norm.
Thus, RKHS interpolation also has the functionality of uncertainty quantification, thanks to the equivalence with GP interpolation.

The posterior standard deviation contracts to zero as the test location is surrounded by more data points. A result of the rate of this contraction is reviewed in terms of the kernel's smoothness.

\paragraph{Chapter~\ref{sec:regression}: Regression}
Regression is the problem of estimating an unknown function from \emph{noisy} observations of its function values, and is fundamental in machine learning and statistics.

The GP-based method for regression is {\em Gaussian Process Regression} (GPR), also known as {\em Kriging}. As for GP interpolation, it models the unknown function as a GP sample and computes its posterior distribution given the observations.
The corresponding RKHS method is {\em Kernel Ridge Regression} (KRR), also known as {\em spline smoothing} for a particular case and as {\em regularized least squares} for a generic case.
KRR estimates the unknown function by searching for a smooth RKHS function with a small mean square error for the observations, where the smoothness is determined by the regularization constant.
This chapter discusses the relations between these two methods.

While regression is usually interpreted as a generalization of interpolation, the opposite interpretation also holds for GPR and KRR: they are particular cases of GP and RKHS interpolations using the {\em regularized} kernel, defined as the original kernel plus the covariance kernel of the noise process.
This interpretation enables applying the equivalence results for interpolation in Chapter~\ref{sec:interpolation} to GPR and KRR.
In particular, it explains the origin of the well-known equivalence between the KRR estimator and the posterior mean function of GPR.

Moreover, the GP posterior variance of the unknown function at a test location, plus the noise variance, is shown to be identical to the squared worst-case error of KRR when the unknown function has unit norm in the regularized kernel's RKHS.
Thus, the GP posterior variance has an RKHS interpretation, equipping KRR with uncertainty quantification functionality.

Lastly, the convergence rates of GPR and KRR to the unknown function are reviewed and compared.
The unknown function is assumed to have the same finite smoothness as GP samples in GPR.
The GP posterior mean function then converges to the unknown function at the minimax optimal rate for increasing sample sizes~\citep{VarZan11}.
The same assumption implies that the RKHS consists of functions at least $d/2$-smoother than the unknown function (where $d$ is the input dimensionality), which may sound like a contradiction. However, under this assumption, the KRR estimator is shown to converge to the unknown function at the {\em same} optimal rate if {\em the regularization constant is set as the noise variance divided by the sample size} \citep{fischer2020sobolev}.
This last condition is identical to that required for the equivalences between GPR and KRR.
In this sense, there is no contradiction between the convergence results for GPR and KRR.

\paragraph{Chapter~\ref{sec:integral_transforms}: Comparison of Probability Distributions}
This chapter provides GP interpretations for RKHS-based methods involving the comparison of probability distributions.
These are mainly studied in the RKHS community, and their GP interpretations are not frequently discussed in the literature.

 {\em Maximum Mean Discrepancy} (MMD)~\citep{GreBorRasSchetal12}, a popular distance metric on probability distributions, quantifies the discrepancy between two probability distributions as the {\em maximum difference} between the means of a function under two distributions when {\em the function has the unit RKHS norm}.
MMD is shown to be identical to the {\em root mean squared difference} between the means of a function under the two distributions when {\em the function is a GP sample}~\citep[Corollary 7, p.~40]{Rit00}.

Similarly, the {\em Hilbert--Schmidt Independence Criterion} (HSIC)~\citep{GreBouSmoSch05}, a widely used measure of the dependence between two random variables, is identical to the {\em expected squared covariance} between the real-valued transformations of the two random variables when the {\em transformation is a GP sample} \citep{SzeRiz09,SejSriGreFuk13}.
Thus, MMD and HSIC have natural GP formulations.

 {\em Sampling} from a probability distribution is to generate a finite sample approximation to that distribution. Sampling methods whose approximation quality is measured by MMD include {\em Quasi Monte Carlo} \citep{Hic98,DicKuoSlo13}, {\em kernel herding} \citep{CheWelSmo10}, {\em kernel thinning} \citep{DwiMac24} and {\em MMD gradient flows} \citep{ArbKorSalGre19}, which are also methods for {\em numerical integration} where the function to be integrated is assumed to be an RKHS function.
When MMD is replaced by its GP formulation, these methods can be naturally interpreted as probabilistic methods known as  {\em Bayesian quadrature} that model the integrand as a GP sample \citep{Oha91,briol2019probabilistic,MahsereciKarvonen2026}.

\section{Related Literature}

We collect a few key related books or articles on GP-based and RKHS-based methods that may be helpful for further reading.

This monograph is a revised and extended version of a shorter technical report by \citet{kanagawa2018gaussian}.
Many results are newly presented in the current monograph, such as the systematic presentation of the relations between GP-based and RKHS-based methods from the Hilbert space viewpoint.

\citet{Berlinet2004} collects classic results on using RKHSs in statistics and probability, such as the earliest contributions by \citet{kolmogorov1941interpolation,Par61,matheron1962traite,kimeldorf1970correspondence,Kallianpur1970RKHS,Lar72}.
Other related books and articles can be found in:
\begin{itemize}
    \item {\bf Machine learning:} Learning with GPs~\citep{RasmussenWilliams} and RKHSs~\citep{scholkopf2002learning}; RKHS representations of probability distributions~\citep{MuaFukSriSch17}.
    \item {\bf Statistics:} Splines~\citep{wahba1990spline}; random fields~\citep{Ad90,AdlJon07}; spatial statistics~\citep{Ste99}.
    \item {\bf Statistical learning theory:} RKHS-based learning~\citep{CuckerZhou2007,Steinwart2008}; Bayesian nonparametric inference with GPs~\citep{VarZan08,ghosal2017fundamentals}.
    \item {\bf Numerical methods:} Theory for numerical algorithms with GPs~\citep{Rit00} and RKHSs~\citep{NovWoz08,NovWoz10}; RKHS-based interpolation~\citep{Wen05,SchWen06,SchSchSch13}; probabilistic numerics~\citep{HenOsbGirRSPA2015,hennig2022probabilistic}; multiresolution probabilistic numerical solvers~\citep{OwhadiScovel2019}.
   \item {\bf Applied mathematics:} Gaussian Hilbert spaces~\citep{Janson1997}; Gaussian measures on Banach spaces~\citep{Bog98}; stochastic differential equations~\citep{da2014stochastic,SarkkaSolin2019}; Bayesian inverse problems~\citep{Stu10}.
\end{itemize}

\section{Notation and
Definitions}
\label{sec:notation}
We collect the notation and definitions used throughout the monograph.

\paragraph{Notation system}
Unless otherwise specified,
\begin{itemize}
    \item {\bf Random} variables are denoted by uppercase letters (e.g., input vector $X$; output $Y$; function $F$);
    \item {\bf Deterministic} quantities by lowercase letters (e.g., input vector $x$; output $y$; function $f$);
    \item {\bf Vectors and matrices} by bold fonts (e.g., vector of outputs ${\bm y}$; vector of weights ${\bm w}$; kernel matrix ${\bm K}$)
    \item {\bf Sets and spaces} by calligraphic letters (e.g., input space $\cX$; output space $\cY$; Hilbert space $\mathcal{H}$).
\end{itemize}
Where defined explicitly, the notation rule has exceptions, such as a probability distribution denoted by a capital letter $P$, etc.

\paragraph{Basics}
Let $\mathbb{N}$ be the set of natural numbers, $\mathbb{N}_0 := \mathbb{N} \cup \left\{ 0 \right\}$ and $\mathbb{N}_0^d$ be the $d$-dimensional Cartesian product of $\mathbb{N}_0$ with $d \in \mathbb{N}$.
For a multi-index $\alpha := (\alpha_1,\dots,\alpha_d)^\top \in \mathbb{N}_0^d$, define $| \alpha | := \sum_{i=1}^d \alpha_i$.

The real line is denoted by $\Re$, the $d$-dimensional Euclidean space for $d \in \mathbb{N}$ by $\Re^d$, and  the Euclidean norm by $\| \cdot \|$.

The complex plane is denoted by $\mathbb{C}$.
For $z \in \mathbb{C}$, let $\bar{z}$ be its complex conjugate.

For $\alpha \in \mathbb{N}_0^d$ and a function $f$ defined on $\Re^d$, let $\partial^\alpha f$ and $D^\alpha f$ be the $\alpha$-th partial derivative and the $\alpha$-th partial weak derivative, respectively.

For $i,j \in \mathbb{N}$, we define $\delta_{ij} \in \{0, 1\}$ as $\delta_{ij} = 1$ if $i = j$, and $\delta_{ij} = 0$ otherwise.

For a matrix (or a vector) ${\bm M}$, its transpose is denoted by ${\bm M}^\top$.

\paragraph{Probability}
A random variable $X$ with probability distribution $P$ is denoted by $X \sim P$.
The expectation with respect to $X$ may be denoted by $\bE_{X}[\cdot]$, $\bE_{X \sim P}[\cdot]$, or $\bE[\cdot]$, depending on the context.

\paragraph{Function spaces}
The set of continuous functions on a topological space $\cX$ is denoted by $C(\cX)$.
For a measure $\nu$ on a measurable space $\cX$ and a constant $1 \leq p < \infty$,  $L_p(\nu)$ denotes the Banach space of ($\nu$-equivalence classes of) $p$-integrable functions with respect to $\nu$:
\begin{equation} \label{eq:lp-space}
L_p(\nu) := \left\{f:\cX \to \Re: \| f \|_{L_p(\nu)}^p := \int |f(x)|^p d\nu(x) < \infty \right\}.
\end{equation}
Write $L_p(\mathcal{X}) := L_p(\nu)$ if $\nu$ is the Lebesgue measure on $\cX \subset \Re^d$.

\chapter{Gaussian Processes and RKHSs: Preliminaries}\label{sec:definition}

As preliminaries, this chapter describes the definitions and basic properties of Gaussian processes (GP), reproducing kernel Hilbert spaces (RKHS), and Gaussian Hilbert spaces (GHS).
We introduce positive definite kernels in Section \ref{sec:pd-kernel}, GPs in Section~\ref{sec:GP-intro}, RKHSs in Section~\ref{sec:RKHS-def}, and GHSs in Section~\ref{sec:GHS}.
The equivalence between a GHS and the corresponding RKHS is explained in Section~\ref{sec:canonical-isomet-GHS-RKHS}.
The description of the RKHS of a shift-invariant kernel, such as Gaussian and Mat\'ern kernels, in terms of the kernel's spectral density is provided in Section~\ref{sec:fourier_RKHS}.

The characterizations of GPs and RKHSs based on orthogonal basis functions are provided in Chapter~\ref{sec:theory}.

\index{GP|see{Gaussian process}}
\index{RKHS|see{reproducing kernel Hilbert space}}
\index{GHS|see{Gaussian Hilbert space}}
\index{Gram matrix|see{kernel matrix}}
\index{Gaussian kernel|see{squared-exponential kernel}}
\index{RBF kernel|see{squared-exponential kernel}}
\index{stationary kernel|see{shift-invariant kernel}}
\section{Positive Definite Kernels} \label{sec:pd-kernel}

\index{positive definite kernel}
A {\em positive definite kernel} is intuitively a function of two input points that outputs their similarity. It is a basis for both GP-based and RKHS-based methods because it appears as the covariance function of a GP and the reproducing kernel of an RKHS.

Precisely, a symmetric function $k:\cX\times\cX \to \Re$ defined on a nonempty set $\cX$ is called a positive definite kernel if it satisfies the inequality
\begin{equation} \label{eq:PD-kernel-def}
    \sum_{i=1}^n \sum_{j=1}^n  c_i c_j k(x_i,x_j) \geq 0
\end{equation}
for any number $n \in \mathbb{N}$, coefficients $c_1,\ldots,c_n \in \Re$ and input points $x_1,\dots,x_n \in \cX$.
This condition is equivalent to the symmetric matrix $$
{\bm K}_n = (k(x_i,x_j)) \in\Re^{n\times n},
$$
whose ($i$,$j$)-th element is the value $k(x_i, x_j)$ for the $i$-th and $j$-th input points $x_i, x_j$, being {\em positive semidefinite} for any number $n \in \mathbb{N}$ and any input points $x_1, \dots, x_n \in \cX$.
\index{kernel matrix}
This matrix ${\bm K}_n$ is called the {\em kernel matrix} or the {\em Gram matrix}.

In this monograph,  \emph{kernel} always means a positive definite kernel unless otherwise stated.
The reader should not confuse it with the notion of {\em smoothing kernels} used in nonparametric statistics, which do not necessarily satisfy the inequality above.

Let us see examples of kernels on the Euclidean space~$\mathbb{R}^d$ of $d \in \mathbb{N}$ dimensions.

\index{squared-exponential kernel}
\begin{example}[Gaussian/squared-exponential Kernels] \label{ex:square-exp-kernel}
A Gaussian or squared-exponential kernel on $\cX \subset \Re^d$ is defined as
\begin{equation} \label{eq:square-exp-kernel}
k_\gamma(x,x') := \exp\left( - \frac{ \| x - x' \|^2 } {\gamma^2 } \right) \quad (x,x' \in \cX),
\end{equation}
where $\gamma > 0$ is a parameter called the {\em bandwidth}.
\end{example}

The kernel \eqref{eq:square-exp-kernel} is often called the \emph{squared-exponential} kernel\footnote{It is also called ``square exponential'' or ``exponentiated quadratic.''} in the GP literature to avoid confusion with ``Gaussian'' processes, while it is often called {\em Gaussian} kernel or \emph{Radial Basis Function} (RBF) kernel in the RKHS literature.
In this monograph, we refer to it as a squared-exponential kernel.

The squared-exponential kernel outputs the highest value $1$ for identical inputs $x = x'$ and a lower value for more dissimilar $x$ and $x'$.
Thus, it is a good example to see that a kernel measures the similarity between the two inputs.
The bandwidth parameter $\gamma$ determines the scale at which the similarity is quantified.
This kernel is an {\em infinitely differentiable} function of the two inputs.
This property will be important for understanding the properties of the associated GP and RKHS, as described later.

Another important class is the Mat\'ern kernels \citep{Mat60} widely used in the spatial statistics literature \citep[Sections 2.7, 2.10]{Ste99}; see also \citet{porcu2024matern} for a recent review.

\index{Mat\'ern kernel}
\begin{example}[Mat\'ern kernels] \label{ex:matern-kernel}
A Mat\'ern kernel on $\cX \subset \Re^d$ with smoothness parameter $\alpha > 0$ is defined as
\begin{equation}
\label{eq:matern-kernel}
\begin{aligned}
k_{\alpha}(x,x')
&:=
\frac{1}{2^{\alpha-1}\Gamma(\alpha)}
\left(
\frac{\sqrt{2\alpha}\|x-x'\|}{h}
\right)^\alpha \\
&\quad\times
K_\alpha\left(
\frac{\sqrt{2\alpha}\|x-x'\|}{h}
\right)
\qquad (x,x'\in\cX),
\end{aligned}
\end{equation}
where $h > 0$ is the bandwidth parameter, $\Gamma$ is the Gamma function, and $K_\alpha$ is the modified Bessel function of the second kind of order~$\alpha$.

If $\alpha$ is a half-integer, i.e., $\alpha = m + 1/2$ for a non-negative integer $m \geq 0$, then the kernel~\eqref{eq:matern-kernel} is expressed in terms of an exponential function and polynomials up to degree $m$~\citep[e.g.,][Section 4.2.1 and Eq.~4.16]{RasmussenWilliams}:
\begin{align*}
k_{m+1/2} (x,x') & = \exp\left( - \frac{ \sqrt{2m + 1} \| x - x' \| }{h} \right) \frac{\Gamma(m + 1)}{\Gamma(2m + 1)} \\
& \times \sum_{i=0}^m \frac{ (m+i)! }{ i! (m-i)! } \left( \frac{ \sqrt{8m+4} \| x - x' \| }{ h } \right)^{m-i} .
\end{align*}
For example, for $\alpha = 1/2$ ($m=0$), $\alpha = 3/2$ ($m=1$) and $\alpha = 5/2$ ($m=2$), we have
\begin{align*}
k_{1/2}(x,x') &= \exp\left( - \frac{ \| x - x' \| }{h} \right),   \\
k_{3/2}(x,x') &= \left( 1 + \frac{\sqrt{3} \| x - x' \| }{ h } \right) \exp\left( - \frac{ \sqrt{3} \| x - x' \| }{h} \right),   \\
k_{5/2}(x,x') &= \left( 1 + \frac{\sqrt{5} \| x - x' \| }{ h } + \frac{ 5 \| x - x' \|^2 }{ 3h^2 } \right) \exp\left( - \frac{ \sqrt{5} \| x - x' \| }{h} \right).
\end{align*}
These are the widely used Mat\'ern kernels.
\index{Laplace kernel}
In particular, the case $\alpha = 1/2$ is called the {\em Laplace kernel} or the {\em exponential kernel}.
\end{example}

The parameter $\alpha$ quantifies the smoothness of the kernel: the kernel is smoother for larger $\alpha$.
In fact, the decay rate of the Mat\'ern kernel's spectral density, which quantifies its smoothness, is $2\alpha + d$; see Example~\ref{ex:RKHS-Matern} for details.
The squared-exponential kernel in Example~\ref{ex:square-exp-kernel} is the limit of Mat\'ern kernel for infinite smoothness $\alpha$ \citep[p.~50]{Ste99}:
$$
\lim_{\alpha \to \infty}k_{\alpha}(x,x') =  \exp \left(- \frac{\| x - x'\|^2}{2 h^2}  \right).
$$
Therefore, the Mat\'ern interpolates between the Laplace kernel ($\alpha = 1/2$) and the squared-exponential kernel ($\alpha = \infty$).
The Laplace kernel is not differentiable at identical inputs $x = x'$ and thus is not a continuously differentiable function.
The bandwidth parameter $h$ determines the scale of inputs.

The Mat\'ern kernels, including the Laplace and squared-exponential kernels as limiting cases, are introduced here to emphasize the notion of smoothness.
As described below, the kernel determines the smoothness of the associated GP and RKHS functions. Thus, choosing a kernel implies choosing the smoothness of the GP and RKHS functions.

There are many other kernels on Euclidean and non-Euclidean spaces; see, e.g.,~\citet{genton2001classes,scholkopf2002learning,SchTsuVer04,RasmussenWilliams,HofSchSmo08}. Kernels can be defined on any set $\cX$ as long as condition~\eqref{eq:PD-kernel-def} is satisfied. This generality allows kernels for structured data. Convolution kernels apply to discrete objects such as strings, trees, and graphs~\citep{Haussler1999Convolution}. Graph kernels compare networks, molecules, and other relational objects through walks, subgraphs, or related features~\citep{Vishwanathan2010GraphKernels}. String kernels compare symbolic sequences, such as text or biological sequences, through shared substrings or subsequences~\citep{Lodhi2002StringKernels}. The Tanimoto kernel is widely used for sparse binary descriptors, especially molecular fingerprints~\citep{Ralaivola2005ChemicalInformatics}. Kernels can also be defined on probability distributions through Hilbert-space embeddings~\citep{MuaFukSriSch17}.

Kernels are also used on geometric domains such as spheres, tori, and Riemannian manifolds. They can be constructed from harmonic analysis, heat kernels, or spectral decompositions of the Laplace--Beltrami operator~\citep{Gneiting2013Spheres,Borovitskiy2020RiemannianMatern}. Neural tangent kernels give another important class, arising as limiting kernels for wide neural networks trained by gradient descent~\citep{jacot2018neural}. These examples show the breadth of the kernel viewpoint. We do not study kernel construction in this monograph. Our focus is the GP and RKHS structure induced by a given positive definite kernel.

\section{Gaussian Processes (GP)} \label{sec:GP-intro}

Intuitively, a GP is a random function whose function values at any finite input points follow a multivariate Gaussian distribution.
The {\em covariance} between the function values at any two input points, referred to as the {\em covariance function} as a function of the input points, is a positive definite kernel.
Formally, a GP may be defined as follows \citep[e.g.,][p.~443]{Dud02}.

\index{Gaussian process}
\index{Gaussian process!covariance function}
\begin{definition}[Gaussian processes] \label{def:GP}
Suppose a real-valued random variable $F(x)$ is defined for each point $x \in \cX$ on a non-empty set~$\cX$.
The set $\{ F(x) \}_{x \in \cX}$ of these random variables is called a {\em Gaussian process} (GP), if there is a function $m: \cX \to \mathbb{R}$ and a positive definite kernel $k: \cX \times \cX \to \mathbb{R}$ such that for any $n \in \mathbb{N}$ and any $x_1, \dots, x_n \in \cX$,
\begin{align*} (\ssf(x_1),\dots,\ssf(x_n))^\top& \sim \N(
{\bm m}_n,{\bm K}_n),
\end{align*}
\begin{align*}
\text{where}\,\,\,  {\bm m}_n = (m(x_1),\dots,m(x_n))^\top \in \Re^n,\,\,\,
   {\bm K}_n =dc (k(x_i,x_j))_{i,j = 1}^n \in \Re^{n \times n},
\end{align*}
and $\N(
{\bm m}_n,{\bm K}_n)$ is the $n$-dimensional Gaussian distribution with mean vector ${\bm m}_n$ and covariance matrix ${\bm K}_n$.
Such an $F$ is denoted by $F \sim \GP(m,k)$, and $m$ and $k$ are called the mean and covariance functions of $F$.

\end{definition}

A GP is interpreted as a random function because any given input point~$x$ is associated with an output $F(x)$ that is a real-valued random variable; a more rigorous discussion is provided below.
The mean function $m(x)$ provides the expected value of $F(x)$ at any $x$, i.e., $m(x) = \mathbb{E}[F(x)]$.

The covariance function $k(x,x')$ at different input points $x$, $x'$ determines how strongly $F(x)$ and $F(x')$ are dependent.
Indeed,  Definition~\ref{def:GP} implies that the covariance between $F(x)$ and $F(x')$ is
\begin{align*}
{\rm Cov}[F(x), F(x')]
& := \mathbb{E} [ (F(x) - m(x)) (F(x') - m(x')) ] \\
& = k(x,x').
\end{align*}
For the squared-exponential and Mat\'ern kernels introduced above, if $x$ and $x'$ are closer, $k(x,x')$ is larger, and $F(x)$ and $F(x')$ correlate more strongly; if $x$ and $x'$ are far apart, $k(x,x')$ is close to zero, and $F(x)$ and $F(x')$ correlate very weakly.

For any function $m$ and positive definite kernel $k$, a GP with mean function $m$ and covariance function $k$ exists \citep[Theorem 12.1.3]{Dud02}.
Thus, choosing specific $m$ and $k$ in an algorithm implies choosing the associated GP and its properties.

For example, assuming that the mean function is zero, the smoothness of a sample function
$$
F: \cX \to \mathbb{R},
$$
which is defined as a realization of the random variables $F(x)$ at all input points $x \in \cX$, is determined primarily by the kernel $k$ (assume here that the mean function is zero).
If $k$ is a squared-exponential kernel, a sample $F$ is almost surely infinitely differentiable (very smooth).
If $k$ is a Mat\'ern kernel, its smoothness is finite and determined by the smoothness parameter~$\alpha$.
See Chapter~\ref{sec:theory} for details.

\section{Reproducing Kernel Hilbert Spaces (RKHS)} \label{sec:RKHS-def}

A Hilbert space is a vector space where an inner product is defined, the norm is induced from the inner product, and any convergent sequence in this norm converges to a member of the space.
An RKHS is a Hilbert space consisting of functions for which the {\em reproducing kernel} exists, which is a positive definite kernel, such that the value of any member function at any input point is reproduced as the inner product between that function and the reproducing kernel indexed by that input point; this is called {\em reproducing property}.
To make this precise, an RKHS is defined as follows.

\index{reproducing kernel Hilbert space}
\begin{definition}[RKHS] \label{def:RKHS}
Let $\cX$ be a nonempty set and $k$ be a positive definite kernel on $\cX$.
A Hilbert space $\cH_k$ consisting of functions on $\cX$ with an inner product $\left< \cdot,\cdot\right>_{\cH_k}$ is called a {\em reproducing kernel Hilbert space} (RKHS) with reproducing kernel $k$, if the following are satisfied:
\begin{enumerate}
 \item For all $x\in\cX$, we have $k(\cdot,x) \in \cH_k$;
 \item For all $x \in \cX$ and for all $f\in\cH_k$,
 \begin{equation} \label{eq:reproducing-prop}
      f(x) = \langle f,k(\cdot,x)\rangle_{\cH_k} \quad {\rm (Reproducing\ property)}.
 \end{equation}
\end{enumerate}
The norm of the RKHS is defined and denoted by
$$
\left\| f \right\|_{\cH_k} := \sqrt{\left< f, f \right>_{\cH_k}} \quad (f \in \cH_k).
$$
\end{definition}

\index{canonical feature map}
\paragraph{Canonical Feature}
The notation $k(\cdot,x)$ for any input point $x$ denotes the function of the first argument with the second one fixed to $x$:
$$
k(\cdot,x):  x' \mapsto k(x',x).
$$
The first property of Definition~\ref{def:RKHS} states that it belongs to the RKHS for any $x$.
It is interpreted as a representation of $x$ in the RKHS, often called a {\em feature vector} or {\em canonical feature map} of $x$.

\index{reproducing property}
\paragraph{Reproducing Property}
The second property of Definition~\ref{def:RKHS}, the reproducing property, states that the function value $f(x)$ of any RKHS function $f$ at any input point $x$ is equal to the inner product of $f$ and the feature vector $k(\cdot, x)$ of~$x$.
This is the defining characteristic of the RKHS and its inner product.

\index{evaluation functional}
\paragraph{Continuity of the Evaluation Functional}
The reproducing property implies the continuity (or the boundedness) of the {\em evaluation functional} at any input point~$x$.
The evaluation functional at fixed $x$ is the function that takes a function~$f$ as an input and outputs the function value $f(x)$.
Its continuity means that the function value $f(x)$ changes continuously when $f$ moves continuously in the RKHS while $x$ is fixed.
This follows from the Cauchy--Schwarz inequality applied to \eqref{eq:reproducing-prop}:
\begin{align*}
    | f(x) | \leq \| f \|_{\cH_k} \| k(\cdot, x) \|_{\cH_k} \quad \text{for all } f \in \cH_k.
\end{align*}
This implies the boundedness of the evaluation functional in the RKHS, which is equivalent to its continuity in the RKHS.

\index{Riesz representation}
\paragraph{Riesz Representation}
Every evaluation functional being continuous in a Hilbert space is equivalent to the Hilbert space being an RKHS.
Indeed, by the Riesz representation theorem \citep[e.g.,][Theorem~5.5.1]{Dud02}, the continuity of the evaluation functional at~$x$ in $\cH_k$ implies that there is a unique element $h_x \in \cH_k$ whose inner product with any $f \in \cH_k$ recovers the evaluation functional:
$$
f(x) = \left<f, h_x \right>_{\cH_k} \quad  \text{for all } f \in \cH_k.
$$
This element $h_x$ is called the {\em Riesz representation} of the evaluation functional in $\cH_k$.
Its uniqueness and the reproducing property~\eqref{eq:reproducing-prop} imply that the reproducing kernel $k(\cdot,x)$ indexed at $x$ is the Riesz representation of the evaluation functional:
$$
k(\cdot, x) = h_x \quad \text{for all } x \in \cX.
$$
The kernel value $k(x, x')$ for any two input points $x, x'$ can then be written as the inner product between the Riesz representations of the evaluation functionals at $x$ and $x'$:
\begin{equation} \label{eq:kernel-inner-reproduce}
k(x, x') = \left< k(\cdot,x), k(\cdot, x') \right>_{\cH_k}
= \left< h_x, h_{x'} \right>_{\cH_k}.
\end{equation}
Thus, the reproducing kernel is induced from the Riesz representations of the evaluation functionals.

\index{Moore--Aronszajn theorem}
\paragraph{One-to-One Relation between Kernels and RKHSs}
 The {\em Moore--Aronszajn theorem} \citep{Aronszajn1950} states that every positive definite kernel $k$ has a uniquely associated RKHS $\cH_k$ whose reproducing kernel is $k$. Thus, kernels and RKHSs are one-to-one.
 This implies that selecting a particular kernel in an algorithm is to select the RKHS. Therefore, understanding the properties of the RKHS induced by a selected kernel is practically important.

\index{reproducing kernel Hilbert space!construction}
\subsection{Explicit Construction of an RKHS}
\label{sec:RKHS-explict-const}
To describe how RKHS functions inherit the kernel's properties,  we describe an ``explicit construction'' of an RKHS from its kernel.
In particular, this shows that any RKHS function can be written as a weighted sum of canonical features or as a limit of such weighted sums.

\paragraph{Pre-Hilbert Space}
Define a pre-Hilbert space $\cH_0$ as the set of all functions written as a weighted sum of canonical features of {\em finitely many} input points:
\begin{align} \label{eq:RKHS-pre-Hilbert}
    & \cH_0 := \left\{ \sum_{i=1}^n c_i k(\cdot,x_i): \   n \in \mathbb{N},\ c_1,\dots,c_n \in \Re,\ x_1,\dots,x_n \in \cX  \right\},
\end{align}
where the weights $c_1, \dots, c_n$ can be negative.
Define the inner product between any two elements
$$f := \sum_{i=1}^n a_i k(\cdot,x_i) \in \cH_0, \quad g := \sum_{j=1}^m b_j k(\cdot,y_j) \in \cH_0$$ as the weighted sum of the kernel evaluations using the weights $a_1,\dots,a_n$ and $b_1,\dots,b_m$ and the evaluation points $x_1,\dots,x_n$ and $y_1,\dots,y_m$  based on which $f$ and $g$ are defined, i.e.,
\begin{equation} \label{eq:inner-prod-pre-hilbert}
\left< f, g \right>_{\cH_0} := \sum_{i=1}^n \sum_{j=1}^m a_i b_j k(x_i, y_j).
\end{equation}
The norm of any $f \in \cH_0$ is then defined as
$$
\| f \|_{\cH_0} := \sqrt{\left< f, f \right>_{\cH_0}}.
$$

The reproducing property holds for the pre-Hilbert space: the $\cH_0$-inner product between any function $f := \sum_{i=1}^n a_i k(\cdot,x_i)$ in the pre-Hilbert space and the reproducing kernel at any point $x$ reproduces the evaluation of $f$ at $x$:
\begin{equation} \label{eq:reprod-pre-hilbert}
\left< f, k(\cdot, x) \right>_{\cH_0} = \sum_{i=1}^n a_i k(x_i,x) = f(x),
\end{equation}
which follows from \eqref{eq:inner-prod-pre-hilbert} with $g := k(\cdot,x)$.
In particular, this implies that, for any two functions $f$ and $g$ in the pre-Hilbert space, the absolute difference between their function values at any point $x$ is bounded by the $\cH_0$-distance between $f$ and $g$ times the square root of the kernel value at $x$:
\begin{align} \label{eq:point-wise-bound-pre-hilbert}
   \left| f(x) - g(x) \right| \leq  \left\| f - g \right\|_{\cH_0} \sqrt{k(x,x)} \quad (f, g \in \cH_0, \ x\in \cX).
\end{align}
which follows from the Cauchy--Schwarz inequality.

\paragraph{RKHS via the Completion of the Pre-Hilbert Space}
The RKHS $\cH_k$ is then given as the completion (or the closure) of the pre-Hilbert space $\cH_0$ \citep[Theorem 4.21]{Steinwart2008}.
In other words, the RKHS is given by adding to the pre-Hilbert space the limiting functions of any convergent Cauchy sequences in the pre-Hilbert space.

More explicitly, take any Cauchy sequence in the pre-Hilbert space $f_1, f_2, \dots \in \cH_0$. That is, for any constant $\varepsilon > 0$ there exists a number $n_\varepsilon$ such that the $\cH_0$-distance between any two functions $f_n$ and $f_m$ whose indices $n$ and $m$ are larger than $n_\varepsilon$ is upper bounded by $\varepsilon$:
\begin{equation} \label{eq:cauchy-def}
    \left\| f_n - f_m \right\|_{\cH_0} < \varepsilon ~~\text{for all}~~ n, m > n_\varepsilon.
\end{equation}
Then, at each point $x$, their function values
$f_1(x), f_2(x), \dots \in \mathbb{R}$ are a Cauchy sequence in
$\mathbb{R}$, because for any constant $\varepsilon' > 0$ we have, by
\eqref{eq:point-wise-bound-pre-hilbert} and \eqref{eq:cauchy-def},
\begin{align*}
|f_n(x) - f_m(x)|
&\leq
\|f_n - f_m\|_{\cH_0}\sqrt{k(x,x)} \\
&<
\frac{\varepsilon'}{1+\sqrt{k(x,x)}} \sqrt{k(x,x)}
\leq \varepsilon'
\qquad \text{for all } n,m>n_\varepsilon,
\end{align*}
where $n_\varepsilon$ is the number in \eqref{eq:cauchy-def} with
$\varepsilon := \varepsilon' / \left(1+\sqrt{k(x,x)}\right)$.
Therefore, the limit of these function values exists in $\mathbb{R}$ for each point $x$, which we denote by $f(x)$:
\begin{equation} \label{eq:pointwise-conv-pre-RKHS}
    f(x) := \lim_{n \to \infty} f_n(x) \quad \text{for all }~ x \in \cX.
\end{equation}
The set of all such limiting functions $f$, including the functions in the pre-Hilbert space $\cH_0$, is the RKHS $\cH_k$:
\begin{align*}
    \cH_k = \{ f: \cX \to \mathbb{R}: ~~ &   f(x) = \lim_{n \to \infty} f_n(x) ~~ \text{for all } x \in \cX  \\
    & \text{for some Cauchy sequence}~ f_1, f_2, \dots  \in \cH_0    \} \supset \cH_0.
\end{align*}

\paragraph{RKHS Inner Product and Norm}
The inner product and the norm of the RKHS $\cH_k$ are defined as the extensions of those of the pre-Hilbert space $\cH_0$.
Let us start from the norm.
For any RKHS function $f$, its norm is defined as the limit of the $\cH_0$-norms of a Cauchy sequence $f_1, f_2, \dots$ in the pre-Hilbert space whose limit defines $f$ pointwise as \eqref{eq:pointwise-conv-pre-RKHS}:
$$
\left\| f \right\|_{\cH_k} := \lim_{n \to \infty} \left\| f_n \right\|_{\cH_0} < \infty,
$$
where the finiteness follows from the sequence being Cauchy in $\cH_0$.
The RKHS norm defined in this way is independent of the choice of the Cauchy sequence representing $f$. This follows from the argument below establishing the corresponding uniqueness of the inner product.

The inner product between two RKHS functions $f$ and $g$ is the limit of the $\cH_0$-inner products between functions from the two Cauchy sequences in the pre-Hilbert space whose limits are $f$ and $g$:
\begin{align} \label{eq:rkhs-inner-def-limit}
\left< f,g \right>_{\cH_k} := \lim_{n, m \to \infty} \left< f_n,  g_m \right>_{\cH_0} \quad (f,g \in \cH_k),
\end{align}
where $f_1, f_2, \dots \in \cH_0$ and $g_1, g_2, \dots \in \cH_0$ are Cauchy sequences whose limits define $f$ and $g$ pointwise:
\begin{align*}
    f(x) = \lim_{n \to \infty} f_n(x), \quad
    g(x) = \lim_{m \to \infty} g_m(x)  \quad (x \in \cX).
\end{align*}

The limit in \eqref{eq:rkhs-inner-def-limit} exists. Indeed, since Cauchy
sequences are bounded,
\begin{align*}
& \left|
\left\langle f_n,g_m\right\rangle_{\cH_0}
-
\left\langle f_p,g_q\right\rangle_{\cH_0}
\right| \\
&\leq
\left\|f_n-f_p\right\|_{\cH_0}\left\|g_m\right\|_{\cH_0}
+
\left\|f_p\right\|_{\cH_0}\left\|g_m-g_q\right\|_{\cH_0}
\to 0
\end{align*}
as $n,p,m,q\to\infty$. Thus,
$(\langle f_n,g_m\rangle_{\cH_0})_{n,m}$ is Cauchy in $\mathbb{R}$.

We next show that the limit in \eqref{eq:rkhs-inner-def-limit} is independent
of the choice of Cauchy sequences representing $f$ and $g$. Let
$\tilde f_1,\tilde f_2,\dots \in \cH_0$ and
$\tilde g_1,\tilde g_2,\dots \in \cH_0$ be other Cauchy sequences representing
$f$ and $g$, respectively. 

Set $h_n:=f_n-\tilde f_n$.
Since $f_1,f_2,\dots$ and $\tilde f_1,\tilde f_2,\dots$ are Cauchy in $\cH_0$,
\[
\|h_n-h_m\|_{\cH_0}
\leq
\|f_n-f_m\|_{\cH_0}
+
\|\tilde f_n-\tilde f_m\|_{\cH_0}
\to0
\]
as $n,m\to\infty$. Thus, $h_1,h_2,\dots$ is Cauchy in $\cH_0$.
Moreover, since both $f_1,f_2,\dots$ and $\tilde f_1,\tilde f_2,\dots$
converge pointwise to $f$,
\[
h_n(x)=f_n(x)-\tilde f_n(x)\to0
\qquad (x\in\cX).
\]
We next show that $\|h_n\|_{\cH_0}\to0$.
Since $h_m\in\cH_0$, for each fixed $m$ there exist
$r\in\mathbb{N}$, $a_1,\dots,a_r\in\mathbb{R}$, and
$x_1,\dots,x_r\in\cX$ such that
$h_m=\sum_{j=1}^r a_j k(\cdot,x_j)$.
By the reproducing property~\eqref{eq:reprod-pre-hilbert},
\[
\left\langle h_n,h_m\right\rangle_{\cH_0}
=
\sum_{j=1}^r a_j h_n(x_j)
\to0
\qquad (n\to\infty).
\]
Given $\varepsilon>0$, choose $M$ such that
$\|h_n-h_m\|_{\cH_0}<\varepsilon$ for all $n,m\geq M$.
For $n,m\geq M$, the triangle and Cauchy--Schwarz inequalities give
\[
\begin{aligned}
\|h_m\|_{\cH_0}^2
&=
\left\langle h_m,h_m-h_n\right\rangle_{\cH_0}
+
\left\langle h_m,h_n\right\rangle_{\cH_0}
\\
&\leq
\varepsilon\|h_m\|_{\cH_0}
+
\left|\left\langle h_m,h_n\right\rangle_{\cH_0}\right|.
\end{aligned}
\]
Letting $n\to\infty$ yields
$\|h_m\|_{\cH_0}\leq\varepsilon$ for all $m\geq M$.
Hence
\[
\|f_n-\tilde f_n\|_{\cH_0}
=
\|h_n\|_{\cH_0}
\to0.
\]
The same argument gives
$\|g_m-\tilde g_m\|_{\cH_0}\to0$. Consequently,
\begin{align*}
\left|
\left\langle f_n,g_m\right\rangle_{\cH_0}
-
\left\langle \tilde f_n,\tilde g_m\right\rangle_{\cH_0}
\right|
&\leq
\left\|f_n-\tilde f_n\right\|_{\cH_0}\left\|g_m\right\|_{\cH_0}
\\
&\quad+
\left\|\tilde f_n\right\|_{\cH_0}
\left\|g_m-\tilde g_m\right\|_{\cH_0}
\to0
\end{align*}
as $n,m\to\infty$, since Cauchy sequences are bounded.
Thus, \eqref{eq:rkhs-inner-def-limit} is well-defined.

The reproducing property can be shown to hold from the definition~\eqref{eq:rkhs-inner-def-limit} of the RKHS inner product.
Take $f$ as an RKHS function defined pointwise as the limit of a Cauchy sequence $f_1, f_2, \dots$ in the pre-Hilbert space as in \eqref{eq:pointwise-conv-pre-RKHS}.
By setting $g := k(\cdot,x)$ in \eqref{eq:rkhs-inner-def-limit} and using the reproducing property~\eqref{eq:reprod-pre-hilbert} of the pre-Hilbert space, we have
\begin{align*}
    \left< f, k(\cdot, x) \right>_{\cH_k}
    & = \lim_{n \to \infty} \left< f_n, k(\cdot,x) \right>_{\cH_0}  = \lim_{n \to \infty} f_n(x) = f(x),
\end{align*}
recovering the reproducing property.

\index{reproducing kernel Hilbert space!norm and smoothness}
\subsection{Smoothness Quantified by the RKHS Norm}
Intuitively, the RKHS norm $\| f \|_{\cH_k}$ of a function $f$ captures not only its magnitude, like the $L_2$ norm, but also its {\em smoothness}: the RKHS norm $\| f \|_{\cH_k}$ becomes large when the magnitude of $f$ is large or when $f$ is not smooth; a smaller RKHS norm $\| f \|_{\cH_k}$ implies the magnitude of $f$ is small and $f$ is smooth.

This smoothness property is important to understand the behavior of an algorithm using the RKHS norm as a regularization.
The following example on the RKHSs of a Mat\'ern kernel, which follows from \citet[Eq.~4.15]{RasmussenWilliams} and \citet[Corollary 10.48]{Wen05}, best illustrates this property.

\begin{example}[RKHSs of Mat\'ern Kernels] \label{ex:matern-rkhs}
Let $\cX \subset \Re^d$ be a domain with Lipschitz boundary,\footnote{For the definition of Lipschitz boundary, see, e.g., \citet[p.~189]{Ste70}, \citet[Definition 4.3]{Tri06} and \citet[Definition 3]{kanagawa2020convergence}.} and $k_{\alpha}$ be the Mat\'ern kernel in Example~\ref{ex:matern-kernel} with smoothness $\alpha > 0$ and bandwidth $h > 0$ such that $s := \alpha + d/2$ is an integer, and $\cH_{k_{\alpha}}$ be the RKHS of $k_{\alpha}$.
Then, for each $f \in \cH_{k_{\alpha}}$, all weak derivatives\footnote{Integration by parts defines a relation between a given function and its derivative through integrals.  A weak derivative is any function that satisfies this integral relation with the given function. A given function that is non-differentiable on a measure-zero set can have a weak derivative, because it is defined through integrals. In this sense, weak derivatives are ``weaker'' than standard derivatives.} up to the order $s$ exist and are square-integrable.
The squared RKHS norm $\left\| f \right\|_{\cH_{k_{\alpha}}}^2$ is upper and lower bounded by, up to constants, the sum of the squared $L_2$ norms of all the weak derivatives of $f$ up to order $s$.

\index{Sobolev space}
In other words, $\cH_{k_\alpha}$ is norm-equivalent\footnote{Normed vector spaces $\cH_1$ and $\cH_2$ are called norm-equivalent, if $\cH_1 = \cH_2$ as a set, and if there are constants $c_1, c_2 > 0$ such that $c_1 \| f \|_{\cH_2} \leq \| f \|_{\cH_1} \leq c_2 \| f \|_{\cH_2}$ holds for all $f \in \cH_1 = \cH_2$, where $\| \cdot \|_{\cH_1}$ and $\| \cdot \|_{\cH_2}$ denote the norms of $\cH_1$ and $\cH_2$, respectively.} to the {\em Sobolev space of order~$s$}~\citep[e.g.,][]{AdaFou03} defined by
\begin{equation*} 
W_2^s (\cX) := \left\{ f \in L_2(\cX) :\ \| f \|_{W_2^s (\cX)}^2 := \sum_{\beta \in \mathbb{N}_0^d: | \beta | \leq s} \left\| D^\beta f \right\|_{L_2(\cX)}^2 < \infty \right\},
\end{equation*}
where $D^\beta$ denotes the weak derivative of order $\beta = (\beta_1, \dots, \beta_d) \in \mathbb{N}_0^d$ with $|\beta| = \sum_{i=1}^d \beta_i \leq s$.
That is,  $\cH_{k_{\alpha}} = W_2^s (\cX)$ as a set of functions, and there exist constants $c_1, c_2 > 0$ such that
\begin{equation*}
c_1 \| f \|_{W_2^s(\cX)} \leq  \| f \|_{\cH_{k_{\alpha}}} \leq c_2 \| f \|_{W_2^s(\cX)} \quad \text{for all}\ f \in \cH_{k_{\alpha}}.
\end{equation*}
\end{example}

Therefore, the RKHS norm $\| f \|_{\cH_{k_{\alpha}}}$ is smaller for a smoother function~$f$ whose weak derivatives' magnitudes are smaller.
The RKHS norm is larger if $f$ is less smooth.
The smoothness parameter $\alpha$ determines the Sobolev order $s = \alpha + d/2$ of the RKHS.

\index{Laplace kernel!RKHS}
\begin{example}[Laplace kernel's RKHS]
If $\alpha = 1/2$ and $d = 1$, which results in the Laplace kernel $k(x, x') = \exp(- | x - x' | / h)$ (see Example~\ref{ex:matern-kernel}), the RKHS's order of smoothness is $s = 1$.
This implies that the Laplace kernel's RKHS consists of functions whose first weak derivatives exist and are square-integrable.
For instance, the reproducing kernel $f(x):= k(x, z) = \exp(- | x - z | / h) $ for any fixed $z \in \cX$ as a function of $x$ has a square-integrable first weak derivative.\footnote{The non-differentiability at the single point $x = z$ does not contradict the existence of the weak derivative because the latter is defined through the integration-by-parts relation, which is not influenced by the non-differentiability on a measure-zero set.}
\end{example}

\index{squared-exponential kernel!RKHS}
\begin{example}[Squared-exponential kernel's RKHS]

The RKHS of the squared-exponential kernel $\exp(- \| x - x' \|^2 / \gamma^2)$ on $\mathbb{R}^d$ for any finite $d$ consists of infinitely differentiable functions, because it is a subspace of Mat\'ern RKHSs for arbitrarily large finite values of the smoothness parameter $\alpha$. This follows from the spectral characterization in Section~\ref{sec:fourier_RKHS} and is consistent with the fact that the squared-exponential kernel is the limit of Mat\'ern kernels as $\alpha \to \infty$ (see Section~\ref{sec:pd-kernel}).

\end{example}

\section{Gaussian Hilbert Spaces (GHS)}
\label{sec:GHS}

\begin{figure}[t]
    \centering
    \resizebox{0.98\linewidth}{!}{
    \begin{tikzpicture}[
        x=0.9cm,
        y=0.9cm,
        >=stealth,
        font=\small,
        vector/.style={->,thick},
        correspondence/.style={->,thin},
        guide/.style={densely dotted,thin}
    ]

    \def\boxw{4.6}
    \def\boxh{3.4}
    \def\boxy{1.0}

    \def\boxLx{0.2}
    \def\boxRx{7.4}

    \draw (\boxLx,\boxy) rectangle ++(\boxw,\boxh);
    \draw (\boxRx,\boxy) rectangle ++(\boxw,\boxh);

    \coordinate (oH) at (2.65,1.72);
    \coordinate (oG) at (9.85,1.72);

    \coordinate (h1) at (1.55,2.58);
    \coordinate (h2) at (3.38,3.05);
    \coordinate (hc) at (2.20,3.62);

    \coordinate (g1) at (8.75,2.58);
    \coordinate (g2) at (10.58,3.05);
    \coordinate (gc) at (9.40,3.62);

    \draw[guide]
        ($(oH)!-0.65!(h1)$) -- ($(oH)!1.90!(h1)$);
    \draw[guide]
        ($(oH)!-0.45!(h2)$) -- ($(oH)!1.85!(h2)$);

    \draw[guide]
        ($(oG)!-0.65!(g1)$) -- ($(oG)!1.90!(g1)$);
    \draw[guide]
        ($(oG)!-0.45!(g2)$) -- ($(oG)!1.85!(g2)$);

    \draw[vector] (oH) -- (h1);
    \draw[vector] (oH) -- (h2);
    \draw[vector] (oH) -- (hc);

    \draw[vector] (oG) -- (g1);
    \draw[vector] (oG) -- (g2);
    \draw[vector] (oG) -- (gc);

    \draw[correspondence] (h1) -- (g1);
    \draw[correspondence] (h2) -- (g2);
    \draw[correspondence] (hc) -- (gc);

    \node[anchor=east] at (1.48,2.35)
        {$k(\cdot,x_1)$};

    \node[anchor=south west] at (3.43,2.96)
        {$k(\cdot,x_2)$};

    \node[anchor=south] at (2.20,3.74)
        {$c_1k(\cdot,x_1)+c_2k(\cdot,x_2)$};

    \node[anchor=east] at (8.68,2.35)
        {$F(x_1)$};

    \node[anchor=south west] at (10.63,2.96)
        {$F(x_2)$};

    \node[anchor=south] at (9.40,3.74)
        {$c_1F(x_1)+c_2F(x_2)$};

    \node at (6.10,1.93)
        {$\psi_k:\cH_k\to\cG_k$};

    \node[anchor=north] at (2.50,0.78)
        {Reproducing Kernel Hilbert Space $\cH_k$};

    \node[anchor=north] at (9.70,0.78)
        {Gaussian Hilbert Space $\cG_k$};

    \end{tikzpicture}
    }
    \caption{
    Schematic illustration of the canonical isometry
    $\psi_k:\cH_k\to\cG_k$ between the reproducing kernel Hilbert
    space $\cH_k$ of $k$ and the Gaussian Hilbert space $\cG_k$
    generated by $F\sim\GP(0,k)$. It maps each canonical feature
    $k(\cdot,x)$ to the Gaussian random variable $F(x)$ and preserves
    linear combinations and inner products.
    }
    \label{fig:RKHS-GHS}
\commentAlt{Figure~\ref{fig:RKHS-GHS}}{Two side-by-side boxes show
the RKHS on the left and the Gaussian Hilbert space on the right.
Horizontal arrows pair two canonical features and their linear
combination with the corresponding Gaussian random variables.}
\end{figure}

\index{Gaussian Hilbert space}
A {\em Gaussian Hilbert space} is a Hilbert space representation of a GP and a probabilistic formulation of the RKHS.
It is a Hilbert space where each element is a zero-mean Gaussian random variable, and the inner product of any two elements is defined by their covariance.
It will serve as a foundation for understanding the relations between a GP and the associated RKHS.
The definition is as follows~\citep[e.g.,][]{Janson1997}.

\begin{definition}

Let $k$ be a positive definite kernel on a non-empty set~$\cX$.
A Hilbert space $\cG_k$ with inner product $\left< \cdot,\cdot\right>_{\cG_k}$ is  called a {\em Gaussian Hilbert space} (GHS) with kernel $k$, if the following are satisfied:
\begin{enumerate}
    \item  Each element $Z \in \cG_k$ is a zero-mean real-valued Gaussian random variable.
    \item The inner product between any two elements $Z, W \in \cG_k$ is defined by their covariance: $\left< Z, W \right>_{\cG_k} = \mathbb{E}[ZW]$.
    The norm of any $Z \in \cG_k$ is its standard deviation: $\| Z \|_{\cG_k} := \sqrt{ \left< Z, Z \right>_{\cG_k} } = \sqrt{  \mathbb{E}[Z^2] }$.
    \item Each input point $x \in \cX$ is associated with a zero-mean Gaussian random variable $Z_x \in \cG_k$, such that the kernel value $k(x,x')$ for any two input points $x, x' \in \cX$ is the covariance between the associated elements $Z_x$ and $Z_{x'}$:
\begin{equation} \label{eq:inner-product-GHS}
    \left< Z_x, Z_{x'}\right>_{\cG_k} = \mathbb{E}[Z_x Z_{x'}] = k(x, x') \quad \text{for all } x, x' \in \mathcal{X}.
\end{equation}
    \item The linear span of these elements $\{ Z_x \}_{x \in \mathcal{X}} \subset \cG_k$ is dense in $\cG_k$, i.e., for any $Z \in \cG_k$ and $\varepsilon > 0$, there exist $n \in \mathbb{N}$, $x_1, \dots, x_n \in \mathcal{X}$ and $c_1, \dots, c_n \in \mathbb{R}$ such that $\| Z - \sum_{i=1}^n c_i Z_{x_i}\|_{\cG_k} < \varepsilon$.
\end{enumerate}

\end{definition}

The set of zero-mean Gaussian random variables $\{Z_x: x \in \cX \}$ satisfying~\eqref{eq:inner-product-GHS}, which will be called {\em elementary Gaussian random variables}, can be defined from a GP, $F \sim \GP(0,k)$, as
$$
Z_x := \ssf(x) \quad \text{for all }\ x \in \mathcal{X}.
$$
As described below, the GHS is ``generated'' from these random variables.
In this sense, the GHS is a Hilbert space representation of the GP.

\subsection{Explicit Construction of GHS}
\label{sec:explict-const-GHS}

In parallel with the RKHS, the GHS can be constructed from the set of elementary Gaussian random variables~$\{Z_x: x \in \cX \}$.
We first define a pre-Hilbert space $\cG_0$ as the set of all finite linear combinations of the elementary Gaussian random variables:
\begin{align} \label{eq:GHS-pre-Hilbert}
    \cG_0
    & := \left\{ \sum_{i=1}^n c_i Z_{x_i}: \ n \in \mathbb{N},\ c_1,\dots,c_n \in \mathbb{R},\ x_1,\dots,x_n \in \cX  \right\}.
\end{align}
Any element in $\cG_0$ is a zero-mean Gaussian random variable.

For any two elements $Z$ and $Z'$ in $\cG_0$, their inner product is defined as their covariance.
Explicitly, since they can be written as
$$
Z := \sum_{i=1}^n a_i Z_{x_i} \in \cG_0, \quad Z' := \sum_{j=1}^m b_j Z_{y_j} \in \cG_0
$$
for some finite numbers $n, m$, coefficients $a_i, b_j$, and input points $x_i, y_j$,
the inner product is given by the weighted sum of kernel evaluations:
\begin{equation} \label{eq:GHS-pre-Hil-inner-prod}
\left< Z, Z' \right>_{\cG_0} := \mathbb{E}[Z Z'] = \sum_{i=1}^n \sum_{j=1}^m a_i b_j k(x_i, y_j).
\end{equation}
The norm of any $Z \in \cG_0$ is equal to its standard deviation:
$$
\| Z \|_{\cG_0} = \sqrt{ \left< Z, Z \right>_{\cG_0}} = \sqrt{ \mathbb{E}[Z^2]}.
$$

The GHS $\cG_k$ is then given as the completion of the pre-Hilbert space~$\cG_0$ with respect to the norm $\| \cdot \|_{\cG_0}$, i.e., $\cG_k$ is given by adding to $\cG_0$ the limits of all Cauchy sequences of elements of $\cG_0$.
More explicitly, let $Z_1, Z_2, \dots \in \cG_0$ be a Cauchy sequence such that
$$
\lim_{n,m \to \infty} \| Z_n - Z_m \|_{\cG_0}^2 = \lim_{n,m \to \infty} \mathbb{E}[ ( Z_n - Z_m )^2 ] = 0.
$$
Then the sequence converges to a random variable $Z$ in the mean-square sense:
$$
\lim_{n \to \infty} \mathbb{E}[ ( Z_n - Z )^2 ] = 0,
$$
which implies that $Z_1, Z_2, \dots$ converge to $Z$ in distribution. Moreover, since all of $Z_1, Z_2, \dots$ are zero-mean Gaussian random variables with finite variances, the limit $Z$ is also zero-mean Gaussian.\footnote{
Indeed, mean-square convergence implies $\mathbb{E}[Z_n^2]\to\mathbb{E}[Z^2]$ as $n\to\infty$.
Since $Z_n\to Z$ in distribution, the characteristic functions of $Z_n$ converge pointwise to that of $Z$. 
On the other hand, since $Z_n\sim\mathcal{N}(0,\mathbb{E}[Z_n^2])$, their characteristic functions converge, for every $t\in\mathbb{R}$, to
$\exp(-t^2\mathbb{E}[Z^2]/2)$; this is therefore the characteristic function of $Z$, and hence $Z\sim\mathcal{N}(0,\mathbb{E}[Z^2])$.}
Thus, the GHS consists of all such zero-mean Gaussian random variables $Z$:
\begin{align*}
    \cG_
    k = \{  Z:   ~~ &   \lim_{n \to \infty} \mathbb{E}[ ( Z_n - Z )^2 ] = 0,   \\
    & \text{for some Cauchy sequence}~ Z_1, Z_2, \dots  \in \cG_0    \} \supset \cG_0.
\end{align*}

The inner product and the norm of $\cG_k$ are the covariance and the standard deviation, respectively, as they are the extensions of the inner product and the norm of $\cG_0$.
Let us first show this for the norm. The squared norm of each element $Z$ in the GHS is the limit of the squared norms = variances of the Cauchy sequence $Z_1, Z_2 \dots \in \cG_0$ that converge to $Z$ in mean square, so it is the variance of $Z$:
$$
\left\| Z \right\|_{\cG_k}^2 := \lim_{n \to \infty} \left\| Z_n \right\|_{\cG_0}^2 =  \lim_{n \to \infty}  \mathbb{E}\left[ Z_n^2 \right] = \mathbb{E}[Z^2] < \infty.
$$

Similarly, the inner product between any two elements $Z$ and $W$ in the GHS is their covariance.
Let $Z_1, Z_2, \dots \in \cG_0$ and $W_1, W_2, \dots \in \cG_0$ be Cauchy sequences that respectively converge to $Z$ and $W$ in mean square.
Then the GHS inner product between $Z$ and $W$ is the limit of the inner products in $\cG_0$ = covariances between $Z_n$ and $W_m$, which is the covariance of $Z$ and $W$:
\begin{align}
\left< Z, W \right>_{\cG_k} & := \lim_{n, m \to \infty} \left< Z_n, W_m \right>_{\cG_0}  \label{eq:GHS-inner-def-limit} \\
& = \lim_{n, m \to \infty} \mathbb{E}[ Z_n W_m ]  = \mathbb{E}[ZW],  \nonumber
\end{align}
where the last identity follows from
\begin{align*}
&  \left(\mathbb{E}[ Z_n W_m ] - \mathbb{E}[ Z W ]  \right)^2 \\
= & \left(\mathbb{E}[ Z_n W_m ] - \mathbb{E}[ Z_n W ] + \mathbb{E}[ Z_n W ] - \mathbb{E}[ Z W ]  \right)^2 \\
\leq & ~ 2 \left(\mathbb{E}[ Z_n W_m ] - \mathbb{E}[ Z_n W ] \right)^2 + 2\left(\mathbb{E}[ Z_n W ] - \mathbb{E}[ Z W ]  \right)^2 \\
\leq & ~ 2 ~\mathbb{E}[Z_n^2]~ \mathbb{E}[(W_m - W)^2] + 2 ~\mathbb{E}[(Z_n-Z)^2] ~\mathbb{E}[W^2],
\end{align*}
which converges to zero as $n$ and $m$ go to infinity.

\section{Canonical Isometry and the GHS--RKHS Equivalence} \label{sec:canonical-isomet-GHS-RKHS}
The Gaussian Hilbert space introduced in Section~\ref{sec:GHS} is equivalent to the RKHS introduced in Section~\ref{sec:RKHS-def} in the sense that there exists an {\em isometric isomorphism} between them, i.e., a {\em bijective linear operator that preserves inner products} (see, e.g., \citealp[Section 4]{Par61};
\citealp[Theorem~35, p.~65]{Berlinet2004};
\citealp[Theorem~8.15]{Janson1997}).
See Figure~\ref{fig:RKHS-GHS} for an illustration.

\index{canonical isometry}
This isometric isomorphism is called {\em canonical isometry} and here denoted by
$$
\psi_k: \cH_k \to \cG_k.
$$
We will define it constructively.
First, it is defined as an operator between the set of all the canonical features $\{ k(\cdot,x):~ x \in \mathcal{X} \}$ and the set of elementary Gaussian random variables $\{ Z_x:~ x \in \mathcal{X} \}$, such that the canonical feature $k(\cdot,x)$ of each point $x$ is mapped to the elementary Gaussian random variable~$Z_x$ at the same point $x$:
\begin{equation} \label{eq:canonical-isometry}
    \psi_k (k(\cdot,x)) := Z_x \quad \text{for all } x \in \mathcal{X}.
\end{equation}
Its inverse maps the elementary Gaussian random variable $Z_x$ at each point~$x$ to the canonical feature $k(\cdot,x)$ of the same point:
$$
 \psi_k^{-1}(Z_x) = k(\cdot,x) \quad \text{for all } x \in \mathcal{X}.
$$
Thus, the canonical isometry is bijective between $\{k(\cdot,x):~ x \in \mathcal{X} \}$ and  $\{ Z_x:~ x \in \mathcal{X} \}$.
It also preserves the inner product between these two sets: the RKHS inner product~\eqref{eq:kernel-inner-reproduce} between the canonical features of any two points $x$ and $x'$ is identical to the GHS inner product~\eqref{eq:inner-product-GHS} between the corresponding elementary Gaussian random variables, as they are both equal to the kernel value $k(x,x')$ between $x$ and $x'$:
\begin{align*}
     \left< k(\cdot,x), k(\cdot,x') \right>_{\cH_k} = \left< Z_x, Z_{x'}  \right>_{\cG_k}
    = k(x, x') \quad \text{for all }~ x, x' \in \cX.
\end{align*}

\subsection{Canonical Isometry between the Pre-Hilbert Spaces}
The canonical isometry extends linearly to an isometry between the pre-Hilbert spaces $\cH_0$ in \eqref{eq:RKHS-pre-Hilbert} and $\cG_0$ in \eqref{eq:GHS-pre-Hilbert}.
Since any element in $\cH_0$ is a finite weighted sum of the canonical features, this element is mapped to the corresponding weighted sum of elementary Gaussian random variables, which is an element in $\cG_0$:
\begin{align*}
\psi_k\left( \sum_{i=1}^n c_i k(\cdot,x_i) \right) &:= \sum_{i=1}^n c_i \psi_k(k(\cdot,x_i) ) =  \sum_{i=1}^n c_i Z_{x_i}\in \cG_0 \\
& \quad \text{for all } ~  \sum_{i=1}^n c_i k(\cdot,x_i) \in \cH_0.
\end{align*}
This definition is well-defined in that the output of $\psi_k$ is unique in $\cG_0$ for any given input in $\cH_0$.\footnote{Indeed, suppose that the same element in $\cH_0$ can be written in two ways, $\sum_{i=1}^n a_i k(\cdot,x_i)=\sum_{j=1}^m b_j k(\cdot,y_j)$.
Then the corresponding outputs under $\psi_k$ are equal in $\cG_0$, since $\left\|\sum_{i=1}^n a_i Z_{x_i}-\sum_{j=1}^m b_j Z_{y_j}\right\|_{\cG_0}^2 = \mathbb{E}\!\left[\left(\sum_{i=1}^n a_i Z_{x_i}-\sum_{j=1}^m b_j Z_{y_j}\right)^2\right] = \left\|\sum_{i=1}^n a_i k(\cdot,x_i)-\sum_{j=1}^m b_j k(\cdot,y_j)\right\|_{\cH_0}^2 =0$,
where the second equality follows from $\mathbb{E}[Z_xZ_y]=k(x,y)$ for any $x,y\in\mathcal{X}$.}

Moreover, $\psi_k$ preserves the inner products between $\cH_0$ and $\cG_0$, since either inner product, defined in \eqref{eq:inner-prod-pre-hilbert} or
\eqref{eq:GHS-pre-Hil-inner-prod}, is the weighted sum of kernel evaluations for which the weights and evaluation points are those defining the two given elements for the inner product.
More explicitly, consider the two elements $Z$ and $W$ in $\cG_0$ that correspond via $\psi_k$ to two elements $f$ and $g$ in $\cH_0$ given as
$$
f := \sum_{i=1}^n a_i k(\cdot,x_i), \quad g := \sum_{j=1}^m b_j k(\cdot,y_j), $$
so that
$$
Z = \psi_k(f) = \sum_{i=1}^n a_i Z_{x_i}, \quad W = \psi_k(g) = \sum_{j=1}^m b_j Z_{y_j}.
$$
Then the inner product of $Z$ and $W$ in $\cG_0$ equals that of $f$ and $g$ in $\cH_0$:
\begin{align} \label{eq:inner-prod-pres-pre-hilb}
\left< Z, W \right>_{\cG_0}
= \sum_{i=1}^n \sum_{j=1}^m a_i b_j k(x_i, y_j) = \left< f, g \right>_{\cH_0}.
\end{align}

The inner-product preservation implies the norm preservation: the distance between any two elements in $\cH_0$ equals the distance between the corresponding two elements in $\cG_0$, i.e.,
\begin{equation} \label{eq:pre-hilbert-norm-pres}
\left\| f - g \right\|_{\cH_0} = \left\| \psi_k(f) - \psi_k(g) \right\|_{\cG_0}  \quad \text{for all } f, g \in \cH_0.
\end{equation}
This implies that $\psi_k$ is injective.
It is also surjective by the definition of $\cG_0$: any element in $\cG_0$ is a finite linear combination of the elementary Gaussian random variables $Z_x$, and is therefore the image under $\psi_k$ of the corresponding finite linear combination of the canonical features $k(\cdot,x)$ in $\cH_0$.
Hence, $\psi_k$ is bijective between $\cH_0$ and $\cG_0$.

\subsection{Canonical Isometry between the RKHS and GHS}
\label{sec:can-iso-def-RKHS-GHS}

We now define the canonical isometry $\psi_k$ as a linear operator between the RKHS $\cH_k$ and the GHS $\cG_k$.
This is based on the ``explicit constructions'' of the RKHS (Section~\ref{sec:RKHS-explict-const}) and the GHS (Section~\ref{sec:explict-const-GHS}).

\paragraph{Definition of the Canonical Isometry}
Let us define the canonical isometry $\psi_k(f)$ for any RKHS element $f \in \cH_k$.
Let $f_1, f_2, \dots \in \cH_0$ be a Cauchy sequence in the pre-Hilbert space whose limit defines $f$.
Let $Z_1, Z_2, \dots \in \cG_0$ be the corresponding sequence in the pre-Hilbert space for the GHS, such that each element $Z_n$ is the canonical isometry applied to $f_n$:
$$
Z_n := \psi_k(f_n) \in \cG_0 \quad (n = 1,2, \dots).
$$
This sequence is Cauchy in $\cG_0$, because $f_1, f_2, \dots$ are Cauchy in $\cH_0$ and $\psi_k$ preserves the norm between $\cH_0$ and $\cG_0$ as in \eqref{eq:pre-hilbert-norm-pres}, so the distance between any two elements in one sequence is identical to that in the other sequence:
\begin{equation} \label{eq:dist-pre-seq-can-iso}
\left\| Z_n - Z_m \right\|_{\cG_0} = \left\| f_n - f_m \right\|_{\cH_0} \quad \text{for all ~} n,~ m = 1, 2, \dots .
\end{equation}
Therefore, the limit of $Z_1, Z_2 \dots$ exists and defines an element $Z$ in the GHS $\cG_k$.
This $Z$ is the definition of $\psi_k(f)$:
$$
\psi_k(f) :=  \lim_{n \to \infty} \psi_k(f_n) \in \cG_k,
$$
where the convergence is in the mean-square sense (see Section~\ref{sec:explict-const-GHS}).

This definition does not depend on the choice of the Cauchy sequence representing $f$.
Indeed, if $f_1,f_2,\dots$ and $\tilde f_1,\tilde f_2,\dots$ are two Cauchy sequences in $\cH_0$ representing the same element $f$ of $\cH_k$, then
\[
\left\|
\psi_k(f_n)-\psi_k(\tilde f_n)
\right\|_{\cG_0}
=
\left\|
f_n-\tilde f_n
\right\|_{\cH_0}
\longrightarrow 0
\quad\text{as } n\to\infty.
\]
Thus the two image sequences have the same limit in $\cG_k$.

\paragraph{Definition of the Inverse}
The inverse of the canonical isometry is defined similarly.
Take an arbitrary element $Z$ from the GHS $\cG_k$, and let $Z_1, Z_2, \dots$ be a Cauchy sequence in the pre-Hilbert space $\cG_0$ whose limit defines $Z$.
Let $f_1, f_2, \dots $ be the corresponding sequence in the pre-Hilbert space $\cH_0$ for the RKHS, where each~$f_n$ is defined as the inverse $\psi_k^{-1}$ of the canonical isometry applied to $Z_n$:
$$
f_n := \psi_k^{-1}(Z_n) \in \cH_0, \quad (n = 1, 2, \dots).
$$
This sequence is Cauchy in $\cH_0$, because the canonical isometry preserves the norm between $\cH_0$ and $\cG_0$ as in \eqref{eq:dist-pre-seq-can-iso}.
Thus, the limit of the sequence $f_1, f_2, \dots$ exists and defines an RKHS element $f \in \cH_k$.
This $f$ is the definition of the inverse of the GHS element $Z$ under the canonical isometry:
$$
\psi_k^{-1}(Z) := \lim_{n \to \infty} \psi_k^{-1}(Z_n) \in \cH_k,
$$
where the convergence is pointwise (see \eqref{eq:pointwise-conv-pre-RKHS}).

By construction, $\psi_k^{-1}$ defined in this way is the inverse of $\psi_k$ as it satisfies
$$
\psi_k^{-1}\left( \psi_k(f)   \right) = f \quad \text{for all } f \in \cH_k.
$$

\subsection{Inner-product Preservation and Bijectivity}
We now show that the canonical isometry $\psi_k$ is bijective and preserves the inner product between the RKHS $\cH_k$ and the GHS $\cG_k$.

\paragraph{Inner-product Preservation}
Let us first show that the canonical isometry preserves the inner products of the RKHS and GHS, which are respectively defined in \eqref{eq:rkhs-inner-def-limit} and
\eqref{eq:GHS-inner-def-limit}.
Let $f \in \cH_k$ and $g \in \cH_k$ be two arbitrary RKHS elements, and let $f_1, f_2, \dots \in \cH_0$ and $g_1, g_2, \dots \in \cH_0$ be Cauchy sequences in the pre-Hilbert space whose limits define $f$ and $g$, respectively.
The canonical isometry is shown to preserve the inner products of the pre-Hilbert spaces $\cH_0$ and $\cG_0$ as in \eqref{eq:inner-prod-pres-pre-hilb}, so
$$
\left< f_n, g_m \right>_{\cH_0} = \left< \psi_k(f_n), \psi_k(g_m) \right>_{\cG_0} \quad \text{for all } n, m = 1, 2, \dots
$$
Thus, the RKHS inner product between $f$ and $g$, which is defined as the limit of these inner products in $\cH_0$ as $n, m \to \infty$ as \eqref{eq:rkhs-inner-def-limit}, is written as
\begin{align*}
  \left< f, g \right>_{\cH_k} & = \lim_{n, m \to \infty}   \left< f_n, g_m \right>_{\cH_0}  \\
  & = \lim_{n, m \to \infty}   \left< \psi_k(f_n), \psi_k(g_m) \right>_{\cG_0} = \left< \psi_k(f), \psi_k(g) \right>_{\cG_k},
\end{align*}
which recovers the definition of the GHS inner product between the GHS elements $\psi_k(f) \in \cG_k$ and $\psi_k(g) \in \cG_k$ in \eqref{eq:GHS-inner-def-limit}.

This inner-product preservation implies the norm preservation. In particular, the distance between any RKHS elements $f$ and $g$ equals that between the corresponding GHS elements $\psi_k(f)$ and $\psi_k(g)$, generalizing the distance preservation for the pre-Hilbert spaces in \eqref{eq:pre-hilbert-norm-pres}:
\begin{align} \label{eq:dist-preserve-can-isom}
    \left\| f - g \right\|_{\cH_k} =     \left\| \psi_k(f) - \psi_k(g) \right\|_{\cG_k} \quad \text{for all }~ f, g \in \cH_k.
\end{align}

\paragraph{Bijectivity}
The bijectivity of $\psi_k$ follows as follows.
Section~\ref{sec:can-iso-def-RKHS-GHS} already shows that the inverse $\psi_k^{-1}(Z)$ exists in the RKHS for each GHS element $Z \in \cG_k$, so $\psi_k$ is surjective.
Therefore, what remains to show is that $\psi_k$ is injective, which is shown as follows.
For any two different RKHS elements $f$ and $g$, their distance is positive, which is preserved under the canonical isometry as in \eqref{eq:dist-preserve-can-isom}, so the distance between the corresponding GHS elements $\psi_k(f)$ and $\psi_k(g)$ is positive, which implies that they are different:
\begin{align*}
 \left\| \psi_k(f) - \psi_k(g) \right\|_{\cG_k} = \left\| f - g \right\|_{\cH_k} > 0 \quad  \text{for all } ~ f \not=g \in \cH_k.
\end{align*}
Thus, $\psi_k$ is shown to be bijective between $\cH_k$ and $\cG_k$.

We have shown that the canonical isometry, as a linear operator from the RKHS to the GHS, preserves the inner product (and thus the distance) and is bijective, so it is an isometric isomorphism.
In this sense, the RKHS and GHS are equivalent.
This will be the key to various equivalences between GP-based and RKHS-based methods.

\section{Spectral Characterization of Shift-Invariant RKHSs}
\label{sec:fourier_RKHS}

Before proceeding further, we describe a precise characterization of the RKHS whose reproducing kernel is  {\em shift-invariant} on $\cX = \Re^d$.
\index{shift-invariant kernel}
A kernel~$k(x, x')$ is shift-invariant (or stationary) if its value depends only on the difference $x - x'$ between the two input vectors $x, x' \in \mathbb{R}^d$, i.e., there is a function $\Phi:\Re^d \to \Re$ such that
\begin{equation} \label{eq:shift-invariant-kernel}
    k(x,x') = \Phi(x-x') \quad \text{for all }\ x, x' \in \mathbb{R}^d.
\end{equation}
The squared-exponential, Mat\'ern, and Laplace kernels (Examples~\ref{ex:square-exp-kernel} and \ref{ex:matern-kernel}) are all shift-invariant kernels.

\index{Fourier transform}
The RKHS of a shift-invariant kernel is characterized by Fourier transforms.
For an integrable function $f \in L_1(\mathbb{R}^d)$, its Fourier
transform $\cF[f] : \mathbb{R}^d \to \mathbb{C}$ is defined by
\[
\cF[f](\omega)
:=
\frac{1}{(2\pi)^{d/2}}
\int_{\mathbb{R}^d}
f(x)
\exp\left(-\sqrt{-1}\,x^\top\omega\right)
\,dx,
\qquad \omega\in\mathbb{R}^d.
\]
It is a spectral representation of $f$, where $\cF[f](\omega)$ represents its component at frequency $\omega$.\footnote{By Plancherel's theorem, $\cF$ extends uniquely to
$L_2(\mathbb{R}^d)$ and agrees with this definition on
$L_1(\mathbb{R}^d)\cap L_2(\mathbb{R}^d)$; see, e.g.,
\citet[Theorem~8.29, p.~252]{Fol99}.}

\index{spectral density}
\index{Bochner's theorem}
If the function $\Phi$ defining the shift-invariant kernel~\eqref{eq:shift-invariant-kernel} is integrable, i.e., $\Phi \in L_1(\mathbb{R}^d)$, then Bochner's theorem \citep{Boc33} implies that its Fourier transform $\cF[\Phi]$ is real-valued, non-negative, and integrable. This Fourier transform is called the {\em spectral density} of $k$.
The decay rate of $\cF[\Phi](\omega) \to 0$ for increasing frequencies $\|\omega\| \to \infty$ determines the kernel's smoothness.
As described below, the decay rate for a Mat\'ern kernel is polynomial and faster for a larger smoothness parameter $\alpha$, and that for the squared-exponential kernel is exponential.

Theorem~\ref{theo:RKHS-shift-invariant} below shows that whether a function $f$ belongs to the RKHS is determined by whether its Fourier transform $\cF[f]$ decays fast enough compared with the spectral density $\cF[\Phi]$ for increasing frequencies.
This implies that the RKHS of a smoother kernel only contains smoother functions.
This result is available in, e.g., \citet[Lemma 3.1]{kimeldorf1970correspondence} and \citet[Theorem 10.12]{Wen05}.

\index{reproducing kernel Hilbert space!spectral characterization}
\begin{theorem}
\label{theo:RKHS-shift-invariant}
Let $k$ be a shift-invariant kernel on $\cX=\mathbb{R}^d$ of the form~\eqref{eq:shift-invariant-kernel}, where $\Phi\in C(\mathbb{R}^d)\cap L_1(\mathbb{R}^d)$ and $\mathcal F[\Phi](\omega) > 0$ for all $\omega \in \mathbb{R}^d$. Then
\[
\cH_k
=
\left\{
f\in L_2(\mathbb{R}^d)\cap C(\mathbb{R}^d)
:
\|f\|_{\cH_k}<\infty
\right\},
\]
where the norm is given by
\begin{equation}
\label{eq:RKHS-shift-invariant}
\|f\|_{\cH_k}^2
=
\frac{1}{(2\pi)^{d/2}}
\int_{\mathbb{R}^d}
\frac{|\cF[f](\omega)|^2}
{\cF[\Phi](\omega)}
\,d\omega,
\end{equation}
and the inner product by
\[
\langle f,g\rangle_{\cH_k}
=
\frac{1}{(2\pi)^{d/2}}
\int_{\mathbb{R}^d}
\frac{
\cF[f](\omega)\overline{\cF[g](\omega)}
}{
\cF[\Phi](\omega)
}
\,d\omega,
\qquad f,g\in\cH_k .
\]
Here, the bar denotes complex conjugation.
\end{theorem}

\eqref{eq:RKHS-shift-invariant} shows that the RKHS norm of a function $f$ is a constant times the integral,  over frequencies $\omega \in \mathbb{R}^d$, of its Fourier squared $|\cF[f](\omega)|^2$ divided by the kernel's spectral density $\cF[\Phi](\omega)$.
This implies that, for $f$ to have a finite RKHS norm, the Fourier transform $\cF[f](\omega)$ should be small for frequencies $\omega$ for which the spectral density $\cF[\Phi](\omega)$ is small.
In particular, the decay rate of $|\cF[f](\omega)|^2$ should be faster than the decay rate of $\cF[\Phi](\omega)$ as $\| \omega \| \to \infty$.

Our first example of Theorem~\ref{theo:RKHS-shift-invariant} is the RKHS of a squared-exponential kernel. In this case, the $L_2$ mass of the high-frequency components of every RKHS function decays exponentially fast, implying infinite smoothness.
\index{squared-exponential kernel!spectral characterization}
\begin{example}[RKHS of a squared-exponential kernel] \label{ex:RKHS-squared-exponential}
Let
$$
k_\gamma(x,x') = \Phi_\gamma(x-x') := \exp\left( - \frac{\| x - x' \|^2}{\gamma^2} \right)
$$
be the squared-exponential kernel on $\mathbb{R}^d$ with bandwidth $\gamma > 0$ in Example \ref{ex:square-exp-kernel}.
The Fourier transform of $\Phi_\gamma$ is given by
\begin{equation*} \label{eq:Gaussian-fourier}
\cF[\Phi_\gamma](\omega) = C_{d,\gamma} \exp\left(- \frac{ \gamma^2 \| \omega \|^2 }{4} \right) \quad \text{for all }\ \omega \in \Re^d,
\end{equation*}
where $C_{d,\gamma}$ is a constant depending only on $d$ and $\gamma$~\citep[e.g.,][Theorem 5.20]{Wen05}.
Thus, the RKHS $\cH_{k_\gamma}$ of $k_\gamma$ is
\[
\cH_{k_\gamma}
=
\left\{
f \in L_2(\mathbb{R}^d) \cap C(\mathbb{R}^d)
:
\|f\|_{\cH_{k_\gamma}} < \infty
\right\},
\]
where
\[
\|f\|_{\cH_{k_\gamma}}^2
=
\frac{1}{(2\pi)^{d/2}C_{d,\gamma}}
\int_{\mathbb{R}^d}
|\cF[f](\omega)|^2
\exp\left(
\frac{\gamma^2\|\omega\|^2}{4}
\right)
\,d\omega .
\]
Therefore, the squared $L_2$ mass of the high-frequency components of any function $f$ in the RKHS decays exponentially fast. This decay is faster for a larger bandwidth $\gamma$.
\end{example}

The next example is the RKHS of a Mat\'ern kernel, where the $L_2$ mass of the high-frequency components of every RKHS function decays at least polynomially fast, with a faster decay rate for a larger smoothness parameter.

\index{Mat\'ern kernel!spectral characterization}
\begin{example}[RKHS of a Mat\'ern kernel] \label{ex:RKHS-Matern}
Let
\begin{equation*}
\begin{aligned}
k_{\alpha}(x,x')
&=
\Phi_{\alpha,h}(x-x') \\
&:=
\frac{2^{1-\alpha}}{\Gamma(\alpha)}
\left(
\frac{\sqrt{2\alpha}\|x-x'\|}{h}
\right)^{\alpha}
K_\alpha\left(
\frac{\sqrt{2\alpha}\|x-x'\|}{h}
\right).
\end{aligned}
\end{equation*}
be the Mat\'ern kernel on $\Re^d$ with smoothness $\alpha > 0$ and bandwidth $h > 0$ in Example \ref{ex:matern-kernel}.
The Fourier transform of $\Phi_{\alpha, h}$ is given by
\begin{equation} \label{eq:Fourier-matern}
\mathcal F[\Phi_{\alpha,h}](\omega)
=
C_{\alpha,h,d}
\left(
\frac{2\alpha}{h^2}+\|\omega\|^2
\right)^{-\alpha-d/2},
\qquad \omega\in\mathbb R^d,
\end{equation}
where $C_{\alpha,h,d}$ is a constant depending only on $\alpha$, $h$ and $d$  \citep[e.g.,][Eq.~4.15]{RasmussenWilliams}.\footnote{This form is obtained from \citet[Eq.~4.15]{RasmussenWilliams} after converting to the Fourier-transform convention used in this section.}
Thus, the RKHS $\cH_{k_{\alpha}}$ of $k_{\alpha}$ is
\[
\mathcal H_{k_\alpha}
=
\left\{
f\in L_2(\mathbb R^d)\cap C(\mathbb R^d)
:
\|f\|_{\mathcal H_{k_\alpha}}<\infty
\right\},
\]
where
\[
\|f\|_{\mathcal H_{k_\alpha}}^2
=
\frac{1}{(2\pi)^{d/2}C_{\alpha,h,d}}
\int_{\mathbb R^d}
|\mathcal F[f](\omega)|^2
\left(
\frac{2\alpha}{h^2}+\|\omega\|^2
\right)^{\alpha+d/2}
\,d\omega .
\]

For a function $f$ to have a finite RKHS norm, the squared $L_2$ mass of its high-frequency Fourier components must decay to zero at least polynomially fast, with a faster decay rate for larger $\alpha$.
From \eqref{eq:Fourier-matern} and \citet[Corollary 10.48]{Wen05}, the RKHS is norm-equivalent to the Sobolev space of order $\alpha+d/2$, consistent with Example~\ref{ex:matern-rkhs}.
\end{example}

\chapter{Sample Paths of Gaussian Processes and RKHSs}
\label{sec:theory}

\index{Gaussian process!sample path}
A GP is a collection of real-valued Gaussian random variables $F(x)$ indexed by a set of locations $x \in \cX$, and their realizations, say $f(x) \in \mathbb{R}$ for each $x \in \cX$, define a real-valued function that maps a location to a real value, $x \mapsto f(x)$; this is called a {\em sample path}, {\em sample function}, or simply {\em sample} of the GP.
Since a GP-based method assumes that an unknown function is a sample of a specified GP, the method's success depends on whether this assumption is satisfied at least approximately.
Therefore, the properties of a GP sample, such as its smoothness, should be understood for a GP-based method to be used successfully.

This chapter discusses the properties of a GP sample through the lens of RKHSs.
We first review {\em Mercer's theorem}, which gives an eigenfunction expansion of the kernel, and the {\em Mercer representation} of the RKHS.
We then describe the corresponding series representation of a GP, known as the {\em Karhunen--Lo\`eve expansion} (Section \ref{sec:hypothesis-mercer}).
The Karhunen--Lo\`eve expansion is a probabilistic counterpart of the Mercer representation and follows from the equivalence between the Gaussian Hilbert space and the RKHS.
These series representations help characterize the properties of a GP sample.

Section \ref{sec:driscol-sample-path} then discusses {\em Driscoll's theorem} \citep{driscoll1973reproducing,LukBed01}, which gives the necessary and sufficient condition for a GP sample to belong to a {\em given} RKHS; this RKHS can be different from the GP kernel's RKHS.
Intuitively, it states that a GP sample belongs to a given RKHS if and only if it is ``sufficiently larger'' than the GP kernel's RKHS.
This indicates that a GP sample is less smooth than RKHS functions.
The immediate consequence is that, as is well-known \citep[e.g.,][]{wahba1990spline}, a GP sample does {\em not} belong to the GP kernel's RKHS with probability one whenever the RKHS is infinite-dimensional (such as the case of squared-exponential and Mat\'ern kernels).

Section \ref{sec:power-RKHS-sample-path} describes that a GP sample belongs to a function space defined as a {\em power} of the RKHS, which is another RKHS interpolating the original RKHS and the space of square-integrable functions~\citep{steinwart2019convergence}.
This result enables quantifying the smoothness of a GP sample.
As corollaries,  sample path properties of GPs associated with squared-exponential kernels and Mat\'ern kernels are described in Section~\ref{sec:examples-sample-path}.

Figure~\ref{fig:GP_intro} shows scaled kernel eigenfunctions, GP sample paths, and unit-norm RKHS functions for the Brownian motion kernel.

\begin{figure}[t]
  \includegraphics[width=\linewidth]{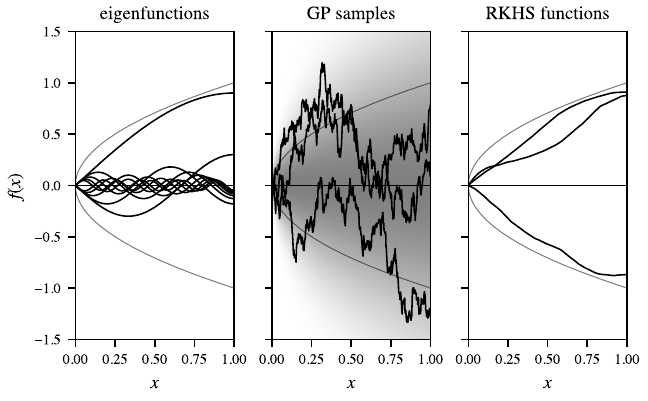}
\caption{
Comparison of scaled kernel eigenfunctions (\textbf{left}), GP samples (\textbf{center}), and unit-norm RKHS functions (\textbf{right}) for the Brownian motion kernel $k(x,x')=\min(x,x')$ on $\cX=[0,1]$, with $\nu$ the uniform measure on $\cX$. The associated GP is Brownian motion, while its RKHS is the first-order Sobolev space described in Example~\ref{ex:stoch-int-motivation}. The \textbf{left} panel shows the first ten eigenfunctions $\phi_i$ of the kernel integral operator~\eqref{eq:integral-eigen-decomp}, each scaled by $\lambda_i^{1/2}$ to have unit RKHS norm; for this kernel, $\phi_i(x)=\sqrt{2}\sin((i-\tfrac12)\pi x)$ and $\lambda_i=((i-\tfrac12)^2\pi^2)^{-1}$ for $i\in\mathbb N$. The \textbf{center} panel shows independent GP samples obtained by truncating the Karhunen--Lo\`eve expansion~\eqref{eq:KL-exp-informal} after 300 terms. The \textbf{right} panel shows random unit-norm RKHS functions $f(\cdot)=\sum_{j=1}^n a_j k(\cdot,x_j)/\sqrt{{\bm a}^\top {\bm K}_n {\bm a}},$ where ${\bm a}=(a_1,\ldots,a_n)^\top\sim\mathcal N(0,I_n)$ and ${\bm K}_n=(k(x_i,x_j))_{i,j=1}^n$ for the plotting grid $x_1,\ldots,x_n$. Each panel also shows $\pm\sqrt{k(x,x)}$, which equals the pointwise GP standard deviation and the maximum of $|f(x)|$ over $\|f\|_{\cH_k}\leq 1$ by Corollary~\ref{coro:std-max-equiv-eval-func}.
}
  \label{fig:GP_intro}
\commentAlt{Figure~\ref{fig:GP_intro}}{Three panels show scaled eigenfunctions of the Brownian motion kernel, Brownian-motion sample paths, and unit-norm functions in the associated first-order Sobolev RKHS. See long description.}
\commentLongAlt{Figure~\ref{fig:GP_intro}}{The left panel shows the first ten scaled kernel eigenfunctions. The center panel shows independent Brownian-motion sample paths obtained from a 300-term Karhunen--Loeve expansion. The right panel shows random unit-norm functions in the associated first-order Sobolev RKHS. All panels include the pointwise standard-deviation envelope.}
\end{figure}

For simplicity, this chapter assumes that $\cX$ is a compact metric space (e.g., a bounded and closed subset of $\Re^d$) and $k$ is a continuous kernel on $\cX$.
Under this assumption, the RKHS $\cH_k$ is separable~\citep[Lemma 4.33]{Steinwart2008} and has countable basis functions.

\section{Orthogonal Basis Expansions of GPs and RKHSs} \label{sec:hypothesis-mercer}

Series representations based on the eigenfunctions and eigenvalues of a kernel integral operator, introduced in Section~\ref{sec:integral-op}, are described for kernels in Section~\ref{sec:Mercer} (Mercer's theorem), for RKHSs in Section~\ref{sec:mercer-RKHS} (Mercer representation), and for GPs in Section~\ref{sec:KL-expansion} (Karhunen--Loève expansion).

Mercer's theorem \citep{Mercer1909,SteSco12} gives a series representation of a kernel based on the eigenfunctions and eigenvalues of a kernel integral operator.
These eigenfunctions are orthonormal in the space of square-integrable functions, and the eigenfunctions multiplied by the eigenvalues' square roots form an orthonormal basis for the RKHS.
Thus, we start by describing the integral operator and its eigendecomposition.

\index{kernel integral operator}
\subsection{Kernel Integral Operator}
\label{sec:integral-op}

Let $\nu$ be a finite Borel measure on $\cX$ with its support being $\cX$, such as the Lebesgue measure on a compact $\cX \subset \Re^d$.
Let $L_2(\nu)$ be the Hilbert space of square-integrable functions\footnote{Strictly, here each $f \in L_2(\nu)$ represents a class of functions that are equivalent $\nu$-almost everywhere.} with respect to $\nu$:
\begin{equation*}
L_2(\nu) := \left\{f:\cX \to \Re: \| f \|_{L_2(\nu)}^2 := \int |f(x)|^2 d\nu(x) < \infty \right\},
\end{equation*}
for which the inner product is defined as
$$
\left<f, g \right>_{L_2(\nu)} := \int f(x)g(x)d\nu(x) \quad \text{for all }\ f, g \in L_2(\nu).
$$
We define an integral operator  $T_k: L_2(\nu) \to L_2(\nu)$ that takes a square-integrable function $f$ and outputs its convolution with the kernel $k$ under the measure $\nu$:
\begin{equation} \label{eq:integral_operator}
T_k f := \int k(\cdot,x) f(x) d \nu(x) \quad \text{for all }\ f \in L_2(\nu).
\end{equation}
This operator smooths an input function by integrating it with the kernel, and the output function's smoothness is determined by the kernel's smoothness.
Thus, it can be interpreted as an operator representation of the kernel.

\paragraph{Eigendecomposition of the Integral Operator}
Under the conditions considered, the integral operator $T_k$ is compact, positive, and self-adjoint, so that the spectral theorem holds~\citep[e.g.,][Theorem A.5.13]{Steinwart2008}.
This implies that the integral operator $T_k$ applied to $f$ can be expressed as a (finite or infinite) sum of the components of $f$ corresponding to different eigenfunctions  of $T_k$ multiplied by the associated eigenvalues:
\begin{equation} \label{eq:integral-eigen-decomp}
T_k f = \sum_{i \in I} \lambda_i \left< \phi_i, f \right>_{L_2(\nu)} \phi_i,
\end{equation}
where the convergence is in $L_2(\nu)$.
Here, $I \subset \mathbb{N}$
 is a set of indices, which is an infinite set if the RKHS of $k$ is infinite-dimensional (e.g., $I = \{1, 2, ... \}$) and a finite set if the RKHS is finite-dimensional.

The functions $\phi_1, \phi_2, \dots \in L_2(\nu)$ are the eigenfunctions of $T_k$ and $\lambda_1 \geq \lambda_2 \geq \cdots > 0$ are the associated eigenvalues such that
\begin{equation} \label{eq:int-op-eig-def}
  T_k  \phi_i  = \lambda_i \phi_i \quad \text{for all }\ i \in I.
\end{equation}
These eigenfunctions form an orthonormal system in $L_2(\nu)$:
\begin{equation} \label{eq:ONB-def}
    \left< \phi_i, \phi_j\right>_{L_2(\nu)} =
\begin{cases}
    1 & \text{if }  i = j \\
    0 & \text{if }  i \not= j.
\end{cases}
\end{equation}
In many common examples, such as the periodic Sobolev kernels below, eigenfunctions with larger eigenvalues are smoother than those with smaller eigenvalues.

\subsection{Mercer's Theorem}

\label{sec:Mercer}

\index{Mercer theorem}
{\em Mercer's theorem} states that the kernel can be expanded in terms of the eigenfunctions and eigenvalues of the integral operator.
Theorem~\ref{theo:mercer} below is due to \citet[Theorem 4.49]{Steinwart2008}, while Mercer's theorem holds under weaker assumptions than those considered here \citep[e.g.,][Section 3]{SteSco12}.

\begin{theorem}[Mercer's theorem] \label{theo:mercer}
Let $\cX$ be a compact metric space, $k:\cX \times \cX \to \Re$ be a continuous kernel, $\nu$ be a finite Borel measure whose support is $\cX$, and $(\phi_i,\lambda_i)_{i \in I}$ be the eigenfunctions and eigenvalues of the integral operator~\eqref{eq:integral-eigen-decomp}.
Then we have
\begin{equation}  \label{eq:mercer}
  k(x,x') =  \sum_{i \in I} \lambda_i \phi_i(x) \phi_i(x'), \quad x,x' \in \cX,
  \end{equation}
  where the convergence is uniform over $x, x' \in \cX$ and absolute.
\end{theorem}

Mercer's theorem implies that the kernel can be written as the Euclidean inner product between feature vectors composed of the eigenfunctions and eigenvalues.
A feature vector $\Psi(x)$ of an input point $x$ can be defined so that the $i$-th component is the $i$-th eigenfunction value $\phi_i(x)$ of $x$ multiplied by the eigenvalue's square root $\sqrt{\lambda_i}$:
\begin{equation} \label{eq:explicit-feature-RKHS}
\Psi(x) :=  (\sqrt{\lambda_i} \phi_i(x) )_{i \in I}.
\end{equation}
Then the kernel value $k(x,x')$ for two input points $x, x'$ is equal to the Euclidean inner product between their feature vectors:
$$
k(x,x') = \Psi(x)^\top \Psi(x').
$$
In this sense, the eigenfunctions and eigenvalues provide an ``explicit'' construction of feature vectors.

From the Mercer expression~\eqref{eq:mercer}, it is easy to see that $\phi_i$ are the eigenfunctions of the integral operator~\eqref{eq:integral_operator} and $\lambda_i$ are the associated eigenvalues, since
\begin{align*}
  T_k \phi_i  & =  \int k(\cdot,x)\phi_i(x)d\nu(x)
    = \int \sum_{j \in I} \lambda_j \phi_j \phi_j(x) \phi_i(x) d\nu(x) \\
    & = \sum_{j \in I} \lambda_j \phi_j \left< \phi_j, \phi_i \right>_{L_2(\nu)} = \lambda_i \phi_i \quad \text{for all }\ i \in I,
\end{align*}
recovering \eqref{eq:int-op-eig-def}, where the last identity follows from that $\phi_i$ are orthonormal functions in $L_2(\nu)$; see \eqref{eq:ONB-def}.

The eigenfunctions, the eigenvalues, and Mercer's theorem are best illustrated by the following example on a kernel whose RKHS is a periodic Sobolev space~\citep[e.g.,][]{wahba1975smoothing,Bac17}.

\index{periodic Sobolev kernel}
\begin{example}[Periodic Sobolev Kernels] \label{ex:periodic-sobolev}
    For a positive integer $s$, the following kernel on the unit interval $\cX = [0,1]$  defines the $s$-th order Sobolev space consisting of periodic functions $f$ on $[0,1]$ whose values are the same at the bounds: $f(0) = f(1)$:
    $$
    k_s(x, x') = 1 + \frac{ (-1)^{s-1} (2\pi)^{2s} }{ (2s)! } B_{2s}(\{ x - x' \}) \quad \text{for all }\ x, x' \in [0,1],
    $$
    where $\{ x - x' \}$ is the fractional part of $x - x'$ and $B_{2s}$ is the $2s$-th Bernoulli polynomial.
    It is periodic on $[0,1]$ as the kernel value for $x = 0$ and $x' = 1$ is the same as that for $x = x' = 0$.
    For instance, for $s = 1, 2$ the kernel takes the form
    \begin{align*}
        k_1 (x, x') &= 1 + 2\pi^2 \left( \{x-x'\}^2 - \{x-x'\} + \frac{1}{6} \right) \\
        k_2 (x, x') &= 1 - \frac{(2\pi)^{4}}{4!} \left(  \{x-x'\}^4 - 2\{x-x'\}^3 + \{x-x'\}^2 - \frac{1}{30} \right).
    \end{align*}

Using the Fourier series expansion of the Bernoulli polynomial~\citep[e.g.,][Section 23.1.18]{abramowitz1964handbook} and a trigonometric identity, the kernel can be expanded as
     \begin{align*}
    & k_s(x,x') =  1 + 2 \sum_{m = 1}^\infty m^{-2s} \cos (2\pi m(x - x') )   \\
    & = 1 + 2 \sum_{m = 1}^\infty m^{-2s}  \left( \cos (2\pi m x ) \cos (2\pi m x' ) + \sin(2\pi mx) \sin(2\pi mx') \right).
     \end{align*}
     The Mercer expansion~\eqref{eq:mercer} is obtained if we define
     \begin{align*}
         \phi_1(x) &  := 1, \quad \lambda_1 := 1, \\
         \phi_i(x) & :=  \sqrt{2} \cos (i \pi  x), \quad \lambda_i := (i/2)^{-2s}, \quad \text{if }\ i = 2, 4, 6, \dots,  \\
         \phi_i(x) &:=  \sqrt{2} \sin ((i-1) \pi  x), \quad \lambda_i := ((i-1)/2)^{-2s}, \quad \text{if }\ i = 3, 5, 7, \dots .
     \end{align*}
    These $\phi_i$ are orthonormal basis functions in $L_2(\nu)$ for $\nu$ the uniform distribution on $\cX = [0,1]$.
    Thus, $\phi_i$ and $\lambda_i$ are the eigenfunctions and eigenvalues of the integral operator~\eqref{eq:integral_operator}.

    The eigenfunctions $\phi_i$ associated with larger eigenvalues $\lambda_i$ have {\em lower frequencies} (smaller $i$), and thus, they are smoother than the eigenfunctions $\phi_i$ with smaller eigenvalues $\lambda_i$ having {\em higher frequencies} (larger $i$).
    A larger Sobolev order $s$ makes the decay of eigenvalues $\lambda_i$ for increasing $i$ faster, and thus lowers smaller eigenvalues $\lambda_i$ and reduces the contributions of the eigenfunctions $\phi_i$ having higher frequencies, making the kernel smoother.
\end{example}

\begin{remark}\rm \label{rem:mercer-uniequness}
The expansion in \eqref{eq:mercer} depends on the measure $\nu$, since $(\phi_i,\lambda_i)_{i \in I}$ is an eigensystem of the integral operator \eqref{eq:integral_operator}, which is defined with $\nu$.
However, the kernel $k$ on the left side is unique, irrespective of the choice of $\nu$.
In other words, a different choice of $\nu$ results in a different eigensystem $(\phi_i,\lambda_i)_{i \in I}$, and thus results in a different basis expression of the same kernel $k$.
\end{remark}

\begin{remark}\rm \label{rem:mercer-measure}
In Theorem \ref{theo:mercer}, the assumption that $\nu$ has $\cX$ as its support is important, since otherwise the equality  \eqref{eq:mercer} may not hold for some $x \in \cX$.
For instance, assume that there is an open set $N \subset \cX$ such that $\nu(N) = 0$.
Then the integral operator \eqref{eq:integral_operator} does not take into account the values of a function $f$ on $N$, and therefore the values of the eigenfunctions $\phi_i$ on $N$ are not determined. Consequently, the equality in \eqref{eq:mercer} is guaranteed to hold on $\cX \backslash N$, but not necessarily on $N$.
We refer to \citet[Corollaries 3.2 and 3.5]{SteSco12} for precise statements of Mercer's theorem in such a case.
\end{remark}

\subsection{Mercer Representation of RKHSs}
\label{sec:mercer-RKHS}

\index{Mercer representation}
The {\em Mercer representation} of the RKHS~\citep[Theorem 4.51]{Steinwart2008}, presented below,  states that the eigenfunctions $(\phi_i)_{i \in I}$ times the eigenvalues' square roots $(\sqrt{\lambda_i})_{i \in I}$, i.e., the feature map in \eqref{eq:explicit-feature-RKHS} giving the kernel's expansion, form an {\em orthonormal basis} of the RKHS.
In other words, each function $f$ in the RKHS is uniquely associated with a square-summable weight sequence $(\alpha_i)_{i \in I}$, whose length may be finite or infinite, such that the function is written as the sum of these weights times the eigenfunctions times the eigenvalues' square roots:
$$
f = \sum_{i \in I} \alpha_i \lambda_i^{1/2} \phi_i.
$$
The RKHS norm of a function is identical to the Euclidean norm of the associated weights; the RKHS inner product between any two functions is the Euclidean inner product between the two associated weight sequences.

\begin{theorem}[Mercer Representation] \label{theo:mercer-RKHS}
Let $\cX$ be a compact metric space, $k:\cX \times \cX \to \Re$ be a continuous kernel with RKHS $\cH_k$, $\nu$ be a finite Borel measure whose support is $\cX$, and $(\phi_i,\lambda_i)_{i \in I}$ be the eigenfunctions and eigenvalues of the integral operator~\eqref{eq:integral-eigen-decomp}.
Then,
\begin{align*}
 \cH_k = \biggl\{f: \cX \to \mathbb{R}:\ & f
 = \sum_{i \in I} \alpha_i \lambda_i^{1/2} \phi_i \ \ \text{for some }  (\alpha_i)_{i \in I} \subset \mathbb{R} \nonumber \\
 & \text{such that}\ \| f \|_{\cH_k}^2 := \sum_{i \in I} \alpha_i^2 < \infty \biggr\}, 
\end{align*}
and the inner product is given by
\begin{align*}
  & \left< f,g \right>_{\cH_k} = \sum_{i \in I} \alpha_i\beta_i \\ & \text{for all}
  \q f := \sum_{i \in I} \alpha_i \lambda_i^{1/2} \phi_i \in \cH_k,\q g := \sum_{i\in I} \beta_i \lambda_i^{1/2} \phi_i \in \cH_k.
\end{align*}
In other words, $(\lambda_i^{1/2} \phi_i )_{i \in I}$ form an orthonormal basis of $\cH_k$.

\end{theorem}

\begin{remark}
The $i$-th basis function $\lambda_i^{1/2} \phi_i$ is associated with a weight sequence where the $i$-th weight is one and the other weights are zero. From this and the Mercer representation, the orthonormality of the basis functions can be confirmed:
\begin{equation*} 
    \left< \lambda_i^{1/2} \phi_i,\  \lambda_j^{1/2} \phi_j \right>_{\cH_k} =
\begin{cases}
    1 & \text{if } i = j \\
    0 & \text{if } i \not= j.
\end{cases}
\end{equation*}
\end{remark}

In many common examples where the eigensystem is naturally ordered by frequency, the Mercer representation helps explain how the smoothness of functions in the RKHS is related to the rate of eigenvalue decay $\lambda_i \to 0$ as $i \to \infty$; faster decay suppresses higher-frequency components more strongly and thus corresponds to smoother RKHS functions.
This is best illustrated by the RKHS of the periodic kernel in Example~\ref{ex:periodic-sobolev}.

\begin{example}[Mercer Representation of a Periodic Sobolev RKHS] \label{ex:Mercer-periodic-Sobolev}
    By the Mercer representation, each function $f$ in the RKHS of the $s$-th order periodic Sobolev kernel in Example~\ref{ex:periodic-sobolev} is written as
\begin{align*}
    f(x) = \alpha_1 & + \sum_{i = 2, 4, 6, \dots} \alpha_i~(i/2)^{-s} \sqrt{2} \cos( i \pi x ) \\
    & +  \sum_{i = 3, 5, 7, \dots} \alpha_i~((i-1)/2)^{-s} \sqrt{2} \sin( (i-1) \pi x ) \quad \text{for all }\ x \in [0,1]
\end{align*}
for some square-summable weight sequence $(\alpha_1, \alpha_2, \dots )$.
Therefore, for a larger order $s$, the Fourier basis functions with larger indices~$i$, which have higher frequencies, contribute less significantly to the composition of the function, making it smoother.
Therefore, the order $s$~quantifies the smoothness of the least-smooth functions in the RKHS.

\end{example}

\subsection{Karhunen--Loève Expansion of Gaussian Processes}
\label{sec:KL-expansion}

\index{Karhunen--Lo\`eve expansion}
We describe a series representation of a Gaussian process based on the eigenfunctions and eigenvalues of the integral operator, known as the \emph{Karhunen--Lo\`eve (KL) expansion} \citep{Karhunen1947,Loeve1948}.
It is the probabilistic counterpart of the RKHS's Mercer representation and holds due to the equivalence between the Gaussian Hilbert space (GHS, Section~\ref{sec:GHS}) and the RKHS.

The following form of the KL expansion follows from \citet[Theorem 3.1]{steinwart2019convergence}.
See also, e.g., \citet[Sections 3.2 and 3.3]{Ad90} and \citet[Section 2.3]{Berlinet2004}.
We provide a proof as it is instructive.

\begin{theorem}[Karhunen--Loève Expansion] \label{theo:KL-expansion}
Let $\cX$ be a compact metric space, $k:\cX \times \cX \to \Re$ be a continuous kernel with RKHS $\cH_k$, $\nu$ be a finite Borel measure whose support is $\cX$, and $(\phi_i,\lambda_i)_{i \in I}$ be the eigenfunctions and eigenvalues of the integral operator~\eqref{eq:integral-eigen-decomp}.
Let $F \sim \GP(0,k)$ be a GP, $\cG_k$ be the associated GHS and $\psi_k: \cH_k \to  \cG_k$ be the canonical isometry~\eqref{eq:canonical-isometry}.

Define $(Z_i)_{i \in I} \subset \cG_k$ as the elements corresponding to the orthonormal basis functions $(\lambda_i^{1/2} \phi_i)_{i \in I}$ of the RKHS via the canonical isometry:
\begin{equation} \label{eq:KL-normal-varibles}
 Z_i := \psi_k ( \lambda_i^{1/2} \phi_i ) \in \cG_k \quad \text{for all }\  i \in I.
\end{equation}
Then the following claims hold:
\begin{enumerate}
\item $(Z_i)_{i \in I}$ are independent standard Gaussian random variables:
\begin{equation} \label{eq:normal-variables-KL-expansion}
Z_i \sim \mathcal{N}(0, 1)\quad \mathrm{and} \quad \bE [ Z_i  Z_j] = 0\quad  \text{if }\ i \not= j \quad \text{for all }\ i, j \in I.
\end{equation}
\item For all $x \in \cX$ and for all finite indices $J \subset I$, we have
\begin{equation} \label{eq:KL-expansion}
  \bE\left[ \left( \ssf(x) - \sum_{i \in J} Z_i \lambda_i^{1/2} \phi_i(x) \right)^2 \right] = k(x,x) - \sum_{j \in J} \lambda_j \phi_j^2(x).
\end{equation}
\item If $I = \mathbb{N}$, we have
\begin{equation} \label{eq:KL-uniform-mean-squre}
 \lim_{n \to \infty} \bE\left[ \left( \ssf(x) - \sum_{i = 1}^n Z_i \lambda_i^{1/2} \phi_i(x) \right)^2 \right] = 0 \quad \text{for all }\ x \in \cX,
\end{equation}
where the convergence is uniform in $x \in \cX$.
\end{enumerate}
\end{theorem}

\begin{proof}
Since the GHS $\cG_k$ consists of zero-mean Gaussian random variables, so does each $Z_i \in \cG_k$.

The identity~\eqref{eq:normal-variables-KL-expansion} holds because i)~the expected product of $Z_i$ and $Z_j$ is their inner product in the GHS, ii) the canonical isometry preserves the inner product, and iii) $(\lambda_i^{1/2} \phi_i)_{i \in I}$ are orthonormal in the RKHS:
\begin{align*}
\mathbb{E}[Z_i Z_j] = \left< Z_i, Z_j \right>_{\cG_k} = \left<  \lambda_i^{1/2} \phi_i, \lambda_j^{1/2} \phi_j \right>_{\cH_k} =
\begin{cases}
    1 & \text{if } i = j \\
    0 & \text{if } i \not= j.
\end{cases}.
\end{align*}
Every finite subcollection of $(Z_i)_{i\in I}$ is jointly Gaussian, since each of its linear combinations belongs to $\cG_k$.
Hence, together with the covariance identity above, $(Z_i)_{i\in I}$ are independent standard Gaussian random variables.

The identity \eqref{eq:KL-expansion} follows because
i)~the expected square is the squared GHS norm,
ii)~the canonical isometry maps the canonical feature $k(\cdot,x)$ to $F(x)$,
iii)~the canonical isometry preserves the norm, and iv)~$(\lambda_i^{1/2} \phi_i)_{i \in I}$ are orthonormal in the RKHS:
\begin{align*}
& \bE\left[ \left( \ssf(x) - \sum_{i \in J} Z_i \lambda_i^{1/2} \phi_i(x) \right)^2 \right] = \left\|  \ssf(x) - \sum_{i \in J} Z_i \lambda_i^{1/2} \phi_i(x)  \right\|_{\cG_k}^2 \\
&= \left\| k(\cdot,x) -  \sum_{i \in J} \lambda_i^{1/2} \phi_i(\cdot)  \lambda_i^{1/2} \phi_i(x)  \right\|_{\cH_k}^2 = k(x,x) - \sum_{i \in J}  \lambda_i \phi_i^2(x),
\end{align*}
where the last equality uses the reproducing property.

The convergence \eqref{eq:KL-uniform-mean-squre} follows from the identity~\eqref{eq:KL-expansion} and Mercer's theorem in Theorem \ref{theo:mercer}.
\end{proof}

Theorem~\ref{theo:KL-expansion} implies that a zero-mean GP with kernel $k$ can be written as the sum of independent standard Gaussian random variables~$Z_i$ times the orthonormal basis functions $\lambda_i^{1/2} \phi_i$ of the RKHS of $k$:
\begin{equation} \label{eq:KL-exp-informal}
\ssf(x) = \sum_{i \in I} Z_i \lambda_i^{1/2} \phi_i(x) \quad \text{for all }\  x \in \cX,
\end{equation}
where the convergence is in the mean-square sense.
This is usually called the KL expansion.

The KL expansion suggests, informally, that a GP sample is smoother if high-frequency eigenfunctions $\phi_i$ have smaller eigenvalues $\lambda_i$.
This is illustrated by the following example of a GP whose covariance is a periodic Sobolev kernel in Example~\ref{ex:periodic-sobolev}.

\begin{example}
The KL expansion of a zero-mean GP whose covariance function is the $s$-th order periodic kernel $k_s$ in Example~\ref{ex:periodic-sobolev} with $s \in \mathbb{N}$ is given by
\begin{align*}
F(x)
&=
Z_1
+
\sum_{i=2,4,6,\dots}
Z_i (i/2)^{-s}\sqrt{2}\cos(i\pi x) \\
&\quad+
\sum_{i=3,5,7,\dots}
Z_i ((i-1)/2)^{-s}\sqrt{2}\sin((i-1)\pi x),
\qquad x\in[0,1],
\end{align*}
where $Z_1, Z_2, \dots \stackrel{i.i.d.}{\sim} \mathcal{N}(0,1)$.
This suggests that, similar to the Mercer representation in Example~\ref{ex:Mercer-periodic-Sobolev}, the GP is smoother if the order $s$ is larger, as the Fourier basis functions with high frequencies (large indices $i$) contribute less significantly.

\end{example}

\begin{remark}
The independent standard Gaussian random variables~$Z_i$ in \eqref{eq:KL-normal-varibles} are not independent of the given process $F$; each $Z_i$ is obtained as the $L_2(\nu)$ inner product between $F$ and the eigenfunction $\phi_i$ divided by the eigenvalue's square root $\sqrt{\lambda_i}$:
\begin{align*}
 &    Z_i = \psi_k(\lambda^{1/2}_i\phi_i) = \psi_k\left(\lambda_i^{-1/2}\int k(\cdot,x) \phi_i(x) d
 \nu(x) \right) \\
&  = \lambda_i^{-1/2}\int \psi_k\left(k(\cdot,x)\right) \phi_i(x) d
 \nu(x)   = \lambda_i^{-1/2} \int \ssf(x) \phi_i(x) d\nu(x),
\end{align*}
where the second identity follows from $\phi_i$ and $\lambda_i$ being an eigenfunction and the associated eigenvalue of the integral operator \eqref{eq:integral_operator}.
\end{remark}

\section{Sample Path Properties and the Zero-One Law} \label{sec:driscol-sample-path}

\index{Driscoll's theorem}
{\em Driscoll's theorem} \citep{driscoll1973reproducing} offers a useful tool to understand the properties of a GP sample.
It provides the necessary and sufficient condition for a GP sample to be included in a {\em given} RKHS, which can differ from the GP kernel's RKHS.
The properties of a GP sample are characterized by the properties of the given RKHS that includes the sample.

\index{zero--one law}
The theorem is called  {\em zero-one law} because the probability that a GP sample belongs to a given RKHS is either zero or one. That is, a statement like ``a GP sample belongs to a given RKHS with probability 0.3'' never holds.

For simplicity, we focus on a GP whose sample path is continuous with probability one.
See, e.g., \citet[Chapter 1]{AdlJon07} and \citet[Section 3.2]{da2014stochastic} and references therein for sufficient (and necessary) conditions for a GP to be continuous, which are rather technical to be presented here.
An intuitive, sufficient condition for continuity is that the eigenvalues of the integral operator decay ``not too slowly,'' as described later.

To state Driscoll's zero-one law, consider an RKHS $\cH_r$ with reproducing kernel $r$ that  subsumes the RKHS of the GP covariance kernel~$k$:
$$
\cH_k \subset \cH_r.
$$
\index{inclusion operator}
Let $I_{kr}: \cH_k \to \cH_r$ be the inclusion operator that maps any element in $\cH_k$ to the identical element in $\cH_r$:
\begin{equation} \label{eq:inclusion-op}
I_{kr}: \cH_k \to \cH_r:  f \mapsto f.
\end{equation}

\index{Hilbert--Schmidt operator}
Theorem~\ref{theo:gen-Driscol-theorem} below states that a GP sample with covariance $k$ belongs to the RKHS $\cH_r$ of kernel $r$ with probability one, if and only if the inclusion operator~\eqref{eq:inclusion-op} is {\em Hilbert--Schmidt}, i.e., the sum of the squared $\cH_r$-norms of the orthonormal basis functions of $\cH_k$ is finite.\footnote{
A bounded linear operator $A :\cH \to \cF$ between two Hilbert spaces $\cH$ and $\cF$ is called {\em Hilbert--Schmidt} if its Hilbert--Schmidt norm is finite: $\| A \|_{\rm HS}^2 := \sum_{i \in I} \left\| A\psi_i \right\|_{\cF}^2 < \infty$, where $(\psi_i)_{i \in I}$ is an arbitrary orthonormal basis of $\cH$.
}
Intuitively, this requires that $\cH_r$ be large enough compared with $\cH_k$, as discussed later.

Theorem~\ref{theo:gen-Driscol-theorem} is a version of the zero-one law of \citet[Theorem 3]{driscoll1973reproducing} derived from a result of \citet[Theorem 7.5]{LukBed01}.
The proof is given in Appendix \ref{sec:Proof-Driscol}.

\begin{theorem}[A generalized Driscoll's theorem] \label{theo:gen-Driscol-theorem}
Let $k$ and $r$ be kernels on a compact metric space $\cX$ with respective RKHSs $\cH_k$ and $\cH_r$  such that  $\cH_k \subset \cH_r$, and  $I_{kr}: \cH_k \to \cH_r$ be the  inclusion operator~\eqref{eq:inclusion-op}.
Let $\ssf \sim \GP(m,k)$ be a GP with mean function $m$ and covariance kernel $k$  such that $m \in \cH_r$ and $\ssf$ is continuous on $\cX$ with probability one.
Then, the following statements are true.
\begin{enumerate}
\item If $I_{kr}$ is Hilbert--Schmidt, then  $F \in \cH_r$ holds with probability one.
\item If $I_{kr}$ is not Hilbert--Schmidt, then $F \in \cH_r$ holds with probability zero.
\end{enumerate}
\end{theorem}

\paragraph{Hilbert--Schmidt Condition}
We now discuss the condition of the inclusion operator $I_{kr}$ being Hilbert--Schmidt.
As mentioned earlier, this condition is that the sum of the squared $\cH_r$-norms of the orthonormal basis functions of $\cH_k$ is finite.
If the kernel $k$ is continuous,  Theorem~\ref{theo:mercer-RKHS} implies that the eigenfunctions times the eigenvalues' square roots, $(\lambda_i^{1/2} \phi_i)_{i \in I}$, are orthonormal basis
functions of $\cH_k$, and thus the condition is written as
\begin{align}
    \| I_{kr} \|_{\rm HS}^2 &:= \sum_{i \in I}  \| I_{kr} \lambda^{1/2}_i \phi_i   \|_{\cH_r}^2 \nonumber \\
    &  =  \sum_{i \in I}   \| \lambda_i^{1/2} \phi_i  \|_{\cH_r}^2 = \sum_{i \in I}   \lambda_i \| \phi_i  \|_{\cH_r}^2   < \infty, \label{eq:HS-norm-inclusion-op}
\end{align}
where $\| I_{kr} \|_{\rm HS}$ denotes the Hilbert--Schmidt norm of $I_{kr}$.
From the last expression, the condition is that
$\lambda_i \| \phi_i \|_{\cH_r}^2$ is summable over $i \in I$.
Therefore, for the condition to hold, the $\cH_r$-norms of the
eigenfunctions should not grow too rapidly relative to the decay of the
eigenvalues.

Let us first consider the case where the kernel $r$ is the same as the GP covariance kernel: $r = k$.
Suppose that the RKHS $\cH_k$ is infinite-dimensional so that the cardinality of the index set $I$ is infinite, as for many commonly used kernels.
Then the condition~\eqref{eq:HS-norm-inclusion-op} is {\em not} satisfied, because
$$
\sum_{i \in I}   \| \lambda_i^{1/2} \phi_i  \|_{\cH_r}^2 = \sum_{i \in I} 1 = \infty,
$$
since $(\lambda_i^{1/2} \phi_i)_{i \in I}$ are orthonormal in $\cH_r = \cH_k$.
\index{Gaussian process!sample path outside covariance RKHS}
Therefore, Theorem~\ref{theo:gen-Driscol-theorem} implies that a GP sample does {\em not} belong to its covariance kernel's RKHS with probability one.
This recovers a classical result stated by \citet{Par61} and proved by \citet{Kallianpur1970RKHS}; see also \citet[p.~5]{wahba1990spline} and \citet[Corollary 7.1]{LukBed01}.
This is because the covariance kernel's RKHS is not large enough to contain a GP sample, which is less smooth than the functions in the RKHS; see Figure~\ref{fig:GP_intro} for an illustration.

\begin{remark}
 \citet[p.~5]{wahba1990spline} heuristically shows that a GP sample does not belong to its covariance kernel's RKHS with probability one based on the KL expansion~\eqref{eq:KL-exp-informal}; see also \citet[p.~66]{Berlinet2004} and \citet[Section 6.1]{RasmussenWilliams}.

For $F \sim \GP(0,k)$, let
$\ssf = \sum_{i=1}^\infty Z_i \lambda_i^{1/2} \phi_i$ be the KL expansion.
For any $m \in \mathbb{N}$, consider its finite truncation with the first $m$ terms $F_m := \sum_{i=1}^m Z_i \lambda_i^{1/2} \phi_i$.
Then, the expected squared RKHS norm of $F_m$ is $m$, since
$$
  \bE[\| F_m \|_{\cH_k}^2] = \bE \left[ \sum_{i=1}^m Z_i^2  \right] =  \sum_{i=1}^m \bE[ Z_i^2] = \sum_{i=1}^m 1 = m.
$$
Therefore, as $m$ goes to infinity the expected squared RKHS norm diverges: $\lim_{m \to \infty}  \bE[\| \ssf_m \|_{\cH_k}^2] = \infty$.
Since the GP sample $F$ is the limit of $F_m$ as $m \to \infty$, one may heuristically expect that the expected squared RKHS norm of $F$ is infinite, i.e.,
$\bE[\|F\|_{\cH_k}^2] = \infty$.
This is consistent with the fact that $F$ does not belong to the RKHS with probability one.

Note that, while this argument may be intuitively plausible, it is {\em not} a proof.
As shown in Theorem~\ref{theo:KL-expansion}, the most basic result for the convergence of the KL expansion $\ssf = \lim_{m \to \infty} \ssf_m$ is in the {\em mean-square sense}, which is {\em weaker} than the convergence in the RKHS norm, and thus it does not imply $ \lim_{m \to \infty} \left\| F - F_m \right\|_{\cH_k} = 0$.
Hence, $\lim_{m \to \infty}  \bE[\| \ssf_m \|_{\cH_k}^2] = \infty$ does {\em not} imply $\bE[\| \ssf \|_{\cH_k}^2] = \infty$.
This shows the importance of understanding the convergence type of the KL expansion, which is investigated and used for establishing GP-sample path properties by \citet{steinwart2019convergence}.
\end{remark}

Let us now consider a kernel $r$ different from the GP's covariance kernel~$k$.
Define $r$ by replacing the eigenvalues $\lambda_1 \geq \lambda_2 \geq \cdots > 0$ in the Mercer expansion of the kernel $k$ in \eqref{eq:mercer} by weights  $\gamma_1 \geq \gamma_2 \geq \cdots > 0 $:
\begin{equation} \label{eq:kernel-scaling}
r(x,x') := \sum_{i \in I} \gamma_i \phi_i(x) \phi_i(x').
\end{equation}
Since this is the Mercer expansion of $r$, its RKHS $\cH_r$ can be expanded as (Theorem~\ref{theo:mercer-RKHS})
\begin{equation} \label{eq:RKHS-with-diff-weights}
 \cH_r = \left\{ f := \sum_{i \in I} \alpha_i \gamma_i^{1/2} \phi_i\,:\ \| f \|_{\cH_r}^2 := \sum_{i \in I} \alpha_i^2 < \infty \right\},
\end{equation}
with $(\gamma_i^{1/2} \phi_i)_{i \in I}$ being an orthonormal basis of $\cH_r$.
In this case, the squared $\cH_r$-norm of the eigenfunction $\phi_i$ is $\gamma_i^{-1}$, as
$$
 \|  \phi_i  \|_{\cH_r}^2 =  \gamma_i^{-1} \| \gamma_i^{1/2} \phi_i  \|_{\cH_r}^2 = \gamma_i^{-1}.
$$
Therefore, the condition~\eqref{eq:HS-norm-inclusion-op} is written as
\begin{align}
\| I_{kr} \|_{\rm HS}^2
 =  \sum_{i \in I}  \lambda_i \gamma_i^{-1}  < \infty. \label{eq:HS-cond-eig}
\end{align}
Hence, the condition holds if most weights $\gamma_i$ are sufficiently larger than the corresponding eigenvalues $\lambda_i$.
This means the RKHS $\cH_r$ should be sufficiently larger than the covariance kernel's RKHS $\cH_k$ to contain the GP sample.
The condition \eqref{eq:HS-cond-eig} is essentially the same as the condition of \citet[Theorem 4.3]{karvonen2023small}.

The condition~\eqref{eq:HS-cond-eig} can be more informally derived from the KL expansion~\eqref{eq:KL-exp-informal}.
Write the GP sample using the KL expansion as
$$
F  = \sum_{i \in I} Z_i \lambda_i^{1/2} \phi_i = \sum_{i \in I} W_i \gamma_i^{1/2} \phi_i,
$$
where we defined the new coefficients $W_i := Z_i \lambda_i^{1/2} \gamma_i^{-1/2}$.
Since $(\gamma_i^{1/2}\phi_i)_{i \in I}$ form an orthonormal basis of $\cH_r$, the sample $F$ belongs to $\cH_r$ if and only if the squared sum of the coefficients is finite:
$
\sum_{i \in I} W_i^2 < \infty.
$
A sufficient condition for this to hold with probability one is that the expectation of the squared sum be finite.
This recovers the condition~\eqref{eq:HS-cond-eig}, since this expectation is identical to the sum of the eigenvalues $\lambda_i$ divided by the weights $\gamma_i$ in the condition~\eqref{eq:HS-cond-eig}:
$$
\mathbb{E}[  \sum_{i\in I} W_i^2 ] = \mathbb{E}[  \sum_{i\in I}  Z_i^2 \lambda_i \gamma_i^{-1} ]  = \sum_{i \in I} \lambda_i \gamma_i^{-1}.
$$

More generally, a necessary and sufficient condition for the embedding \(I_{kr}:\cH_k\to\cH_r\) to be Hilbert--Schmidt is given in terms of a metric-entropy integral \citep[Theorem~A]{GonDud93}. See also \citet[Corollary~5.4]{Ste17} for a related condition.\footnote{The 2017
preprint is cited because it differs substantially from the published
version \citep{steinwart2019convergence}.}

\begin{remark}
Theorem~\ref{theo:gen-Driscol-theorem} gives a necessary and sufficient condition for whether a sample path of a given GP belongs almost surely to a specified RKHS. A separate question is whether there exists any RKHS to which the GP sample path belongs almost surely. See \citet{Steinwart2026PathsRKHS} for general non-existence criteria.
\end{remark}

\section{A Powered RKHS as a GP Sample Space} \label{sec:power-RKHS-sample-path}

\index{powered RKHS}
Under appropriate conditions, an RKHS $\cH_r$ to which a GP sample belongs almost surely can be constructed as a {\em power} of the GP's RKHS $\cH_k$ \citep{steinwart2019convergence}.
This powered RKHS interpolates between the GP's RKHS and the space of square-integrable functions, and is parametrized by a constant $0 < \theta \leq 1$ that quantifies where the powered RKHS is ``located'' between the two spaces \citep{SteSco12}; a larger $\theta$ indicates that the powered RKHS is closer to the original RKHS.
This $\theta$ quantifies the GP sample's relative smoothness compared with the functions in the RKHS.

A powered RKHS is defined as follows \citep[Definition 4.1]{SteSco12}.
\begin{definition}[Powered RKHS] \label{def:power-RKHS}
Let $\cX$ be a compact metric space, $k$ be a continuous kernel on $\cX$ with RKHS $\cH_k$, and $\nu$ be a finite Borel measure whose support is $\cX$.
Let $0 < \theta \leq 1$ be a constant such that
\begin{equation} \label{eq:cond-powered-RKHS}
    \sum_{i \in I} \lambda_i^\theta \phi_i^2(x) < \infty \quad \text{for all}\ x \in \cX,
\end{equation}
where $(\lambda_i, \phi_i )_{i \in I}$ are the eigenvalues and eigenfunctions of the integral operator in \eqref{eq:integral-eigen-decomp}.
Then the $\theta$-th power of $\cH_k$ is defined as
\begin{align}
 \cH_k^\theta := \biggl\{f: \cX \to \mathbb{R}:\ & f
 = \sum_{i \in I} \alpha_i \lambda_i^{\theta/2} \phi_i \ \ \text{for some }  (\alpha_i)_{i \in I} \subset \mathbb{R} \nonumber \\
 & \text{such that}\ \| f \|_{\cH_k^\theta}^2 := \sum_{i \in I} \alpha_i^2 < \infty \biggr\}, \label{eq:power-RKHS}
\end{align}
and the inner product is given by
\begin{align*}
  & \left< f,g \right>_{\cH_k^\theta} = \sum_{i \in I} \alpha_i\beta_i \\ & \text{for all}
  \q f := \sum_{i \in I} \alpha_i \lambda_i^{\theta/2} \phi_i \in \cH_k^\theta,\q g := \sum_{i\in I} \beta_i \lambda_i^{\theta/2} \phi_i \in \cH_k^\theta.
\end{align*}

\end{definition}

The powered RKHS~\eqref{eq:power-RKHS} itself is an RKHS with the reproducing kernel given by
\begin{equation*} 
k^\theta(x,x') := \sum_{i\in I} \lambda_i^\theta \phi_i(x) \phi_i(x'), \quad x, x' \in \cX.
\end{equation*}
This is called the $\theta$-th power of the kernel $k$.
The condition~\eqref{eq:cond-powered-RKHS} is necessary for the powered kernel to be well-defined.

The powered RKHS and kernel are the special case of the RKHS~\eqref{eq:RKHS-with-diff-weights} and the kernel~\eqref{eq:kernel-scaling} where the weights are the $\theta$-th power of the eigenvalues: $\gamma_i = \lambda_i^\theta$.
If the power $\theta$ is close to zero, the powered eigenvalues $\lambda_i^\theta$ become close to one, making small eigenvalues larger.
Thus, as discussed earlier, the powered RKHS becomes larger.
On the other hand, if the power $\theta$ is one, then the powered RKHS is identical to the original RKHS.

\index{interpolation space}
More precisely, the powered RKHS $\cH_k^\theta$ is an {\em interpolation space} between the original RKHS $\cH_k$ and the space $L_2(\nu)$ of square-integrable functions, where the power $0 < \theta \leq 1$ determines how close $\cH_k^\theta$ is to $\cH_k$ \citep[Theorem 4.6]{SteSco12}.
If $\theta = 1$ we have $\cH_k^\theta = \cH_k$, and $\cH_k^\theta$ gets larger as $\theta$ decreases.  Indeed,  $\cH_k^\theta$ is nesting with respect to $\theta$:
\begin{equation*} \label{eq:nest_pow}
 \cH_k = \cH_k^1 \subset \cH_k^\theta \subset \cH_k^{\theta'}\subset L_2(\nu) \quad {\rm for\ all} \ \ 0 < \theta' < \theta < 1.
\end{equation*}
The powered RKHS contains the RKHS and functions that are less smooth than the RKHS functions. As the power $\theta$ becomes smaller, the powered RKHS contains many more less smooth functions.
This is illustrated by the following example on a power of a periodic Sobolev RKHS in Examples~\ref{ex:periodic-sobolev} and \ref{ex:Mercer-periodic-Sobolev}.

\index{periodic Sobolev space!powered RKHS}
\begin{example}[Power of a Periodic Sobolev RKHS]  \label{ex:power-periodic-sobolev-RKHS}
Let us consider the $\theta$-th power of the $s$-th order periodic Sobolev RKHS in Examples~\ref{ex:periodic-sobolev} and \ref{ex:Mercer-periodic-Sobolev} for $s \in \mathbb{N}$ and $0 < \theta < 1$.
Since the eigenvalues decay at the rate $\lambda_i = O(i^{-2s})$ and the eigenfunctions $\phi_i$ are uniformly bounded, the condition~\eqref{eq:cond-powered-RKHS} is satisfied if $
\sum_{i=1}^\infty i^{-2 s \theta} < \infty$,
which holds if $\theta > 1/2s$.
Thus, if $s = 1$ then $\theta$ should be greater than $1/2$; if $s = 2$ then $\theta$ should be greater than $1/4$, etc.

Each function $f$ in the powered RKHS $\cH_{k_s}^\theta$ is then written as
\begin{align*}
f(x)
&=
\alpha_1
+
\sum_{i=2,4,6,\dots}
\alpha_i (i/2)^{-s\theta}\sqrt{2}\cos(i\pi x) \\
&\quad+
\sum_{i=3,5,7,\dots}
\alpha_i ((i-1)/2)^{-s\theta}\sqrt{2}\sin((i-1)\pi x),
\qquad x\in[0,1],
\end{align*}
for some square-summable weight sequence $(\alpha_1, \alpha_2, \dots )$.
Hence, the powered RKHS is the {\em periodic Sobolev space of smoothness $s\theta$}.
In particular, if $s\theta$ is an integer, the powered RKHS is the RKHS of the $s\theta$-th order periodic Sobolev kernel:
$$
\cH_{k_s}^\theta = \cH_{k_{s\theta}}.
$$
For example, if $s = 2$ and $\theta = 1/2$, the original RKHS $\cH_{k_s}$ consists of square-integrable {\em twice-differentiable} functions, while the powered RKHS $\cH_{k_s}^\theta$ consists of square-integrable {\em once-differentiable} functions; thus the latter is much bigger than the former.

\end{example}

\index{Gaussian process!sample path in powered RKHS}
Theorem~\ref{theo:gp-path-power-RKHS} below, which follows from Theorem \ref{theo:gen-Driscol-theorem}, shows that a GP sample $F$ with covariance $k$ belongs to the powered RKHS $\cH_k^\theta$ with probability one if the sum of eigenvalues powered to $1-\theta$ is finite: $\sum_{i \in I} \lambda_i^{1-\theta} < \infty$.
This condition is weaker for smaller $\theta$, and the largest possible $\theta$ that satisfies the condition quantifies the relative smoothness of the GP sample $F$ compared to the original RKHS~$\cH_k$.
 This result is a special case of \citet[Theorem 5.2]{steinwart2019convergence}, which holds under weaker assumptions on $\cX$, $k$ and $\nu$.
\begin{theorem} \label{theo:gp-path-power-RKHS}
Let $\cX$ be a compact metric space, $k$ be a continuous kernel on $\cX$ with RKHS $\cH_k$ such that $F \sim \GP(0,k)$ is continuous with probability one, and $\nu$ be a finite Borel measure whose support is $\cX$.
Let $0 < \theta < 1$ be a constant such that $\sum_{i \in I} \lambda_i^\theta \phi_i^2(x) < \infty$ holds for all $x \in \cX$, where $(\lambda_i, \phi_i )_{i \in I}$ are the eigenvalues and eigenfunctions of the integral operator in \eqref{eq:integral-eigen-decomp}.
Then, the following statements are equivalent.
\begin{enumerate}
\item $\sum_{i \in I} \lambda_{i}^{1-\theta} < \infty$.
\item The inclusion operator $I_{k k^\theta}: \cH_k \to \cH_k^\theta$ is Hilbert--Schmidt.
\item   $F \in \cH_k^\theta$ with probability one.
\end{enumerate}
\end{theorem}

\begin{proof}
The equivalence between 2.~and 3.~follows from Theorem \ref{theo:gen-Driscol-theorem} and that $\cH_k^\theta$ is an RKHS with $k^\theta$ being its kernel.
The equivalence between 1.~and 2.~follows from that the Hilbert--Schmidt norm of the inclusion operator $I_{k k^\theta}$  is  $\sum_{i \in I} \lambda_i^{1-\theta}$:
$$
\| I_{k k^\theta} \|_{\rm HS}^2
= \sum_{i \in I} \| I_{k k^\theta} \lambda_i^{1/2} \phi_i \|_{\cH_k^\theta}^2
= \sum_{i \in I} \| \lambda_i^{(1-\theta)/2} \lambda_i^{\theta/2} \phi_i \|_{\cH_k^\theta}^2
=  \sum_{i \in I} \lambda_i^{1-\theta},
$$
where the last equality follows from that $(\lambda_i^{\theta/2}\phi_i)_{i \in I}$ are orthonormal in $\cH_k^\theta$.
\end{proof}

As mentioned, Theorem~\ref{theo:gp-path-power-RKHS} implies that the largest possible $\theta$ satisfying the condition $\sum_{i \in I} \lambda_i^{1-\theta} < \infty$ quantifies the relative smoothness of the GP sample $F$ compared with the RKHS $\cH_k$.
The following example on a GP with a periodic Sobolev kernel illustrates this.

\begin{example} \label{ex:periodic-sobolev-GP-smoothness}
Let $s\geq 2$.
As mentioned in Example~\ref{ex:power-periodic-sobolev-RKHS}, the eigenvalues associated with the $s$-th order periodic Sobolev kernel $k_s$ decay at the rate $\lambda_i \asymp i^{-2s}$, and the $\theta$-th powered RKHS $\cH_{k_s}^\theta$ is the $s\theta$-th order periodic Sobolev space.
The powered RKHS is well-defined for $\theta>1/(2s)$.
In this case, the eigenvalue condition in Theorem~\ref{theo:gp-path-power-RKHS} becomes
\[
    \sum_{i=1}^\infty i^{-2s(1-\theta)} < \infty,
\]
which holds if and only if
\[
    \theta < \frac{2s-1}{2s}.
\]
Thus, for
\[
    \frac{1}{2s}<\theta<\frac{2s-1}{2s},
\]
a sample $F\sim\GP(0,k_s)$ belongs almost surely to the periodic Sobolev space of order $s\theta$.

The upper bound on $\theta$ corresponds to the Sobolev order
\[
    \frac{2s-1}{2s}s
    =
    s-\frac12.
\]
Hence, by the nesting of Sobolev spaces, the sample belongs almost surely to every periodic Sobolev space of order strictly smaller than $s-1/2$, whereas the RKHS has Sobolev order $s$.
\end{example}

Other examples of GP sample properties are given for squared-exponential and Mat\'ern kernels in the next section.

\section{Examples of GP Sample Properties} \label{sec:examples-sample-path}

Based on the above generic results, we derive sample properties of GPs with squared-exponential and Mat\'ern kernels.

Corollary~\ref{coro:gauss-sample-path} below on a squared-exponential kernel follows from Theorem~\ref{theo:gp-path-power-RKHS} and the fact that the eigenvalues $\lambda_i$ in this case decay faster than any polynomial rate as $i \to \infty$.
The proof is given in Appendix~\ref{sec:proof-gp-sample-path-se}.

\index{squared-exponential kernel!GP sample paths}
\begin{corollary}[Sample Properties for Squared-Exponential Kernels] \label{coro:gauss-sample-path}
Let $\cX \subset \Re^d$ be a compact set with Lipschitz boundary, $\nu$ be the Lebesgue measure on $\cX$, $k_\gamma(x,x') = \exp(- \|x - x' \|^2 / \gamma^2)$ be the squared-exponential kernel on $\cX$ with bandwidth $\gamma > 0$, and $\cH_{k_\gamma}$ be its RKHS.
Then for all $0 < \theta < 1$, the $\theta$-th power $\cH_{k_\gamma}^\theta$ of $\cH_{k_\gamma}$ in Definition \ref{def:power-RKHS} is well-defined.
Moreover, for $\ssf \sim \GP(0,k_\gamma)$,  we have  $F \in \cH_{k_\gamma}^\theta$ with probability one for all $0 < \theta < 1$.
\end{corollary}

Corollary~\ref{coro:gauss-sample-path} shows that a GP sample $F \sim \GP(0,k_\gamma)$ having a squared-exponential covariance $k_\gamma$ belongs to the $\theta$-th powered RKHS $\cH_{k_\gamma}^\theta$ with probability one for any power $\theta$ arbitrarily close to one.
Thus, the ``relative smoothness'' of the GP sample compared with the original RKHS is one. Indeed, the $\theta$-th powered RKHS here consists of infinitely differentiable functions, so the GP sample is also infinitely differentiable.
While the GP sample does not belong to the original RKHS, the GP sample has very similar properties to it.

Corollary \ref{coro:matern-sample-path} below provides GP sample properties for Mat\'ern kernels.
It requires the {\em interior cone condition} \citep[Definition 3.6]{Wen05} that there is no ``pinch point'' (i.e.,~a $\prec$-shaped region) on the boundary of the input space $\cX \subset \mathbb{R}^d$.\footnote{A set $\cX \subset \Re^d$ is said to satisfy the interior cone condition if there exist an angle $\theta \in (0,\pi/2)$ and a radius $R > 0$ such that every $x \in \cX$ is associated with a unit vector $\xi(x)$ so that the cone $C(x, \xi(x), \theta, R)$ is contained in $\cX$,
where
$C(x, \xi(x), \theta, R) := \{ x + a y:\ y \in \Re^d,\ \|y\| = 1,\ \left<y, \xi(x) \right> \geq \cos \theta,\ a \in [0,R] \}$.
}
For example, this condition is satisfied if $\cX$ is a cube or a ball in $\mathbb{R}^d$.

\index{Mat\'ern kernel!GP sample paths}
\begin{corollary}[Sample Properties for Mat\'ern  Kernels] \label{coro:matern-sample-path}
Let $\cX \subset \Re^d$ be a compact subset having a Lipschitz boundary and satisfying the interior cone condition.
Let $k_{\alpha}$ be the Mat\'ern kernel on $\cX$ in Example~\ref{ex:matern-kernel} with smoothness $\alpha > 0$ and bandwidth $h > 0$ such that $\alpha + d/2 \in \mathbb{N}$.
Similarly, let $k_{\alpha'}$ be the Mat\'ern kernel on $\cX$ with smoothness $\alpha' > 0$ and bandwidth $h' > 0$ such that $\alpha' + d/2 \in \mathbb{N}$, and $\cH_{k_{\alpha'}}$ be its RKHS.
Then if
\begin{equation} \label{eq:Matern-smoothness-cond}
    \alpha  > \alpha' + d/2,
\end{equation}
a GP sample $\ssf \sim \GP(0,k_{\alpha})$ satisfies $F \in \cH_{k_{\alpha'}}$ with probability one.
\end{corollary}
\begin{proof}
Let $W_2^s(\cX)$ and $W_2^\beta(\cX)$ be the Sobolev spaces of order
$$
s := \alpha + d/2 \quad \text{and} \quad  \beta := \alpha' + d/2,
$$
respectively, as defined in Example~\ref{ex:matern-rkhs}.
Since $\cX$ satisfies the interior cone condition and  $s - \beta > d/2$,  Maurin's theorem~\citep[Theorem 6.61]{AdaFou03} implies that the embedding $W_2^s(\cX) \to W_2^\beta(\cX)$ is Hilbert--Schmidt.
Moreover, since the boundary of $\cX$ is Lipschitz, the RKHS $\cH_{k_{\alpha}}$ of $k_{\alpha}$ is norm-equivalent to $W_2^s(\cX)$, and $\cH_{k_{\alpha'}}$ is norm-equivalent to $W_2^\beta(\cX)$ \citep[Corollary 10.48]{Wen05}
 (See also Example \ref{ex:matern-rkhs}).
Therefore the embedding $\cH_{k_{\alpha}} \to \cH_{k_{\alpha'}} $ is also Hilbert--Schmidt.
The assertion then follows from Theorem \ref{theo:gen-Driscol-theorem}.
\end{proof}

Corollary~\ref{coro:matern-sample-path} essentially shows that the smoothness parameter $\alpha$ of the Mat\'ern kernel $k_\alpha$ quantifies the smoothness of the resulting GP sample $F \sim \GP(0, k_\alpha)$.
Let $\alpha'$ be the largest possible constant satisfying the condition~\eqref{eq:Matern-smoothness-cond} so that
$$
\alpha \approx \alpha' + d/2.
$$
This value $\alpha'+d/2$ is the smoothness of the RKHS $\cH_{k_{\alpha'}}$ of the Mat\'ern kernel $k_{\alpha'}$ with smoothness parameter $\alpha'$.
Therefore, the GP sample $F \sim \GP(0, k_\alpha)$ belonging to $\cH_{k_{\alpha'}}$ implies that its smoothness is $\alpha'+d/2$, which is arbitrarily close to $\alpha$.

\index{Gaussian process!sample path smoothness}
\index{reproducing kernel Hilbert space!smoothness versus GP sample paths}
The smoothness $\alpha$ of the GP sample $F \sim \GP(0, k_\alpha)$ is $d/2$-lower than the smoothness of the covariance kernel's RKHS $\cH_{k_\alpha}$, which is norm-equivalent to the Sobolev space of order $\alpha + d/2$.
This is consistent with Example~\ref{ex:periodic-sobolev-GP-smoothness} where a GP sample of a periodic Sobolev kernel on a one-dimensional input space is shown to have smoothness $1/2$-smaller than the covariance kernel's RKHS.

\begin{remark}
The condition $\alpha > \alpha' + d/2$ in \eqref{eq:Matern-smoothness-cond} is sharp: if $\alpha = \alpha'+ d/2$, the sample $\ssf \sim \GP(0,k_{\alpha})$ does {\em not} belong to $\cH_{k_{\alpha'}}$ with probability one \citep[Corollary 5.7, ii]{steinwart2019convergence}.
The condition $\alpha' + d/2 \in \mathbb{N}$ may also be removed; $\alpha$ can be any positive real satisfying $\alpha > \alpha' + d/2$  \citep[Corollary 5.7, i]{steinwart2019convergence}.
\end{remark}

\chapter[Linear Functionals of GPs and RKHS Representations]{Linear Functionals of Gaussian Processes and RKHS Representations}
\label{sec:fund-equiv-gp-RKHS}

\index{linear functional!of a Gaussian process}
This chapter discusses {\em linear functionals} of Gaussian processes, which are fundamental in many applications of Gaussian processes.
A linear functional of a GP is its real-valued linear function.
Denoting such a linear functional by $A(F)$ with $F$ being a GP, examples include
\begin{itemize}
    \item   Function value evaluation: $A(F) = F(x)$ for any fixed $x \in \cX$;
    \item Integral: $A(F) := \int F(x) dP(x)$ for a probability measure $P$ on a measurable $\cX$;
    \item Partial derivative: $A(F) := \frac{\partial F}{\partial x_i}(x)$ for any fixed $x = (x_1,\dots,x_d)^\top \in \cX \subset \mathbb{R}^d$ and for $i = 1,\dots,d$.
\end{itemize}
To define such a linear functional, one might first consider the sample space of a GP and then define the functional on it.
However, this approach requires understanding the GP sample properties, which may not be straightforward. Moreover, there exist many important linear functionals that can never be defined in this way, such as stochastic integrals, as described later.

This chapter describes how linear functionals of a GP can be defined using the corresponding RKHS representations in a unifying manner.\footnote{An early development of this viewpoint appears in \citet[Sections~3--4]{Par61}.}  Based on this, we will observe a fundamental equivalence between the GP and RKHS approaches: the root mean square of a linear functional  with respect to the GP, $\sqrt{ \mathbb{E}_{F \sim \GP(0,k)} [A(F)^2] }$,  is identical to the maximum absolute value of the linear functional over the unit ball of the RKHS, $\sup_{\| f \|_{\cH_k} \leq 1 } \left| A(f) \right|$:
$$
\sqrt{ \mathbb{E}_{F \sim \GP(0,k)} [A(F)^2] } = \sup_{\| f \|_{\cH_k} \leq 1 } \left| A(f) \right|.
$$
This equivalence leads to different equivalences between the GP and RKHS approaches, as described in later sections.

Section~\ref{sec:lin-func-GP} first provides a unifying definition of linear functionals of a GP, with a motivating example of stochastic integrals.
Section~\ref{sec:examples-functionals} then describes how this definition leads to the natural definition of linear functionals in individual cases, such as function-value evaluation, integrals, and derivatives.
Section~\ref{sec:master-theorem} formalizes and discusses the equivalence between the root mean square of a linear functional for a GP and the maximum absolute value of the linear functional over the unit ball of the RKHS.

\section{Linear Functionals of a Gaussian Process}
\label{sec:lin-func-GP}

A naive way of defining a linear functional  $A(F)$ of a Gaussian process  $F \sim \GP(0,k)$ might be to first consider its samples and then define the functional by applying it to those samples.
However, this two-step approach causes some technical issues.
One issue is that, in this way, if one takes an element $h$ in the RKHS $\mathcal{H}_k$, one cannot define the inner product between the GP sample $F$ and $h$, i.e.,
\begin{equation}\label{eq:rkhs_inner_product_functional}
A(F) = \left< F, h \right>_{\cH_k}
\end{equation}
is not well-defined because $F \sim \GP(0,k)$ is {\em not} included in the RKHS $\cH_k$ almost surely (when the RKHS $\cH_k$ is infinite-dimensional), as we saw in the previous chapter.
However, there are important situations where precisely the functionals of the form \eqref{eq:rkhs_inner_product_functional} are of interest, a notable example being  {\em stochastic integrals}.

\index{stochastic integral}
\index{Brownian motion}
\begin{example}[Stochastic Integrals]
\label{ex:stoch-int-motivation}
Let $F \sim \GP(0,k)$ be the Wiener process (or the Brownian motion; see Figure~\ref{fig:GP_intro}) on $\cX = [0,1]$. In this case, the kernel is $k(x,x') = \min(x,x')$ and the RKHS $\cH_k$ is (see, e.g., \citet[p.~68]{AdlJon07})
\begin{align}\label{eq:RKHS-Brownian}
\cH_k = \biggl\{ & f \in L_2([0,1]):
D f \ \text{exists,}\\
& f(x) = \int_0^x Df(t) dt , \ \text{and}\ \int_0^1 (D f (x))^2 dx < \infty
  \biggr\}, \nonumber
\end{align}
where $D f$ denotes the weak derivative of $f$. Namely, the RKHS consists of functions having square-integrable weak derivatives (i.e., the first-order Sobolev space). Its inner product is given by
$$
 \left< f, h \right>_{\cH_k} := \int_0^1 Df (x) Dh (x) dx \quad \text{for}\quad f, h \in \cH_k.
$$

\index{Paley--Wiener stochastic integral}
\index{Wiener integral|see{Paley--Wiener stochastic integral}}
\index{It\=o integral}
For the Wiener process $F \sim \GP(0,k)$, the {\em Paley--Wiener stochastic integral}\footnote{This is often called the {\em Wiener integral}.}~\citep{PalWieZyg33} of a function $g \in L_2([0,1])$ (or the {\em It{\=o} integral}, when $g$ is a square-integrable adapted stochastic process) is the integral of $g$ with respect to infinitesimal changes of $F$ and is usually denoted by $\int_0^1 g(x) dF(x)$. 
Thus, it is essentially the $L_2$ inner product of $g$ and the weak derivative of $F$, i.e., $\int_0^1 g(x) DF(x) dx$.
In other words, if we define an RKHS function $h \in \cH_k$ by  $h(x) = \int_0^x g(t)dt$ so that $g = Dh$, then the stochastic integral is given by the RKHS inner product between $h$ and $F$:
\begin{equation} \label{eq:stoch-int-informal}
    \int_0^1 g(x) dF(x) := \left< F, h \right>_{\cH_k} , \quad \text{where} \quad h(x) = \int_0^x g(t)dt.
\end{equation}
However, since  $F \sim \GP(0,k)$ does {\em not} belong to $\cH_k$ almost surely,  $F$ does not have a square-integrable derivative $DF$. Therefore, the stochastic integral cannot be defined in this two-step approach.

\end{example}

\paragraph{A ``Direct Definition''}
Is there a way to ``directly define'' the inner product $\left<F, h \right>_{\cH_k}$?
This is possible using the Gaussian Hilbert space $\cG_k$ introduced in Section~\ref{sec:GHS}, which is equivalent to the RKHS $\cH_k$ through the canonical isometry $\psi_k: \cH_k \to \cG_k$.
To provide an intuitive idea, let us consider the KL expansion \eqref{eq:KL-exp-informal}\footnote{Here, we assume that the index set $I$ in Chapter~\ref{sec:theory} is the set of natural numbers, $I = \{1, 2, \dots,\}$, without loss of generality. Recall that $\cX$ being a compact metric space and $k$ being continuous imply that $\cH_k$ is separable~\citep[Lemma 4.33]{Steinwart2008}.} of $F \sim \GP(0,k)$ in Chapter~\ref{sec:theory}, assuming that $k$ is continuous on a compact metric space $\cX$:
$$
F = \sum_{i=1}^\infty Z_i \lambda_i^{1/2} \phi_i, \quad Z_1, Z_2, \dots \stackrel{i.i.d.}{\sim} \mathcal{N}(0,1),
$$
where $(\lambda_i^{1/2} \phi_i)_{i=1}^\infty$ form an orthonormal basis of $\cH_k$, so that
$$
h = \sum_{i=1}^\infty \lambda_i^{1/2} \phi_i \left< \lambda_i^{1/2} \phi_i, h \right>_{\cH_k} \quad \text{and} \quad Z_i = \psi_k ( \lambda_i^{1/2} \phi_i ) \in \cG_k.
$$
Let $h_n$ and $F_n$ be, respectively, the truncations of the orthonormal expansion of $h$ and the KL expansion of $F$ after the first $n$ terms:
$$
h_n :=  \sum_{i=1}^n \lambda_i^{1/2} \phi_i \left< \lambda_i^{1/2} \phi_i, h \right>_{\cH_k} \quad \text{and} \quad \quad F_n := \sum_{i=1}^n Z_i \lambda_i^{1/2} \phi_i
$$
The $F_n$ belongs to the RKHS $\cH_k$ with probability one, and thus its inner product with the RKHS element $h$ is well defined:
\begin{align*}
    & \left<F_n, h \right>_{\cH_k}
    = \left<\sum_{i=1}^n Z_i \lambda_i^{1/2} \phi_i, h \right>_{\cH_k}
    = \sum_{i=1}^n Z_i  \left< \lambda_i^{1/2} \phi_i, h \right>_{\cH_k}  \\
    & = \sum_{i=1}^n \psi_k(\lambda_i^{1/2} \phi_i)  \left< \lambda_i^{1/2} \phi_i, h \right>_{\cH_k}
    = \psi_k\left(  \sum_{i=1}^n \lambda_i^{1/2} \phi_i \left< \lambda_i^{1/2} \phi_i, h \right>_{\cH_k}  \right)  \\
    & = \psi_k\left(  h_n \right).
\end{align*}
Taking the limit as $n$ goes to infinity,
$$
\lim_{n \to \infty} \left< F_n, h \right>_{\cH_k} = \lim_{n \to \infty} \psi_k\left(  h_n \right) =  \psi_k\left( \lim_{n \to \infty}  h_n \right) = \psi_k\left( h \right),
$$
where the second equality holds because the canonical isometry $\psi_k$ is norm-preserving.
\index{canonical isometry!linear functionals}
Thus, the inner product $\left<F, h \right>_{\cH_k}$ between a Gaussian process $F$ and the RKHS element $h$ may be {\em defined} as the Gaussian random variable corresponding to $h$ via the canonical isometry:
$$
\left<F, h \right>_{\cH_k} := \lim_{n \to \infty} \left< F_n, h \right>_{\cH_k} = \psi_k\left( h \right) \in \cG_k.
$$
Again, the first equality is a definition and not an identity, because the limit $\lim_{n \to \infty} F_n = F$ does not belong to the RKHS.

\index{Riesz representation!of a linear functional}
Note that, by the Riesz representation theorem, {\em any} bounded linear functional $A$ on the RKHS $\cH_k$ can be written as the inner product with a uniquely associated RKHS element, say $f_A \in \cH_k$, called the Riesz representation of $A$:
\begin{equation} \label{eq:Riesz-repr-functional}
    A(f) = \left< f, f_A \right>_{\cH_k} \quad \text{for all } ~ f \in \cH_k.
\end{equation}
Thus, the linear functional $A$ applied to a Gaussian process $F$ can be defined as the Gaussian random variable corresponding to the Riesz representation via the canonical isometry, as formulated below.

\index{bounded linear functional}
\begin{definition}[Linear Functionals of a Gaussian Process] \label{def:GP-linear-functional}
Let $F \sim \GP(0,k)$ be a GP with covariance kernel $k$ on a non-empty set $\cX$, $\cG_k$ be the associated Gaussian Hilbert space, $\cH_k$ be the RKHS of $k$, and $\psi_k: \cH_k \to \cG_k$ be the canonical isometry (see Section~\ref{sec:canonical-isomet-GHS-RKHS}).
Let $A: \cH_k \to \mathbb{R}$ be a bounded linear functional on $\cH_k$ with Riesz representation $f_A \in \cH_k$ satisfying~\eqref{eq:Riesz-repr-functional}.
Then the linear functional $A$ applied to $F$ is defined as the Gaussian random variable in $\cG_k$ corresponding to $f_A$ via the canonical isometry $\psi_k$:
$$
A(F) := \psi_k (f_A) \in \cG_k.
$$
\end{definition}

Definition~\ref{def:GP-linear-functional} is generically applicable since no assumption is needed except that $\cX$ is a non-empty set and $k$ is a kernel on $\cX$.

\begin{remark}
    \index{measurable linear functional}
    It can be shown that a linear functional defined as in Definition~\ref{def:GP-linear-functional} is a ``measurable linear functional'' in the theory of Gaussian measures (\citealp[Definition~2.10.1]{Bog98};
 \citealp[p.~158]{kukush2020gaussian}).
\end{remark}

While Definition~\ref{def:GP-linear-functional} may appear abstract, it leads to natural definitions of many linear functionals, as described in the following.
As a preparation, Proposition~\ref{prop:functional-explicit-const} below shows that defining the linear functional $A(F)$ as in Definition~\ref{def:GP-linear-functional} is equivalent to defining it as a weighted sum of (potentially infinitely many) function values of the Gaussian process $F$.

\index{mean-square convergence}
\begin{proposition} \label{prop:functional-explicit-const}
Let $F \sim \GP(0,k)$ be a GP with covariance kernel $k$ on a non-empty set $\cX$,  $\cH_k$ be the RKHS of $k$,  and $A: \cH_k \to \mathbb{R}$ be a bounded linear functional on $\cH_k$.
Let $A(F)$ be the linear functional $A$ applied to $F$ as defined in Definition~\ref{def:GP-linear-functional}.
Then there exist some coefficients $(c^{(n)}_i)_{i=1}^n \subset \mathbb{R}$ and input points $(x^{(n)}_i)_{i=1}^n \subset \cX$ for each $n \in \mathbb{N}$ such that
\begin{equation} \label{eq:linear-functional-limit-form}
A(\ssf)  = \lim_{n \to \infty} \sum_{i=1}^n c^{(n)}_i  \ssf (x_i^{(n)}).
\end{equation}
where the convergence is in the mean-square sense.
\end{proposition}

\begin{proof}
Let $f_A \in \cH_k$ be the Riesz representation of $A$ on $\cH_k$ and $\psi_k:~\cH_k \to \cG_k$ be the canonical isometry.
Since $f_A \in \cH_k$, there exist some $(c^{(n)}_i)_{i=1}^n \subset \mathbb{R}$ and $(x^{(n)}_i)_{i=1}^n \subset \cX$ such that (see Section~\ref{sec:RKHS-def})
$$
\lim_{n \to \infty} \left\| f_A - \sum_{i=1}^n c^{(n)}_i k(\cdot,x^{(n)}_i) \right\|_{\cH_k} = 0.
$$
Recall $\psi_k (k(\cdot,x)) = \ssf(x) \in \cG_k$ for any $x \in \cX$. Therefore,
\begin{align*}
\left\| f_A - \sum_{i=1}^n c^{(n)}_i k(\cdot,x^{(n)}_i) \right\|_{\cH_k}^2
& = \left\| \psi_k(f_A) - \sum_{i=1}^n c^{(n)}_i \psi_k( k(\cdot,x^{(n)}_i) ) \right\|_{\cG_k}^2 \\
& = \mathbb{E} \left[ \left( A(\ssf) - \sum_{i=1}^n c^{(n)}_i  \ssf(x_i^{(n)}) \right)^2 \right],
\end{align*}
and the assertion follows.
\end{proof}

For example, a stochastic integral is usually defined in the limit form of \eqref{eq:linear-functional-limit-form}, where for each $n$ the weighted sum is a finite approximation of the integral with $n$ function evaluations.
Example~\ref{ex:stoch-integrals} below shows that this usual definition is recovered by defining the stochastic integral as in Definition~\ref{def:GP-linear-functional}.

\begin{example}[Stochastic Integrals (Contd.)] \label{ex:stoch-integrals}
Consider the same notation and setting as Example~\ref{ex:stoch-int-motivation} on stochastic integrals.
For a square-integrable function $g \in L_2([0,1])$, we wish to define the integral of $g$ with respect to infinitesimal changes along the path of a Brownian motion $F \sim \GP(0,k)$ on the interval $[0,1]$, where $k(x,x') = \min(x,x')$. 
Suppose for simplicity that $g$ is Lipschitz continuous. 
Such a stochastic integral is usually defined as
\begin{equation} \label{eq:stoch-int-limit-def}
\int_0^1 g(x) dF(x) := \lim_{n \to \infty} \sum_{i=1}^n g(t_i) \left( F(t_i) - F(t_{i-1}) \right).
\end{equation}
where the convergence is in the mean-square sense and
$$
0 = t_0 < t_1 < t_2 < \cdots < t_n = 1
$$
for each $n \in \mathbb{N}$ are a partition of interval $[0,1]$ such that $|t_i - t_{i-1} | \leq C/n$ for  all $i=1,\dots,n$ for a constant $C > 0$ (e.g., $t_i = i/n$).
This is the usual mean-square limit representation of the Paley--Wiener integral and agrees with the It{\=o} integral for deterministic $g$.
Let us recover this definition from Definition~\ref{def:GP-linear-functional}.

 Define an RKHS function $h \in \cH_k$ through the integral of $g$ as $h(x) = \int_0^x g(t)dt$ so that $g$ is the weak derivative of $h$.
Let $A$ be the bounded linear functional on the RKHS $\cH_k$ defined as the inner product with $h$:
$$
A(f) := \left<  h,  f \right>_{\cH_k} = \int_0^1  g(x) Df(x) dx  \quad \text{for all} \quad f \in \cH_k.
$$
Then, as suggested in \eqref{eq:stoch-int-informal}, the stochastic integral may be defined as the linear functional $A$ applied to the Brownian motion $F \sim \GP(0,k)$.
According to Definition~\ref{def:GP-linear-functional}, this is defined as the Gaussian random variable corresponding to the Riesz representation $h$ of the functional via the canonical isometry:
$$
A(F) := \psi_k(h) \in \cG_k.
$$
We will show that this is identical to the limit definition~\eqref{eq:stoch-int-limit-def}.

To this end, suppose for simplicity that $g$ is Lipschitz continuous with Lipschitz constant $L \geq 0$.
Define $h_n \in \cH_k$ by
$$
h_n(\cdot) := \sum_{i=1}^n g(t_i) \left(  k(\cdot,t_i) -  k(\cdot,t_{i-1}) \right).
$$
We first show that  $\lim_{n \to \infty} \left\| h - h_n \right\|_{\cH_k}^2 = 0$.
The weak derivative of $h_n$ is
\begin{align*}
  Dh_n &= \sum_{i=1}^n g(t_i) \left(  D k(\cdot,t_i) -  D k(\cdot,t_{i-1}) \right) \\
  & = \sum_{i=1}^n g(t_i) \left(  1_{[0, t_i]}  -  1_{[0, t_{i-1}]}  \right) = \sum_{i=1}^n g(t_i) 1_{[t_{i-1}, t_i]}(\cdot),
\end{align*}
where $1_{B}(\cdot)$ for a subset $B \subset [0,1]$ is the indicator function of $B$.
Thus, by the property of the RKHS $\cH_k$ (see \eqref{eq:RKHS-Brownian}), we have
\begin{align*}
    & \left\| h - h_n \right\|_{\cH_k}^2
     = \left\| Dh - Dh_n \right\|_{L_2([0,1])}^2 \\
    & = \left\| g(\cdot) - \sum_{i=1}^n g(t_i) 1_{[t_{i-1}, t_i]}(\cdot) \right\|_{L_2([0,1])}^2
    = \sum_{i=1}^n  \int_{t_{i-1}}^{t_i} ( g(t) - g(t_i) )^2 dt \\
    & \leq \sum_{i=1}^n L^2 \int_{t_{i-1}}^{t_i} ( t - t_i )^2 dt   \leq \sum_{i=1}^n L^2 \frac{C^2}{n^2} \frac{C}{n}  = L^2 \frac{C^3}{n^2},
\end{align*}
which implies $\lim_{n \to \infty} \left\| h - h_n \right\|_{\cH_k}^2 = 0$.
Therefore, we have
$$\lim_{n \to \infty} \left\| \psi_k(h) - \psi_k(h_n) \right\|_{\cG_k}^2 = 0,
$$
because $\left\| \psi_k(h) - \psi_k(h_n) \right\|_{\cG_k}^2 = \left\| h - h_n \right\|_{\cH_k}^2$.

On the other hand,
\begin{align*}
 \psi_k(h_n)
& = \psi_k\left( \sum_{i=1}^n g(t_i) \left(  k(\cdot,t_i) -  k(\cdot,t_{i-1}) \right) \right) \\
& = \sum_{i=1}^n g(t_i) \left( F(t_i) - F(t_{i-1}) \right).
\end{align*}
The limit of this last expression as $n \to \infty$ is the usual definition of the stochastic integral~\eqref{eq:stoch-int-limit-def}, and this limit is $A(F) := \psi_k(h)$ as shown above.
Therefore, we have recovered the stochastic integral from Definition \ref{def:GP-linear-functional}.

Note that the assumption that $g$ is Lipschitz continuous is introduced just for simplicity, and not necessary to define the stochastic integral. Indeed, Definition \ref{def:GP-linear-functional} only requires that $h \in \cH_k$, which means $g \in L_2([0,1])$.

\end{example}

\section{Examples of Linear Functionals} \label{sec:examples-functionals}
We now describe how Definition \ref{def:GP-linear-functional} leads to natural definitions of linear functionals of a GP in individual cases.
Our first example is ``function value evaluation,'' which recovers the Gaussian random variable $F(x_0)$ at a specific location $x_0$.
\index{evaluation functional}
\begin{example}[Function Value Evaluation]
\label{ex:function-value-evaluation}
Let $x_0 \in \cX$ be an arbitrary input point, and define a linear functional on $\cH_k$ by
$$
A(f) := f(x_0) \quad \text{for all } ~ f \in \cH_k.
$$
By definition of the RKHS $\cH_k$, $A$ is bounded on $\cH_k$ and the Riesz representation of $A$ is given by the canonical feature map: $f_A := k(\cdot,x_0)$, as we have $\left< f_A, f \right>_{\cH_k} =  f(x_0)$ for all $f \in \cH_k$ by the reproducing property.
By Definition \ref{def:GP-linear-functional}, the linear functional  for the Gaussian process $\ssf \sim \GP(0,k)$ is then defined by
$$
A(\ssf) := \psi_k (k(\cdot,x_0))  = \ssf(x_0) \in \cG_k.
$$

\end{example}

The above example serves as a ``sanity check'' of Definition~\ref{def:GP-linear-functional}. It holds for any GP, because the evaluation functional $f \mapsto f(x_0)$ is always bounded on the RKHS $\cH_k$ by the definition of the RKHS. The reason why $F(x_0)$ is obtained is due to the definition of the canonical isometry $\psi_k$, which maps $k(\cdot,x) \in \cH_k$ to $F(x) = \psi_k( k(\cdot,x)) \in \cG_k$ for any $x \in \cX$.

\paragraph{Lebesgue Integration}
Our next example is (Lebesgue) integrals of a GP, which are fundamental in Bayesian computation. As it turns out, the Riesz representations in this case are {\em kernel mean embeddings}, discussed in detail in Chapter~\ref{sec:integral_transforms}.

\begin{example}[Integrals]
\label{example:integral}
Let $\cX$ be a measurable space and $P$ be a probability measure on it.
Let $k$ be a measurable kernel on $\cX$ with $\cH_k$ being its RKHS,
such that
\[
\int \sqrt{k(x,x)}\,dP(x)<\infty.
\]
Under the last condition, the linear functional
\[
A(f):=\int f(x)\,dP(x)
\qquad\text{for all } f\in\cH_k
\]
is bounded on $\cH_k$, since by the reproducing property
\begin{align*}
\left|\int f(x)\,dP(x)\right|
&=
\left|
\int
\left\langle f,k(\cdot,x)\right\rangle_{\cH_k}
\,dP(x)
\right| \\
&\leq
\int
\|f\|_{\cH_k}
\|k(\cdot,x)\|_{\cH_k}
\,dP(x) \\
&=
\|f\|_{\cH_k}
\int\sqrt{k(x,x)}\,dP(x).
\end{align*}
We write the Riesz representer of $A$ as
\[
f_A=\int_{\cX}k(\cdot,x)\,dP(x).
\]
If $\cH_k$ is separable, Lemma~\ref{lem:bochner-kernel-mean} in Appendix~\ref{sec:kmean-est-realization} shows that the integral on the right-hand side is the Bochner integral of the feature map $x\mapsto k(\cdot,x)$.
\index{kernel mean embedding}
The element $f_A$ is called the kernel mean embedding of $P$ in
$\cH_k$ \citep{SmoGreSonSch07}.
Then, for a Gaussian process $\ssf\sim\GP(0,k)$, the integral is
defined as
\begin{equation}
\label{eq:def-stoch-integ}
\int \ssf(x)\,dP(x)
:=
A(\ssf)
:=
\psi_k(f_A)
\in\cG_k.
\end{equation}
\end{example}

To understand this definition, we represent the integral on the left-hand side of \eqref{eq:def-stoch-integ} as a mean-square limit.
If, in addition, $\cH_k$ is separable, Proposition
\ref{prop:kmean-est-realization} in Appendix
\ref{sec:kmean-est-realization} shows that there exist points
$x_1,x_2,\dots\in\cX$ such that
\[
\lim_{n\to\infty}\left\|
f_A-\frac{1}{n}\sum_{i=1}^n k(\cdot,x_i)
\right\|_{\cH_k}
= 0.
\]
This implies that, for every $f\in\cH_k$,
\[
\lim_{n\to\infty} \left|
\int f(x)\,dP(x)
-
\frac{1}{n}\sum_{i=1}^n f(x_i)
\right|
= 0.
\]
Hence the empirical measures
\[
P_n:=\frac{1}{n}\sum_{i=1}^n\delta_{x_i}
\]
approximate $P$ with $\cH_k$ as the class of test functions; see
Chapter~\ref{sec:integral_transforms}.

Moreover, by the canonical isometry,
\[
\mathbb E\left[
\left(
\int \ssf(x)\,dP(x)
-
\frac{1}{n}\sum_{i=1}^n \ssf(x_i)
\right)^2
\right]
=
\left\|
f_A-\frac{1}{n}\sum_{i=1}^n k(\cdot,x_i)
\right\|_{\cH_k}^2
\longrightarrow 0.
\]
Thus, the integral of $\ssf\sim\GP(0,k)$ is the mean-square limit
of the corresponding empirical averages:
\[
\int \ssf(x)\,dP(x)
=
\lim_{n\to\infty}
\frac{1}{n}\sum_{i=1}^n \ssf(x_i).
\]

\index{Gaussian process!derivative}
\paragraph{Differentiation}
For the next example of linear functionals, we turn to {\em differentiation}.
We will start by introducing some notation.
Let $\cX \subset \mathbb{R}^d$ be an open set, and $k: \cX \times \cX \to \mathbb{R}$ be a kernel.
For $m \in \mathbb{N}$ and a multi-index $\alpha := (\alpha_1,\dots,\alpha_d) \in \mathbb{N}_{0}^d$ with $|\alpha| := \sum_{i=1}^d \alpha_i \leq m$, denote by $\partial^\alpha f := \frac{\partial^{\alpha_1} \cdots \partial^{\alpha_d}}{\partial^{|\alpha|}} f$  the $\alpha$-th order partial derivative of $f: \cX \to \mathbb{R}$ (assuming that this derivative exists).
For multi-indices $\alpha, \beta \in \mathbb{N}_0^d$ and $g: \mathbb{R}^{2d} \to \mathbb{R}$ , let $\partial ^{\alpha, \beta} g$ be the $(\alpha,\beta)$-th partial derivative of $g$, where $(\alpha,\beta) \in \mathbb{N}_0^{2d}$ is a multi-index with $2d$ components.
Then the kernel $k$ is defined to be {\em $m$-times continuously differentiable}, if $\partial^{\alpha,\alpha} k$ exists and is continuous for all multi-indices $\alpha \in \mathbb{N}_0^d$ with $| \alpha | \leq m$ \citep[Definition 4.5]{Steinwart2008}.

\begin{example}[Derivatives]  \label{example:derivative}
Let $\cX \subset \mathbb{R}^d$ be an open set.
For $m \in \mathbb{N}$, let $k$ be a $m$-times continuously differentiable kernel on $\cX$.
For a multi index $\alpha \in \mathbb{N}_{0}^d$ with $|\alpha|  \leq m$ and $x_0 \in \cX$, define a linear functional by
$$
A(f) := \partial^\alpha f(x_0) \quad \text{for all } ~ f \in \cH_k.
$$
Then this functional is bounded on $\cH_k$, and the Riesz representation is given as
$$
f_A = \partial^{0,\alpha} k(\cdot,x_0)
$$
\citep[Corollary 4.36]{Steinwart2008}, and it follows that $\| f_A \|_{\cH_k}^2 = \partial^{\alpha,\alpha}k(x_0,x_0)$.
For a Gaussian process $\ssf \sim \GP(0,k)$, the $\alpha$-th derivative at $x_0 \in \cX$ is then defined as
$$
\partial^\alpha \ssf(x_0) := A(\ssf) := \psi_k (f_A) = \psi_k (\partial^{0,\alpha} k(\cdot,x_0)) \in \cG_k.
$$
\end{example}

To provide an intuition for Example \ref{example:derivative}, consider the case of a first-order partial derivative, i.e., $\alpha \in \mathbb{N}_0^d$ with $\alpha_i = 1$ for some $i = 1,\dots,d$ and $\alpha_j = 0$ for $j \neq i$.
Then for a sequence $(h_n)_{n=1}^\infty \subset \mathbb{R} \backslash \{ 0 \}$ such that $\lim_{n \to \infty} h_n = 0$, the Riesz representation is given by definition as
$$
 \partial^{0,\alpha} k(\cdot,x_0) = \lim_{n \to \infty} \frac{1}{h_n} \left( k(\cdot,x_0 + h_n \alpha) - k(\cdot,x_0) \right),
$$
where the convergence is in the RKHS norm;  see the proof of \citet[Lemma 4.34]{Steinwart2008}.
Therefore, the partial derivative of the Gaussian process $\ssf \sim \GP(0,k)$ is defined as
$$
\partial^\alpha \ssf(x_0) := \psi_k(\partial^{0,\alpha} k(\cdot,x_0)) = \lim_{n \to \infty} \frac{1}{h_n} \left( \ssf(x_0 + h_n \alpha) - \ssf(x_0) \right),
$$
where the convergence is in the mean-square sense.

\index{Gaussian process--RKHS equivalence!linear functionals}
\section{Fundamental GP--RKHS Equivalence}
\label{sec:master-theorem}

A fundamental equivalence result between GP-based and RKHS-based methods is now described.

A GP's linear functional $A(F)$ is the Gaussian random variable in the GHS corresponding via the canonical isometry $\psi_k: \cH_k \to \cG_k$ to the functional's Riesz representation $f_A$ in the RKHS, i.e., $A(F) = \psi_k(f_A)$, given that the functional is bounded on the RKHS (Definition~\ref{def:GP-linear-functional}).
This implies that the {\em standard deviation} of the GP's linear functional is equal to the RKHS norm of the Riesz representation, since the canonical isometry is norm-preserving and the GHS norm  is the standard deviation:
$$
\sqrt{ \mathbb{E} [A(F)^2] }
=
\left\| \psi_k(f_A) \right\|_{\cG_k}
= \left\| f_A \right\|_{\cH_k}.
$$

On the other hand, the RKHS norm of the Riesz representer is identical to the maximum absolute value of the functional over unit-norm RKHS functions:
\begin{align*}
\left\| f_A \right\|_{\cH_k}
& =
\sup_{f \in \cH_k:\,\|f\|_{\cH_k}\leq 1}
\left\langle f_A,f\right\rangle_{\cH_k}\\
& =
\sup_{f \in \cH_k:\,\|f\|_{\cH_k}\leq 1}
A(f) = \sup_{f \in \cH_k:\,\|f\|_{\cH_k}\leq 1}|A(f)|,
\end{align*} 
which is the {\em operator norm} of $A$. \index{operator norm}
The last identity follows from $A$ being linear and the unit RKHS ball being symmetric.

Therefore, {\em the standard deviation of a GP's linear functional is identical to the functional's maximum absolute value over unit-norm RKHS functions}, as summarized in Theorem~\ref{theo:master-theorem} below.
 \citet[Corollary 7, p.~40]{Rit00} provides a special case of this result in which a linear functional is defined in terms of function values.

\begin{theorem} \label{theo:master-theorem}
Let $k$ be a positive definite kernel with RKHS $\cH_k$, and let $A: \cH_k \to \mathbb{R}$ be a bounded linear functional with Riesz representation $f_A \in \cH_k$.
For a Gaussian process $\ssf \sim \GP(0,k)$, define its linear functional $A(\ssf)$ as in Definition~\ref{def:GP-linear-functional}.
Then we have
\begin{equation*} 
\sqrt{\mathbb{E}_{\ssf \sim \GP(0,k)}[A(\ssf)^2] }
= \sup_{f \in \cH_k: \| f \|_{\cH_k} \leq 1} \left| A(f) \right|
= \| f_A \|_{\cH_k}.
\end{equation*}

\end{theorem}

Theorem~\ref{theo:master-theorem} is a fundamental result leading to the equivalences of GP-based and RKHS-based methods in many individual problems, by defining an {\em error metric}  as a linear functional.
Simple examples are given below.

\subsection{Function Values}
Let us first consider an evaluation functional in Example~\ref{ex:function-value-evaluation}, such that the functional $A$ gives the value of an input function $f$ evaluated at a fixed input point $x$, i.e., $A(f) := f(x)$.
The evaluation functional is bounded on any RKHS $\cH_k$ by definition, and its Riesz representation is the canonical feature $k(\cdot,x)$.
Thus, the following result immediately follows from Theorem~\ref{theo:master-theorem}.

\begin{corollary}
\label{coro:std-max-equiv-eval-func}
Let $k$ be a positive definite kernel on a non-empty set $\cX$ with RKHS $\cH_k$ and $\ssf \sim \GP(0,k)$ be a Gaussian process.
Then, for any $x \in \cX$, we have
$$
\sqrt{\mathbb{E} \left[ F(x)^2 \right] }
= \sup_{ \| f \|_{\cH_k} \leq 1} \left| f(x) \right|
= \sqrt{k(x,x)}.
$$
\end{corollary}

Corollary~\ref{coro:std-max-equiv-eval-func} shows that the standard deviation of a GP's function value is identical to the maximum absolute value of the unit-norm RKHS functions; see Figure~\ref{fig:GP_intro} for a simulated example.
Intuitively, the maximum equals the standard deviation here because RKHS functions are smoother than GP samples.

\subsection{Interpolation}
\label{sec:interpo-sec4}

Let us apply Theorem~\ref{theo:master-theorem} to the problem of {\em interpolation}, which will be discussed in more detail in Chapter~\ref{sec:interpolation}, by defining an interpolation error metric as a linear functional.

For a given unknown function $f$, the task is to estimate its value $f(x)$ at any test point $x$ using the observations $f(x_1), \dots, f(x_n)$ at finite points $x_1, \dots, x_n$ as training data.
Consider an estimator of $f(x)$ given as a weighted sum of the observations
$$
\hat{f}(x) := \sum_{i=1}^n w_i f(x_i)
$$
for some weights $w_1, \dots, w_n$ computed appropriately.
Then the error of the estimator is defined as the difference between $f(x)$ and $\hat{f}(x)$ and is given as a linear functional:
$$
A(f) := f(x) - \sum_{i=1}^n w_i f(x_i).
$$
This is a bounded linear functional on any RKHS $\cH_k$ and its Riesz representation is $f_A = k(\cdot,x) - \sum_{i=1}^n w_i k(\cdot,x_i)$, as can be easily shown from the reproducing property.
Theorem~\ref{theo:master-theorem} then leads to the following result.

\begin{corollary} \label{coro:interpolation-error-equiv}
Let $k$ be a positive definite kernel on a non-empty set $\cX$ with RKHS $\cH_k$ and $\ssf \sim \GP(0,k)$ be a Gaussian process.
Let $x_1, \dots, x_n \in \cX$ and $w_1, \dots, w_n \in \mathbb{R}$ be arbitrary with any $n \in \mathbb{N}$.
Then for any point $x \in \cX$,
we have
\begin{align*}
\sqrt{\mathbb{E} \left( F(x) - \sum_{i=1}^n w_i F(x_i) \right)^2 } &= \sup_{ \| f \|_{\cH_k} \leq 1} \left| f(x) - \sum_{i=1}^n w_i f(x_i) \right| \\
& = \left\| k(\cdot,x) - \sum_{i=1}^n w_i k(\cdot,x_i) \right\|_{\cH_k}  .
\end{align*}

\end{corollary}

Corollary~\ref{coro:interpolation-error-equiv} shows the equivalence between two error metrics, one probabilistic and the other non-probabilistic, for an interpolation estimator.
\index{average-case error}
\index{worst-case error}
The probabilistic one is the {\em average-case error} where the true function is assumed to be a GP.
The non-probabilistic one is the {\em worst-case error} over unit-norm RKHS functions.

Optimizing the weight coefficients $w_1, \dots,w_n$ to minimize the average-case error leads to {\em GP interpolation}, while minimizing the worst-case error yields {\em RKHS interpolation}.
The equivalence of the two error metrics implies the equivalence of the two interpolation estimators.
Their relations will be discussed in more detail in Chapter~\ref{sec:interpolation}.
The analysis is extended to {\em regression} where observations contain noise in Chapter~\ref{sec:regression}.

\subsection{Distance between Probability Distributions}

\index{maximum mean discrepancy}
Theorem~\ref{theo:master-theorem} yields a probabilistic interpretation of the {\em Maximum Mean Discrepancy (MMD)} \citep{GreBorRasSchetal12}, a distance metric between probability distributions based on RKHS, and is discussed in Chapter~\ref{sec:integral_transforms}.

Let $\cX$ be a measurable space (e.g., a compact subset of $\mathbb{R}^d$), and let $P$ and $Q$ be arbitrary probability distributions on $\cX$.
We consider a linear functional that outputs the difference between the integrals of a function under the two distributions:
$$
A(f) := \int f(x) dP(x) - \int f(x) dQ(x).
$$
The MMD between $P$ and $Q$ quantifies their discrepancy as the maximum absolute value of this functional over unit-norm RKHS functions
$$
{\rm MMD}_k(P, Q) := \sup_{ \left\| f \right\|_{\cH_k} \leq 1} \left| A(f) \right|.
$$

It can be shown that the functional $A(f)$ is bounded on an RKHS~$\cH_k$ whose reproducing kernel $k$ is bounded, i.e., $\sup_{x \in \cX} k(x,x) < \infty$.
Its Riesz representation is given as the difference between the {\em kernel mean embeddings} of $P$ and $Q$:
$$
f_A = \int k(\cdot, x)dP(x) - \int k(\cdot, x)dQ(x).
$$
Theorem~\ref{theo:master-theorem} then yields the following probabilistic interpretation of MMD in terms of a Gaussian process.

\begin{corollary} \label{coro:mmd-gp}
Let $\cX$ be a measurable space, $k$ be a bounded kernel on $\cX$ with RKHS $\cH_k$, and $F \sim \GP(0,k)$ be a Gaussian process.
Let $P$ and $Q$ be arbitrary probability distributions on $\cX$. Then we have
\begin{align*}
&\sqrt{\mathbb{E}\left[
\left(
\int F(x)\,dP(x)-\int F(x)\,dQ(x)
\right)^2
\right]}\\
& \quad =
\sup_{\|f\|_{\cH_k}\leq 1}
\left|
\int f(x)\,dP(x)-\int f(x)\,dQ(x)
\right|\\
& \quad =
\left\|
\int k(\cdot,x)\,dP(x)
-
\int k(\cdot,x)\,dQ(x)
\right\|_{\cH_k}.
\end{align*}
\end{corollary}

Corollary~\ref{coro:mmd-gp} shows that the MMD between two probability distributions is the root mean squared difference between the integrals of a GP under the two distributions.
Thus, the GP's properties characterize MMD's ability to distinguish probability distributions. Chapter~\ref{sec:integral_transforms} discusses probabilistic interpretations of RKHS methods for comparing probability distributions in more detail.

\subsection{It{\=o} Isometry}

\index{It\=o isometry}
Lastly, we show that Theorem~\ref{theo:master-theorem} applied to a stochastic integral (Example~\ref{ex:stoch-integrals}) leads to the {\em It{\=o} isometry}, a fundamental result in stochastic calculus.
The It{\=o} isometry states that the expected square of the stochastic integral of a function is equal to the integral of that function squared.

\begin{corollary}[It{\=o} Isometry]
Let $F \sim \GP(0,k)$ be a Wiener process on $\cX = [0,1]$ with $k(x,x') = \min(x,x')$.
For a square-integrable function $g \in L_2([0,1])$, define $ \int_0^1 g(x) dF(x)$ as the stochastic integral of $g$ with respect to $F$ in Example~\ref{ex:stoch-integrals}.
Then we have
\begin{align*}
 & \mathbb{E} \left[ \left( \int_0^1 g(x)  dF(x) \right)^2 \right]  = \int_0^1 g(x)^2 dx.
\end{align*}

\end{corollary}
\begin{proof}
Let $\cH_k$ be the RKHS of $k$, and define $h \in \cH_k$ by
$$
h(x) := \int_0^x g(t)dt ~~\text{for all}~~ x \in [0,1].
$$
Define a bounded linear functional
$$
A(f) := \left< f, h \right>_{\cH_k} \quad \text{for all }~ f \in \cH_k
$$
so that $h$ is the Riesz representation of $A$.
Then, the stochastic integral of $g$ with respect to $F$ in Example~\ref{ex:stoch-integrals} is defined as in Definition~\ref{def:GP-linear-functional}:
$$
\int_0^1 g(x) dF(x) := A(F).
$$
Theorem~\ref{theo:master-theorem} implies that
\begin{align*}
 & \mathbb{E} \left[ \left( \int_0^1 g(x)  dF(x) \right)^2 \right]   =  \mathbb{E} \left[ A(F)^2 \right]
 = \left\| h \right\|_{\cH_k}^2 = \int_0^1 g(x)^2 dx,
\end{align*}
where the last identity follows from the definition of $h$ and the property of $\cH_k$ in \eqref{eq:RKHS-Brownian}.
\end{proof}

In this chapter, we have seen that a linear functional of a GP is defined as the Gaussian random variable corresponding via the canonical isometry to the functional's Riesz representation in the RKHS.
The root mean square of a GP's linear functional is then shown to be identical to the functional's maximum absolute value over unit RKHS functions.
This generic equivalence implies the equivalences between GP-based and RKHS-based methods in many individual problems.

The next three chapters will discuss the GP--RKHS connections in more detail, focusing on fundamental statistical problems: interpolation (Chapter~\ref{sec:interpolation}), regression (Chapter~\ref{sec:regression}) and comparison of distributions (Chapter~\ref{sec:integral_transforms}).

\chapter{Interpolation}
\label{sec:interpolation}

\index{interpolation}
This chapter discusses the GP and RKHS approaches to {\em interpolation}, which are illustrated in Figure~\ref{fig:RKHS_v_GP} and will be respectively referred to as {\em GP interpolation} and {\em RKHS interpolation}.
The problem is to estimate an unknown function based on its noise-free function values at finite points, which is fundamental in many fields such as statistics, machine learning, and numerical analysis.
To be precise, let $f^*: \cX \to \mathbb{R}$ be a function defined on a nonempty set $\cX$.
We do not know $f^*$, but we can observe its function values $f^*(x_1), \dots, f^*(x_n)$ at $n \in \mathbb{N}$ locations $x_1, \dots, x_n \in \cX$: write these observations as
\begin{equation} \label{eq:interpolation-4645}
    y_i = f^*(x_i), \quad i = 1,\dots, n.
\end{equation}
The task of interpolation is to estimate $f^*$ based on the data
$$
(x_1, y_1), \dots, (x_n, y_n).
$$
We will discuss the regression problem, where observations are noisy, in Chapter~\ref{sec:regression}.

Noise-free observations are available, for example, when measurement equipment is very accurate, or when the function values are given by computer experiments, in which case
\index{emulation}
performing interpolation is called {\em emulation} in the literature \citep{Sacks89,KenHag01}.
In these situations, the function to be estimated is typically costly to evaluate, so estimation can only use a small number of function evaluations.
The GP and RKHS interpolation methods are routinely used in this situation, as estimation can benefit from prior knowledge about the function incorporated via the choice of a kernel.

\begin{figure}[t]
    \centering
    \includegraphics[width=\textwidth]{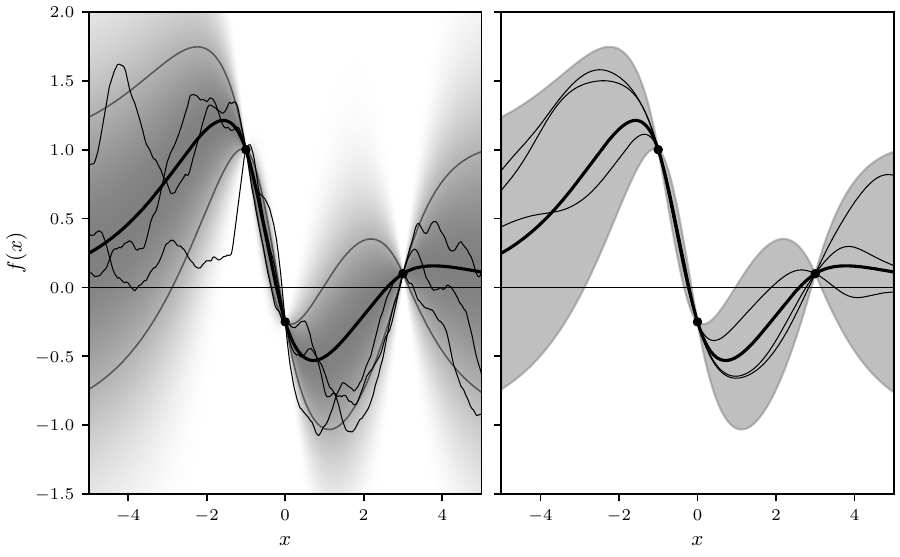}
    \caption{Direct comparison of GP interpolation (\textbf{left}) and RKHS interpolation (\textbf{right}), using the Mat\'ern kernel with smoothness $\alpha = 3/2$ (Example~\ref{ex:matern-kernel}).
    In both panels, the three black dots are the same training observations.
    The thick black curve represents the GP posterior mean function (\textbf{left}), which, by Corollary~\ref{coro:equivalnce-interp}, is identical to the RKHS interpolant (\textbf{right}).
    The two thin black curves are the posterior mean plus-minus one GP posterior standard deviation (\textbf{left}), which, by Theorem~\ref{theo:equivalence-post-cov-worst-average}, is identical to the worst-case error of RKHS interpolation over unit-RKHS functions (\textbf{right}).
    The \textbf{left} panel shows three samples from the GP posterior as thin, rough lines, as well as the marginal probability density for function values as gray shading.
    The \textbf{right} panel shows three random unit-norm RKHS functions (constructed the same way as in Figure~\ref{fig:GP_intro}, see explanation in the caption there), which necessarily lie entirely within the worst-case bound (gray shading).
    }
    \label{fig:RKHS_v_GP}
        \commentAlt{Figure~\ref{fig:RKHS_v_GP}}{Two panels compare Gaussian-process and RKHS interpolation from the same three observations, showing the common interpolant, uncertainty bounds, and representative functions. See long description.}
\commentLongAlt{Figure~\ref{fig:RKHS_v_GP}}{The left panel shows the Gaussian-process posterior mean, one-standard-deviation bounds, three posterior sample paths, and marginal uncertainty shading. The right panel shows the identical RKHS interpolant, the corresponding worst-case bounds, and three random unit-norm RKHS functions. Both panels use the same three observations.}
\end{figure}

This chapter is organized as follows.
Section~\ref{sec:GP-interp-new} derives GP interpolation using a geometric argument based on the Gaussian Hilbert space.
Section~\ref{sec:kernel-interpolation-new} introduces RKHS interpolation as a solution to the minimum-norm interpolation and derives its analytic expression based on an RKHS geometric argument.
Section~\ref{sec:equiv-GP-kernel-interp} summarizes the equivalence between the two approaches and explains how it arises from the equivalence between the Gaussian Hilbert space and the RKHS.
Section~\ref{sec:RKHS-interp-GP-UQ} studies the posterior (co-)variance of GP interpolation, which is used for uncertainty quantification and is a distinguishing functionality of the GP approach, from an RKHS viewpoint.
We provide two RKHS interpretations, one geometric and the other as the worst-case error of RKHS interpolation.
Lastly, Section~\ref{sec:upper-bound-variance} discusses how the GP posterior variance contracts as the sample size increases.

\index{Gaussian process interpolation}
\section{Gaussian Process Interpolation}
\label{sec:GP-interp-new}

In GP interpolation, the unknown function $f^*$ is modeled as a realization of a Gaussian process
\begin{equation} \label{eq:GP-prior-ell}
F \sim \GP(m,k),
\end{equation}
where the mean function $m: \cX \to \mathbb{R}$ and the covariance function $k: \cX \times \cX \to \mathbb{R}$ are chosen so that $f^*$ can be a sample realization of $F$.
Function-value observations $y_i = f^*(x_i)$ in \eqref{eq:interpolation-4645} are thus assumed to be realizations of the Gaussian random variables $F(x_1), \dots, F(x_n)$ at input points $x_1, \dots, x_n$.

GP interpolation is to derive the posterior distribution, i.e., the conditional distribution of $F$ given the conditioning $F(x_i) = y_i$ for $i=1,\dots,n$:
\begin{equation*} 
    F \mid \left( F(x_1) = y_1, \dots, F(x_n) = y_n \right).
\end{equation*}
As is well known, this posterior distribution is another GP whose mean and covariance functions are analytically available, as summarized below.

\begin{theorem} \label{theo:GP-interpolation}
Let $\cX$ be a non-empty set, $m: \mathcal{X} \to \mathbb{R}$ be a  function, $k: \mathcal{X} \times \mathcal{X} \to \mathbb{R}$ be a kernel, and $F \sim \GP(m, k)$ be a Gaussian process.
Let $x_1, \dots, x_n \in \cX$ be such that the kernel matrix ${\bm K}_n = (k(x_i,x_j))_{i,j=1}^n \in \Re^{n \times n}$ is invertible.
Then, for almost every $y_1, \dots, y_n \in \mathbb{R}$ in the support of the joint random variables $F(x_1), \dots, F(x_n)$, we have
$$
F \ |  \left( F(x_1) = y_1, \dots, F(x_n) = y_n \right) \ \sim \ \GP(\bar{m},\bar{k}),
$$
where $\bar{m}:\cX \to \Re$ and $\bar{k}:\cX \times \cX \to \Re$ are the posterior mean and covariance functions, given by
\begin{align}
 \bar{m}(x) &= m(x) + {\bm k}_n(x)^\top {\bm K}_n^{-1}({\bm y}_n -  {\bm m}_n) \quad (x \in \cX), \label{eq:posteior_mean_interp} \\
  \bar{k}(x,x') &= k(x,x') - {\bm k}_n(x)^\top {\bm K}_n^{-1} {\bm k}_n(x') \quad (x,x' \in \cX), \label{eq:posterior-variance_interp}
\end{align}
with
\begin{align*}
     {\bm y}_n &:= (y_1, \dots, y_n)^\top \in \mathbb{R}^n, \\
   {\bm m}_n &:= (m(x_1), \dots, m(x_n))^\top \in \mathbb{R}^n, \\
   {\bm k}_n(x) &:= (k(x_1, x), \dots, k(x_n, x))^\top \in \mathbb{R}^n.
\end{align*}
\end{theorem}

\begin{remark} \label{rem:disintegration}
    Since $F(x_1), \dots, F(x_n)$ are continuous random variables, the probability of the conditioning event $F(x_1) = y_1, \dots, F(x_n) = y_n$ is zero.
    Thus, rigorously, the conditioning is only defined for almost every $y_1, \dots, y_n$ in the support of the joint random variables $F(x_1), \dots, F(x_n)$  \citep{ChaPol97,cockayne2019bayesian}.
\end{remark}

The posterior mean function $\bar{m}(x)$ is used as an estimate of the unknown function value $f^*(x)$, the uncertainty of which can be quantified by using the posterior standard deviation $\sqrt{\bar{k}(x,x)}$, as will be discussed later.
See Figure~\ref{fig:RKHS_v_GP} for an illustration.

Theorem~\ref{theo:GP-interpolation} is well-known in the literature~\citep[e.g.,][]{RasmussenWilliams}.
Below, we will describe a geometric derivation of this result based on the Gaussian Hilbert space (GHS) in Section~\ref{sec:GHS}.
This derivation helps us understand the equivalence between GP and RKHS interpolants.

\index{conditional expectation!orthogonal projection}
\index{posterior covariance!projection residual}
\subsection{Geometrically Deriving the Conditional Mean and Variance}

We first consider the case where the prior mean function is zero:
$$
F \sim \GP(0, k).
$$
The generic case~\eqref{eq:GP-prior-ell} will be considered later.
Throughout, the kernel matrix ${\bm K}_n = (k(x_i, x_j))_{i, j =1}^n \in \mathbb{R}^{n \times n}$ is assumed to be invertible.

Let $\cG_k$ be the Gaussian Hilbert space generated from $F \sim \GP(0,k)$, where each element is a zero-mean Gaussian random variable given as a weighted sum of the GP's function values at finitely or infinitely many locations (see Section~\ref{sec:GHS}).
Let $\cU_n$ be the $n$-dimensional subspace of $\cG_k$ spanned by the GP's function values $F(x_1), \dots, F(x_n)$ at the training locations, and $\cU_n^\perp$ be its orthogonal complement, consisting of elements of $\cG_k$ orthogonal to all elements in~$\cU_n$:
\begin{align}
&\cU_n
:= \left\{ c_1 F(x_1) + \cdots + c_n F(x_n):\   c_1, \dots, c_n \in \mathbb{R}  \right\}   \subset \cG_k \label{eq:subspace-GHS}  \\
&\cU_n^\perp := \left\{  V \in \cG_k:\  \left<  V, Y   \right>_{\cG_k} = \mathbb{E}[V Y] = 0\ \text{for all }\ \ssy \in \cU_n \right\}  \subset \cG_k.
\nonumber
\end{align}
Since the orthogonality in the GHS implies statistical independence, each element in the orthogonal complement $\cU_n^\perp$ is independent of any element of the subspace $\cU_n$ and thus independent of the GP's function values $F(x_1), \dots, F(x_n)$ at the training locations.

\begin{figure}[t]
    \centering
\resizebox{0.9\linewidth}{!}{
\begin{tikzpicture}[
  >=latex,
  font=\small,
  plane/.style={fill=gray!6, draw=gray!60, dashed, line width=0.55pt},
  vec/.style={->, line width=0.65pt},
  resid/.style={->, line width=0.75pt},
  lbl/.style={fill=white, inner sep=1.2pt}
]
  \coordinate (O)   at (10.6,2.0);
  \coordinate (X1)  at (5.3,2.0);
  \coordinate (X2)  at (7.9,0.35);
  \coordinate (C)   at ($(X1)+(X2)-(O)$);

  \coordinate (Px)  at (4.65,1.05);
  \coordinate (Pxp) at (7.05,0.90);

  \coordinate (Fx)  at (4.65,3.75);
  \coordinate (Fxp) at (7.05,3.55);

  \draw[plane] (O) -- (X1) -- (C) -- (X2) -- cycle;

  \draw[vec] (O) -- (X1);
  \draw[vec] (O) -- (X2);

  \draw[vec] (O) -- (Px);
  \draw[vec] (O) -- (Pxp);

  \draw[vec] (O) -- (Fx);
  \draw[vec] (O) -- (Fxp);

  \draw[resid] (Px) -- (Fx);
  \draw[resid] (Pxp) -- (Fxp);

  \draw[line width=0.45pt] (Px) -- ++(0.22,0) -- ++(0,0.22) -- ++(-0.22,0);
  \draw[line width=0.45pt] (Pxp) -- ++(0.22,0) -- ++(0,0.22) -- ++(-0.22,0);

  \foreach \P in {O,X1,X2,Px,Pxp,Fx,Fxp}
    \fill (\P) circle (1.25pt);

  \node[lbl,right=2pt] at (O) {$0$};
  \node[lbl,above=2pt] at (X1) {$F(x_1)$};
  \node[lbl,below right=1pt] at (X2) {$F(x_2)$};

  \node[lbl,below left=2pt] at (Px) {$F_n(x)$};
  \node[lbl,below=2pt] at (Pxp) {$F_n(x')$};

  \node[lbl,above=2pt] at (Fx) {$F(x)$};
  \node[lbl,above=2pt] at (Fxp) {$F(x')$};

  \node[lbl,left=0.5pt] at ($(Px)!0.65!(Fx)$) {$F_n^\perp(x)$};
  \node[lbl,left=0.5pt] at ($(Pxp)!0.65!(Fxp)$) {$F_n^\perp(x')$};

  \node[lbl,below=5pt] at ($(C)!0.52!(X2)$)
    {$\mathcal{U}_n=\operatorname{span}\{F(x_1),F(x_2)\}$};
\end{tikzpicture}
}
\caption{
Geometric interpretation of the conditional expectation and covariance in GP interpolation in Lemma~\ref{lemma:interp-post-mean-var-ghs-gen}.
The figure illustrates the GHS subspace $\cU_n$ spanned by the two GP function values $F(x_1)$ and $F(x_2)$.
For a GP function value $F(x)$ at a test location $x$, its conditional expectation given $F(x_1)$ and $F(x_2)$ is the orthogonal projection $F_n(x)$ onto the subspace, while its conditional variance is the squared norm of the projection residual $F_n^\perp(x)$.
The conditional covariance between $F(x)$ and $F(x')$ at locations $x$ and $x'$ is the inner product of the two projection residuals $F_n^\perp(x)$ and $F_n^\perp(x')$.
}
    \label{fig:GHS-projections-cov}
\commentAlt{Figure~\ref{fig:GHS-projections-cov}}{A geometric diagram shows two Gaussian-process function values projected onto the subspace spanned by two observations, together with their orthogonal residuals. See long description.}

\commentLongAlt{Figure~\ref{fig:GHS-projections-cov}}{A dashed plane represents the Gaussian Hilbert subspace spanned by two observed function values. Two test function values lie above the plane. Their orthogonal projections lie in the plane, and vertical arrows from the projections to the test values show the residuals. Right-angle markers indicate orthogonality. The residual lengths determine the conditional variances, and their inner product determines the conditional covariance.}
\end{figure}

Lemma~\ref{lemma:interp-post-mean-var-ghs-gen} below shows that
\begin{enumerate}
    \item the conditional expectation of the GP's function value $F(x)$ at any test location $x$ given the GP's function values $F(x_1), \dots, F(x_n)$ at the training locations is  the {\em orthogonal projection} of $F(x)$ onto the subspace $\cU_n$ spanned by  $F(x_1), \dots, F(x_n)$:
    \begin{equation} \label{eq:Gauss-projection}
     F_n(x) := \arg\min_{Y \in \cU_n} \| F(x) - Y \|_{\cG_k} \in \cU_n  \quad (x \in \cX).
    \end{equation}
    In other words, the conditional expectation is the {\em best approximation} of $F(x)$ as a linear combination of $F(x_1), \dots, F(x_n)$ in the mean-square sense;

    \item the conditional covariance of the GP's function values $F(x)$ and $F(x')$ at any two locations $x$ and $x'$ given $F(x_1), \dots, F(x_n)$ is the  inner product between the {\em residuals} of the orthogonal projections of $F(x)$ and $F(x')$ onto the subspace $\cU_n$ spanned by $F(x_1), \dots, F(x_n)$:
    \begin{equation} \label{eq:GHS-residuals}
F_n^\perp(x) := F(x) - F_n(x) \in \cU_n^\perp \quad (x \in \cX).
    \end{equation}
    In particular, the conditional {\em variance} of $F(x)$ given $F(x_1), \dots, F(x_n)$ is the {\em squared norm} of the difference between  $F(x)$ and its projection onto the subspace $\cU_n$.
\end{enumerate}
Figure \ref{fig:GHS-projections-cov} illustrates these geometric interpretations.

\begin{lemma}
\label{lemma:interp-post-mean-var-ghs-gen}
Let $k$ be a kernel on a non-empty set $\cX$ and $\cG_k$ be the GHS associated with a Gaussian process $F \sim \GP(0, k)$.
Let $x_1, \dots, x_n \in \cX$ be such that the kernel matrix ${\bm K}_n := (k(x_i, x_j))_{i,j = 1}^n \in \mathbb{R}^{n \times n}$ is invertible.
Then, for all $x, x' \in \cX$, we have
\begin{align}
\mathbb{E}[F(x) \mid F(x_1), \dots, F(x_n)] &=  F_n(x), \label{eq:post-mean-interp-4760} \\
   {\rm Cov}[F(x), F(x') \mid F(x_1), \dots, F(x_n)] & =  \left< F_n^\perp(x), F_n^\perp(x') \right>_{\cG_k}, \label{eq:post-cov-interp-4761}
\end{align}
where $F_n(x)$ is the orthogonal projection~\eqref{eq:Gauss-projection} of $F(x)$ onto the subspace $\cU_n \subset \cG_k$ spanned by $F(x_1), \dots, F(x_n)$ and  $F_n^\perp(x
)$ is its residual~\eqref{eq:GHS-residuals}.
\end{lemma}

\begin{proof}
The identity~\eqref{eq:post-mean-interp-4760} follows from a more generic result of Lemma~\ref{lemma:post-mean-var-ghs-gen} in Chapter~\ref{sec:integral_transforms} with $Z := F(x)$.
Thus, we shall prove the second identity~\eqref{eq:post-cov-interp-4761}.
The residual  $F_n^\perp(x)$ is statistically independent of the conditioning variables $F(x_1), \dots, F(x_n)$, because $F_n^\perp(x)$ belongs to the orthogonal complement $\cU_n^\perp$ and is orthogonal to $F(x_1), \dots, F(x_n)$.
Hence, the residual $F_n^\perp(x)$ conditioned on $F(x_1), \dots, F(x_n)$ is identical to  $F_n^\perp(x)$ without conditioning:
\begin{equation*}
F_n^\perp(x) \mid F(x_1), \dots, F(x_n) =  F_n^\perp(x).
\end{equation*}
Using this and the first identity~\eqref{eq:post-mean-interp-4760} yields
\begin{align*}
& \mathrm{Cov}[F(x), F(x') \mid F(x_1), \dots, F(x_n)]   \\
& =  \mathbb{E}\left[ \left (F(x) - F_n(x))  (F(x') - F_n(x') \right) \mid F(x_1), \dots, F(x_n) \right]  \\
&=  \mathbb{E}[ F_n^\perp(x) F_n^\perp (x') \mid F(x_1), \dots, F(x_n) ]  \\
&= \mathbb{E}[ F_n^\perp(x) F_n^\perp (x')   ]\\
& = \left< F_n^\perp(x), F_n^\perp(x') \right>_{\cG_k},
\end{align*}
proving \eqref{eq:post-cov-interp-4761}.
\end{proof}

\index{posterior mean!Gaussian process interpolation}
\subsection{Closed-form Expression for the Posterior Mean}
The first identity~\eqref{eq:post-mean-interp-4760} leads to a closed-form expression of the conditional expectation.
The orthogonal projection $F_n(x)$ can be written as a weighted sum of $F(x_1), \dots, F(x_n)$:
\begin{align} \label{eq:GP-post-mean-interp-proj}
F_n(x) = \sum_{i=1}^n w_i(x) F(x_i),
\end{align}
where the weights $ {\bm w}(x) := (w_1(x), \dots, w_n(x))^\top \in \mathbb{R}^n$ are those that make $F_n(x)$ the best approximation of $F(x)$:
\begin{align*}
 {\bm w}(x)
 & := \argmin_{ {\bm w} \in \mathbb{R}^n } \left\|  \sum_{i=1}^n w_i F(x_i) - F(x) \right\|_{\cG_k}^2 
 \end{align*}
 \begin{align*}
 & = \argmin_{ {\bm w} \in \mathbb{R}^n } \mathbb{E}\left[ \left( \sum_{i=1}^n w_i F(x_i) - F(x) \right)^2   \right]   \\
 & = \argmin_{{\bm w} \in \mathbb{R}^n} k(x,x) - 2 {\bm w}^\top {\bm k}_n(x) + {\bm w}^\top {\bm K}_n {\bm w} = {\bm K}_n^{-1}{\bm k}_n(x),
\end{align*}
where ${\bm k}_n(x) = (k(x_1, x), \dots, k(x_n, x))^\top \in \mathbb{R}^n$.
By \eqref{eq:post-mean-interp-4760}, the closed-form expression~\eqref{eq:GP-post-mean-interp-proj} is that of the conditional expectation.

So far, the conditioning is done for random variables $F(x_1), \dots, F(x_n)$, but what we want is the conditioning on the actual, fixed observed data $y_1, \dots, y_n$.
This is done by ``fixing'' the random variables to the observed data, $F(x_i) = y_i$, leading to the closed-form expression~\eqref{eq:posteior_mean_interp} of the posterior mean (with zero prior mean):
\begin{align} \label{eq:GP-inter-post-mean4819}
\mathbb{E}[F(x) \mid F(x_1) = y_1, \dots, F(x_n) = y_n] = \sum_{i=1}^n w_i(x) y_i.
\end{align}

\index{posterior covariance!Gaussian process interpolation}
\subsection{Closed-form Expression for the Posterior Covariance}

The second identity~\eqref{eq:post-cov-interp-4761} shows that the conditional covariance of $F(x)$ and $F(x')$ given $F(x_1), \dots, F(x_n)$ is  the GHS inner product between the residuals $F_n^\perp(x) = F(x) - F_n(x)$ and $F_n^\perp(x') = F(x') - F_n(x')$ of the orthogonal projections $F_n(x)$ and $F_n(x')$.
Thus, using the closed-form expression~\eqref{eq:GP-post-mean-interp-proj} for the projections, the conditional covariance is expressed as
\begin{align}
&  {\rm Cov}[F(x), F(x') \mid F(x_1), \dots, F(x_n)]  \nonumber \\
& = \left< F_n^\perp(x),\  F_n^\perp(x')  \right>_{\cG_k}   =  \mathbb{E} \left[ F_n^\perp(x) F_n^\perp(x')   \right] \nonumber \\
& = \mathbb{E} \left[ \left( F(x) - \sum_{i=1}^n w_i(x) F(x_i) \right) \left( F(x') - \sum_{i=1}^n w_i(x') F(x_i) \right)  \right] \nonumber \\
& = k(x,x') - {\bm k}_n^\top(x) {\bm K}_n^{-1}{\bm k}_n(x').    \label{eq:post-cov-expres-4844}
\end{align}

The expression~\eqref{eq:post-cov-expres-4844} does {\em not} depend on the conditioning variables $F(x_1), \dots, F(x_n)$.
Geometrically, this is because the residuals $F_n^\perp(x)$ and $F_n^\perp(x')$ are orthogonal to $F(x_1), \dots, F(x_n)$ in the GHS and thus statistically independent of $F(x_1), \dots, F(x_n)$.
Therefore, the conditioning $F(x_1) = y_1, \dots, F(x_n) = y_n$ by the actual observations $y_1, \dots, y_n$ also does not affect the posterior covariance:\footnote{This independence from observations $y_1, \dots, y_n$ is, intuitively, the consequence of assuming that the GP model is well-specified.
If the model is misspecified, which is usually the case in practice, this independence from observations can be problematic as it can lead to miscalibrated uncertainty estimates. Therefore, in practice, one needs to adapt the kernel to the specific observations by, e.g., optimizing it by marginal likelihood maximization. See Remark~\ref{remark:post-indep-observations} for a further discussion.} 
\begin{align}
& {\rm Cov}[F(x), F(x') \mid F(x_1) = y_1, \dots, F(x_n) = y_n] \nonumber \\
& =     k(x,x') - {\bm k}_n^\top(x) {\bm K}_n^{-1}{\bm k}_n(x'),  \label{eq:post-cov-GPI-4828}
\end{align}
yielding the closed form expression~\eqref{eq:posterior-variance_interp} of the posterior covariance.
In particular, the posterior variance is obtained by setting $x = x'$:
\begin{align}
& {\rm Var}[F(x) \mid F(x_1) = y_1, \dots, F(x_n) = y_n] \nonumber \\
& = k(x,x) - {\bm k}_n^\top(x) {\bm K}_n^{-1}{\bm k}_n(x).  \label{eq:post-var-4855}
\end{align}

\subsection{Closed-form Expressions with a Non-zero Prior Mean Function}

Now, consider the generic case of a GP prior $F \sim \GP(m,k)$ with a non-zero prior mean function $m: \cX \to \mathbb{R}$.
Define a centered GP
$$
F_0(x) := F(x) - m(x),
$$
so that $F_0 \sim \GP(0, k)$ is a zero-mean GP.
The conditioning $F(x_i) = y_i$ then becomes
$$
F_0(x_i) = y_i - m(x_i) \quad (i = 1,\dots,n).
$$
Therefore, the posterior mean of $F(x)$ given the conditioning by the actual observations $F(x_1) = y_1, \dots, F(x_n) = y_n$ and the GP prior is
\begin{align*}
   & \mathbb{E}[F(x) \mid F(x_1) = y_1, \dots, F(x_n) = y_n]  \nonumber \\
   & = m(x) + \mathbb{E}[F_0(x)  \mid F_0(x_1) = y_1 - m(x_1), \dots, F_0(x_n) = y_n - m(x_n)] \nonumber \\
   & = m(x) + \sum_{i=1}^n w_i(x) (y_i - m(x_i)), 
\end{align*}
where the last expression follows from the closed-form expression~\eqref{eq:GP-inter-post-mean4819} for the posterior mean with zero prior mean.
This proves \eqref{eq:posteior_mean_interp}.

The posterior covariance for the GP prior with a non-zero prior mean function is the same as the zero-mean case~\eqref{eq:post-cov-GPI-4828} since the mean shift does not change the covariance.
Thus, we obtain \eqref{eq:posterior-variance_interp}.

\index{RKHS interpolation}
\section{RKHS Interpolation}
\label{sec:kernel-interpolation-new}

This section describes the interpolation method using an RKHS, which we refer to as {\em RKHS interpolation}.
With this method, the unknown function is estimated as the smoothest RKHS function that interpolates the training data.

Let $k$ be a kernel on a non-empty set $\cX$ and $\cH_k$ be its RKHS.
For the unknown function $f^*$ in \eqref{eq:interpolation-4645} generating the data $(x_1, y_1)$, $\dots$, $(x_n, y_n)$, where $y_i = f^*(x_i)$, \index{minimum-norm interpolation}
its estimator is the solution to the following {\em minimum-norm interpolation} problem:
\begin{equation} \label{eq:kernel-interp-opt}
\hat{f} := \argmin_{f \in \cH_k} \| f  \|_{\cH_k}\quad {\rm subject\ to}\ \quad   f(x_i) = y_i, \quad i = 1,\dots,n.
\end{equation}
That is, the estimate $\hat{f}$ has the smallest RKHS norm among RKHS functions that interpolate the data $(x_1, y_1), \dots, (x_n, y_n)$.
Since the RKHS norm $\| f \|_{\cH_k}$ quantifies the smoothness of function $f$ (i.e., smaller $\| f \|_{\cH_k}$ implies that $f$ is smoother), we can interpret $\hat{f}$ as the smoothest interpolant in the RKHS $\cH_k$.
See the thick black curve in Figure~\ref{fig:RKHS_v_GP}.
Therefore, one should choose the kernel $k$ so that the RKHS $\cH_k$ would contain the unknown function $f^*$.

As is well known \citep[e.g.,][Theorem 58 in p.~112]{Berlinet2004}, the solution $\hat{f}$ can be computed by simple matrix operations, as summarized below.

\begin{theorem} \label{theo:kernel-interpolation}
Let $k: \mathcal{X} \times \mathcal{X} \to \mathbb{R}$ be a kernel.
Let $(x_i, y_i)_{i=1}^n \subset \mathcal{X} \times \mathbb{R}$ be such that the kernel matrix ${\bm K}_n = (k(x_i, x_j))_{i,j=1}^n \in \mathbb{R}^{n \times n}$ is invertible.
Then, the solution of the optimization problem \eqref{eq:kernel-interp-opt} is given by
\begin{equation} \label{eq:KRR_interpolation}
   \hat{f} (x)  = {\bm k}_n(x)^\top {\bm K}_n^{-1}  {\bm y}_n \quad (x \in \cX),
\end{equation}
where
\begin{align*}
    {\bm y}_n & :=(y_1,\dots,y_n)^\top \in \mathbb{R}^n \\
    {\bm k}_n(x) &:= (k(x_1, x), \dots, k(x_n, x))^\top \in \mathbb{R}^n.
\end{align*}
\end{theorem}

Let us derive the analytic solution $\hat{f}$ in Theorem~\ref{theo:kernel-interpolation} by a geometric argument, which helps study the relation with GP interpolation.

\paragraph{The Solution Lies in the Subspace}

Let $\cS_n \subset \cH_k$ be the  subspace of the RKHS spanned by the canonical features $k(\cdot, x_1), \dots, k(\cdot, x_n)$ of training locations $x_1, \dots, x_n$, and $\cS_n^\perp \subset \cH_k$ be its orthogonal complement:
\begin{align} \label{eq:subspace-RKHS}
& \cS_n := \left\{ c_1 k(\cdot, x_1) + \cdots + c_n k(\cdot, x_n) : \   c_1, \dots, c_n \in \mathbb{R}  \right\}, \\
& \cS_n^\perp := \left\{ h \in \cH_k :\ \left< h, f  \right>_{\cH_k} = 0 \text{ for all } f \in \cS_n \right\}. \nonumber
\end{align}
We first show that the solution $\hat{f}$ to \eqref{eq:kernel-interp-opt} lies in the subspace $\cS_n$, i.e., it is a weighted sum of the canonical features:
\begin{equation} \label{eq:kernel-interp-finite-basis}
    \hat{f} \in \cS_n\ \Longleftrightarrow \  \hat{f} = \sum_{j=1}^n \alpha_j k(\cdot, x_j) \ \ \text{for some}\ \ {\bm \alpha}_n := (\alpha_1, \dots, \alpha_n)^\top \in \mathbb{R}^n.
\end{equation}
By orthogonal decomposition, any $f \in \cH_k$ can be written as
\begin{align}
    & f = f_n + f_n^\perp, \quad \text{where}\quad f_n \in \cS_n, \quad f_n^\perp \in \cS_n^\perp, \nonumber \\
    & \| f \|_{\cH_k}^2  = \| f_n \|_{\cH_k}^2 + \| f_n^\perp \|_{\cH_k}^2.  \label{eq:orth-decomp-norm-5009}
\end{align}
Here, $f_n \in \cS_n$ is the orthogonal projection of $f$ onto the subspace $\cS_n$ and thus is given in the form~\eqref{eq:kernel-interp-finite-basis}, while $f_n^\perp = f - f_n \in \cS_n^\perp$ is its projection residual.
By the reproducing property and the definition of the orthogonal complement $\cS_n^\perp$, we have
$$
f_n^\perp(x_i) = \left< k(\cdot, x_i), f_n^\perp \right>_{\cH_k} = 0 \quad \text{for all} \quad i = 1,\dots, n.
$$
That is, the residual $f_n^\perp$ vanishes at training input points $x_1,\dots,x_n$.
This implies that
\begin{align*}
f(x_i) = f_n(x_i), \quad i = 1,\dots, n.
\end{align*}
Therefore, if $f$ satisfies the constraint $f(x_i) = y_i$, so does the projection~$f_n$, while the RKHS norm of $f_n$ is not greater than the RKHS norm of $f$ from \eqref{eq:orth-decomp-norm-5009}:
$$
f_n(x_i) = y_i, \quad i = 1,\dots, n, \quad \| f_n \|_{\cH_k} \leq \| f \|_{\cH_k}.
$$
Hence, the optimal solution $\hat{f}$ should lie in $\cS_n$ and thus is given in the form~\eqref{eq:kernel-interp-finite-basis}.

\paragraph{Closed-form Solution}
We now derive the analytic expression of the solution $\hat{f}$ in Theorem~\ref{theo:kernel-interpolation}.
    The constraint $\hat{f}(x_i) = y_i$ for $i=1,\dots,n$ leads to a linear equation for the coefficients ${\bm \alpha}_n = (\alpha_1, \dots, \alpha_n)^\top $ in \eqref{eq:kernel-interp-finite-basis}, which is
$$
{\bm K}_n {\bm \alpha}_n = {\bm y}_n,
$$
where ${\bm K}_n = (k(x_i, x_j))_{i,j=1}^n \in \mathbb{R}^{n \times n}$ is the kernel matrix and ${\bm y}_n = (y_1, \dots, y_n)^\top \in \mathbb{R}^n$.
Assuming that ${\bm K}_n$ is invertible, the solution is uniquely determined as ${\bm \alpha}_n = {\bm K}_n^{-1} {\bm y}_n$, which leads to the unique solution to the minimum-norm interpolation problem~\eqref{eq:kernel-interp-opt}.
Thus Theorem~\ref{theo:kernel-interpolation} is proven.

\index{Gaussian process--RKHS equivalence!interpolation}
\section{Equivalence between the GP and RKHS Interpolants}
\label{sec:equiv-GP-kernel-interp}

The following equivalence between the posterior mean of GP interpolation and the solution to RKHS interpolation follows from their analytic expressions \eqref{eq:posteior_mean_interp} and \eqref{eq:KRR_interpolation}, as is well known~\citep{kimeldorf1970correspondence}.

\begin{corollary} \label{coro:equivalnce-interp}
Let $k: \mathcal{X} \times \mathcal{X} \to \mathbb{R} $ be a kernel and $(x_i, y_i)_{i=1}^n \subset \cX \times \Re$ be such that the kernel matrix ${\bm K}_n = (k(x_i, x_j))_{i,j=1}^n \in \mathbb{R}^{n \times n}$ is invertible.
Then we have $\bar{m} = \hat{f}$, where
\begin{itemize}
\item $\bar{m}$ is the posterior mean function \eqref{eq:posteior_mean_interp} of GP interpolation based on  the zero-mean GP prior $F \sim \GP(0,k)$;
\item $\hat{f}$ is the solution \eqref{eq:KRR_interpolation} to RKHS interpolation \eqref{eq:kernel-interp-opt} based on the RKHS $\cH_k$.
\end{itemize}
\end{corollary}

\begin{figure}[t]
    \centering
\resizebox{1.0\linewidth}{!}{
    \begin{tikzpicture}[
  >=latex,
  font=\large,
  plane/.style={fill=gray!6, draw=gray!60, dashed, line width=0.7pt},
  vec/.style={->, line width=0.9pt},
  resid/.style={->, line width=1.0pt},
  map/.style={->, line width=0.8pt},
  lbl/.style={fill=white, inner sep=1.4pt},
  maplbl/.style={fill=white, inner sep=1.8pt, font=\large}
]
  \coordinate (OL)  at (6.10,2.05);
  \coordinate (L1)  at (1.00,2.05);
  \coordinate (L2)  at (3.45,-0.65);
  \coordinate (CL)  at ($(L1)+(L2)-(OL)$);
  \coordinate (Pk)  at (1.80,0.70);
  \coordinate (Kx)  at (1.80,4.45);

  \draw[plane] (OL) -- (L1) -- (CL) -- (L2) -- cycle;
  \draw[vec] (OL) -- (L1);
  \draw[vec] (OL) -- (L2);
  \draw[vec] (OL) -- (Pk);
  \draw[vec] (OL) -- (Kx);
  \draw[resid] (Pk) -- (Kx);

  \draw[line width=0.6pt]
    (Pk) -- ++(0.23,0) -- ++(0,0.23) -- ++(-0.23,0);

  \foreach \P in {OL,L1,L2,Pk,Kx}
    \fill (\P) circle (1.55pt);

  \node[lbl,below=3pt] at (L1) {$k(\cdot,x_1)$};
  \node[lbl,below=3pt] at (L2) {$k(\cdot,x_2)$};
  \node[lbl,below left=2pt] at (Pk) {$k_n(\cdot,x)$};
  \node[lbl,above=2pt] at (Kx) {$k(\cdot,x)$};
  \node[lbl,left=5pt] at ($(Pk)!0.61!(Kx)$)
    {$k(\cdot,x)-k_n(\cdot,x)$};
  \node[
  lbl,
  below=18pt,
  xshift=10pt,
  font=\large,
  align=center
] at ($(CL)!0.39!(L2)$)
{
  $\mathcal{S}_n=\operatorname{span}\{k(\cdot,x_1),k(\cdot,x_2)\}$\\[-1pt]
  $\mathcal{S}_n\subset\mathcal{H}_k$
};

  \coordinate (OR)  at (13.35,2.05);
  \coordinate (R1)  at (8.25,2.05);
  \coordinate (R2)  at (10.95,-0.65);
  \coordinate (CR)  at ($(R1)+(R2)-(OR)$);
  \coordinate (PF)  at (9.10,0.70);
  \coordinate (Fx)  at (9.10,4.45);

  \draw[plane] (OR) -- (R1) -- (CR) -- (R2) -- cycle;
  \draw[vec] (OR) -- (R1);
  \draw[vec] (OR) -- (R2);
  \draw[vec] (OR) -- (PF);
  \draw[vec] (OR) -- (Fx);
  \draw[resid] (PF) -- (Fx);

  \draw[line width=0.6pt]
    (PF) -- ++(0.23,0) -- ++(0,0.23) -- ++(-0.23,0);

  \foreach \P in {OR,R1,R2,PF,Fx}
    \fill (\P) circle (1.55pt);

  \node[lbl,below=3pt] at (R1) {$F(x_1)$};
  \node[lbl,below=3pt] at (R2) {$F(x_2)$};
  \node[lbl,below left=2pt] at (PF) {$F_n(x)$};
  \node[lbl,above=2pt] at (Fx) {$F(x)$};
  \node[lbl,left=5pt] at ($(PF)!0.61!(Fx)$)
    {$F(x)-F_n(x)$};
  \node[
  lbl,
  below=18pt,
  xshift=10pt,
  font=\large,
  align=center
] at ($(CR)!0.39!(R2)$)
{
  $\mathcal{U}_n=\operatorname{span}\{F(x_1),F(x_2)\}$\\[-1pt]
  $\mathcal{U}_n\subset\mathcal{G}_k$
};
  \draw[map] ($(Kx)+(0.12,0)$) -- ($(Fx)+(-0.12,0)$);
  \node[maplbl,above=4pt] at ($(Kx)!0.5!(Fx)$)
    {$\psi_k:\mathcal{H}_k\to\mathcal{G}_k$};

  \draw[map] ($(Pk)+(0.12,0)$) -- ($(PF)+(-0.12,0)$);
\node[maplbl,above=4pt,xshift=25pt] at ($(Pk)!0.5!(PF)$)
  {$\psi_k:\mathcal{H}_k\to\mathcal{G}_k$};
\end{tikzpicture}
}
\caption{
Geometric interpretation of the equivalence between the RKHS interpolant and the GP interpolant discussed in Sections~\ref{sec:equiv-GP-kernel-interp} and~\ref{sec:RKHS-interp-GP-UQ}.
Under the canonical isometry $\psi_k:\cH_k\to\cG_k$, the Riesz representer $k(\cdot,x)$ of the evaluation functional at a test location~$x$ corresponds to the GP function value $F(x)$.
The RKHS subspace $\cS_n$ spanned by $k(\cdot,x_1),\dots,k(\cdot,x_n)$ corresponds to the GHS subspace $\cU_n$ spanned by $F(x_1),\dots,F(x_n)$.
The orthogonal projection $k_n(\cdot,x)$ of $k(\cdot,x)$ onto $\cS_n$ corresponds to the orthogonal projection $F_n(x)$ of $F(x)$ onto $\cU_n$.
Likewise, the RKHS residual $k(\cdot,x)-k_n(\cdot,x)$ corresponds to the GHS residual $F(x)-F_n(x)$.
Their common norm is the posterior standard deviation of GP interpolation at~$x$.
}
    \label{fig:equiv-interpolation}
\commentAlt{Figure~\ref{fig:equiv-interpolation}}{Parallel projection diagrams compare RKHS and Gaussian Hilbert space representations under the canonical isometry. See long description.}
\commentLongAlt{Figure~\ref{fig:equiv-interpolation}}{The left diagram shows a kernel feature projected orthogonally onto the subspace spanned by two observed kernel features. The right diagram shows the corresponding Gaussian-process value projected onto the subspace spanned by the two observed function values. Horizontal arrows identify the feature, its projection, and its residual with their Gaussian Hilbert space counterparts. Right-angle markers show that both residuals are orthogonal to their respective subspaces.}
\end{figure}

The geometric interpretation of this equivalence is developed below and summarized in Figure~\ref{fig:equiv-interpolation}.

\index{evaluation functional!approximation in interpolation}
\subsection{RKHS Interpolation as Approximate Evaluation Functional}
The problem of interpolation is to estimate the function value $f^*(x)$ of the unknown true function $f^*$ at any test location $x$.
Assuming that $f^*$ belongs to the RKHS of kernel $k$, the function value is written as the RKHS inner product between $f^*$ and the Riesz representation $k(\cdot,x)$ of the evaluation functional $f \mapsto f(x)$ at $x$:
\begin{equation*} 
f^*(x) = \left< f^*, k(\cdot,x) \right>_{\cH_k}.
\end{equation*}
This suggests two ways to estimate the unknown function value $f^*(x)$.
One is to estimate $f^*$ using the training data by minimum-norm interpolation, as explained above, and take the inner product of the resulting estimate $\hat{f}$ and the Riesz representation $k(\cdot,x)$ of the evaluation functional.
The other, less intuitive way is to {\em approximate the Riesz representation $k(\cdot,x)$ using the training data} and take the inner product with $f^*$.
We will show that these two ways are equivalent, and the second way gives a geometric interpretation for the equivalence to GP interpolation in Corollary~\ref{coro:equivalnce-interp}.

Let us describe the second way.
Denote by $k_n(\cdot,x)$ the best approximation of the evaluation functional's Riesz representation $k(\cdot,x)$ based on the Riesz representations $k(\cdot,x_1), \dots, k(\cdot, x_n)$ at the training locations:
\begin{equation} \label{eq:RKHS-projection}
    \quad k_n(\cdot, x) := \argmin_{f \in \cS_n} \left\| k(\cdot,x) - f \right\|_{\cH_k} \in \cS_n,
\end{equation}
which is the orthogonal projection of $k(\cdot,x)$ onto the subspace $\cS_n$ in \eqref{eq:subspace-RKHS} spanned by $k(\cdot,x_1), \dots, k(\cdot,x_n)$.
Let $k_n^\perp(\cdot, x)$ be the residual of this projection:
\begin{align*}
 k_n^\perp(\cdot, x) := k(\cdot, x) - k_n(\cdot, x) \in \cS_n^\perp,
\end{align*}
which belongs to the orthogonal complement $\cS_n^\perp$ of the subspace $\cS_n$.
See Figure~\ref{fig:equiv-interpolation} for an illustration.

Given that $k(\cdot,x_1), \dots, k(\cdot,x_n)$ are linearly independent (which is equivalent to that the kernel matrix ${\bm K}_n$ is invertible), the best approximation $k_n(\cdot,x)$ can be uniquely written as a weighted sum of $k(\cdot,x_1), \dots, k(\cdot,x_n)$:
$$
k_n(\cdot,x) = \sum_{i=1}^n w_i(x) k(\cdot,x_i),
$$
where the weights $w_1(x),\dots,w_n(x)$ make $k_n(\cdot,x)$ the best approximation of $k(\cdot,x)$ in the RKHS:
\begin{align}
(w_1(x),\dots,w_n(x))^\top & := \argmin_{w_1, \dots, w_n \in \mathbb{R}} \left\| k(\cdot,x) - \sum_{i=1}^n w_i k(\cdot,x_i) \right\|_{\cH_k}^2  \nonumber \\
 =  {\bm K}_n^{-1} {\bm k}_n(x) \in \mathbb{R}^n, \label{eq:weights-RKHS-int-4405}
\end{align}
with  ${\bm k}_n(x) := (k(x_1, x), \dots, k(x_n, x))^\top \in \mathbb{R}^n$.
Notice the exact parallel with the orthogonal projection~\eqref{eq:GP-post-mean-interp-proj} in GP interpolation, which will be discussed shortly.

Assuming that the unknown true function $f^*$  belongs to the RKHS, its RKHS inner product with the best approximation $k_n(\cdot,x)$ of the evaluation functional's Riesz representation $k(\cdot,x)$ is given by
\begin{align}
    \left<f^*, k_n(\cdot,x)  \right>_{\cH_k}
    & =  \left<f^*, \sum_{i=1}^n w_i(x) k(\cdot,x_i)  \right>_{\cH_k} \nonumber \\
    & = \sum_{i=1}^n w_i(x) f^*(x_i) = \hat{f}(x). \label{eq:RKHS-interp-approx-eval}
\end{align}
This is identical to the analytic solution~\eqref{eq:KRR_interpolation} of RKHS interpolation, since $y_i = f^*(x_i)$.
Thus, RKHS interpolation approximates the evaluation functional to approximately evaluate the function value $f^*(x)$.

\index{canonical isometry!interpolation}
\subsection{Geometric Interpretation of the Equivalence}

The geometric formulations of GP and RKHS interpolations enable understanding their equivalence in Corollary~\ref{coro:equivalnce-interp}.

Let us first summarize the correspondences between the relevant RKHS and GHS quantities; see also Figure~\ref{fig:equiv-interpolation}.
As before, let $\psi_k :\cH_k \to \cG_k$ be the canonical isometry such that the evaluation functional's Riesz representation $k(\cdot,x)$ at any location $x$ is mapped to the GP function value $F(x)$ at that location:
$$
\psi_k  \left( k(\cdot, x) \right) = F(x) \quad  \left(x \in \cX\right).
$$
Then the RKHS subspaces $\cS_n$ and $\cS_n^\perp$ in \eqref{eq:subspace-RKHS} correspond to the GHS subspaces $\cU_n$ and $\cU_n^\perp$ in \eqref{eq:subspace-GHS}, respectively, via the canonical isometry:
\begin{align*}
    \cU_n &= \psi_k (\cS_n) := \left\{ \psi_k( f ) :\ f \in \cS_n \right\}, \\
    \cU_n^\perp & = \psi_k (\cS_n^\perp) := \left\{ \psi_k( h ) :\ h \in \cS_n^\perp \right\}.
\end{align*}
Likewise, the best approximation
$k_n(\cdot,x)$ of $k(\cdot,x)$ in \eqref{eq:RKHS-projection} and its residual $k_n^\perp(\cdot,x)$ correspond to the best projection $F_n(x)$ of $F(x)$ in \eqref{eq:Gauss-projection} and its residual $F_n^\perp(x)$, respectively:
\begin{align} \label{eq:projection-equivalences}
F_n(x) &= \psi_k\left( k_n(\cdot, x) \right),
 \\
F_n^\perp(x) &= \psi_k\left( k_n^\perp(\cdot, x) \right). \nonumber
\end{align}

The identity~\eqref{eq:projection-equivalences} is the source of the equivalence between the GP and RKHS interpolants in Corollary~\ref{coro:equivalnce-interp}.
Since the RKHS inner product between $k_n(\cdot,x)$ and the true unknown function $f^*$ is the RKHS interpolant $\hat{f}(x)$ as shown in \eqref{eq:RKHS-interp-approx-eval}, and the canonical isometry preserves this inner product, the identity~\eqref{eq:projection-equivalences} yields
\begin{align}
\hat{f}(x) & = \left< f^*, k_n(\cdot, x) \right>_{\cH_k} \nonumber \\
& = \left< \psi_k(f^*), \psi_k\left( k_n(\cdot, x) \right) \right>_{\cG_k} \nonumber \\
& = \left< \psi_k(f^*), F_n(x) \right>_{\cG_k}  = \mathbb{E}\left[ \psi_k(f^*) F_n(x)  \right].  \label{eq:last-express-der-4493}
\end{align}

The GHS element $\psi_k(f^*)$ is a Gaussian random variable whose covariance with the GP function value $F(x')$ at any location $x'$ recovers the function value $f^*(x')$ at that location:
$$
\mathbb{E}[ \psi_k(f^*) F(x') ] = f^*(x') \quad \text{for all }\ x' \in \cX,
$$
which is the GHS version of the reproducing property.
Recall that $F_n(x)$ is the conditional expectation of $F(x)$ given $F(x_1), \dots, F(x_n)$, as shown in \eqref{eq:post-mean-interp-4760}, and is given as a weighted sum of $F(x_1), \dots, F(x_n)$:
\begin{align*}
    F_n(x) = \mathbb{E}[F(x) \mid F(x_1), \dots, F(x_n)] = \sum_{i=1}^n w_i(x) F(x_i).
\end{align*}
Thus, the covariance of $\psi_k(f^*)$ and $F_n(x)$ in \eqref{eq:last-express-der-4493} ``fixes'' the conditioning variables  $F(x_1), \dots, F(x_n)$ to the actual observed values $f^*(x_1), \dots, f^*(x_n)$, leading to the GP posterior mean $\bar{m}(x)$:
\begin{align*}
    \mathbb{E}\left[ \psi_k(f^*) F_n(x)  \right]
    & = \mathbb{E}[F(x) \mid F(x_i) = f^*(x_i),\ i=1,\dots,n] \\
    &= \bar{m}(x).
\end{align*}

We have shown that the equivalence between the GP and RKHS interpolants in Corollary~\ref{coro:equivalnce-interp} follows from the  equivalence~\eqref{eq:projection-equivalences} between the best approximation in the GHS of a GP function value and the best approximation in the RKHS of the corresponding evaluation functional's Riesz representation.
The next section studies these best approximations' {\em residuals}, which provide geometric interpretations of GP posterior (co)variances and their RKHS counterparts.

\index{uncertainty quantification!RKHS interpretation}
\section{RKHS Interpretations of GP Uncertainty Estimates}

\label{sec:RKHS-interp-GP-UQ}

GP interpolation quantifies the uncertainty of the value $f^*(x)$ of the unknown $f^*$ at any test location $x$ by using its {\em posterior variance}~(see \eqref{eq:posterior-variance_interp} and \eqref{eq:post-var-4855}):
\begin{align} \label{eq:post-var-5209}
\bar{k}(x,x) & =  k(x, x) - {\bm k}_n(x)^\top {\bm K}_n^{-1} {\bm k}_n(x) \\
 &  = {\rm Var}[F(x) \mid F(x_i) = f^*(x_i),\ i = 1, \dots, n] . \nonumber
\end{align}
Its square root, $\sqrt{\bar{k}(x,x)}$, is the posterior standard deviation and may quantify the error of the posterior mean $\bar{m}(x)$ as an estimator of $f^*(x)$.
\index{credible interval}
Thus,  a credible interval for $f^*(x)$ may be constructed as
\begin{equation} \label{eq:GP-interval-5214}
    \left[  \bar{m}(x) - \alpha \sqrt{\bar{k}(x,x)},\quad   \bar{m}(x) + \alpha \sqrt{\bar{k}(x,x)}  \right],
\end{equation}
where $\alpha > 0$ is a constant (e.g., $\alpha = 1.96$ for $95\%$ credible intervals).
The two thin black curves in Figure~\ref{fig:RKHS_v_GP} describe the credible band for $\alpha = 1$.
This way of uncertainty quantification is the basis of GP interpolation's applications, such as Bayesian optimization~\citep{garnett2023bayesian}, probabilistic numerics~\citep{hennig2022probabilistic}, inverse problems \citep{Stu10} and so on.

This section describes two RKHS interpretations of the GP posterior variance.
One is geometric, showing it to be identical to the squared distance between the evaluation functional's Riesz representation and its best approximation (Section~\ref{sec:geomet-int-gp-var}).
\index{worst-case error!RKHS interpolation}
The second interpretation shows it to be the square of the {\em worst-case error} of RKHS interpolation when the unknown true function has the unit RKHS norm (Section~\ref{sec:worst-int-gp-var}).
These offer non-probabilistic interpretations of GP uncertainty estimates.

\begin{remark} \label{rem:post-ver-GP-interp}
In this vanilla form, the {\em width} of the credible interval~\eqref{eq:GP-interval-5214} is independent of the unknown true function $f^*$, since the posterior variance~\eqref{eq:post-var-5209} is independent of the values $y_i = f^*(x_i)$ of $f^*$.
Therefore, unless $f^*$ is a sample of the GP prior with the specified kernel $k$, the credible interval would not be well-calibrated, i.e., it does not contain the function value $f^*(x)$ of $f^*$ with the designed probability (e.g., 95\%).
To make it well-calibrated, the kernel (or its hyperparameters) must be selected based on the actual function values $f^*(x_1), \dots, f^*(x_n)$ of $f^*$ using a method such as maximum likelihood and cross-validation.
See, e.g., \citet{naslidnyk2025comparing} for a further discussion.
\end{remark}

\subsection{Geometric Interpretation}
\label{sec:geomet-int-gp-var}

By Lemma~\ref{lemma:interp-post-mean-var-ghs-gen} and the second identity in \eqref{eq:projection-equivalences}, the GP posterior covariance at any two points $x$ and $x'$ is identical to the inner product between the projection residuals for the evaluation functionals at $x$ and $x'$ in the RKHS:
\begin{align*}
&  k(x,x') - {\bm k}_n(x)^\top {\bm K}_n^{-1} {\bm k}_n(x')  \nonumber  \\
 & =  \left< F(x) - F_n(x),~~ F(x') - F_n(x') \right>_{\cG_k} \nonumber \\
& = \left< k(\cdot, x) - k_n(\cdot, x),~~
k(\cdot, x') - k_n(\cdot, x') \right>_{\cH_k}. 
\end{align*}
Therefore, the posterior covariance becomes near zero when the two projection residuals are nearly orthogonal in the RKHS or when their norms are near zero.

Accordingly, the posterior variance is the squared distance between the evaluation functional's Riesz representation and its best approximation in the RKHS:
\begin{align}
 & k(x, x) - {\bm k}_n(x)^\top {\bm K}_n^{-1} {\bm k}_n(x) \nonumber \\
 & =  \left\| F(x) - F_n(x)  \right\|_{\cG_k}^2   = \left\| k(\cdot, x) - k_n(\cdot, x)  \right\|_{\cH_k}^2, \label{eq:GP-post-cov-residual}
\end{align}
as illustrated in Figure~\ref{fig:equiv-interpolation}.
In other words, the GP posterior variance represents the squared {\em approximation error} of approximating the evaluation functional at test location $x$ using the $n$ evaluation functionals at training locations $x_1, \dots, x_n$.
Hence, it becomes small if $x$ is surrounded more densely by $x_1, \dots, x_n$.
See Section~\ref{sec:upper-bound-variance} for quantitative bounds on the contraction of the posterior variance.

\subsection{Worst-Case Error Interpretation}
\label{sec:worst-int-gp-var}

To describe the worst-case error interpretation of the GP posterior variance, let us define an {\em error functional} for interpolation as in Section~\ref{sec:interpo-sec4}.

For a function $f$, we define $A(f)$ as the difference between its value $f(x)$ at test location $x$ and its GP/RKHS interpolation estimate based on training observations $f(x_1), \dots, f(x_n)$:
\begin{align*}
    A(f) :=  f(x) - \sum_{i=1}^n w_i(x)f(x_i) .
\end{align*}
where the weights $w_1(x), \dots, w_n(x)$ are given as \eqref{eq:weights-RKHS-int-4405}.
This is bounded as a linear functional on the RKHS, and its Riesz representation is the difference between the Riesz representation of the evaluation functional at $x$ and its best approximation:
\begin{equation*} 
A(f) = \left<  f,\ k(\cdot,x) - \sum_{i=1}^n w_i(x) k(\cdot,x_i) \right>_{\cH_k}  \quad (f \in \cH_k).
\end{equation*}

  Theorem~\ref{theo:master-theorem} applied to this error functional and the identity~\eqref{eq:GP-post-cov-residual} lead to Theorem~\ref{theo:equivalence-post-cov-worst-average} below; see also Corollary~\ref{coro:interpolation-error-equiv}.
  It shows that the GP posterior standard deviation is identical to the maximum absolute value of the error functional over unit-norm RKHS functions, i.e., the worst-case error of RKHS interpolation.
  See Figure~\ref{fig:RKHS_v_GP} for an illustration.

\begin{theorem} \label{theo:equivalence-post-cov-worst-average}
Let $k: \cX \times \cX \to \mathbb{R} $ be a kernel and $\cH_k$ be its RKHS.
Let $x_1, \dots, x_n \in \cX $ be such that the kernel matrix $ {\bm K}_n = (k(x_i, x_j))_{i,j =1}^n \in \mathbb{R}^{n \times n} $ is invertible.
Then for all $x \in \cX$, we have
\begin{align*}
    & \sqrt{ k(x, x) - {\bm k}_n(x)^\top {\bm K}_n^{-1} {\bm k}_n(x) }\\
   & =  \sup_{f \in \cH_k:\ \| f \|_{\cH_k} \leq 1 } \left|f(x) - \sum_{i=1}^n w_i(x) f(x_i) \right|,
\end{align*}
where $w_1(x), \dots, w_n(x) \in \mathbb{R}$ are given as \eqref{eq:weights-RKHS-int-4405}.
\end{theorem}

With this interpretation, and using the identity between the GP interpolant $\bar{m}(x)$ and the RKHS interpolant $\hat{f}(x)$, the upper and lower bounds of the credible interval~\eqref{eq:GP-interval-5214} can be written as $\hat{f}(x)$ plus-minus a constant times the worst-case error:
\begin{align*}
 &     \hat{f}(x) \pm \alpha \sup_{ \| f \|_{\cH_k} \leq 1 } \left| f(x) - \sum_{i=1}^n w_i(x) f(x_i) \right|.
\end{align*}
Its width becomes large at a test point $x$ where the worst-case error for RKHS interpolation is large.
This way, the ``uncertainty'' estimate can be defined without using probabilistic arguments.

Another consequence of  Theorem~\ref{theo:equivalence-post-cov-worst-average} is that the posterior standard deviation gives an upper bound of the error of the RKHS interpolant if the unknown function $f^*$ belongs to the RKHS.
\begin{corollary} \label{coro:kernel-interp-upper-5310}
Let $k: \cX \times \cX \to \mathbb{R} $ be a kernel and $\cH_k$ be its RKHS.
Let $x_1, \dots, x_n \in \cX $ be such that the kernel matrix ${\bm K}_n = (k(x_i, x_j))_{i,j =1}^n \in \mathbb{R}^{n \times n} $ is invertible.
Then for all $f^* \in \cH_k$ and for all $x \in \cX$, we have
$$
| f^*(x) - \hat{f}(x) | \leq \| f^* \|_{\cH_k} \sqrt{  k(x, x) - {\bm k}_n(x)^\top {\bm K}_n^{-1} {\bm k}_n(x) },
$$
where
$\hat{f}(x) := \sum_{i=1}^n w_i(x) f^*(x_i)$ with $w_1(x), \dots, w_n(x) \in \mathbb{R}$ in \eqref{eq:weights-RKHS-int-4405}.
\end{corollary}

\begin{proof}
Using $f^* \in \cH_k$ and Theorem~\ref{theo:equivalence-post-cov-worst-average}, we have
\begin{align*}
 \left| f^*(x) - \hat{f}(x) \right|  & \leq \sup_{ \| f \|_{\cH_k} \leq \left\| f^* \right\|_{\cH_k} } \left| f(x) - \sum_{i=1}^n w_i(x) f(x_i) \right| \\
 & = \left\| f^* \right\|_{\cH_k} \sup_{\| f \|_{\cH_k} \leq 1 } \left| f(x) - \sum_{i=1}^n w_i(x) f(x_i) \right| \\
 & = \left\| f^* \right\|_{\cH_k} \sqrt{  k(x, x) - {\bm k}_n(x)^\top {\bm K}_n^{-1} {\bm k}_n(x) },
\end{align*}
thereby proving the claim.
\end{proof}

Corollary~\ref{coro:kernel-interp-upper-5310} is a well-known result in the RKHS interpolation literature \citep[e.g.,][]{SchWen06,Wen05}, \index{power function}
where the posterior standard deviation is known as the {\em power function}.
Its asymptotic properties have been extensively studied, as it provides an upper bound for the error of RKHS interpolation.
Such results will be described in the next section.

\begin{remark}
    Corollary~\ref{coro:kernel-interp-upper-5310} can also be derived using the reproducing property and the Cauchy--Schwarz inequality:
\begin{align*}
 \left| f^*(x) - \hat{f}(x) \right| &= \left< f^*, k(\cdot, x) - \sum_{i=1}^n w_i(x) k(\cdot,x_i) \right>_{\cH_k}     \\
& \leq \left\| f^* \right\|_{\cH_k} \left\|  k(\cdot, x) - \sum_{i=1}^n w_i(x) k(\cdot,x_i) \right\|_{\cH_k} \\
& = \left\| f^* \right\|_{\cH_k} \sqrt{  k(x, x) - {\bm k}_n(x)^\top {\bm K}_n^{-1} {\bm k}_n(x) }.
\end{align*}

\end{remark}

\index{posterior variance!contraction}
\section{Contraction of Posterior Variance} \label{sec:upper-bound-variance}

Lastly, we briefly review basic results from the RKHS interpolation literature on the large-sample asymptotics of the posterior variance $\bar{k}(x,x) = k(x, x) - {\bm k}_n(x)^\top {\bm K}_n^{-1} {\bm k}_n(x)$ as the sample size $n$ increases  \citep[e.g.,][]{Wen05,SchWen06,SchSchSch13}.
Such results have been used in asymptotic analyses of GP/kernel-based algorithms \citep[e.g.,][]{Bul11,briol2019probabilistic,TuoWu16,StuTec18}.
Here we assume that the input space is $\cX \subset \mathbb{R}^d$.

\index{fill distance}
The key quantity is the {\em fill distance}, which quantifies the denseness of training input points $X^n := \{ x_1, \dots, x_n \} \subset \Re^d$ around the test input point $x \in \mathbb{R}^d$.
Let $B(x, \rho) := \{ \tilde{x} \in \mathbb{R}^d:~  \| x - \tilde{x} \| \leq \rho  \}$ be the ball around $x$ of radius $\rho > 0$.
Then, for any $\rho > 0$, the fill distance at $x \in \mathbb{R}^d$ is defined as
\begin{equation} \label{eq:local-fill-dist}
h_{\rho,X^n}(x) := \sup_{\tilde{x} \in B(x, \rho) } \min_{i = 1,\dots,n} \| \tilde{x} - x_i \|.
\end{equation}
In other words,  $h_{\rho,X^n}(x)$ is the radius of the largest ``hole'' in $B(x, \rho)$ in which no points from $X^n$ are contained.
Thus, $h_{\rho,X^n}(x)$ being smaller implies that more points exist around~$x$.

Theorem~\ref{theo:upper-bound-post-var} below gives an upper bound on the posterior variance in terms of the fill distance \eqref{eq:local-fill-dist}.
This result is a simplified version of \citet[Theorem 5.14]{WuSch93}.
For completeness and pedagogical reasons, we prove it in
Appendix~\ref{sec:proof-post-var-contraction} using the proof strategy of \citet{WuSch93}.

\begin{theorem} \label{theo:upper-bound-post-var}
Let $k: \mathbb{R}^d \times \mathbb{R}^d \to \mathbb{R}$ be a kernel for which there exists
$\Phi \in C(\mathbb{R}^d) \cap L_1(\mathbb{R}^d)$
such that
$k(x,x') = \Phi(x-x')$
for all $x, x' \in \mathbb{R}^d$
and its Fourier transform $\cF[\Phi]$ satisfies
\begin{equation} \label{eq:kernel-Fourier-5410}
    0 < \cF[\Phi](\omega)
    \leq C_\Phi (1 + \| \omega \|)^{-2s}
    \quad (\omega \in \mathbb{R}^d)
\end{equation}
for some $C_\Phi > 0$ and $s > d/2$.
Then, for any $\rho > 0$, there exist constants $h_0 > 0$ and $C > 0$ such that the following holds.

For any $n \in \mathbb{N}$,  $X^n = \{x_1, \dots, x_n \} \subset \mathbb{R}^d$ such that the kernel matrix ${\bm K}_n = (k(x_i, x_j))_{i,j =1}^n \in \mathbb{R}^{n \times n} $ is invertible, and $x \in \mathbb{R}^d$ such that $h_{\rho,X^n}(x) \leq \min(h_0, 1)$, we have
\begin{equation} \label{eq:bound-post-var}
k(x, x) - {\bm k}_n(x)^\top {\bm K}_n^{-1} {\bm k}_n(x) \leq C   h_{\rho,X^n}(x)^{2s - d},
\end{equation}
where  ${\bm k}_n(x) := (k(x_1, x), \dots, k(x_n, x))^\top \in \mathbb{R}^n$.

\end{theorem}

Theorem~\ref{theo:upper-bound-post-var} holds for shift-invariant kernels $k(x,x') = \Phi(x-x')$ on $\mathbb{R}^d$ that satisfy the condition \eqref{eq:kernel-Fourier-5410}.
This condition requires that  the tail of the Fourier transform $\cF[\Phi](\omega)$ decays at the rate $\| \omega \|^{-2s}$ as $\| \omega \| \to \infty$, and thus $s > d/2$ quantifies the smoothness of $\Phi$.
For example, it is satisfied by Mat\'ern kernels of order $\alpha > 0$, for which we have $s = \alpha + d/2$ (see Example~\ref{ex:RKHS-Matern}).
In this case, the upper bound in \eqref{eq:bound-post-var} becomes $C h_{\rho,X^n}(x)^{2\alpha}$.
The condition~\eqref{eq:kernel-Fourier-5410} is also satisfied by squared-exponential kernels for any $s > d/2$ (see Example~\ref{ex:RKHS-squared-exponential}), and thus the bound~\eqref{eq:bound-post-var} holds for arbitrarily large $s$.
(Note, however, that the constant $C$ depends on $s$ and becomes larger for larger $s$.)

\begin{remark}
For a squared-exponential kernel, it is possible to obtain an upper bound of exponential order for the posterior variance  \citep[Theorem 11.22]{Wen05}.
To present this result, let us assume that $\cX$ is a cube in $\mathbb{R}^d$ and define the fill distance globally on $\cX$ as
$$
h_{X^n} := \sup_{y \in \cX} \min_{i = 1, \dots, n} \| y - x_i \|.
$$
Then, the proof of \citet[Theorem 11.22]{Wen05}\footnote{The power function $P_{\Phi,X}^2(x)$ corresponds to the posterior variance.} shows that there exists a constant $c > 0$ that does not depend on $h_{X^n}$ such that
$$
k(x, x) - {\bm k}_n(x)^\top {\bm K}_n^{-1} {\bm k}_n(x)  \leq  \exp( - c / h_{X^n} )
$$
for any $X^n = \{x_1, \dots, x_n\} \subset \cX$ with sufficiently small $h_{X^n}$.

\end{remark}

Theorem \ref{theo:upper-bound-post-var} implies that the posterior variance's contraction rate is determined by the smoothness $s$ of the kernel and how quickly the fill distance $h_{\rho, X^n}(x)$ goes to $0$ (i.e., how quickly the denseness of  $x_1,\dots,x_n$ around $x$ increases).
We summarize this implication in the following corollary.

\begin{corollary} \label{coro:post-var-contrat}
Suppose that the conditions of Theorem \ref{theo:upper-bound-post-var} hold.
Moreover, suppose that there exists a constant $b > 0$ such that $h_{\rho,X^n}(x) = O(n^{-b})$ as $n \to \infty$.
Then we have
\begin{equation*} 
k(x, x) - {\bm k}_n(x)^\top {\bm K}_n^{-1} {\bm k}_n(x) = O(n^{-b(2s - d)}) \quad (n \to \infty).
\end{equation*}
\end{corollary}

\begin{remark}
If $x_1,\dots,x_n$ are equally-spaced grid points in the ball $B(x,\rho)$, then  $h_{\rho,X}(x) = O(n^{-1/d})$ as $n \to \infty$, i.e., $b = 1/d$ in Corollary~\ref{coro:post-var-contrat}.
In this case, the posterior standard deviation, which is identical to the worst-case error of RKHS interpolation (see Section~\ref{sec:RKHS-interp-GP-UQ}),  contracts as
\begin{align}
\sqrt{k(x, x) - {\bm k}_n(x)^\top {\bm K}_n^{-1} {\bm k}_n(x)}
&= \sup_{\| f \|_{\cH_k} \leq 1} | f(x) - \hat{f}(x) | \nonumber \\
& = O(n^{-(s/d - 1/2) }), \label{eq:post-cov-optimal-rate-grids}
\end{align}
where $\hat{f}(x)$ is the RKHS interpolant based on observations $f(x_1)$, $\dots$, $f(x_n)$.
This rate matches the minimax optimal rate for the worst-case error of interpolation in the Sobolev space of order $s$ \citep{novak1988deterministic}.

More precisely, let $\cX = [0,1]^d$ and suppose that the RKHS $\cH_k$ is norm-equivalent to the Sobolev space of order $s$ (e.g., when $k$ is the Mat\'ern kernel of order $\alpha = s - d/2$).
Denote by $\mathcal{A}_n$ the set of all interpolation algorithms based on function evaluations at $x_1, \dots, x_n$.
Then, \citet[Proposition in Section 1.3.11]{novak1988deterministic} shows that
$$
\inf_{\hat{f} \in \mathcal{A}_n} \sup_{\| f \|_{\cH_k} \leq 1} \sup_{x \in [0,1]^d} | f(x) - \hat{f}(x) |  = \Theta( n^{- (s/d-1/2)} ),
$$
where $\Theta$ denotes asymptotic lower and upper bounds.
Thus, the rate in \eqref{eq:post-cov-optimal-rate-grids} cannot be improved if we also consider the supremum over $x \in [0,1]^d$.
\end{remark}

\begin{remark}[Choice of smoothness]
The bounds above show that smoother kernels can give smaller posterior standard deviations, or equivalently smaller worst-case errors. This does not mean that the smoothest kernel is always preferable.

\index{model misspecification!kernel smoothness}
A smoother kernel typically gives a smaller assumed model class. This increases the risk of model misspecification when the true function is rougher than assumed. Under suitable conditions on the misspecification and the design points, interpolation may still converge to the true function~\citep[e.g.,][]{wynne2021convergence}. However, the rate and the interpretation of the posterior standard deviation may change~\citep[e.g.,][]{naslidnyk2025comparing}.

Smoother kernels can also lead to ill-conditioned kernel matrices, especially for dense designs and unregularized interpolation. The resulting linear system may then be sensitive to numerical error. This issue is well studied in radial basis function interpolation through power functions and interpolation matrices; see, for example, \citet[Chapter~12]{Wen05}, \citet{Schaback1995}, and \citet{DiederichsIske2019}.

Kernel smoothness should therefore be chosen with both approximation and numerical stability in mind. We return to data-driven kernel choice in Chapter~\ref{chap:conclusion}.
\end{remark}

\chapter{Regression}
\label{sec:regression}

\index{regression}
We compare here GP-based and RKHS-based methods for {\em regression}, where noisy observations of function values are available for function estimation.
Let $f^*:\cX \to \Re$ be an unknown function defined on an input space $\cX$.
The task of regression is to estimate $f^*$ based on data $(x_1, y_1), \dots, (x_n, y_n) \in \cX \times \Re$ generated as
\begin{equation}  \label{eq:regres_noise_model}
 y_i = \ff(x_i) + \xi_i,\qq i=1,\dots,n,
\end{equation}
where $\xi_i \in \mathbb{R}$ represents a noise or the variation unexplained by the input features.

\index{Gaussian process regression}
\index{kernel ridge regression}
\index{GPR|see{Gaussian process regression}}
\index{KRR|see{kernel ridge regression}}
\index{Kriging|see{Gaussian process regression}}
\index{Wiener--Kolmogorov prediction|see{Gaussian process regression}}
\index{spline smoothing|see{kernel ridge regression}}
\index{regularized least squares|see{kernel ridge regression}}
We refer to the GP-based method  {\em Gaussian Process Regression (GPR)} and the RKHS-based one {\em Kernel Ridge Regression (KRR)}, following the convention in machine learning.
The former is also known as  ``Kriging'' or ``Wiener--Kolmogorov prediction,'' and the latter as ``spline smoothing'' or ``regularized least-squares.''

This chapter is organized as follows.
GPR is introduced in Section~\ref{sec:GPR}.
It is shown that GPR is a special case of GP interpolation when using a {\em regularized kernel} in Section~\ref{sec:GP-regres-interpo}.
KRR is introduced in Section \ref{sec:KRR}, which is shown to be a special case of RKHS interpolation in Section \ref{sec:KRR-as-KI}.
The equivalence between GPR and KRR is summarized and interpreted via their interpolation formulations in Section~\ref{sec:GP-KRR-equivalence}.
An RKHS interpretation of the posterior variance of GPR as a worst-case error of KRR is given in Section~\ref{sec:geometric-worst-post-var}.
Theoretical convergence rates for GPR and KRR are reviewed and compared in Section~\ref{sec:rates-and-posterior-contraction}.

 \section{Gaussian Process Regression} \label{sec:GPR}
Gaussian process regression, also known as \emph{Kriging} or \emph{Wiener--Kolmogorov prediction}, is a Bayesian nonparametric method for regression.
It starts from defining a prior distribution for the unknown regression function $f^*$ in \eqref{eq:regres_noise_model} using a Gaussian process
\begin{equation} \label{eq:GP-prior}
\ssf \sim \GP(m,k),
\end{equation}
where $m: \mathcal{X} \to \mathbb{R}$ is the mean function and $k: \mathcal{X} \times \mathcal{X} \to \mathbb{R}$ is the covariance function, which are specified so that $f^*$ can be a realization of the random function $F$.

Replacing $f^*$ by $F$ in \eqref{eq:regres_noise_model}, the regression model becomes
\begin{equation} \label{eq:regress-model-random}
Y_i = \ssf(x_i) + \xi_i,\qq i=1,\dots,n,
\end{equation}
where we also assume\footnote{In general, however, GPR allows for the noise variables to be correlated Gaussian with varying magnitudes of variances. We assume i.i.d.~Gaussian noises for simplicity.} that $\xi_1, \dots, \xi_n$ are zero-mean i.i.d.~Gaussian noises with variance $\sigma^2 > 0$:
\begin{equation} \label{eq:Gaussian-noise}
\xi_1, \dots, \xi_n \iid \N(0,\sigma^2).
\end{equation}
Since both $F(x_i)$ and $\xi_i$ are random variables, $Y_i$ in \eqref{eq:regress-model-random} is also random variable.
Thus, the observed outputs $y_1, \dots, y_n$, which are fixed, should be regarded as realizations of random variables $Y_1, \dots, Y_n$.

In GPR, the unknown $f^*$ is inferred by deriving the {\em posterior distribution} of the latent function $F$ in \eqref{eq:regress-model-random} given the conditioning
$$
Y_1 = y_1,\ \dots,\ Y_n = y_n
$$
on the actual observations $y_1, \dots, y_n$.
As can be shown by standard Gaussian calculus \citep[e.g.,][Section 2.2]{RasmussenWilliams}, this posterior distribution is again a GP and analytically given as follows.

\begin{theorem}\label{theo:GP-posterior}
 Let $\cX$ be a non-empty set, $m: \mathcal{X} \to \mathbb{R}$ be a  function, $k: \mathcal{X} \times \mathcal{X} \to \mathbb{R}$ be a kernel, and $F \sim \GP(m, k)$.
 Let $x_1, \dots, x_n \in \cX$ and $Y_i := F(x_i) + \xi_i$ for all $i = 1,\dots,n$, where $\xi_1, \dots, \xi_n \stackrel{i.i.d.}{\sim} \mathcal{N}(0,\sigma^2)$ with $\sigma^2 > 0$.
Then, for almost every ${\bm y}_n := (y_1, \dots, y_n)^\top \in \mathbb{R}^n$ in the support of the random vector $(Y_1, \dots, Y_n)^\top \in \mathbb{R}^n$, we have
\begin{equation*}
\ssf \mid \left(Y_1 = y_1, \dots, Y_n = y_n\right) \sim \GP(\bar{m},\bar{k}),
\end{equation*}
where $\bar{m}:\cX \to \mathbb{R}$ and $\bar{k}:\cX \times \cX \to \mathbb{R}$ are given by
\begin{align}
 \bar{m}(x) &= m(x) + {\bm k}_n(x)^\top ({\bm K}_n + \sigma^2 {\bm I}_n)^{-1}({\bm y}_n -  {\bm m}_n)\quad (x \in \cX),  \label{eq:posteior_mean} \\
  \bar{k}(x,x') &= k(x,x') - {\bm k}_n(x)^\top({\bm K}_n + \sigma^2 {\bm I}_n)^{-1} {\bm k}_n(x') \quad (x,x' \in \cX),
  \label{eq:posterior-variance}
\end{align}
where ${\bm m}_n := (m(x_1), \dots, m(x_n))^\top \in \mathbb{R}^n$,  ${\bm K}_n := (k(x_i, x_j))_{i,j = 1}^n \in \mathbb{R}^{n \times n}$,  and ${\bm k}_n(x) := (k(x_1,x),\ldots,k(x_n,x))^\top \in~\mathbb{R}^n$.
\end{theorem}

Let us discuss connections to GP interpolation in Section~\ref{sec:GP-interp-new}.
First, setting the noise variance to zero, $\sigma^2 = 0$, in Theorem~\ref{theo:GP-posterior} leads to Theorem~\ref{theo:GP-interpolation} on the analytic expression for GP interpolation.
This is unsurprising given that interpolation is a regression problem with noise-free observations.

On the other hand, we showed in Section~\ref{sec:GP-interp-new} that the posterior mean and covariance functions in GP interpolation are geometrically characterized via {\em orthogonal projections} in the Gaussian Hilbert space.
A natural question would be whether a similar geometric characterization exists for GPR.
This is an important question to understand the equivalence between GPR and KRR, summarized below. The equivalence between GP-based and RKHS-based methods generally holds because the GP-based method's formulation in the Gaussian Hilbert space is equivalent to the RKHS-based method's geometric formulation in the RKHS.

The next section shows that a geometric characterization of GPR is possible by making an opposite interpretation of GPR as a special case of GP interpolation.

\index{Gaussian process regression!as interpolation}
\section{GP Regression as GP Interpolation}
\label{sec:GP-regres-interpo}

We describe GP regression as a special case of GP interpolation when the latter uses a {\em regularized kernel}, defined below as the original kernel plus the noise kernel.

\subsubsection{Latent-plus-noise Process and the Regularized Kernel}

As for the interpolation setting,\footnote{This assumption was implicitly made for GP interpolation by assuming the kernel matrix is invertible.} we henceforth assume that input points $x_1,\dots,x_n$ are pairwise distinct:
\begin{equation} \label{eq:distint-inputs-5887}
x_i \not = x_j \quad {\rm for} \quad i \not= j, \quad i,j=1,\dots,n.
\end{equation}

First, define $\delta:\cX \times \cX \to \Re$ as the Kronecker delta function
\begin{equation*} 
\delta(x,x') :=
\begin{cases}
1 \quad (x = x') \\
0 \quad (x \not= x')
\end{cases}.
\end{equation*}
As this is a positive definite kernel,\footnote{That is, it satisfies $\sum_{i,j=1}^m c_i c_j \delta(x_i,x_j) \geq 0$ for all $m \in \mathbb{N}$, $c_1, \dots, c_m \in \mathbb{R}$ and $x_1, \dots, x_m \in \cX$. This can be shown as follows. Let $I_1, \dots, I_N \subset \{1, \dots, m\}$ be disjoint subsets such that $I_1 \cup \cdots \cup I_N = \{1,\dots, m\}$, $x_i = x_j$ for all $i,j \in I_a$ for all $a = 1,\dots,N$, and $x_i \not= x_j$ for all $i \in I_a$ and $j \in I_b$ with $a \not= b$.
Then $\sum_{i,j=1}^m c_i c_j \delta(x_i,x_j) = \sum_{a=1}^N (\sum_{i \in I_a} c_i)^2 \geq 0$.}
 $\delta$ multiplied by the noise variance $\sigma^2$ is also positive definite.
\index{noise process}
 Define $\xi: \cX \to \mathbb{R}$ as a zero-mean GP with this kernel:
\begin{equation} \label{eq:noise-process-5905}
    \xi \sim \GP(0, \sigma^2 \delta),
\end{equation}
which is interpreted as the noise process.
Under the assumption~\eqref{eq:distint-inputs-5887}, the evaluations of $\xi$ at $x_1, \dots, x_n$ are i.i.d.~zero-mean Gaussian variables with variance $\sigma^2$:
$$
\xi(x_1), \dots, \xi(x_n) \stackrel{i.i.d.}{\sim} \mathcal{N}(0, \sigma^2).
$$
Thus, the noise variables $\xi_1, \dots, \xi_n$ in \eqref{eq:Gaussian-noise} can be interpreted as the evaluations of $\xi$ at $x_1, \dots, x_n$.

\index{latent-plus-noise process}
Now, let us define a Gaussian process $G :\cX \to \mathbb{R}$ as the sum of the latent process $F \sim \GP(m,k)$ in \eqref{eq:GP-prior} and the noise process in \eqref{eq:noise-process-5905}:
\begin{equation} \label{eq:noise-contamination}
\ssg  := \ssf + \xi \sim \GP(m, k^\sigma),
\end{equation}
where $k^\sigma: \cX \times \cX \to \mathbb{R}$ is the kernel defined as the sum of $k$ and $\sigma^2 \delta$:
\begin{equation} \label{eq:reg-kernel}
  k^\sigma(x,x') := k(x,x') + \sigma^2 \delta(x,x') \quad (x,x' \in \cX),
\end{equation}
\index{regularized kernel}
which we call {\em regularized kernel}.

Then, the output variables $Y_1, \dots, Y_n$ of the GP regression model in \eqref{eq:regress-model-random} can be formulated as the values of this latent-plus-noise process $G$ at the training locations $x_1,\dots,x_n$:
\begin{align*} 
Y_i := \ssg(x_i) &= \ssf(x_i) + \xi(x_i) \quad (i=1,\dots, n).
\end{align*}

GP interpolation using the latent-plus-noise process~\eqref{eq:noise-contamination} recovers GP regression.
Let us first see this with the posterior mean.

\subsubsection{Posterior Mean}
Since the noise $\xi(x)$ at a test location $x$ is independent of both latent $F(x_i)$ and noise $\xi(x_i)$ values at the training locations, the conditional expectation of the latent-plus-noise $G(x) = F(x) + \xi(x)$ at the test location is the same as the conditional expectation of the latent $F(x)$ without the noise.
This is summarized below.

\begin{lemma}
    \label{lemma:relation-post-mean-gp-int-reg}
Assume \eqref{eq:distint-inputs-5887}.
Let $F$ be the latent process in \eqref{eq:GP-prior} and $G = F + \xi$ be the latent-plus-noise process in \eqref{eq:noise-contamination}.
Then for all $x \in \mathcal{X}$ with $x \not= x_1, \dots, x_n$,  we have
\begin{equation} \label{eq:reg-int-mean-equiv}
 \mathbb{E}[ G(x)  \mid G(x_1), \dots, G(x_n)] = \mathbb{E} [ F(x)  \mid G(x_1), \dots, G(x_n) ].
\end{equation}
\end{lemma}

\begin{proof}
Since the noise process $\xi$ is independent of $F$ and $x\notin\{x_1,\dots,x_n\}$, the random variable $\xi(x)$ is independent of $G(x_1),\dots,G(x_n)$. Hence,
\[
\mathbb{E}[ \xi(x) \mid G(x_1), \dots, G(x_n) ]
=
\mathbb{E}[ \xi(x) ] = 0.
\]
Therefore,
\begin{align*}
  \mathbb{E}[ \ssg(x) \mid  G(x_1), \dots, G(x_n) ]
&= \mathbb{E}[  \ssf(x) + \xi(x) \mid  G(x_1), \dots, G(x_n) ]  \\
&=  \mathbb{E} [  F(x) \mid G(x_1), \dots, G(x_n)],
\end{align*}
and so \eqref{eq:reg-int-mean-equiv} is proved.
\end{proof}

 The identity~\eqref{eq:reg-int-mean-equiv} implies that the posterior mean of GP interpolation using the latent-plus-noise process conditioning on the actual output values $y_1, \dots, y_n$ is identical to that of GP regression:
 \begin{align*}
 & \mathbb{E}[ G(x)  \mid G(x_1) = y_1, \dots, G(x_n) = y_n] \\
 & = \mathbb{E} [ F(x)  \mid Y_1 = y_1, \dots,  Y_n  = y_n] = \bar{m}(x),
\end{align*}
where $Y_1, \dots, Y_n$ are the output random variables of the GP regression model~\eqref{eq:regress-model-random}, the conditioning of which is equivalent to the conditioning of the values $G(x_1), \dots, G(x_n)$ of the latent-plus-noise process.
In other words, the posterior mean of GP interpolation using the regularized kernel $k^\sigma$ in \eqref{eq:reg-kernel} is identical to the posterior mean of GP regression using the original kernel $k$.

 Therefore, the formula~\eqref{eq:posteior_mean} for the posterior mean of GP regression can be derived from the formula~\eqref{eq:posteior_mean_interp} for the posterior mean of GP interpolation.
 Under \eqref{eq:distint-inputs-5887}, and for $x \notin \{x_1, \dots, x_n\}$, we have
\begin{align}
     {\bm k}^\sigma_n(x) &:=  (k^\sigma(x,x_1), \dots, k^\sigma(x,x_n))^\top =  {\bm k}_n(x), \nonumber    \\
      {\bm K}_n^\sigma &:= (k^\sigma(x_i, x_j))_{i,j=1}^n = {\bm K}_n + \sigma^2 {\bm I}_n.  \label{eq:identities-vector-matrix}
\end{align}
Using these identities in the formula~\eqref{eq:posteior_mean_interp}, we obtain
 \begin{align*}
& \mathbb{E}[ F(x)  \mid Y_1 = y_1, \dots, Y_n = y_n] \\
& = \mathbb{E}[ G(x)  \mid G(x_1) = y_1, \dots, G(x_n) = y_n] \\
& = m(x) + {\bm k}_n^\sigma(x)^\top ({\bm K}_n^\sigma)^{-1}({\bm y}_n -  {\bm m}_n)  \\
& =  m(x) + {\bm k}_n(x)^\top ({\bm K}_n + \sigma^2 {\bm I}_n)^{-1} ({\bm y}_n -  {\bm m}_n),
 \end{align*}
which is identical to the formula~\eqref{eq:posteior_mean} for GP regression.

Importantly, through the identity~\eqref{eq:reg-int-mean-equiv}, the conditional expectation in GP regression admits a geometric interpretation as the orthogonal projection in the Gaussian Hilbert space of the latent-plus-noise process, or as the best approximation of the unknown latent-plus-noise value $G(x)$ at the test location $x$ by a linear combination of the latent-plus-noise values $G(x_1), \dots, G(x_n)$ at the training locations $x_1, \dots, x_n$; see Lemma~\ref{lemma:interp-post-mean-var-ghs-gen}.

\subsubsection{Posterior Covariance}

Lemma~\ref{theo:relation-two-cov-funcs} below shows that the conditional covariance of the latent-plus-noise process $G$ at test points, given its values $G(x_1)$, $\dots$, $G(x_n)$ at training locations, is the sum of the conditional covariance of the latent process $F$ given  $G(x_1), \dots, G(x_n)$, which is calculated in GP regression, and the covariance of the noise process $\xi$.

\begin{lemma}
    \label{theo:relation-two-cov-funcs}
Assume \eqref{eq:distint-inputs-5887}.
Let $F$ be the latent process in \eqref{eq:GP-prior} and $G = F + \xi$ be the latent-plus-noise process in \eqref{eq:noise-contamination}.
Then for all $x, x' \in \mathcal{X}$ with $x, x' \not= x_1, \dots, x_n$,  we have
\begin{align*}
& \mathrm{Cov}[ G(x), G(x') \mid G(x_1), \dots, G(x_n) ]   \nonumber \\
&  =   \mathrm{Cov}[ F(x), F(x') \mid G(x_1), \dots, G(x_n) ] + \sigma^2 \delta(x,x'). 
\end{align*}
\end{lemma}

\begin{proof}

Since $x, x' \not= x_1,\dots, x_n$,
it follows that $\xi (x)$  and $\xi(x')$ are independent of $G(x_1), \dots, G(x_n)$, so we have $$
\mathrm{Cov}[\xi (x), \xi(x') \mid G(x_1), \dots, G(x_n)] = \mathrm{Cov}[\xi (x), \xi(x')] = \sigma^2 \delta(x,x').
$$
Note also that $\xi(x)$ is independent of $F(x'), G(x_1), \dots, G(x_n)$, and $\xi(x')$ is independent of $F(x), G(x_1), \dots, G(x_n)$. Therefore,
\begin{align*}
 & \mathrm{Cov}[ \ssg(x), \ssg(x') \mid G(x_1), \dots, G(x_n) ]  \\
 & =   \mathrm{Cov}[ (\ssf(x) + \xi(x) ), ( \ssf(x') + \xi(x') ) \mid  G(x_1), \dots, G(x_n) ]  \\
& = \mathrm{Cov}[\ssf(x), \ssf(x') \mid G(x_1), \dots, G(x_n) ] \\
& \quad + \mathrm{Cov}[\xi (x), \xi(x') \mid G(x_1), \dots, G(x_n)] \\
& = \mathrm{Cov}[\ssf(x), \ssf(x') \mid G(x_1), \dots, G(x_n) ] + \sigma^2 \delta(x,x')
\end{align*}
and the result is proved.
\end{proof}

Lemma~\ref{theo:relation-two-cov-funcs} implies that the posterior covariance in GP regression of the latent process $F$ is recovered from the posterior covariance in GP interpolation of the latent-plus-noise process $G$ by subtracting the noise covariance:
 \begin{align*}
\bar{k}(x,x') & = \mathrm{Cov}[ F(x), F(x') \mid Y_1 = y_1, \dots, Y_n = y_n ] \\
 & =\mathrm{Cov}[ G(x), G(x') \mid G(x_1) = y_1, \dots, G(x_n) = y_n ] - \sigma^2 \delta(x,x').
 \end{align*}
Indeed, the formula~\eqref{eq:posterior-variance} of the posterior covariance in GP regression can be obtained from the formula~\eqref{eq:posterior-variance_interp} of the posterior covariance in GP interpolation using the regularized kernel~\eqref{eq:reg-kernel}, which is the covariance of the latent-plus-noise process.
Using the identities in \eqref{eq:identities-vector-matrix}, we have
 \begin{align*}
& \mathrm{Cov}[ G(x), G(x') \mid G(x_1) = y_1, \dots, G(x_n) = y_n ] - \sigma^2 \delta(x,x') \\
& = k^\sigma(x,x') - {\bm k}_n^\sigma(x)^\top ({\bm K}_n^\sigma)^{-1} {\bm k}_n^\sigma(x') - \sigma^2 \delta(x,x')  \\
& = k(x,x') - {\bm k}_n(x)^\top ({\bm K}_n + \sigma^2 {\bm I}_n)^{-1} {\bm k}_n(x'),
 \end{align*}
which is identical to the formula~\eqref{eq:posterior-variance} for GP regression.

Lemmas~\ref{lemma:relation-post-mean-gp-int-reg}~and~\ref{theo:relation-two-cov-funcs} show that GP regression is a specific case of GP interpolation using the latent-plus-noise process, whose covariance is the regularized kernel.
This formulation of GP regression enables understanding its equivalence to kernel ridge regression, introduced below, from the viewpoint of the equivalence between GP interpolation and RKHS interpolation in Chapter~\ref{sec:interpolation}.

\index{kernel ridge regression!regularized least squares}
\section{Kernel Ridge Regression}
\label{sec:kernel-ridge-regression}
\label{sec:KRR}

Kernel ridge regression (KRR), also known as \emph{spline smoothing} \citep{wahba1990spline} for a specific case, is a regression method based on a positive definite kernel $k$ and its RKHS $\cH_k$.
Its estimator of the unknown regression function $f^*$ in \eqref{eq:regres_noise_model} is the solution to the following regularized least-squares problem:
\begin{equation}
 \label{eq:square-loss}
 \hat{f} = \argmin_{f\in\cH_k} \frac{1}{n} \sum_{i=1}^n ( f(x_i) - y_i )^2 + \lambda \|f\|_{\cH_k} ^2,
\end{equation}
where, as before, $(x_1,y_1), \dots, (x_n, y_n) \in \cX \times \mathbb{R}$ are training data, and  $\lambda > 0$ is a regularization constant.

In the objective function~\eqref{eq:square-loss}, the first term is the squared error, and the second term is a regularization term.
Intuitively, since the RKHS norm $\| f \|_{\cH_k}$ increases as $f$ gets less smooth (see Example \ref{ex:matern-rkhs} and Section \ref{sec:theory}), this regularization term penalizes an over-complicated~$f$.
The constant $\lambda$, which is a hyperparameter of KRR, controls the strength of this penalization.

\index{representer theorem}
\paragraph{Representer Theorem}
While the optimization problem~\eqref{eq:square-loss} is over the RKHS $\cH_k$, which can be infinite-dimensional, its solution is given as a finite linear combination of $k(\cdot,x_1), \dots, k(\cdot,x_n)$,
\begin{equation} \label{eq:KRR-estimator-weighted}
\hat{f} = \sum_{i=1}^n \alpha_i k(\cdot,x_i),
\end{equation}
for some coefficients $\alpha_1,\dots,\alpha_n \in \Re$.
This is an instance of the {\em representer theorem} \citep[Theorem 1]{SchHerSmo01} and can be shown as follows.

Similar to Section~\ref{sec:kernel-interpolation-new} on RKHS interpolation, let $\cS_n \subset \cH_k$ be the RKHS subspace spanned by $k(\cdot, x_1), \dots, k(\cdot, x_n)$ and $\cS_n^\perp \subset \cH_k$ be its orthogonal complement:
\begin{align*}
& \cS_n := \left\{ c_1 k(\cdot, x_1) + \cdots + c_n k(\cdot, x_n):~   c_1, \dots, c_n \in \mathbb{R}  \right\}, \\
& \cS_n^\perp := \left\{ h \in \cH_k:~ \left< h, f  \right>_{\cH_k} = 0 \text{ for all } f \in \cS_n \right\}. \nonumber
\end{align*}
Any $f \in \cH_k$ can be  decomposed uniquely as
$$
f = f_n + f_n^\perp \quad \text{with} \quad f_n \in S_n \quad \text{and} \quad f_n^\perp \in S_n^\perp .
$$
Then for any $i = 1,\dots, n$, we have, by the reproducing property and $\left< f_n^\perp, k(\cdot,x_i) \right>_{\cH_k} = 0$,
\begin{align*}
 f(x_i) & = \left<f, k(\cdot,x_i) \right>_{\cH_k} = \left<f_n + f_n^\perp, k(\cdot,x_i) \right>_{\cH_k}  \\
& = \left< f_n, k(\cdot,x_i) \right>_{\cH_k} = f_n(x_i).
\end{align*}
Therefore, the first term in \eqref{eq:square-loss} depends only on the component $f_n$ of $f$ that lies in the subspace $\cS_n$.
On the other hand, the RKHS norm of $f_n$ is not greater than $f$, since
$$
\| f \|_{\cH_k}^2 = \| f_n \|_{\cH_k}^2 + \| f_n^\perp \|_{\cH_k}^2 \geq \| f_n \|_{\cH_k}^2.
$$
If $f$ has non-zero orthogonal component $f_n^\perp \in \cS_n^\perp$, then this inequality becomes strict, i.e., $\| f \|_{\cH_k} > \| f_n \|_{\cH_k}$ and the objective value~\eqref{eq:square-loss} for $f_n$ becomes strictly smaller than that for $f$.
Thus, the optimal solution $\hat{f}$ of \eqref{eq:square-loss} should not have the orthogonal component and must lie in the subspace $\cS_n$; this implies that $\hat{f}$ must be given in the form~\eqref{eq:KRR-estimator-weighted}.

\paragraph{Analytic Solution}
Now, the optimization problem~\eqref{eq:square-loss} can be solved by optimizing the coefficients $\alpha_1, \dots, \alpha_n \in \mathbb{R}$ in \eqref{eq:KRR-estimator-weighted}. This can be done straightforwardly using the reproducing property.
The resulting analytic solution, well known in the literature, is summarized below.
For completeness, we provide the remaining proof.

\begin{theorem} \label{theo:KRR-estimator}
Let $k: \mathcal{X} \times \mathcal{X} \to \mathbb{R}$ be a kernel, $\cH_k$ be its RKHS, and $(x_i, y_i)_{i=1}^n \subset \mathcal{X} \times \mathbb{R}$.
If $\lambda > 0$, then the solution $\hat{f}$ to the optimization problem~\eqref{eq:square-loss} is unique as an element in $\cH_k$ and is given by
\begin{align}
   & \hat{f} (x) = {\bm k}_n(x)^\top({\bm K}_n + n \lambda {\bm I}_n)^{-1} {\bm y}_n \quad (x \in \cX),  \label{eq:KRR_estimator}
\end{align}
where ${\bm K}_n := (k(x_i, x_j))_{i,j = 1}^n \in \mathbb{R}^{n \times n}$, ${\bm k}_n(x)$ $:=$ $(k(x_1,x),\ldots,k(x_n,x))^\top $ $\in \mathbb{R}^n$, and ${\bm y}_n :=(y_1,\dots,y_n)^\top \in \mathbb{R}^n$.
\end{theorem}

\begin{proof}

Let ${\bm \alpha} := (\alpha_1,\dots,\alpha_n)\Trans \in \Re^n$.
By substituting the expression \eqref{eq:KRR-estimator-weighted} in \eqref{eq:square-loss} and using the reproducing property, the optimization problem now becomes
\begin{equation} \label{eq:KRR-finite-dim-objective}
\min_{{\bm \alpha} \in \Re^n} \frac{1}{n} \left[ {\bm \alpha}\Trans {\bm K}_n^2 {\bm \alpha} - 2 {\bm \alpha}\Trans {\bm K}_n {\bm y}_n + \| {\bm y}_n \|^2 \right] + \lambda {\bm \alpha}\Trans {\bm K}_n {\bm \alpha}.
\end{equation}
Differentiating this objective function with respect to ${\bm \alpha}$, setting it equal to $0$ and arranging the resulting equation yields
\begin{equation} \label{eq:KRR-finite-dim-opt}
{\bm K}_n ({\bm K}_n + n \lambda {\bm I}_n) {\bm \alpha} = {\bm K}_n {\bm y}_n.
\end{equation}
The vector
$$
{\bm \alpha} = ({\bm K}_n + n \lambda \Id_n)^{-1} {\bm y}_n
$$
is one of the solutions to \eqref{eq:KRR-finite-dim-opt}.
Since \eqref{eq:KRR-finite-dim-objective} is a convex function of ${\bm \alpha}$ (while it may not be strictly convex unless ${\bm K}_n$ is strictly positive definite or invertible), this ${\bm \alpha}$ attains the minimum of \eqref{eq:KRR-finite-dim-objective}.
Thus, the function \eqref{eq:KRR-estimator-weighted} with ${\bm \alpha} = ({\bm K}_n + n \lambda \Id_n)^{-1} {\bm y}_n$ attains the minimum of \eqref{eq:square-loss}.
Moreover, it is the unique solution as an element in $\cH_k$ since the square loss is convex\footnote{A loss function $L: \Re \times \Re \to \Re$ is called convex, if $L(y,\cdot): \Re \to \Re$ is convex for all fixed
$y \in \Re$ \citep[Definition 2.12]{Steinwart2008}} \citep[Theorem 5.5]{Steinwart2008}.

\end{proof}

As a preparation for studying the equivalence between GPR and KRR, we shall show that KRR can be interpreted as a special case of RKHS interpolation in the next section.
This is parallel to what we have shown for GPR in Section~\ref{sec:GP-regres-interpo}.

\begin{remark} \label{rem:KI-as-KRR}
    It is also possible to interpret KRR as subsuming RKHS interpolation.
    Dividing the objective function~\eqref{eq:square-loss} by the regularization constant $\lambda > 0$, the KRR objective can be written as
    \begin{equation*}
 \hat{f} = \argmin_{f\in\cH_k} \frac{1}{n \lambda} \sum_{i=1}^n ( f(x_i) - y_i )^2 + \|f\|_{\cH_k} ^2.
\end{equation*}
By taking the limit $\lambda \to 0$, the first term becomes the hard constraint that $f(x_i) = y_i$ for all $i = 1,\dots, n$, leading to the minimum-norm interpolation problem.

\end{remark}

\index{kernel ridge regression!as interpolation}
\section{Kernel Ridge Regression as RKHS Interpolation}
\label{sec:KRR-as-KI}

We will show that KRR~\eqref{eq:square-loss} is a special case of RKHS interpolation~\eqref{eq:kernel-interp-opt} when using the regularized kernel in \eqref{eq:reg-kernel}.
To this end, let us first study the RKHS of the regularized kernel.

\subsubsection{RKHS of the Regularized Kernel}

We rewrite the regularized kernel in \eqref{eq:reg-kernel} as the sum of the original kernel and the noise covariance kernel:
$$
k^\sigma(x,x') = k(x,x') + \ell_{\sigma^2}(x,x'),
$$
where $\ell_{\sigma^2} : \cX \times \cX \to \mathbb{R}$ is the noise covariance kernel and defined as
\begin{equation} \label{eq:delta-kernel-sigma-6354}
    \ell_{\sigma^2} (x,x') := \sigma^2 \delta(x,x') = \begin{cases}
    \sigma^2 & {\rm if}~~  x = x' \\
    0 & {\rm if}~~  x \not= x'
\end{cases}.
\end{equation}
The noise kernel takes the value $\sigma^2$ only when the two input locations coincide, and takes zero for distinct input locations.
Its RKHS, denoted here by $\mathcal{H}_{\sigma^2}$, consists of weighted sums of such ``spikes'':
\begin{equation} \label{eq:RKHS-noise-6369}
    \mathcal{H}_{\sigma^2}
=  \left\{
 \eta = \sum_{x \in \mathcal{X}} c_x \ell_{\sigma^2}(\cdot,x): ~  (c_x)_{ x \in \cX} \subset \mathbb{R},~~  \left\| \eta   \right\|_{\mathcal{H}_{\sigma^2}}^2 = \sigma^2 \sum_{x \in \cX} c_x^2 < \infty   \right\},
\end{equation}
where its inner product is given by
\begin{align*}
    & \left< \eta, \eta' \right>_{\mathcal{H}_{\sigma^2}} = \sigma^2 \sum_{x \in \cX} c_x c_x', \\
    &\text{for any}\quad  \eta = \sum_{x \in \mathcal{X}} c_x \ell_{\sigma^2}(\cdot,x) \in \mathcal{H}_{\sigma^2},\quad  \eta' = \sum_{x \in \mathcal{X}} c_x' \ell_{\sigma^2}(\cdot,x) \in \mathcal{H}_{\sigma^2}.
\end{align*}
Therefore, the spikes of height $\sigma$, the noise standard deviation,
$$
(\sigma^{-1} \ell_{\sigma^2}(\cdot,x))_{x \in \cX} = (\sigma \delta(\cdot,x))_{x \in \cX}
$$
are an orthonormal basis of $\cH_{\sigma^2}$ \citep[e.g.,][Example 3.9]{SteSco12}.\footnote{Thus, if $\cX$ is uncountable (e.g., $\cX = [0,1])$, then $\cH_{\sigma^2}$ is not separable; see \citet[Example 3.9]{SteSco12}.}
\index{noise RKHS}
We will call $\cH_{\sigma^2}$ {\em noise RKHS}.

\index{regularized kernel!RKHS}
Since the regularized kernel is the sum of the original and noise kernels, the regularized kernel's RKHS  consists of functions that can be written as a sum of a function from the original RKHS and a function from the noise RKHS \citep[e.g.,][Section 6]{Aronszajn1950}:
\begin{align} \label{eq:RKHS-new-kernel-6382}
\cH_{k^\sigma} & = \left\{g = f + \eta:  \ f \in \cH_k,\ \eta \in \cH_{\sigma^2} \right\},
\end{align}
where the squared norm of $g \in \cH_{k^\sigma}$ is the minimum of the sum of the squared norm of $f \in \cH_k$ and the squared norm of $\eta \in \cH_{\sigma^2}$, where the minimum is over pairs of functions $f$ and $\eta$ whose sum is $g$:
$$
 \| g \|_{k^\sigma}^2 = \min_{\substack{f \in \cH_k, \eta \in \cH_{\sigma^2 }, \\ g = f + \eta}} \| f \|_{\cH_k}^2 + \| \eta \|_{\cH_{\sigma^2}}^2 .
$$
In other words, each function $g$ in the regularized kernel's RKHS is an original RKHS function $f$ plus a noise function $\eta$.

\subsubsection{Minimum-norm Interpolation with the Regularized Kernel}

Consider applying RKHS interpolation~\eqref{eq:kernel-interp-opt} using the regularized kernel's RKHS~\eqref{eq:RKHS-new-kernel-6382}, which consists of original-plus-noise functions, to noisy training observations $(x_1,y_1), \dots, (x_n,y_n)$:
\begin{equation}   \label{eq:min-norm-interp-aug-6353}
\hat{g} := \argmin_{g \in \cH_{k^\sigma}} \| g  \|_{\cH_{ k^\sigma}}\quad {\rm subject\ to}\ \quad   g(x_i) = y_i, \quad i = 1,\dots,n.
\end{equation}
Note that the kernel matrix ${\bm K}_n^
\sigma := ( k^\sigma(x_i,x_j) )_{i,j=1}^n$ is invertible, as
$$
{\bm K}_n^
\sigma
= ( k(x_i,x_j) + \sigma^2 \delta(x_i, x_j) )_{i,j=1}^n
= {\bm K}_n + \sigma^2 {\bm I}_n,
$$
where the second identity follows from the assumption  $x_i \not= x_j$ for $i \not= j$.
Therefore, Theorem~\ref{theo:kernel-interpolation} implies that the solution $\hat{g}$ to \eqref{eq:min-norm-interp-aug-6353} is unique.

By using\footnote{Set $\ell := \ell_{\sigma^2}$ and $v := k^\sigma$ in Lemma~\ref{lemma:sum-kernel-min-norm}, for which the condition~\eqref{eq:cond-lemma-8146} is satisfied by taking, e.g., $g := \sum_{i=1}^n y_i \delta(\cdot, x_i)$ and $h := \sum_{i=1}^n (y_i - f(x_i) ) \delta(\cdot,x_i)$.} Lemma~\ref{lemma:sum-kernel-min-norm} in Appendix~\ref{sec:proof-KRR-as-interpolation}, the solution $\hat{g}$  can be written as the sum of an original RKHS function $\hat{f}$ and a noise RKHS function $\hat{\eta}$ such that the sum of their respective squared RKHS norms is minimized while the interpolation constraints are satisfied:
\begin{align}
&\hat{g} = \hat{f} + \hat{\eta}, \nonumber \\ \text{where}\quad & (\hat{f}, \hat{\eta}) := \argmin_{ \substack{ f \in \cH_k,~ \eta \in \cH_{\sigma^2}, \\ {\bm f}_n + {\bm \eta}_n = {\bm y}_n } }  \| f \|_{\cH_k}^2 + \| \eta \|_{\cH_{\sigma^2}}^2,  \label{eq:regression-decomposition}
\end{align}
where ${\bm f}_n + {\bm \eta}_n = {\bm y}_n$ represents the interpolation constraints as
\begin{align*}
& {\bm f}_n := (f(x_1), \dots, f(x_n))^\top, \quad {\bm \eta}_n := (\eta(x_1), \dots, \eta(x_n))^\top   \\
& \qquad\qquad\qquad\text{and}\,\,\,{\bm y}_n := (y_1,\dots,y_n)^\top.
\end{align*}
Thus, RKHS interpolation~\eqref{eq:min-norm-interp-aug-6353} using the regularized kernel's RKHS is equivalent to the joint optimization~\eqref{eq:regression-decomposition} of the original and noise RKHS functions.

A representer theorem holds for the optimal noise function $\hat{\eta}$ in
\eqref{eq:regression-decomposition} as it is given as a weighted sum of ``spikes'' only at the training locations, as summarized below. See Appendix~\ref{sec:proof-noise-represent} for the proof.

\begin{lemma} \label{lemma:noise-represent}
Let $\hat{\eta} \in \cH_{\sigma^2}$ be the solution in
\eqref{eq:regression-decomposition}.
Then we have
\begin{equation} \label{eq:noise-func-form-6423}
    \hat{\eta} = \sum_{i=1}^n \hat{b}_i \sigma \delta(\cdot, x_i) \quad {\rm for~some}~~ \hat{b}_1, \dots, \hat{b}_n \in \mathbb{R},
\end{equation}

\end{lemma}

Using Lemma~\ref{lemma:noise-represent}, the optimization problem~\eqref{eq:regression-decomposition} can be simplified to the joint minimization of a function's RKHS norm and the squared sum of noise levels at the training locations, subject to the constraints that each training observation is the sum of the function's value and the noise at that training location.
More precisely, set the noise function in the form of \eqref{eq:noise-func-form-6423},
$$
\eta = \sum_{i=1}^n b_i \sigma \delta(\cdot,x_i), \quad \text{for which}~~~ \| \eta \|_{\cH_{\sigma^2}}^2 = \sum_{i=1}^n b_i^2,
$$
where $b_1, \dots, b_n \in \mathbb{R}$ represent the noise levels at the training locations $x_1, \dots, x_n$.
Then the problem~\eqref{eq:regression-decomposition} is equivalent to optimizing an RKHS function $f$ and noise levels $b_1, \dots, b_n$ such that
\begin{align} \label{eq:KRR-KI-connection-optim}
\hat{f},~ \hat{b}_1, \dots, \hat{b}_n = & \argmin_{f \in \cH_k,\ b_1,\dots,b_n \in \mathbb{R}}  \| f \|_{\cH_k}^2 +  \sum_{i=1}^n b_i^2 \\
& \text{\rm subject to}~~    y_i = f(x_i) + b_i \sigma \ \  (i=1,\dots,n). \nonumber
\end{align}
The RKHS interpolant~\eqref{eq:min-norm-interp-aug-6353} using the regularized kernel is then given as  the sum of the optimal RKHS function $\hat{f}$ and the noise function with the optimal noise levels $\hat{b}_1, \dots, \hat{b}_n$:
\begin{equation*} 
\hat{g} = \hat{f} + \sum_{i=1}^n \hat{b}_i \sigma \delta (\cdot, x_i),
\end{equation*}

\subsubsection{KRR as RKHS Interpolation}

A further simplification of the optimization problem~\eqref{eq:KRR-KI-connection-optim} shows that the optimal RKHS function $\hat{f}$ is a KRR estimator.
The constraints in \eqref{eq:KRR-KI-connection-optim} mean that the noise level $b_i$ at each training location $x_i$ is the difference between the corresponding observation $y_i$ and function value~$f(x_i)$ divided by the noise standard deviation $\sigma$:
$$
b_i = \sigma^{-1} (y_i - f(x_i)) \quad (i = 1,\dots, n).
$$
Therefore, the noise levels represent the {\em training residuals} of a function, scaled by the noise standard deviation.

Using this expression, the optimization problem~\eqref{eq:KRR-KI-connection-optim} is reduced to the {\em joint minimization of a function's RKHS norm and the squared sum of the training residuals}, i.e., the {\em regularized least squares}:
\begin{align*}
\hat{f}
&= \argmin_{f \in \cH_k} \| f \|_{\cH_k}^2 + \sum_{i=1}^n  \sigma^{-2} (f(x_i) - y_i)^2 \\
&= \argmin_{f \in \cH_k} \sigma^2 \| f \|_{\cH_k}^2 + \sum_{i=1}^n    (f(x_i) - y_i)^2.
\end{align*}
Indeed, this is identical to the KRR optimization problem~\eqref{eq:square-loss} with kernel $k$ and regularization constant $\lambda = \sigma^2 / n$.

Thus, the minimum RKHS-norm interpolation using the regularized kernel's RKHS, where each function is the sum of an original RKHS function and a noise function, recovers the regularized least squares of KRR.
We summarize this below.

\begin{theorem} \label{coro:connection-KRR-KI}
Let $(x_i, y_i)_{i=1}^n \subset \mathcal{X} \times \mathbb{R}$ and assume \eqref{eq:distint-inputs-5887}.
Let $k: \cX \times \cX \to \mathbb{R}$ be a kernel and  $k^\sigma : \cX \times \cX \to \mathbb{R}$ be the regularized kernel in \eqref{eq:reg-kernel} with $\sigma > 0$.
Let $\hat{g}$ be the RKHS interpolant~\eqref{eq:min-norm-interp-aug-6353} using $k^\sigma$, and $\hat{f}$ be the KRR estimator~\eqref{eq:square-loss} using $k$ and $\lambda := \sigma^2 / n$.
Then we have
\begin{equation} \label{eq:KRR-KI-connection-solution}
\hat{g}  = \hat{f} + \sum_{i=1}^n (y_i - \hat{f}(x_i) ) \delta(\cdot, x_i).
\end{equation}
\end{theorem}

Theorem~\ref{coro:connection-KRR-KI} implies that the RKHS interpolant using the regularized kernel is identical to the KRR estimator at any test location different from the training locations:
$$
\hat{g}(x) = \hat{f}(x) \quad \text{for all }~ x \not= x_1, \dots, x_n,
$$
since the second term in \eqref{eq:KRR-KI-connection-solution} vanishes for such $x$.

Theorem~\ref{coro:connection-KRR-KI} is similar to Lemma~\ref{theo:relation-two-cov-funcs} showing that GP interpolation subsumes GP regression.
We will next discuss this connection.

\begin{remark}

The analytic expression of KRR in Theorem~\ref{theo:KRR-estimator} can be obtained using Theorems~\ref{coro:connection-KRR-KI} and \ref{theo:kernel-interpolation} and the identities in \eqref{eq:identities-vector-matrix}, assuming \eqref{eq:distint-inputs-5887} to hold.

\end{remark}

\index{Gaussian process--RKHS equivalence!regression}
\section{Equivalence between the GPR Posterior Mean and KRR}
\label{sec:GP-KRR-equivalence}

The equivalence between the GPR posterior mean function and the KRR estimator, as is well known in the literature~\citep{kimeldorf1970correspondence}, follows immediately from their analytic expressions in \eqref{eq:posteior_mean} and \eqref{eq:KRR_estimator}, as summarized below.

\begin{corollary}
     \label{theo:equivalnce}
Let $k: \mathcal{X} \times \mathcal{X} \to \mathbb{R}$ be a kernel and $(x_i, y_i)_{i=1}^n \subset \mathcal{X} \times \mathbb{R}$.
Then $\bar{m} = \hat{f}$, where
\begin{itemize}
    \item $\bar{m}: \cX \to \mathbb{R}$ is the posterior mean function~\eqref{eq:posteior_mean} of GPR with zero prior mean, covariance function $k$, and noise variance $\sigma^2 > 0$;
    \item $\hat{f}: \cX \to \mathbb{R}$ is the KRR estimator~\eqref{eq:KRR_estimator} using kernel $k$ and  regularization constant $\lambda = \sigma^2/n$.
\end{itemize}

\end{corollary}

This equivalence is understood as a special case of the equivalence between the GP and RKHS interpolants, which itself is a consequence of a more generic equivalence between the RKHS and GHS (Section~\ref{sec:interpolation}).
The posterior mean function of GPR is that of GP interpolation using the regularized kernel (Lemma~\ref{lemma:relation-post-mean-gp-int-reg} and its consequences).
The KRR estimator is the RKHS interpolant using the regularized kernel (Theorem~\ref{coro:connection-KRR-KI}).
Since the posterior mean function of GP interpolation is the RKHS interpolant for any kernel (Corollary~\ref{coro:equivalnce-interp}), the equivalence in Corollary~\ref{theo:equivalnce} holds.

We will next discuss RKHS interpretations of the posterior (co)variance in GPR, which also follow from those of the posterior covariance in GP interpolation.

\begin{remark}
\index{regularization!relation to noise variance}
The requirement $\lambda = \sigma^2 / n$ in Corollary~\ref{theo:equivalnce} has two interpretations.
First, if the GPR model \eqref{eq:regress-model-random}--\eqref{eq:Gaussian-noise} is correct, the regularization constant for KRR should decay at the rate $n^{-1}$ as the sample size $n$ increases; see Section~\ref{sec:convergence-GP-kernel} for details.
Second, assuming the existence of noise is equivalent to assuming that the latent function is smoother than the observed latent-plus-noise function, thus equivalent to performing smoothing = regularization.

\end{remark}

\index{posterior variance!worst-case interpretation}
\section{RKHS Interpretations of the Posterior Variance}
\label{sec:geometric-worst-post-var}

RKHS interpretations of the posterior (co)variance in GP regression are obtained as corollaries to those for GP interpolation in Section~\ref{sec:RKHS-interp-GP-UQ}, as GP regression is a particular case of GP interpolation using the regularized kernel.
Focusing on the posterior variance, we will derive the RKHS interpretations step-by-step, assuming \eqref{eq:distint-inputs-5887} as before.

Let $G$ be the latent-plus-noise process~\eqref{eq:noise-contamination} whose covariance is the regularized kernel $k^\sigma$ in \eqref{eq:reg-kernel}:
\begin{align*}
 & G = F + \xi \sim \GP(0, k^\sigma),  \\
 & \text{where }~~ F \sim \GP(0,k),\quad \xi \sim \GP(0, \sigma^2 \delta).
\end{align*}
The conditional variance of the latent process's value $F(x)$ at a test location~$x$, which is the posterior variance of GPR, plus the noise variance $\sigma^2$, is identical to the conditional variance of the latent-plus-noise process's value $G(x)$ given its values $G(x_1), \dots, G(x_n)$ at training locations (Lemma~\ref{theo:relation-two-cov-funcs}):
\begin{align*}
    & k(x,x) - {\bm k}_n(x)^\top ({\bm K}_n + \sigma^2 {\bm I}_n)^{-1} {\bm k}_n(x) + \sigma^2 \\
    & =  {\rm Var}[F(x) \mid G(x_1), \dots, G(x_n) ] + {\rm Var}[\xi(x)  ] \\
    & =  {\rm Var}[G(x) \mid G(x_1), \dots, G(x_n) ].
\end{align*}
The latter is identical to the squared distance between $G(x)$ and its best approximation $G_n(x)$ by $G(x_1), \dots, G(x_n)$ in the GHS $\cG_{k^\sigma}$ of $G$ (Lemma~\ref{lemma:interp-post-mean-var-ghs-gen}):
\begin{align*}
      {\rm Var}[G(x) \mid G(x_1), \dots, G(x_n) ]
    & = \left\| G(x) - G_n(x) \right\|^2_{\cG_{k^\sigma}}.
\end{align*}
By the equivalence between the GHS and the RKHS, this is further identical to the squared distance between the evaluation functional's Riesz representation $k^\sigma(\cdot, x)$ and its best approximation $k_n^\sigma(\cdot, x)$ by the evaluation functionals $k^\sigma(\cdot, x_1), \dots,  k^\sigma(\cdot, x_n)$ at the training locations in the RKHS of the regularized kernel (Section~\ref{sec:geomet-int-gp-var}):
\begin{align*}
 \left\| G(x) - G_n(x) \right\|^2_{\cG_{k^\sigma}}
    & = \left\| k^\sigma(\cdot, x) - k_n^\sigma(\cdot, x) \right\|^2_{\cH_{k^\sigma}}.
\end{align*}
Lastly, this RKHS distance equals the worst-case error of RKHS interpolation in the RKHS of the regularized kernel (Section~\ref{sec:worst-int-gp-var}):
\begin{align} \label{eq:worst-case-interp-reg-6529}
   \left\| k^\sigma(\cdot, x) - k_n^\sigma(\cdot, x) \right\|_{\cH_{k^\sigma}}   & = \sup_{\| g \|_{\cH_{k^\sigma}} \leq 1 } \left| g(x) - \sum_{i=1}^n w_i^\sigma(x) g(x_i) \right|,
\end{align}
where the weights $w_1^\sigma(x), \dots, w_n^\sigma(x)$ are those giving the best approximation:
\begin{align}
(w_1^\sigma(x),\dots,w_n^\sigma(x))^\top
& := \argmin_{w_1, \dots, w_n \in \mathbb{R}} \left\| k^\sigma(\cdot,x) - \sum_{i=1}^n w_i k^\sigma(\cdot,x_i) \right\|_{\cH_{k^\sigma}} \nonumber   \\
& =  ({\bm K}_n^\sigma)^{-1} {\bm k}_n^\sigma(x)  \nonumber \\
& =  ({\bm K}_n + \sigma^2 {\bm I}_n)^{-1} {\bm k}_n(x),  \label{eq:weights-KRR-proj}
\end{align}
which are identical to the weights for the KRR estimator with $\lambda = \sigma^2 / n$.
This is unsurprising as the KRR estimator is identical to the RKHS interpolant using the regularized kernel.

Therefore, the expression
$\sum_{i=1}^n w_i^\sigma(x) g(x_i)$
 in \eqref{eq:worst-case-interp-reg-6529} is the KRR estimator based on observations $(x_1, g(x_1)), \dots,  (x_n, g(x_n))$ of an unknown function~$g$ from the unit ball of the RKHS $\cH_{k^\sigma}$.
As discussed in Section~\ref{sec:KRR-as-KI}, such $g$ can be written as
\begin{align*}
g = f + \eta, & ~~ {\rm where}~~ f \in \cH_k ~~ {\rm and} ~~ \eta \in \cH_{\sigma^2}\\
& ~~ {\rm with}  ~~ \| f \|_{\cH_k}^2 + \| \eta \|_{\cH_{\sigma^2}}^2 \leq 1.
\end{align*}
The $\eta$ is a noise function, and thus the $f$ is the noiseless ``true'' component of $g$.
The difference
$$
g(x) - \sum_{i=1}^n w_i^\sigma(x) g(x_i) = f(x) - \sum_{i=1}^n w_i^\sigma(x) g(x_i) + \eta(x)
$$
is thus the error of the KRR prediction with respect to the unknown observation $g(x)$ that includes the noise $\eta(x)$ at $x$.

\index{worst-case error!kernel ridge regression}
Therefore,  {\em the square root of the GPR posterior variance plus the noise variance} is identical to {\em the worst-case error of the KRR estimator in predicting an RKHS  function plus noise}, as summarized below.
\begin{corollary}  \label{coro:post-var-worst-noisy}
Let $k: \cX \times \cX \to \mathbb{R}$ be a kernel, and $k^\sigma: \cX \times \cX \to \mathbb{R}$ be the regularized kernel~\eqref{eq:reg-kernel} with $\sigma^2 > 0$ with RKHS $\cH_{k^\sigma}$.
Suppose $x_1, \dots, x_n \in \cX$ satisfy \eqref{eq:distint-inputs-5887}.
Then for all $x \not= x_1,\dots,x_n$,  we have
\begin{align*}
    & \sqrt{ k(x,x) - {\bm k}_n(x)^\top ({\bm K}_n + \sigma^2 {\bm I}_n)^{-1} {\bm k}_n(x) + \sigma^2 }\nonumber \\
    & =  \quad \sup_{\| g \|_{\cH_{k^\sigma}} \leq 1 } \left| g(x) - \sum_{i=1}^n w_i^\sigma(x) g(x_i) \right|, 
\end{align*}
where
the weights $w_1^\sigma(x), \dots, w_n^\sigma(x) \in \mathbb{R}$ are given in \eqref{eq:weights-KRR-proj}.
\end{corollary}

The worst-case KRR prediction error quantifies the uncertainty about an unknown test observation in a non-probabilistic way. Corollary~\ref{coro:post-var-worst-noisy} shows that this non-probabilistic uncertainty estimate can be computed analytically via the equivalence with the square root of the GPR posterior variance plus the noise variance.

\index{convergence rate!regression}
\section{Convergence Rates} \label{sec:rates-and-posterior-contraction}

\label{sec:convergence-GP-kernel}

This section studies the relations between convergence results for GPR and KRR in the large sample asymptotics.
Specifically, we compare the posterior contraction rates for GPR by \citet{VarZan11} and the convergence rates for KRR by \citet{fischer2020sobolev}, which are both {\em minimax optimal} but under (seemingly) different assumptions.
We aim to clarify how these two results are related, including their assumptions.
We shall show that, by using GP sample path properties in Chapter~\ref{sec:theory}, the rates of \citet{VarZan11} can be recovered from those of \citet{fischer2020sobolev}.
This shows that the GPR-KRR equivalence extends to their convergence results.

\paragraph{Notation}
Let $L_2(P_X)$ be the Hilbert space of square-integrable functions for a probability distribution $P_X$; see \eqref{eq:lp-space}.
For $\beta >0$, let $C^\beta([0,1]^d)$ be the H\"older space of order $\beta$.

\paragraph{Literature}
While we focus here on \citet{VarZan11} and \citet{fischer2020sobolev}, many analyses exist for the convergence of GPR or KRR  \citep[e.g.,][]{KleVaa06,CapDev07,SteHusSco09,wynne2021convergence,li2024towards}, which we do not attempt to cover.

\index{posterior contraction!Gaussian process regression}
\subsubsection{Posterior Contraction Rates of GPR}
The setting considered by \citet{VarZan11} is as follows.
Let $\cX = [0,1]^d$ be the input space with $d \in \mathbb{N}$, and $f^*: \cX \to \Re$ be the unknown regression function as in \eqref{eq:regres_noise_model}.

Let $X \in \cX$ be an input random variable with distribution $P_X$ whose density function is bounded away from zero and infinity.
Define an output random variable $Y \in \mathbb{R}$ as
\begin{equation} \label{eq:GP-model-theory}
Y := f^*(X) + \xi  \quad {\rm where} \quad \xi \sim \N(0,\sigma^2),
\end{equation}
where $\sigma^2 > 0$ is the noise variance.
Training data
\begin{equation} \label{eq:gp-post-rate-data}
\mathcal{D}_n := \{ (X_1,Y_1), \dots, (X_n, Y_n) \} \subset \cX \times \mathbb{R}
\end{equation}
 are then assumed to be i.i.d.~with $(X,Y)$.

 Let $s > d/2$ and denote by $k_s: \cX \times \cX \to \mathbb{R}$ the Mat\'ern kernel whose RKHS $\cH_{k_s}$ is norm-equivalent to the Sobolev space $W_2^s([0,1]^d)$ of smoothness~$s$ (see Example~\ref{ex:matern-rkhs}).\footnote{
In the notation of Example~\ref{ex:matern-kernel}, this corresponds to $\alpha := s - d/2$.}

GPR is performed using the training data $\mathcal{D}_n$, the noise variance~$\sigma^2$ (or the likelihood model~\eqref{eq:GP-model-theory}), and the zero-mean GP prior  $\GP(0, k_s)$ with covariance kernel $k_s$.
This GP prior is such that a sample path $F \sim \GP(0,k_s)$ has (essentially) the smoothness $$\alpha := s - d/2.$$ See Corollary~\ref{coro:matern-sample-path}.

Under these assumptions,
the posterior distribution of GPR contracts around the ground truth $f^*$ for increasing sample sizes, provided it satisfies a certain regularity condition, as stated below \citep[Theorems 2 and 5]{VarZan11}.

\begin{theorem} \label{theo:GP-rate-van}
Let $k_s$ be a Mat\'ern kernel on $[0,1]^d$ whose RKHS is norm-equivalent to the Sobolev space of order $s := \alpha + d/2$ with $\alpha > d/2$. Suppose that, for $\beta > d/2$ the ground truth $f^*$ in \eqref{eq:GP-model-theory} satisfies
\begin{equation} \label{eq:regularity-ground-truth}
    f^* \in W_2^\beta ([0,1]^d) \cap C^\beta([0,1]^d).
\end{equation}
Then, as $n \to \infty$,
\begin{equation} \label{eq:post-contraction}
\bE_{\mathcal{D}_n|f^*} \left[ \mathbb{E}_{F | \mathcal{D}_n } \| F - f^* \|_{L_2(P_X)}^2  \right] = O(n^{- 2 \min(\alpha,\beta) / (2
\alpha + d) }),
\end{equation}
where $\bE_{\mathcal{D}_n|f^*}$ denotes the expectation for training data $\mathcal{D}_n$ generated as \eqref{eq:gp-post-rate-data}
 for $f^*$, and $ \mathbb{E}_{F | \mathcal{D}_n }$ the expectation for the posterior GP sample $F$ from GPR performed with $k_s$ given $\mathcal{D}_n$.
\end{theorem}

\begin{remark}
Let $\bar{m}_n: \cX \to \mathbb{R}$ be the posterior mean function~\eqref{eq:posteior_mean}:
$$
\bar{m}_n := \mathbb{E}_{F | \mathcal{D}_n} [F].
$$
By the convexity of the error functional $\left\| \cdot - f^* \right\|_{L_2(P_X)}$ and Jensen's inequality, it holds that
$$
 \| \bar{m}_n - f^* \|_{L_2(P_X)}^2 \leq  \mathbb{E}_{F | \mathcal{D}_n} [ \| F - f^* \|_{L_2(P_X)}^2 ].
$$
Therefore, the contraction rate \eqref{eq:post-contraction} implies the convergence rate of the posterior mean function $\bar{m}_n$ to the ground truth $f^*$,
\begin{equation} \label{eq:rate-post-mean}
\bE_{\mathcal{D}_n|f^*} \left[ \| \bar{m}_n - f^* \|_{L_2(P_X)}^2 \right] = O(n^{- 2 \min(\alpha,\beta) / (2
\alpha + d) }) \quad (n \to \infty).
\end{equation}
\end{remark}

The assumption~\eqref{eq:regularity-ground-truth} indicates that the smoothness of the ground truth $f^*$ is $\beta$.
The fastest rate for \eqref{eq:rate-post-mean} is attained when the smoothness $\alpha = s-d/2$ of the GP prior is specified as $\alpha = \beta$, resulting in the rate
\begin{equation} \label{eq:GPR-rate-minmax-opt}
    n^{-2\beta / (2 \beta + d)}.
\end{equation}
This is minimax optimal for regression of a ground-truth function belonging to $W_2^\beta([0,1]^d)$ \citep{Sto80}. \index{minimax rate!regression}
In other words, the optimal rate is attained when the GP prior's smoothness $\alpha$ matches the true function's smoothness~$\beta$.

For the optimal choice $\alpha = \beta$, the RKHS of the Mat\'ern kernel $k_s$ becomes norm-equivalent to the Sobolev space $W_2^s([0,1]^d)$ of order $s = \beta + d/2$, which is {\em $d/2$-smoother} than the smoothness $\beta$ of the ground truth $f^*$.
This means that the posterior mean $\bar{m}_n$ is $d/2$-smoother than $f^*$.
\index{oversmoothing}
Why does this {\em oversmoothing} lead to the optimal rate?
This question can be answered by considering the corresponding convergence result of KRR, as described next.

\begin{remark}
 GPR can be consistent even when the Gaussian noise assumption~\eqref{eq:GP-model-theory} is not satisfied \citep{KleVaa06}.
\end{remark}

\index{convergence rate!kernel ridge regression}
\subsubsection{Convergence Rates of KRR}
We next discuss the convergence rates of KRR by \citet[Corollary 5]{fischer2020sobolev}, which hold under the following conditions.

Let $\cX \subset \mathbb{R}^d$ be a non-empty, open, connected and bounded set with an infinitely differentiable boundary (e.g., an open ball in $\mathbb{R}^d$).
Let $X \in \mathcal{X}$ be a random vector with distribution $P_X$ whose density function is bounded away from zero and infinity, as for GPR.
For a ground truth function $f^*:\cX \to \mathbb{R}$, define a random variable $Y \in \mathbb{R}$ by
$$
Y := f^*(X) + \xi,
$$
where $\xi \in \mathbb{R}$ is a random variable representing the noise at $X$, such that the  conditional expectation of $\xi$ given $X$ is zero.
The $\xi$ and $X$ may be dependent.
\citet[Eq.~MOM]{fischer2020sobolev} assume that all the moments of $\xi$ given $X$ are bounded, in the sense that there exist constants $\sigma_M, L_M > 0$ such that for all $m \geq 2$ it holds that
$$
\mathbb{E}[~ |\xi|^m \mid X ] \leq \frac{1}{2} m! \sigma_M^2 L_M^{m-2}
$$
almost surely for $X$.
For instance, this moment condition is satisfied if $\xi$ is uniformly bounded, and if $\xi$ is a zero-mean Gaussian random variable whose variance is uniformly bounded.
The latter subsumes the setting of \citet{VarZan11} where $\xi$ is zero-mean Gaussian independent of $X$.

Training data  $(X_1, Y_1),\dots,(X_n,Y_n)$ are  assumed to be i.i.d.~realizations of $(X,Y)$.
As for GPR, consider the Mat\'ern kernel $k_s: \cX \times \cX \to \mathbb{R}$ whose RKHS $\cH_{k_s}$ is norm-equivalent to the Sobolev space $W_2^s(\cX)$ of order~$s > d/2$, which consists of functions of smoothness $s$.
Let $\hat{f}_n$ be the KRR estimator using the training data, the kernel $k_s$, and a regularization constant $\lambda > 0$. (Here, the dependence of $\hat{f}_n$ on the sample size $n$ is made explicit.)

The ground truth $f^*$ is assumed to have smoothness $\beta$, in the sense defined below.
Under that regularity condition, a specific case\footnote{Set $t = 0$ in \citet[Corollary 5]{fischer2020sobolev}, which makes the Besov norm with smoothness parameter $t$ equivalent to the $L_2$ norm. Note also that the Besov space $B_{2,2}^s(\cX)$ with $s \in \mathbb{N}$ and $s > d/2$ is equivalent to the Sobolev space $W_2^s(\cX)$.} of \citet[Corollary 5]{fischer2020sobolev} yields the following convergence rate of $\hat{f}_n$ toward $f^*$, when the regularization constant is decreased appropriately as the sample size increases.

\begin{theorem} \label{theo:rate-krr-steinwart}
Suppose the conditions specified above hold.
Let $k_s$ be a Mat\'ern kernel on $\cX \subset \mathbb{R}^d$ whose RKHS is norm-equivalent to the Sobolev space of order $s > d/2$.
Assume that the ground truth $f^*: \cX \to \mathbb{R}$ is bounded and satisfies
$$
f^* \in W_2^\beta(\cX) ~~ \text{for some}~~  0 < \beta < s .
$$
For a constant $c > 0$, set
\begin{equation} \label{eq:lambda_sob_spe}
\lambda = c n^{-2s/(2\beta+d)}.
\end{equation}
Then it holds that
\begin{equation} \label{eq:rate-KRR-sobolev}
\| \hat{f}_n - f^* \|_{L_2(P_X)}^2 = O_p(n^{-2\beta / (2\beta + d)}) \quad (n \to \infty).
\end{equation}
\end{theorem}

\begin{remark}
The above result of \citet{fischer2020sobolev} is a special case of their results that hold under more generic settings. They also provide rates in stronger norms than $L_2$.
\end{remark}

The rate in \eqref{eq:rate-KRR-sobolev} is identical to the rate~\eqref{eq:GPR-rate-minmax-opt} of GPR when the prior GP's smoothness $\alpha$ matches the smoothness $\beta$ of the ground truth $f^*$.
Thus, as mentioned already, this is minimax optimal in the considered setting.

Theorem \ref{theo:rate-krr-steinwart} does {\em not} assume that $f^*$ belongs to the RKHS of the kernel $k_s$, as the smoothness $\beta$ of $f^*$ is less than the smoothness $s$ of the RKHS.
However, the minimax rate is still attained.
In this sense, it is misleading to call the RKHS ``hypothesis space.''
The optimal estimation of a function outside the RKHS is possible because the KRR estimator's flexibility is controlled not only by the kernel but also by the regularization constant $\lambda$.
If $\lambda$ is set smaller, the KRR estimator becomes more flexible.
Indeed, the optimal decay rate~\eqref{eq:lambda_sob_spe} of $\lambda$ depends on both $s$ and $\beta$: If $\beta$ is smaller, the regularization constant must decay more quickly so that the KRR estimator becomes more flexible to approximate the less smooth ground truth.

To study the relation to GPR,
let us state the following corollary of Theorem~\ref{theo:rate-krr-steinwart}.
\begin{corollary} \label{coro:KRR-rate}
Suppose that the same conditions as Theorem~\ref{theo:rate-krr-steinwart} hold with $\beta = s - d/2$.
For a constant $c > 0$, set
$$
\lambda := c n^{-1}.
$$
Then it holds that
\begin{equation*}
\| \hat{f}_n - f^* \|_{L_2(P_X)}^2 = O_p(n^{-2\beta / (2\beta + d)}) \quad (n \to \infty).
\end{equation*}
\end{corollary}

The setting $\beta = s - d/2$ corresponds to the optimal setting for GPR when the prior GP's smoothness $\alpha = s-d/2$ matches the true function's smoothness~$\beta$.
In this case, Corollary~\ref{coro:KRR-rate} shows the optimal decay rate of the regularization constant for KRR is $\lambda = c / n$.
If the constant $c$ is set as the noise variance $\sigma^2$, this is exactly the condition $\lambda = \sigma^2 / n$ required for the GPR-KRR equivalence in Corollary~\ref{theo:equivalnce}.
Therefore, there is no contradiction between the optimal convergence results for GPR and KRR.

\chapter{Comparison of Probability Distributions}
\label{sec:integral_transforms}

\index{probability distribution comparison}
This chapter studies GP and RKHS approaches to statistical tasks involving the comparison of probability distributions.
Specifically, we consider the following three tasks:
\begin{enumerate}
\item Quantifying the discrepancy between two probability distributions (Section~\ref{sec:GPD-MMD}).
\item Numerical integration and sampling (Section~\ref{sec:kernel_and_bayesian_quadrature}).
\item Quantifying the dependency between two random variables (Section~\ref{sec:dependence}).
\end{enumerate}
These tasks are fundamental for many applications in data analysis and machine learning.
\index{two-sample test}
For example, the first task is related to the {\em two-sample problem} in which one tests whether two given datasets are generated from the same probability distribution or not.
The second task involves approximating a target probability distribution by an empirical distribution, the latter then being used for computing the integral of a given function.
\index{independence testing}
The third task is fundamental in {\em independence testing} in which one tests whether two random variables are independent or not from their paired samples.

\index{kernel mean embedding}
For these tasks, there is a unifying methodology known as {\em kernel mean embedding of distributions} \citep{SmoGreSonSch07,MuaFukSriSch17}.
Its key idea is to represent probability distributions as their embeddings into an RKHS and compare these representations using the RKHS geometry.
It enables the above statistical tasks to be solved using kernels, offering nonparametric methods that perform well in practice and have strong theoretical guarantees.
Examples include the {\em Maximum Mean Discrepancy} (MMD) for Task~1 \citep{GreBorRasSchetal12}, {\em Kernel Quadrature} for Task~2 \citep{Lar70}, and {\em Hilbert--Schmidt Independence Criterion} (HSIC)  for Task~3 \citep{GreBouSmoSch05}.

A natural question would be whether there exists a GP-based methodology corresponding to kernel mean embeddings.
This chapter investigates this question.
For each task, we first describe a GP-based approach and then present its equivalence to a kernel mean embedding method.
Specifically, we introduce the {\em Gaussian Process Discrepancy} (GPD) for Task~1, which is equivalent
to the MMD (Section~\ref{sec:GPD-MMD}); {\em Bayesian quadrature} for Task~2, equivalent to kernel quadrature (Section~\ref{sec:kernel_and_bayesian_quadrature}); and the {\em Gaussian Process Independence Criterion} (GPIC) for Task~3, equivalent to the HSIC (Section~\ref{sec:dependence}).

\index{GPD|see{Gaussian process discrepancy}}
\index{MMD|see{maximum mean discrepancy}}
\index{GPIC|see{Gaussian process independence criterion}}
\index{HSIC|see{Hilbert--Schmidt independence criterion}}

\section{Gaussian Processes and MMD}
\label{sec:GPD-MMD}

Suppose we want to quantify the {\em distance}, denoted here as $\|P - Q \|$,  between two probability distributions $P$ and $Q$ on a measurable space $\cX$.
Such a distance is fundamental in many applications.
For example, if $Q$ is an empirical distribution $Q := \frac{1}{n} \sum_{i=1}^n \delta_{X_i}$ of sample points $X_1, \dots, X_n \in \cX$, then the distance $\| P - Q \|$ quantifies how accurately the empirical distribution $Q$ approximates the true distribution $P$.
If both $P$ and $Q$ are empirical distributions, then the distance $\| P - Q \|$ indicates how similar the underlying distributions are.
If $P$ is the joint probability distribution of two random variables $X$ and $Y$ and $Q$ is the product of the marginal distributions of $X$ and $Y$, then the distance $\| P - Q \|$ quantifies the statistical dependence between $X$ and $Y$.

To provide an intuition, let us consider the case $\cX \subset \mathbb{R}$ for now.
The simplest way to compare $P$ and $Q$ is to compare their means.
That is, if the means $\int x~dP(x)$ and $\int x~dQ(x)$ are different, then $P$ and $Q$ are different.
On the contrary, if the means are the same, this does not imply that $P$ and $Q$ are the same.
For example, $P$ and $Q$ may differ in their (uncentered) second moments, $\int x^2~dP(x)$ and $\int x^2~dQ(x)$, in their third moments, $\int x^3~dP(x)$ and $\int x^3~dQ(x)$, or in other higher-order moments.
More generally, $P$ and $Q$ may differ in  nonlinear transforms such as $\int \sin(x) dP(x)$ and $\int \sin(x) dQ(x)$,   $\int \sin(2x) dP(x)$ and $\int \sin(2x) dQ(x)$, and so on.

In this way, one can quantify the difference between $P$ and $Q$ by applying many different transforms and comparing their expectations under $P$ and $Q$.
In practice, however, one cannot perform such comparisons for all possible transforms.
Instead, one could generate transforms {\em randomly}, compare their expectations under $P$ and $Q$, and then compute the expectation of the differences under the random transforms.
More specifically, if we define a Gaussian process $F \sim \GP(0,k)$ with a kernel~$k$ on a measurable space $\cX$, then we can consider
\begin{align}
& {\rm GPD}_k(P,Q) \nonumber \\
& :=  \sqrt{ \bE_{F \sim \GP(0,k)} \left[ \left( \int F(x) dP(x) - \int F(x)dQ(x) \right)^2 \right]}.  \label{eq:GP-distance-measures}
\end{align} 
\index{Gaussian process discrepancy}
We will call \eqref{eq:GP-distance-measures} the {\em Gaussian Process Discrepancy} (GPD) between $P$ and $Q$.
This quantifies the difference between $P$ and $Q$ by taking a GP sample path $F$ as a nonlinear transform, computing its expectations $\int F(x) dP(x)$ and $\int F(x)dQ(x)$ under $P$ and $Q$, respectively, and then averaging the squared difference $\left(\int F(x) dP(x) - \int F(x)dQ(x)\right)^2$ under $F \sim \GP(0,k)$.

Provided that GPD~\eqref{eq:GP-distance-measures} can be computed, the question is whether it is a distance between $P$ and $Q$.
In particular, does GPD become $0$ if and {\em only if} $P =Q$?
Of course, this depends on the choice of the covariance kernel $k$.
For example, if the kernel is linear, $k(x,x') = xx'$ for $x, x' \in \mathbb{R}$, then $F \sim \GP(0,k)$ is given as $F(x) = w x$ with $w \sim N(0,1)$ being standard Gaussian.
Then $\int F(x) dP(x) - \int F(x) dQ(x) = w \left( \int x~dP(x) - \int x~dQ(x) \right)$ and GPD~\eqref{eq:GP-distance-measures} becomes simply the distance between the means of $P$ and $Q$.
Therefore, GPD does not become a distance between $P$ and $Q$ unless one uses a kernel $k$ that makes sample paths of $F \sim \GP(0,k)$ diverse nonlinear transforms.

To answer which choice of kernel $k$ makes GPD~\eqref{eq:GP-distance-measures} a distance between $P$ and $Q$, it is convenient to view it from the RKHS perspective.
As we saw in Corollary~\ref{coro:mmd-gp} in Section~\ref{sec:master-theorem}, assuming that the kernel $k$ is bounded, \eqref{eq:GP-distance-measures} is identical to
\begin{align} \label{eq:MMD-sec-mean-embed}
& \left\| \int k(\cdot,x)dP(x) - \int k(\cdot,x)dQ(x) \right\|_{\cH_k} \\
& =   \sup_{f \in \cH_k: \| f \|_{\cH_k} \leq 1}  \left|   \int f(x)dP(x) - \int f(x)dQ(x)   \right|, \nonumber
\end{align}
where $\cH_k$ is the RKHS of $k$.
\index{maximum mean discrepancy}
This is the MMD between $P$ and $Q$ \citep{GreBorRasSchetal12}.
As shown on the right-hand side, it considers transformations~$f$ from the unit ball of the RKHS, $\{ f \in \cH_k: \| f \|_{\cH_k} \leq 1 \}$, and takes the one that maximizes the absolute difference between the expectations $\int f(x)dP(x) - \int f(x)dQ(x)$ under $P$ and $Q$.
As shown on the left-hand side, this is equal to the RKHS distance between the {\em kernel mean embeddings}
$$
  \int k(\cdot,x)dP(x)  \quad \text{and}\quad   \int k(\cdot,x)dQ(x)
$$
of $P$ and $Q$ \citep{SmoGreSonSch07,MuaFukSriSch17}.
If $\cH_k$ is separable, these kernel mean embeddings are the Bochner integrals of the feature map $x\mapsto k(\cdot,x)$; see Lemma~\ref{lem:bochner-kernel-mean} in Appendix~\ref{sec:kmean-est-realization}.

Therefore, the question of whether GPD~\eqref{eq:GP-distance-measures} is a distance or not can be reduced to the question of whether the embedding operator $\Phi_k : \mathcal{P} \to \cH_k$ from the set $\mathcal{P}$ of all probability measures on $\cX$ into the RKHS $\cH_k$
\begin{equation} \label{eq:RKHS-measure-embedding}
    \Phi_k(P) :=  \int k(\cdot,x)dP(x) \in \cH_k, \quad P \in \mathcal{P},
\end{equation}
is {\em injective} or not.
If it is injective, $\int k(\cdot,x)dP(x) = \int k(\cdot,x)dQ(x)$ holds if and {\em only if} $P = Q$, and thus GPD~\eqref{eq:GP-distance-measures} becomes zero if and only if $P = Q$.
\index{characteristic kernel}
In this case, the kernel $k$ is called {\em characteristic}~\citep{FukumizuGrettonSunScholkopf2008}.\footnote{The same injectivity property was introduced under the term {\em probability-determining} by \citet[p.~94]{FukumizuBachJordan2004}.}
Whether a given kernel is characteristic or not is a well-studied topic in the literature \citep[e.g.,][]{SriGreFukSchetal10,SimSch18}.
For example, Gaussian, Mat\'ern and Laplace kernels on $\cX \subset \mathbb{R}^d$ are known to be characteristic, while polynomial kernels are not.
To summarize, we have obtained the following corollary.
\begin{corollary} \label{coro:GPD-characteristic}
Let $\cX$ be a measurable space, $\mathcal{P}$ be the set of all probability measures on $\cX$, and $k: \cX \times \cX \to \mathbb{R}$ be a bounded characteristic kernel, i.e., the embedding operator in \eqref{eq:RKHS-measure-embedding} is injective.
Then GPD in \eqref{eq:GP-distance-measures} is a distance metric on $\mathcal{P}$, i.e., it holds that
\begin{enumerate}
    \item ${\rm GPD}_k(P,Q) \geq 0$ for all $P, Q \in \mathcal{P}$ and ${\rm GPD}_k(P,Q) = 0$ if and only if $P = Q$;
    \item  ${\rm GPD}_k(P,Q) = {\rm GPD}_k(Q,P)$ for all $P, Q \in \mathcal{P}$ (symmetry);
    \item ${\rm GPD}_k(P,Q)
    \leq {\rm GPD}_k(P, R) + {\rm GPD}_k(R, Q)$ for all $P, Q, R \in \mathcal{P}$ (triangle inequality).
\end{enumerate}

\end{corollary}

\begin{remark}
The triangle inequality in Corollary~\ref{coro:GPD-characteristic} is a consequence of the equivalence between the GPD and the RKHS distance between the embeddings:
\begin{align*}
 {\rm GPD}_k(P,Q)
 & = \left\| \Phi_k (P) - \Phi_k (Q) \right\|_{\cH_k} \\
& \leq  \left\| \Phi_k (P) - \Phi_k (R) \right\|_{\cH_k} + \left\| \Phi_k (R) - \Phi_k (Q) \right\|_{\cH_k}  \\
& = {\rm GPD}_k(P,R) + {\rm GPD}_k(R,Q),
\end{align*}
where the triangle inequality in the RKHS holds because $\left\| \cdot \right\|_{\cH_k}$ is a distance metric on $\cH_k$.

\end{remark}

Intuitively, a kernel $k$ being characteristic implies that the capacity of the RKHS $\cH_k$ is large enough to distinguish all the probability distributions.
In turn, this implies that GP sample paths $F \sim \GP(0,k)$ are diverse enough so that integrals $\int F(x)dP(x)$, which can be understood as ``features'' of distribution $P$, capture all the characteristics of $P$.
This also suggests that one can get insights about GP sample paths by understanding the properties of the RKHS.

\index{maximum mean discrepancy!empirical estimator}
\paragraph{Empirical estimation of GPD}

As is well known, the square of the MMD~\eqref{eq:MMD-sec-mean-embed}, and thus the square of the GPD~\eqref{eq:GP-distance-measures}, can be written in terms of the integrals of the kernel:
\begin{align*}
 &   \int \int  k(x,x') dP(x)dP(x')
    +  \int \int  k(y,y') dQ(y)dQ(y') \\
 &   - 2 \int \int  k(x,y) dP(x)dQ(y).
\end{align*}
Thus, if we have i.i.d.~samples $X_1, \dots, X_n$ from $P$ and $Y_1, \dots, Y_m$ from $Q$, then, as for the MMD, an unbiased estimator of the squared GPD can be defined as
$$
\frac{1}{n(n-1)} \sum_{i \not= j}k(X_i, X_j)
+ \frac{1}{m(m-1)} \sum_{i \not= j}k(Y_i, Y_j)
-  \frac{2}{nm} \sum_{i, j}k(X_i, Y_j),
$$
which can be used as a test statistic for two-sample hypothesis testing, where one wants to test $P = Q$ or $P
\not= Q$ based on the samples \citep{GreBorRasSchetal12}. \index{two-sample test!MMD statistic}

On the other hand, if $P$ is a known probability measure and $Q$ is a weighted signed measure $Q = \sum_{i=1}^n w_i \delta_{x_i}$ with sample points $x_1, \dots, x_n \in \cX$ and weights $w_1, \dots, w_n \in \mathbb{R}$, and $Q$ is constructed to approximate $P$, then the GPD (=MMD) quantifies the {\em error} of $Q$ approximating $P$.\footnote{The definitions of kernel mean embeddings and MMD (and thus GPD) extend naturally to finite signed measures whenever the relevant integrals exist; see, e.g., \citet{SriFukLan11}.}
In this context, the equivalence between the GPD~\eqref{eq:GP-distance-measures} and the MMD~\eqref{eq:MMD-sec-mean-embed} \index{average-case error!probability distribution approximation}
\index{worst-case error!probability distribution approximation}
can be interpreted as the equivalence between the {\em average-case error} and the {\em worst-case error}:
\begin{align*}
& \sqrt{ \bE_{F \sim \GP(0,k)} \left[ \left( \int F(x) dP(x) - \sum_{i=1}^n w_i F(x_i)  \right)^2 \right]} \\
& = \quad \sup_{ \| f \|_{\cH_k} \leq 1}  \left|   \int f(x)dP(x) - \sum_{i=1}^n w_i f(x_i)   \right|.
\end{align*}
This equivalence is classic and known in the numerical integration literature \citep[e.g.,][Corollary 7, p.~40]{Rit00}.

\index{numerical integration}
\section{Bayesian and Kernel Quadrature}

\label{sec:kernel_and_bayesian_quadrature}

We next discuss the problem of {\em numerical integration}.
Suppose that there is a function $f: \cX \to \mathbb{R}$ and we are interested in computing its integral
\begin{equation} \label{eq:integral-3369}
    \int f(x) dP(x)
\end{equation}
for a known probability measure $P$ on $\cX$.
For example, in Bayesian inference, $x$ may be the model parameters, $f$ the likelihood function, and $P$ the prior distribution; the integral~\eqref{eq:integral-3369} is then the marginal likelihood, or model evidence.

The task of numerical integration is to {\em estimate} the value of the integral~\eqref{eq:integral-3369} by appropriately selecting input points $x_1, \dots, x_n \in \cX$ and evaluating the corresponding function values $f(x_1), \dots, f(x_n)$.
For example, if we generate $x_1, \dots, x_n$ by i.i.d.~random sampling from $P$, then the empirical average $\frac{1}{n}\sum_{i=1}^n f(x_i)$ is a Monte Carlo estimate of the integral, whose error decays at the rate $n^{-1/2}$ as $n$ increases.
However, Monte Carlo is inappropriate if the evaluation of $f(x)$ takes time or is expensive.
For example, suppose $f(x)$ is the output of a computationally expensive simulator, where $x$ represents the input parameters. For instance, $f(x)$ could be the simulated global average temperature 50 years in the future, with $x$ encoding relevant environmental and socio-economic inputs (e.g., greenhouse gas emission trajectories, population growth, or technological development).
In this case, one may need to run a supercomputer for days to weeks to compute one output $f(x)$ for a given $x$.
In such cases, one must carefully select the evaluation locations $x_1, \dots, x_n$ to compute the function values $f(x_1), \dots, f(x_n)$, as the total number $n$ of evaluations cannot be large.

To select $x_1, \dots, x_n$ wisely, one should exploit not only the knowledge about the probability measure $P$ but also prior knowledge or assumption about the function $f$.
For example, one could assume the continuity of $f$ such that if we change $x$ slightly, the corresponding simulation output $f(x)$ also changes slightly.
Such a smoothness assumption can be naturally encoded by modeling that $f$ is a sample path of a specific GP.
Based on this GP model, one can decide the evaluation locations $x_1, \dots, x_n$ and estimate the integral~\eqref{eq:integral-3369}; \index{Bayesian quadrature}
this approach is called {\em Bayesian quadrature} \citep[e.g.,][]{diaconis1988bayesian,Oha91,GhaZou03,briol2019probabilistic,MahsereciKarvonen2026}.
Alternatively, one could express prior knowledge about the function $f$ using an RKHS.
By assuming that $f$ belongs to a specific RKHS, one could select the evaluation points $x_1, \dots, x_n$ and estimate the integral; \index{kernel quadrature}
this is the approach of {\em kernel quadrature} \citep[e.g.,][]{CheWelSmo10,DicKuoSlo13,Bac17,kanagawa2020convergence}.\footnote{For early work related to what are now called kernel and Bayesian quadrature, see \citet{Sul59,Sul60,Lar70,Lar72}; see \citet{OatSul19} and \citet[Section~2.3.1]{MahsereciKarvonen2026} for historical discussions.}

What is the relation between  Bayesian quadrature and kernel quadrature?
As one could anticipate, there is a certain equivalence between them \citep{huszar2012optimally}, while their modeling assumptions seem different (i.e., $f \sim \GP(0,k)$ versus $f \in \cH_k$).
Here, we will discuss this equivalence.

\index{Bayesian quadrature!geometric interpretation}
\subsection{Bayesian Quadrature from a Geometric Viewpoint}

\begin{figure}[t]
    \centering
\resizebox{1.0\linewidth}{!}{
\begin{tikzpicture}[
  >=latex,
  font=\large,
  plane/.style={fill=gray!6, draw=gray!60, dashed, line width=0.7pt},
  vec/.style={->, line width=0.9pt},
  resid/.style={->, line width=1.0pt},
  map/.style={->, line width=0.8pt},
  lbl/.style={fill=white, inner sep=1.4pt},
  maplbl/.style={fill=white, inner sep=1.8pt, font=\large}
]
  \coordinate (OL)  at (6.10,2.05);
  \coordinate (L1)  at (1.00,2.05);
  \coordinate (L2)  at (3.45,-0.65);
  \coordinate (CL)  at ($(L1)+(L2)-(OL)$);
  \coordinate (Pk)  at (1.80,0.70);
  \coordinate (Kx)  at (1.80,4.45);

  \draw[plane] (OL) -- (L1) -- (CL) -- (L2) -- cycle;
  \draw[vec] (OL) -- (L1);
  \draw[vec] (OL) -- (L2);
  \draw[vec] (OL) -- (Pk);
  \draw[vec] (OL) -- (Kx);
  \draw[resid] (Pk) -- (Kx);

  \draw[line width=0.6pt]
    (Pk) -- ++(0.23,0) -- ++(0,0.23) -- ++(-0.23,0);

  \foreach \P in {OL,L1,L2,Pk,Kx}
    \fill (\P) circle (1.55pt);

  \node[lbl,below=3pt] at (L1) {$k(\cdot,x_1)$};
  \node[lbl,below=3pt] at (L2) {$k(\cdot,x_2)$};
  \node[lbl,below left=2pt] at (Pk) {$\mu_{P_n}$};
  \node[lbl,above=2pt] at (Kx) {$\mu_P$};
  \node[lbl,left=5pt] at ($(Pk)!0.61!(Kx)$)
    {$\mu_P-\mu_{P_n}$};

  \node[
    lbl,
    below=18pt,
    xshift=10pt,
    font=\large,
    align=center
  ] at ($(CL)!0.39!(L2)$)
  {
    $\mathcal{S}_n=\operatorname{span}\{k(\cdot,x_1),k(\cdot,x_2)\}$\\[-1pt]
    $\mathcal{S}_n\subset\mathcal{H}_k$
  };

  \coordinate (OR)  at (13.35,2.05);
  \coordinate (R1)  at (8.25,2.05);
  \coordinate (R2)  at (10.95,-0.65);
  \coordinate (CR)  at ($(R1)+(R2)-(OR)$);
  \coordinate (PF)  at (9.10,0.70);
  \coordinate (Fx)  at (9.10,4.45);

  \draw[plane] (OR) -- (R1) -- (CR) -- (R2) -- cycle;
  \draw[vec] (OR) -- (R1);
  \draw[vec] (OR) -- (R2);
  \draw[vec] (OR) -- (PF);
  \draw[vec] (OR) -- (Fx);
  \draw[resid] (PF) -- (Fx);

  \draw[line width=0.6pt]
    (PF) -- ++(0.23,0) -- ++(0,0.23) -- ++(-0.23,0);

  \foreach \P in {OR,R1,R2,PF,Fx}
    \fill (\P) circle (1.55pt);

  \node[lbl,below=3pt] at (R1) {$F(x_1)$};
  \node[lbl,below=3pt] at (R2) {$F(x_2)$};
  \node[lbl,below left=2pt] at (PF) {$Z_n$};
  \node[lbl,above=2pt] at (Fx) {$Z$};
  \node[lbl,left=5pt] at ($(PF)!0.61!(Fx)$)
    {$Z-Z_n$};

  \node[
    lbl,
    below=18pt,
    xshift=10pt,
    font=\large,
    align=center
  ] at ($(CR)!0.39!(R2)$)
  {
    $\mathcal{U}_n=\operatorname{span}\{F(x_1),F(x_2)\}$\\[-1pt]
    $\mathcal{U}_n\subset\mathcal{G}_k$
  };

  \draw[map] ($(Kx)+(0.12,0)$) -- ($(Fx)+(-0.12,0)$);
  \node[maplbl,above=4pt] at ($(Kx)!0.5!(Fx)$)
    {$\psi_k:\mathcal{H}_k\to\mathcal{G}_k$};

  \draw[map] ($(Pk)+(0.12,0)$) -- ($(PF)+(-0.12,0)$);
  \node[maplbl,above=4pt,xshift=25pt] at ($(Pk)!0.5!(PF)$)
    {$\psi_k:\mathcal{H}_k\to\mathcal{G}_k$};
\end{tikzpicture}
}
\caption{
Geometric interpretation of the equivalence between kernel quadrature and Bayesian quadrature.
Under the canonical isometry $\psi_k:\cH_k\to\cG_k$, the kernel mean embedding
$\mu_P=\int k(\cdot,x)\,dP(x)$ corresponds to the GP integral
$Z=\int F(x)\,dP(x)$.
The RKHS subspace $\cS_n$ spanned by $k(\cdot,x_1),\dots,k(\cdot,x_n)$
corresponds to the GHS subspace $\cU_n$ spanned by $F(x_1),\dots,F(x_n)$.
The best approximation $\mu_{P_n}$ of $\mu_P$ in $\cS_n$ corresponds to the
best approximation $Z_n$ of $Z$ in $\cU_n$.
Likewise, the residual $\mu_P-\mu_{P_n}$ corresponds to the residual $Z-Z_n$.
Their common squared norm is the posterior variance of Bayesian quadrature.
}
\label{fig:equiv-quadrature}
\commentAlt{Figure~\ref{fig:equiv-quadrature}}{Parallel projection diagrams compare kernel quadrature in the RKHS with Bayesian quadrature in the Gaussian Hilbert space under the canonical isometry. See long description.}
\commentLongAlt{Figure~\ref{fig:equiv-quadrature}}{The left diagram shows the kernel mean embedding projected orthogonally onto the subspace spanned by two kernel features. The right diagram shows the corresponding Gaussian-process integral projected onto the subspace spanned by two observed function values. Horizontal arrows identify the target, its projection, and its residual across the two spaces. Right-angle markers show that both residuals are orthogonal to their respective subspaces.}
\end{figure}

To simplify the discussion, we suppose here that the evaluation points $x_1, \dots, x_n$ have already been given, and focus on estimating the integral~\eqref{eq:integral-3369} using the observed function values $f(x_1), \dots, f(x_n)$.
In Bayesian quadrature, we suppose that the function $f$ is a sample of $\GP(0,k)$, and thus
\begin{equation} \label{eq:conditioning-BQ-3415}
    F(x_i) = f(x_i)  \quad (i = 1,\dots,n), \quad F \sim \GP(0,k).
\end{equation}
What we want is the posterior distribution of the integral
\begin{equation} \label{eq:BQ-integral-4319}
    \int F(x) dP(x),
\end{equation}
given the conditioning~\eqref{eq:conditioning-BQ-3415}.
This posterior is Gaussian, because the integral~\eqref{eq:BQ-integral-4319} without conditioning is Gaussian and the conditioning variables $F(x_1), \dots, F(x_n)$ are also Gaussian.
Therefore, to derive the posterior, it is sufficient to identify the conditional expectation and variance of \eqref{eq:BQ-integral-4319} given the conditioning~\eqref{eq:conditioning-BQ-3415}.

Here, recall that the integral~\eqref{eq:BQ-integral-4319} is an element of the Gaussian Hilbert space $\cG_k$, which is induced from the GP, that corresponds to the kernel mean embedding of $P$ in the RKHS $\cH_k$ (see Example~\ref{example:integral}):
\begin{equation} \label{eq:GP-integral-embedd-3427}
    \int F(x) dP(x) = \psi_k \left( \int k(\cdot, x)dP(x)  \right) \in \cG_k,
\end{equation}
\index{canonical isometry!quadrature}
where $\psi_k: \cH_k \to \cG_k$ is the canonical isometry.

The following lemma provides formulas and geometric interpretations for the conditional expectation and variance of a generic element $Z \in \cG_k$ in the Gaussian Hilbert space.
It is a generalization of Lemma~\ref{lemma:interp-post-mean-var-ghs-gen} in which $Z$ is the GP function value $F(x)$ at a given point $x$.
We prove it for completeness.

\begin{lemma} \label{lemma:post-mean-var-ghs-gen}
Let $k: \cX \times \cX \to \mathbb{R}$ be a kernel and $ \cG_k$ be the Gaussian Hilbert space associated with $F \sim \GP(0, k)$.
Let $x_1, \dots, x_n \in \cX$ such that the kernel matrix $ {\bm K}_n := (k(x_i, x_j))_{i,j = 1}^n \in \mathbb{R}^{n \times n} $ is invertible.
Define
$$
\cU_n
:= \left\{ c_1 F(x_1) + \cdots + c_n F(x_n):~   c_1, \dots, c_n \in \mathbb{R}  \right\}
\subset \cG_k
$$
as the subspace spanned by $F(x_1), \dots, F(x_n) \in \cG_k$.
Let $Z \in \cG_k$ be arbitrary, and define $Z_n \in \cG_k$ and $Z_n^\perp$ as the orthogonal projection of $Z$ onto $\cU_n$ and its residual, respectively:
$$
Z_n := \arg\min_{Y \in \cU_n} \| Z - Y \|_{\cG_k}, \quad Z_n^\perp := Z - Z_n.
$$
Then we have
\begin{align} \label{eq:3696}
  Z_n &= \mathbb{E}[Z \mid F(x_1), \dots, F(x_n)], \\
 \| Z_n^\perp \|_{\cG_k}^2 &=  {\rm Var}[Z \mid F(x_1), \dots, F(x_n)]. \nonumber
\end{align}
\end{lemma}

\begin{proof}
Because $Z_n^\perp$ is orthogonal to any element in $\cU_n$ by definition, we have $\left< Z_n^\perp, F(x_i) \right>_{\cG_k} = \bE[ Z_n^\perp F(x_i) ] = 0$ for all $i = 1, \dots, n$.
Thus $Z_n^\perp$ is independent of $F(x_1), \dots, F(x_n)$ and
 $$
\mathbb{E}[ Z_n^\perp \mid F(x_1), \dots, F(x_n)  ] = \mathbb{E}[Z_n^\perp] = 0,
 $$
where $\mathbb{E}[Z_n^\perp] = 0$ follows from $\cG_k$ consisting of zero-mean Gaussian random variables.

Since $Z_n \in \cU_n$,  it can be written as $Z_n = c_1^* F(x_1) + \cdots + c_n^* F(x_n)$ for some constants $c_1^*, \dots, c_n^* \in \mathbb{R}$.
Therefore,
\begin{align*}
& \mathbb{E}[Z \mid F(x_1), \dots, F(x_n)]
 =
 \mathbb{E}[ Z_n + Z_n^\perp \mid F(x_1), \dots, F(x_n)  ] \\
 & =    \mathbb{E}[ Z_n \mid F(x_1), \dots, F(x_n)  ] \\
 & =  \mathbb{E}\left[ \sum_{i=1}^n c_i^* F(x_i) \mid F(x_1), \dots, F(x_n)  \right] = \sum_{i=1}^n c_i^* F(x_i) = Z_n.
\end{align*}
This proves the first identity in \eqref{eq:3696}.
For the second identity in \eqref{eq:3696}, we have
\begin{align*}
 & {\rm Var}[Z \mid F(x_1), \dots, F(x_n)]
  = \mathbb{E}[ (Z-Z_n)^2 \mid F(x_1), \dots, F(x_n) ] \\
 & = \mathbb{E}[ (Z_n^\perp)^2 \mid F(x_1), \dots, F(x_n) ] =  \mathbb{E}[ (Z_n^\perp)^2  ]  = \left\| Z_n^\perp \right\|_{\cG_k}^2,
\end{align*}
which completes the proof.
\end{proof}

Lemma~\ref{lemma:post-mean-var-ghs-gen} shows that the conditional expectation of any element $Z \in \cG_k$ is its orthogonal projection $Z_n := \arg\min_{Y \in \cU_n} \| Z - Y \|_{\cG_k}$ onto the subspace $\cU_n$ spanned by the conditioning variables $F(x_1), \dots, F(x_n) \in \cG_k$.
(Here, note that $\mathbb{E}[Z \mid F(x_1), \dots, F(x_n)]$ is a random variable, as $F(x_1), \dots, F(x_n)$ are random variables.)
Writing $Y = c_1 F(x_1) + \cdots + c_n F(x_n)$, the orthogonal projection $Z_n$ can be written as $Z_n = c_1^* F(x_1) + \cdots + c_n^*F(x_n)$, where the coefficients $c_1^*, \dots, c_n^* \in \mathbb{R}$ are given by
\begin{align} \label{eq:BQ-coefficients-3489}
 c_1^*, \dots, c_n^*
& := \argmin_{c_1, \dots, c_n \in \mathbb{R}} \left\| \sum_{i=1}^n c_i F(x_i) - Z \right\|_{\cG_k}^2 \\
& = \argmin_{c_1, \dots, c_n \in \mathbb{R}} \mathbb{E} \left[ \left(\sum_{i=1}^n c_i F(x_i) - Z\right)^2 \right], \nonumber
\end{align}
i.e., $Z_n$ solves the least-squares problem of approximating $Z$ as a linear combination of $F(x_1), \dots, F(x_n)$ and thus $Z_n$ is the best approximation of $Z$ in this sense.

On the other hand, the conditional variance ${\rm Var}[Z \mid F(x_1), \dots, F(x_n)]$ is the squared norm $\| Z_n^\perp \|_{\cG_k}^2$ of the residual $Z_n^\perp = Z - Z_n$, i.e., the squared distance between the original element $Z$ and its projection $Z_n \in \cU_n$.
This can be understood as the mean-square error of the best approximation $Z_n$ against $Z$:
\begin{align*}
  {\rm Var}[Z \mid F(x_1), \dots, F(x_n)]
 = \left\| Z_n - Z \right\|_{\cG_k}^2
 = \mathbb{E} \left[\left( Z_n - Z \right)^2 \right].
\end{align*}

Let us now apply Lemma~\ref{lemma:post-mean-var-ghs-gen} to the case $Z :=
\int F(x) dP(x)$ to derive its conditional distribution.
Figure \ref{fig:equiv-quadrature} visualizes Lemma~\ref{lemma:post-mean-var-ghs-gen} in this case.
To derive the conditional expectation, we need to compute the coefficients $c_1^*, \dots, c_n^*$ in \eqref{eq:BQ-coefficients-3489}.
Recalling the equivalence between the GP integral $\int F(x)dP(x) \in \cG_k$ and the kernel mean embedding $\int k(\cdot, x)dP(x) \in \cH_k$ via the canonical isometry $\psi_k: \cH_k \to \cG_k$ (see \eqref{eq:GP-integral-embedd-3427}), we can write
\begin{align*}
& \left\| \sum_{i=1}^n c_i F(x_i) - \int F(x)dP(x) \right\|_{\cG_k}^2 
=   \left\| \sum_{i=1}^n c_i k(\cdot, x_i) - \int k(\cdot, x)dP(x) \right\|_{\cH_k}^2  \\
& = \sum_{i,j =1}^n c_i c_j k(x_i,x_j) - 2 \sum_{i=1}^n c_i \int k(x_i, x)dP(x) + \int \int k(x,x')dP(x)dP(x').
\end{align*}
Minimizing this with respect to $c_1, \dots, c_n$ leads to the solution
\begin{align*}
    & (c_1^*, \dots, c_n^*)^\top := {\bm K}_n^{-1} {\bm \mu}_n, \\
    &\text{where}~~  {\bm \mu}_n := \left( \int k(x_1, x)dP(x), \dots,  \int k(x_n, x)dP(x) \right)^\top \in \mathbb{R}^n,
\end{align*}
where ${\bm K}_n := (k(x_i, x_j))_{i,j=1}^n \in \mathbb{R}^{n \times n}$ is the kernel matrix.
With these coefficients $c_1^*, \dots, c_n^*$, the conditional expectation of the integral is given by
\begin{align}
& \mathbb{E}\left[ \left. \int F(x) dP(x) \right| F(x_1), \dots, F(x_n) \right]
= Z_n = \sum_{i=1}^n c_i^* F(x_i).   \label{eq:3755}
\end{align}
For the conditional variance, we can write
\begin{align}
& {\rm Var}\left[ \int F(x) dP(x) \mid F(x_1), \dots, F(x_n) \right]  =  \left\| Z - Z_n \right\|_{\cG_k}^2 \nonumber \\
& = \left< Z - Z_n, Z - Z_n \right>_{\cG_k}
 \stackrel{(*)}{=} \left< Z, Z - Z_n \right>_{\cG_k}  \nonumber \\
& = \left< \int k(\cdot, x)dP(x),\ \int k(\cdot, x)dP(x) - \sum_{i=1}^n c_i^* k(\cdot, x_i) \right>_{\cH_k} \nonumber \\
& = \left\| \int k(\cdot, x)dP(x) \right\|_{\cH_k}^2  - \sum_{i=1}^n c_i^* \int k(x_i, x)dP(x) \nonumber 
\end{align}
\begin{align}
& = \int \int k(x, x')dP(x)dP(x') - {\bm \mu}_n^\top {\bm K}_n^{-1} {\bm \mu}_n, \label{eq:3761}
\end{align}
where $\stackrel{(*)}{=}$ follows from $Z_n$ and $Z-Z_n$ being orthogonal in $\cG_k$.

Recall that, as stated in \eqref{eq:conditioning-BQ-3415}, we want the conditional distribution of the integral given the conditioning $F(x_i) = f(x_i)$, $i=1,\dots,n$.
Substituting this conditioning in
\eqref{eq:3755}, \index{Bayesian quadrature!posterior mean}
we thus obtain the conditional expectation (or the posterior mean) of the integral
\begin{align} \label{eq:posterior-mean-BQ-3570}
& \mathbb{E}\left[ \left. \int F(x) dP(x) \right| F(x_i) = f(x_i), i=1,\dots,n \right]
=   \sum_{i=1}^n c_i^* f(x_i).
\end{align}
On the other hand, notice that the last expression of the posterior variance~ \eqref{eq:3761} does {\em not} depend on $F(x_1), \dots, F(x_n)$; thus, the conditioning $F(x_i) = f(x_i)$ by the actual observed function values $f(x_1), \dots, f(x_n)$ does {\em not} affect the posterior variance; see Remarks~\ref{rem:post-ver-GP-interp} and \ref{remark:post-indep-observations}.

\index{Bayesian quadrature!posterior variance}
To summarize, the posterior distribution of the integral $\int F(x)dP(x)$, given the prior $F \sim \GP(0,k)$ and the conditioning $F(x_i) = f(x_i)$ for $i = 1, \dots, n$, is the Gaussian distribution with  mean~\eqref{eq:posterior-mean-BQ-3570} and variance~\eqref{eq:3761}.
This is how Bayesian quadrature estimates the integral and quantifies its uncertainty.

While we assumed that the evaluation points $x_1, \dots, x_n$ have already been given, it is also possible to {\em optimize} them by using
the variance~\eqref{eq:3761} as an objective function to minimize \citep[e.g.,][]{gessner2020active,pronzato2023performance}.
We will see next that this optimal selection of the design points $x_1, \dots, x_n$ is equivalent to the minimization of the MMD, the fact first pointed out by \citet{huszar2012optimally}.

\begin{remark} \label{remark:post-indep-observations}
The posterior variance~\eqref{eq:3761} being independent of observed function values is not specific to Bayesian quadrature but is common for Gaussian process regression-based methods {\em without} hyperparameter selection for the kernel. Suppose the kernel parameters (or the kernel itself) are selected based on observed function values (e.g., via marginal likelihood maximization). In that case, the posterior variance will depend on the function values, and thus, it can adapt to the unknown ground-truth function~$f$. See, e.g., \citet{naslidnyk2025comparing} for an asymptotic argument about this point.

\end{remark}

\begin{remark}
To compute the posterior mean
~\eqref{eq:posterior-mean-BQ-3570} and variance~\eqref{eq:3761}, one needs an analytic expression of $\int k(\cdot, x)dP(x)$ for the kernel $k$ and distribution $P$, \index{kernel--measure compatibility}
thus requiring a  compatibility between $k$ and $P$.
For example, a closed-form expression is available when $k$ and $P$ are Gaussian on $\cX = \Re^d$.
For other examples, including the kernel Stein discrepancy \citep{OatGirCho17}, we refer to \citet{briol2025dictionary} and references therein.
\end{remark}

\index{maximum mean discrepancy!Bayesian quadrature}
\subsection{MMD Interpretation of the Quadrature Posterior Variance}

Let us now discuss the equivalence between kernel quadrature and Bayesian quadrature.
As can be seen from Lemma~\ref{lemma:post-mean-var-ghs-gen} and Figure~\ref{fig:equiv-quadrature}, the posterior variance~\eqref{eq:3761} is the squared distance between $Z = \int F(x)dP(x)$ and its projection $Z_n$.
Here, note that
$$
Z = \psi_k\left( \int k(\cdot, x) dP(x) \right), \quad
Z_n = \sum_{i=1}^n c_i^* F(x_i) = \psi_k \left( \sum_{i=1}^n c_i^* k(\cdot, x_i) \right).
$$
Since the canonical isometry $\psi_k: \cH_k \to \cG_k$ preserves the distance, we have
\begin{align}
&     {\rm Var}\left[ \int F(x) dP(x) \mid F(x_1), \dots, F(x_n) \right]
 = \left\| Z - Z_n \right\|_{\cG_k}^2 &  \nonumber \\
 & = \left\| \int k(\cdot, x)dP(x) - \sum_{i=1}^n c_i^* k(\cdot, x_i) \right\|_{\cH_k}^2 = {\rm MMD}_k^2(P, P_n),&  \label{eq:equiv-BQ-post-var-MMD}
\end{align}
where the last expression is the squared MMD between $P$ and the finite signed measure
$$
P_n := \sum_{i=1}^n c_i^* \delta_{x_i}.
$$
Here,
$$
\sum_{i=1}^n c_i^* k(\cdot, x_i)
$$ is the orthogonal projection of the  mean embedding $\int k(\cdot, x)dP(x)$ onto the subspace
$$
S_n := \left\{ c_1 k(\cdot, x_1) + \cdots + c_n k(\cdot, x_n):~ \ c_1, \dots, c_n \in \mathbb{R}   \right\} \subset \cH_k
$$
spanned by $k(\cdot, x_1), \dots, k(\cdot, x_n)$ (see \eqref{eq:BQ-coefficients-3489}):
$$
c_1^*, \dots, c_n^* = \argmin_{c_1, \dots, c_n \in \mathbb{R}} \left\|  \int k(\cdot, x)dP(x) - \sum_{i=1}^n c_i k(\cdot, x_i) \right\|_{\cH_k}^2.
$$
In this sense, $\sum_{i=1}^n c_i^* k(\cdot,x_i)$ is the best approximation of $\int
k(\cdot, x)dP(x)$.

Therefore, using the posterior variance~\eqref{eq:3761} as an objective function for selecting $x_1, \dots, x_n$ is equivalent to selecting $x_1, \dots, x_n$ to minimize the MMD between $P$ and its best approximation $P_n = \sum_{i=1}^n c_i^* \delta_{x_i}$.
\index{kernel quadrature!MMD formulation}
Kernel quadrature is a methodology for numerical integration and sampling using MMD in such a way.

Lastly, note that the squared distance between $Z$ and $Z_n$ can be written as the squared GPD~\eqref{eq:GP-distance-measures} between $P$ and $P_n$:
$$
\left\| Z - Z_n \right\|_{\cG_k}^2
= \mathbb{E}\left[ \left( \int F(x)dP(x) - \sum_{i=1}^n c_i^* F(x_i) \right)^2 \right]
= {\rm GPD}^2_k(P, P_n).
$$
Therefore, by the equivalence between GPD and MMD discussed in~Section~\ref{sec:GPD-MMD}, we have
$$
\left\| Z - Z_n \right\|_{\cG_k}^2 = {\rm GPD}^2_k(P, P_n) = {\rm MMD}_k^2(P, P_n).
$$
Thus, the equivalence in \eqref{eq:equiv-BQ-post-var-MMD} also follows from the GPD--MMD equivalence.

\index{statistical dependence}
\section{Gaussian Processes and HSIC}\label{sec:dependence}

Suppose that, given two random variables $X \in \cX$ and $Y \in \cY$, where $\cX$ and $\cY$ are measurable spaces (which can be different), we want to quantify the strength (or the lack) of the statistical dependence between $X$ and $Y$.
(In practice, one must estimate the dependence of $X$ and $Y$ based on data $(X_1,Y_1), \dots, (X_n, Y_n)$; we will discuss this later.)
If $\cX = \cY = \mathbb{R}$, the most basic way is to compute the covariance between $X$ and $Y$,
$$
{\rm Cov}[X, Y]  := \mathbb{E}\left[ \left(X - \mathbb{E}[X] \right) \left(Y - \mathbb{E}[Y] \right)  \right]
$$
and the correlation,
$
{\rm Corr}[X, Y] := \frac{{\rm Cov}[X,Y]}{ \sqrt{{\rm Var}[X]} \sqrt{{\rm Var}[Y]} }.
$
The covariance and correlation only quantify the {\em linear} dependence between $X$ and $Y$.
Therefore, they cannot capture the {\em nonlinear} dependence between $X$ and $Y$ even when it exists.

For example, consider the following simple setting.
Let $Z \sim {\rm unif}[-2, 2]$ be a uniform random variable on the interval $[-2, 2] \subset \mathbb{R}$ and define random variables $X$ and $Y$ by
\begin{equation} \label{eq:example-nonlinear-dependence}
X := \begin{cases}
1 & \text{if } Z \in [0,1) \\
-1 & \text{if } Z \in (-1,0) \\
0 & \text {otherwise}
\end{cases},
\quad
Y := \begin{cases}
-1 & \text{if } Z \in [1,2] \\
1 & \text{if } Z \in [-2,-1] \\
0 & \text {otherwise}
\end{cases}.
\end{equation}
In this case, there is dependence between $X$ and $Y$, since $X$ takes non-zero values only when $Y$ is zero (thus $Y$ provides information about $X$).
However,  it is easy to see that the covariance between $X$ and $Y$ is zero, ${\rm Cov}[X, Y] = 0$ (and so ${\rm Corr}[X,Y] = 0$).
Thus, the covariance cannot capture the dependence between $X$ and $Y$ in this example.

One way to capture the nonlinear dependence between $X$ and $Y$ is to apply {\em nonlinear transforms} to $X$ and $Y$ and compute the covariance between the transformed variables.
That is, if we appropriately choose nonlinear functions $f: \cX \to \mathbb{R}$ and $g: \cY \to \mathbb{R}$ so that the covariance ${\rm Cov}[ f(X), g(Y) ]$ between $f(X)$ and $g(Y)$ is non-zero, we can detect the dependence between $X$ and $Y$.
Let us consider the example in \eqref{eq:example-nonlinear-dependence} again.
Define $f$ as $f(x) := 1(x \not= 0)$ and $g$ as $g(y) = 1(y=0)$, where $1(A)$ denotes the indicator function such that $1(A) = 1$ if statement $A$ is true and $1(A) = 0$ otherwise.
Then we have ${\rm Cov}[f(X), g(Y)] = 1/4$, thus detecting the nonlinear dependence between $X$ and $Y$.

\index{Gaussian process independence criterion}
\subsection{Gaussian Process Independence Criterion}
The question is how to find appropriate nonlinear transforms $f$ and $g$.
As in the above example, to choose nonlinear transforms appropriately, one needs to know how $X$ and $Y$ are related.
This is a chicken-and-egg problem because to detect the dependence between $X$ and $Y$, one needs to know how they are dependent.
Therefore, in practice, one must try various nonlinear transforms and compute the resulting covariances.

Like GPD in Section~\ref{sec:GPD-MMD}, one possible way to consider various nonlinear transforms $f$ and $g$ is to generate them {\em randomly}, compute the covariance ${\rm Cov}[f(X), g(Y)]$ for each generated $f$ and $g$, and take the expectation of the squared covariance.
Specifically, if we consider two independent GPs, $F \sim \GP(0,k)$ and $G \sim \GP(0,\ell)$, where $k: \cX \times \cX \to \mathbb{R}$ and $\ell: \cY \times \cY \to \mathbb{R}$, as nonlinear transforms, we can define the following {\em Gaussian Process Independence Criterion} (GPIC):
\begin{equation} \label{eq:GPIC}
    {\rm GPIC}_{k,\ell}(X,Y)
:= \mathbb{E}_{F \sim \GP(0,k),\ G \sim \GP(0, \ell)} \left[ {\rm Cov}^2 [ F(X), G(Y) ]  \right].
\end{equation}
Intuitively, if $F \sim \GP(0, k)$ and $G \sim \GP(0, \ell)$ are diverse enough, there is a possibility to generate $F$ and $G$ such that ${\rm Cov}[F(X), G(Y)]$ is non-zero, making GPIC  positive and thus detecting the nonlinear dependence between $X$ and $Y$.
On the other hand, if $X$ and $Y$ are independent, then ${\rm Cov}[F(X), G(Y)] = 0$ for any $F$ and $G$ and the GPIC becomes zero.

Now the question is which kernels $k$ and $\ell$ make the GPIC~\eqref{eq:GPIC} work as an independence criterion. That is, when does the following hold?
\begin{equation} \label{eq:GPIC-independence-4178}
    {\rm GPIC}_{k,\ell}(X,Y) = 0\  \text{holds if and {\em only if}}\ X \ \text{and}\ Y \ \text{are independent}.
\end{equation}
Similar to the case of GPD, it is convenient to view the problem from an RKHS perspective.
In the following, we will see that the GPIC~\eqref{eq:GPIC} is identical to the {\em Hilbert--Schmidt Independence Criterion} (HSIC) \citep{GreBouSmoSch05}, a kernel measure of statistical dependence widely studied in the literature.

\begin{remark} \label{remark:GPIC-brownian}
 \index{Brownian distance covariance}
GPIC~\eqref{eq:GPIC} was proposed as the Brownian Distance Covariance by \citet[Definitions 4 and 5]{SzeRiz09}, focusing on the setting where $\cX = \mathbb{R}^p$ and $\cY = \mathbb{R}^q$ with $p, q \in \mathbb{N}$ are Euclidean spaces and the kernels $k$ and $\ell$ are the Brownian motion covariance kernels defined as
\begin{align*}
    k(x,x') &:= \| x \| + \| x' \| -  \| x - x' \|, \quad x,x' \in \Re^p, \\
\ell(y,y') &:=  \| y \| + \| y' \| -  \| y - y' \|, \quad y,y' \in \Re^q.
\end{align*}
These kernels make $F \sim \GP(0,k)$ and $G \sim \GP(0,\ell)$ Brownian motions.
The definition~\eqref{eq:GPIC} generalizes the Brownian distance covariance to be defined for generic GPs, allowing us to make connections to its RKHS counterpart, the HSIC.

\end{remark}

Before proceeding, let us show that
the GPIC can be expressed analytically in terms of the kernels $k$ and $\ell$.
To this end, below $\bE_{X,Y}$ denotes that the expectation is taken with respect to the joint distribution $P_{X,Y}$ of $X$ and $Y$; $\bE_{X}$ (or $\bE_{Y}$) denotes that the expectation is with respect to the marginal distribution $P_X$ of $X$ (or the marginal distribution $P_Y$ of $Y$).
Let $(X', Y')$ be independent copies of $(X,Y)$.
Then, we can rewrite the GPIC as follows.

\begin{proposition} \label{prop:GPIC-kernel-expression}
Let $\cX$ and $\cY$ be measurable spaces, and $(X,Y) \in \cX \times \cY$ be random variables, and $k: \cX \times \cX \to \mathbb{R}$ and $\ell: \cY \times \cY \to \mathbb{R}$ be kernels.
Suppose $\mathbb{E}_X [k(X,X)] < \infty$ and $\mathbb{E}_Y [\ell(Y,Y)] < \infty$.
Then, for the GPIC defined in \eqref{eq:GPIC}, we have
\begin{align}
    {\rm GPIC}_{k,\ell}(X,Y)
=&   \bE_{X,Y} \bE_{X',Y'} [ k(X,X') \ell(Y,Y') ] \nonumber \\
& +  \bE_{X}  \bE_{X'}   \bE_{Y} \bE_{Y'}[ k(X,X') \ell(Y, Y') ] \nonumber \\
&  -  2 \bE_{X, Y}  \bE_{X'}  \bE_{Y'} [ k(X,X') \ell(Y,Y') ]. \label{eq:GPIC-kernel-expression}
\end{align}
\end{proposition}
\begin{proof}
    See Appendix~\ref{sec:proof-GPIC-kernel-expression}.
\end{proof}

Therefore, if one has i.i.d.~data $(X_1,Y_1), \dots, (X_n,Y_n) \sim P_{X,Y}$, then one can estimate the GPIC by estimating the expectations in \eqref{eq:GPIC-kernel-expression}:
\begin{align}
&   \frac{1}{n (n-1)} \sum_{i \not= j}  k(X_i, X_j) \ell(Y_i, Y_j) \nonumber \\
& +  \left( \frac{1}{n (n-1)}  \sum_{i \not= j} k(X_i,X_j) \right) \left( \frac{1}{n (n-1)}   \sum_{i \not= j} \ell(Y_i,Y_j) \right) \nonumber \\
& - \frac{2}{n} \sum_{i=1}^n \left( \frac{1}{n-1} \sum_{j \not= i} k(X_i, X_j) \right) \left( \frac{1}{n-1} \sum_{j \not= i} \ell(Y_i, Y_j) \right).  \label{eq:GPIC-estimate}
\end{align}

The expression~\eqref{eq:GPIC-kernel-expression} is identical to the kernel expression of HSIC \citep[Lemma 1]{GreBouSmoSch05}, so the equivalence between GPIC and HSIC has already been shown.
In the following, we delve into this equivalence from an RKHS viewpoint.

\index{Hilbert--Schmidt independence criterion}
\subsection{Hilbert--Schmidt Independence Criterion}

To study the equivalence between the GPIC and HSIC, \index{product kernel}
let us introduce the {\em product kernel} $k \otimes \ell: (\cX \times \cY)  \times (\cX \times \cY) \to \mathbb{R}$ of $k$ and $\ell$ defined on the product space $\cX \times \cY$:
$$
k \otimes \ell \left( (x, y), (x',y') \right) := k(x,x') \ell(y,y'),  \quad (x,y), (x', y') \in \cX \times \cY.
$$
This quantifies the similarity between two pairs $(x,y)$ and $(x',y')$.
It is a positive definite kernel on $\cX \times \cY$ and thus is the covariance function of a GP on $\cX \times \cY$ such that
$$
H \sim \GP(0, k \otimes \ell), \quad \bE\left[ H(x,y) H(x', y') \right] = k(x,x') \ell(y,y').
$$
Therefore, using this in the expression~\eqref{eq:GPIC-kernel-expression}, the GPIC can be rewritten as
\begin{align}
    {\rm GPIC}_{k,\ell}(X,Y)
= &\    \bE_{X,Y} \bE_{X',Y'} [ \bE_H \left[ H(X,Y) H(X', Y') \right]  ] \nonumber \\
& +  \bE_{X}  \bE_{X'}   \bE_{Y} \bE_{Y'}[ \bE_H \left[ H(X,Y) H(X', Y') \right] ] \nonumber \\
&  -  2 \bE_{X, Y}  \bE_{X'}  \bE_{Y'} [\bE_H \left[ H(X,Y) H(X', Y') \right]  ] \nonumber \\
=&\ \bE_H \left[ \left( \bE_{X,Y} [H(X,Y)] - \bE_{X'} \bE_{Y'}[H(X',Y')] \right)^2 \right] \nonumber \\
=&\ {\rm GPD}^2_{k \otimes \ell}(P_{X,Y}, P_X \otimes P_Y), \label{eq:GPIC-GPD-4290}
\end{align}
where the second equality follows from Fubini's theorem, which is applicable if $k$ and $\ell$ are bounded.
The expression~\eqref{eq:GPIC-GPD-4290} shows that the GPIC computes the expectations of $H$ under the joint distribution $P_{X,Y}$ and the product of the marginal distributions $P_X \otimes P_Y$,\footnote{Precisely, $P_X \otimes P_Y$ is defined as the distribution on $\cX \times \cY$ such that $P_X \otimes P_Y (A \times B) = P_X(A) P_Y(B)$ for all measurable sets $A \subset \cX$ and $B \subset \cY$.} takes their squared difference, and then averages over $H \sim \GP(0, k \otimes \ell)$.
Therefore, GPIC measures the discrepancy between $P_{X,Y}$ and $P_X \otimes P_Y$ using $H \sim \GP(0, k\otimes\ell)$.
Indeed, as written in \eqref{eq:GPIC-GPD-4290}, this is the squared GPD between $P_{X,Y}$ and $P_X \otimes P_Y$ (see \eqref{eq:GP-distance-measures}).
Since $X$ and $Y$ are independent if and only if $P_{X,Y} = P_X \otimes P_Y$, measuring the discrepancy between $P_{X,Y}$ and $P_X \otimes P_Y$ is to quantify the dependence between $X$ and $Y$; this is what GPIC does.

From the equivalence between the GPD and MMD discussed in Section~\ref{sec:GPD-MMD}, the GPIC can be written as the MMD between $P_{X,Y}$ and $P_X \otimes P_Y$; this is the HSIC \citep[Section 2.3]{SmoGreSonSch07}:
\begin{align}
    & {\rm GPIC}_{k, \ell}(X,Y)
    = {\rm GPD}_{k \otimes \ell}^2(P_{X,Y}, P_X \otimes P_Y) \nonumber \\
    & = {\rm MMD}^2_{k \otimes \ell}(P_{X,Y}, P_X \otimes P_Y) \nonumber \\
    &=  \left\|  \Phi_{k \otimes \ell} (P_{X,Y}) -  \Phi_{k \otimes \ell} (P_X \otimes P_Y)  \right\|_{\cH_{k \otimes \ell}}^2 = {\rm HSIC}_{k, \ell}(X,Y), \label{eq:HSIC-MMD-4304}
\end{align}
where $\cH_{k \otimes \ell}$ is the RKHS of the product kernel $k \otimes \ell$,  $ \Phi_{k \otimes \ell}: \mathcal{P}_{\cX \times \cY} \to \cH_{k \otimes \ell}$ is the  mean embedding map for the set of all joint probability distributions on $\cX \times \cY$, denoted by $\mathcal{P}_{\cX \times \cY}$:
$$
 \Phi_{k \otimes \ell} (P) := \int (k \otimes \ell) (\cdot, (x,y)) dP(x,y), \quad P \in \mathcal{P}_{\cX \times \cY}.
$$
Therefore, whether the GPIC is qualified as an independence criterion is determined by the capacity of the RKHS $\cH_{k \otimes \ell}$.
A sufficient condition is that $k \otimes \ell$ is a characteristic kernel,~i.e.,
$$
\Phi_{k \otimes \ell} (P) = \Phi_{k \otimes \ell} (Q) \ \text{if and only if} \ P = Q \quad (P, Q \in \mathcal{P}_{\cX \times \cY}).
$$
However, for the GPIC (= HSIC) to be an independence criterion as in \eqref{eq:GPIC-independence-4178}, the following weaker condition is sufficient:
\begin{equation} \label{eq:Independence-mean-embeds}
    \Phi_{k \otimes \ell} (P) = \Phi_{k \otimes \ell} (P_\cX \otimes P_\cY) \ \text{if and only if} \ P = P_\cX \otimes P_\cY \quad  (P \in \mathcal{P}_{\cX \times \cY}),
\end{equation}
where $P_\cX$ and $P_\cY$ are the marginal distributions of a generic joint distribution $P$ on $\cX$ and $\cY$, respectively.\footnote{Here, $P_\cX$ and $P_\cY$ should not be confused with $P_X$ and $P_Y$, which are the marginal distributions of random variables $X$ and  $Y$.}

The condition~\eqref{eq:Independence-mean-embeds} is satisfied when $k$ and $\ell$ are characteristic for probability distributions on their respective domains $\cX$ and $\cY$, given that $\cX$ and $\cY$ are separable metric spaces; see \citet[Theorem 2]{gretton2015simpler} and \citet[Theorem 3 (i)]{szabo2018characteristic}.
For example, if $k$ and $\ell$ are Gaussian, Laplace or Mat\'ern kernels on $\cX \subset \mathbb{R}^p$ and $\cY \subset \mathbb{R}^q$ with $p, q \in \mathbb{N}$, then the GPIC becomes an independence criterion as in \eqref{eq:GPIC-independence-4178}, and thus its estimator~\eqref{eq:GPIC-estimate} can be used as a test statistic for testing the independence between $X$ and $Y$ \citep{GreFukTeoSonSchSmo08}.

Lastly, the MMD expression of HSIC in \eqref{eq:HSIC-MMD-4304} can be written as the square of the ``covariance'' between the two RKHS-valued random variables $k(\cdot, X) \in \cH_{k}$ and $\ell(\cdot, Y) \in \cH_{\ell}$.
To show this, note that
\begin{align*}
\Phi_{k \otimes \ell} (P_{X,Y})
& = \bE_{X,Y} [k(\cdot, X) \otimes \ell(\cdot, Y)]  \in \cH_{k \otimes \ell}, \\
\Phi_{k \otimes \ell} (P_X \otimes P_Y)
& = \bE_{X} \bE_{Y} [k(\cdot, X) \otimes \ell(\cdot, Y)] \\
& =  \bE_{X} [k(\cdot, X)] \otimes \bE_{Y}[ \ell(\cdot, Y)]  \in \cH_{k \otimes \ell}.
\end{align*}
Therefore,
\begin{align}
& \Phi_{k \otimes \ell} (P_{X,Y}) - \Phi_{k \otimes \ell} (P_X \otimes P_Y) \nonumber \\
& = \bE_{X,Y} [k(\cdot, X) \otimes \ell(\cdot, Y)] - \bE_{X} [k(\cdot, X)] \otimes \bE_{Y}[ \ell(\cdot, Y)] \nonumber \\
& = \bE_{X,Y} \left[\left(\ k(\cdot, X) - \bE_{X} [k(\cdot, X)]  \ \right) \otimes \left(\ \ell(\cdot, Y) - \bE_{Y}[ \ell(\cdot, Y)] \ \right) \right] \nonumber \\
& =: {\rm Cov}[ k(\cdot, X), \ell(\cdot, Y) ] \in \cH_{k \otimes \ell}.  \label{eq:cov-op-def}
\end{align}
Thus, the HSIC~\eqref{eq:HSIC-MMD-4304} can be expressed as
\begin{align} \label{eq:HSIC-Cov-express-4419}
& {\rm HSIC}_{k, \ell}(X, Y)
= \left\|   {\rm Cov}[ k(\cdot, X), \ell(\cdot, Y) ]  \right\|_{\cH_{ k \otimes \ell } }^2.
\end{align}
This expression is essentially the original definition of HSIC by \citet{GreBouSmoSch05}.
\index{cross-covariance operator}
The ``covariance'' defined in \eqref{eq:cov-op-def}, which is an element in $\cH_{k\otimes\ell}$, is identical to the {\em cross-covariance operator} $C_{X,Y}: \cH_k \to \cH_\ell$ when viewed as the operator from $\cH_k$ to $\cH_\ell$ such that
\begin{align*}
& C_{X,Y} f \qquad (f \in \cH_k) \\
& := \bE_{X,Y} \left[\left<\ f,  \ k(\cdot, X) - \bE_{X} [k(\cdot, X)]  \ \right>_{\cH_k}  \left(\ \ell(\cdot, Y) - \bE_{Y}[ \ell(\cdot, Y)] \ \right) \right]  \\
& ~= \bE_{X,Y} \left[\left(\ f(X) - \bE_{X} [f(X)]  \ \right)  \left(\ \ell(\cdot, Y) - \bE_{Y}[ \ell(\cdot, Y)] \ \right) \right]  ~~ \in \cH_\ell.
\end{align*}
Then, the expression of HSIC in \eqref{eq:HSIC-Cov-express-4419} is identical to the Hilbert--Schmidt norm of $C_{X,Y}$: this is why it is called ``Hilbert--Schmidt Independence Criterion'' \citep{GreBouSmoSch05}. \index{Hilbert--Schmidt norm}
To summarize, we have shown the following equivalence between the GPIC and HSIC:
\begin{align*}
{\rm GPIC}_{k,\ell}(X,Y)
& = \mathbb{E}_{F \sim \GP(0,k),\ G \sim \GP(0, \ell)} \left[ {\rm Cov}^2 [ F(X), G(Y) ]  \right] \\
& =  \left\|   {\rm Cov}[ k(\cdot, X), \ell(\cdot, Y) ]  \right\|_{\cH_{ k \otimes \ell } }^2 = {\rm HSIC}_{k, \ell}(X, Y)  .
\end{align*}

\begin{remark}
    The equivalence between GPIC and HSIC implies that the Distance Covariance, a popular measure of statistical dependence proposed by \citet{szekely2007measuring},  is a special case of HSIC, a fact first pointed out by \citet{SejSriGreFuk13}.
    The reasoning is as follows:
    The Distance Covariance is identical to the Brownian Distance Covariance \citep[Theorem 8]{SzeRiz09}, and the latter is a special case of GPIC  (see Remark~\ref{remark:GPIC-brownian}), which is identical to HSIC.
    Thus, we have recovered the result of \citet[Theorem 24]{SejSriGreFuk13} through the GPIC--HSIC equivalence.
\end{remark}

\chapter{Conclusions}
\label{chap:conclusion}

This monograph studied the connections and equivalences between GP-based and reproducing kernel-based methods for fundamental problems in machine learning, statistics, and numerical analysis, including regression, interpolation, integration, and statistical discrepancies.
It was motivated by parallel developments of the two approaches from the corresponding almost separate research communities.
Understanding the relations between the two approaches enables either community to learn from the developments in the other community more easily.
This monograph serves this understanding.

The equivalences between the two approaches arise from the canonical isometry between the Gaussian Hilbert space (GHS) and the reproducing kernel Hilbert space (RKHS). In particular, a bounded linear functional of a GP is represented by the GHS element corresponding to its Riesz representation in the RKHS. This yields the fundamental identity between the root mean square of the GP functional and its maximum absolute value over the unit ball of the RKHS. In specific estimation problems, the corresponding estimates are projections onto data subspaces in the GHS and RKHS, and the canonical isometry identifies these projections.

The uncertainty of a quantity of interest in a GP method, measured by the posterior standard deviation, is the distance between that quantity and its projection onto the data subspace in the GHS. By the canonical isometry, this distance equals the corresponding distance in the RKHS, which can in turn be expressed as the worst-case error of the RKHS method over the unit ball. This provides a natural RKHS interpretation of GP uncertainty.

Moreover, the sample path properties of a GP are determined by the RKHS of the covariance kernel.
A GP sample does not lie in the RKHS of its covariance kernel with probability one when that RKHS is infinite-dimensional.
However, a GP sample is shown to lie in an RKHS ``slightly larger'' than the original RKHS with probability one.
In this sense, a GP sample is slightly less smooth than the RKHS functions.
Understanding GP sample properties is important as they are the properties of the hypothesis space of a GP method.
The RKHS offers such an understanding.

There are many topics on the relations between GP-based and reproducing kernel-based methods not covered in this monograph.
\begin{itemize}
    \item
    \index{kernel selection}
\index{hyperparameter selection}
\textbf{Kernel and hyperparameter selection.}
The equivalences discussed in this monograph are stated for a fixed kernel. In practice, the kernel is often selected from data. This matters because choosing a kernel means choosing the GP prior and the RKHS, hence the assumed smoothness and other structural properties of the unknown function. \index{model misspecification}
\index{uncertainty calibration}
Data-driven kernel choice is therefore important for reducing model misspecification and for calibrating uncertainty estimates.
\index{marginal likelihood}
For GP methods, kernels and their hyperparameters are commonly chosen by marginal likelihood or predictive likelihood on held-out data. \index{cross-validation}
For RKHS-based methods, hyperparameters are often chosen by cross-validation. These criteria are related, but not identical: likelihood-based criteria involve predictive variances, while standard cross-validation estimates prediction error. Some RKHS interpretations are available for scale-parameter estimation and posterior-standard-deviation calibration; see, for example, \citet{naslidnyk2025comparing}. \index{kernel learning|see{kernel selection}}
A broader theory connecting kernel learning in the GP and RKHS views remains an important topic.

    \index{approximate solution}
\item {\bf Approximate solutions.}
    Generally, GP-based and RKHS-based methods do not scale to large amounts of data.
Their exact solutions require computing the kernel matrix of the data size, the computational complexity of which increases quadratically with the data size, and may require computing the inverse of the (regularized) kernel matrix, whose complexity increases cubically with the data size.
    Computationally-efficient approximations exist for either approach, such as data reduction or summarization \citep[e.g.,][]{williams2000using,titsias2009variational} \index{finite-basis approximation}
and finite-basis approximations of the kernel \citep[e.g.,][]{rahimi2007random,hensman2018variational}.
    These approximations for GP-based and RKHS-based methods are developed separately in the two communities, while there are connections and equivalences between them \citep{wild2021connections,Wynne-Veit-VGP-2022}.
    A better understanding of their relations would help develop more efficient approximations for either approach.

    \item {\bf Other problems.} There are many other problems for which GP-based and RKHS-based methods are developed but not covered in this monograph, \index{global optimization}
such as (global) optimization \citep[e.g.,][]{jones1998efficient,srinivas2010gaussian,chowdhury2017kernelized,rudi2024finding,garnett2023bayesian}, \index{partial differential equation}
\index{physics-informed machine learning}
solving partial differential equations or physics-informed machine learning \citep[e.g.,][]{chen2021solving,pfortner2022physics,hennig2022probabilistic,doumeche2024physics}, \index{infinite-width neural network}
and analyzing the infinite-width limits of neural networks \citep[e.g.,][]{neal1996bayesian,bach2017breaking,matthews2018gaussian,jacot2018neural,lee2018deep}.
    The relations between the GP-based and RKHS-based methods in these problems are fruitful topics for future research.
\end{itemize}

\oneappendix

\chapter{Proofs and Auxiliary Results}

\section{Bochner Integrability and Empirical Approximation of Kernel Mean Embeddings}
\sectionmark{Kernel Mean Embeddings}
\label{sec:kmean-est-realization}

This appendix establishes two facts used in Example~\ref{example:integral}. First, the feature map $x\mapsto k(\cdot,x)$ is Bochner integrable, and its Bochner integral coincides with the Riesz representation of
\[
f\longmapsto\int_{\cX}f(x)\,dP(x).
\]
Second, this kernel mean embedding can be approximated in the RKHS
norm by empirical kernel means.

\begin{lemma}[Bochner representation of the kernel mean embedding]
\label{lem:bochner-kernel-mean}
Let $\cX$ be a measurable space, $P$ a probability measure on
$\cX$, and $k$ a measurable kernel on $\cX$ with separable RKHS
$\cH_k$. Assume that
\[
\int_{\cX}\sqrt{k(x,x)}\,dP(x)<\infty.
\]
Then the feature map
\[
\Phi_k:\cX\to\cH_k,
\qquad
\Phi_k(x):=k(\cdot,x),
\]
is Bochner integrable with respect to $P$. Moreover, if $f_A$ is
the Riesz representer of
\[
A(f):=\int_{\cX}f(x)\,dP(x),
\qquad f\in\cH_k,
\]
then
\[
f_A
=
\int_{\cX}\Phi_k(x)\,dP(x)
=
\int_{\cX}k(\cdot,x)\,dP(x),
\]
where the integrals are Bochner integrals.
\end{lemma}

\begin{proof}
For each $f\in\cH_k$, choose functions of the form
\[
f_m=\sum_{j=1}^{N_m}a_{m,j}k(\cdot,x_{m,j})
\]
such that $f_m\to f$ in $\cH_k$. Since $k$ is measurable,
each $f_m$ is measurable. Moreover, by the reproducing property,
\[
|f_m(x)-f(x)|
\leq
\|f_m-f\|_{\cH_k}\sqrt{k(x,x)}
\longrightarrow 0
\]
for every $x\in\cX$. Thus every $f\in\cH_k$ is measurable, and
\[
x\longmapsto
\langle f,\Phi_k(x)\rangle_{\cH_k}
=
f(x)
\]
is measurable. Hence $\Phi_k$ is weakly measurable. Since $\cH_k$ is separable, Pettis' measurability theorem
\citep[Theorem~4.21]{vanNeerven22} implies that $\Phi_k$ is
strongly measurable. Moreover,
\[
\int_{\cX}\|\Phi_k(x)\|_{\cH_k}\,dP(x)
=
\int_{\cX}\sqrt{k(x,x)}\,dP(x)
<\infty.
\]
Together with strong measurability, this implies that $\Phi_k$ is
Bochner integrable with respect to $P$
\citep[Proposition~1.51]{vanNeerven22}. Let
\[
b:=\int_{\cX}\Phi_k(x)\,dP(x)\in\cH_k
\]
denote the Bochner integral.

The functional $A$ is bounded, since
\[
|A(f)|
\leq
\|f\|_{\cH_k}
\int_{\cX}\sqrt{k(x,x)}\,dP(x).
\]
For each $f\in\cH_k$, the mapping
\[
h\longmapsto\langle f,h\rangle_{\cH_k}
\]
is a bounded linear functional. Since bounded linear functionals
commute with Bochner integration
\citep[Section~1.5.b]{vanNeerven22}, the reproducing property gives
\begin{align*}
\langle f,b\rangle_{\cH_k}
&=
\int_{\cX}
\langle f,\Phi_k(x)\rangle_{\cH_k}\,dP(x)\\
&=
\int_{\cX}f(x)\,dP(x)
=
A(f).
\end{align*}
Thus $b$ is the Riesz representer of $A$, and hence $b=f_A$.
\end{proof}

\begin{proposition}
\label{prop:kmean-est-realization}
Under the assumptions of
Lemma~\ref{lem:bochner-kernel-mean}, let $f_A$ denote the Riesz
representer of
\[
A(f):=\int_{\cX}f(x)\,dP(x),
\qquad f\in\cH_k.
\]
Thus $f_A$ is the kernel mean embedding
\[
f_A=\int_{\cX}k(\cdot,x)\,dP(x),
\]
where the integral is a Bochner integral. There exist points
$x_1,x_2,\dots\in\cX$ such that
\[
\left\|
f_A-\frac{1}{n}\sum_{i=1}^n k(\cdot,x_i)
\right\|_{\cH_k}
\longrightarrow 0
\qquad\text{as }n\to\infty.
\]
\end{proposition}
\begin{proof}
By Lemma~\ref{lem:bochner-kernel-mean}, the feature map
\[
\Phi_k(x):=k(\cdot,x)
\]
is Bochner integrable with respect to $P$, and
\[
f_A=\int_{\cX}\Phi_k(x)\,dP(x).
\]

Let $X_1,X_2,\dots$ be independent random variables with
distribution $P$. The Bochner expectation of $\Phi_k(X_1)$ is
\[
\mathbb{E}[\Phi_k(X_1)]
:=
\int_{\cX}\Phi_k(x)\,dP(x)
=
f_A.
\]
Set
\[
Y_i:=\Phi_k(X_i)-f_A,
\qquad i\geq1.
\]
Then $Y_1,Y_2,\dots$ are independent and identically distributed,
with
\[
\mathbb{E}[Y_1]=0
\]
and
\[
\mathbb{E}\|Y_1\|_{\cH_k}
\leq
\mathbb{E}\|\Phi_k(X_1)\|_{\cH_k}
+
\|f_A\|_{\cH_k}
<\infty.
\]
The strong law of large numbers for Banach-space-valued random
variables
\citep[Corollary~7.10, p.~189]{LedTal91}
therefore gives
\[
\left\|
\frac{1}{n}\sum_{i=1}^nY_i
\right\|_{\cH_k}
\longrightarrow0
\qquad\text{almost surely as }n\to\infty.
\]
Equivalently,
\[
\left\|
f_A-\frac{1}{n}\sum_{i=1}^n\Phi_k(X_i)
\right\|_{\cH_k}
\longrightarrow0
\qquad\text{almost surely as }n\to\infty.
\]
Choose $\omega$ in the probability-one event on which this
convergence holds, and set
\[
x_i:=X_i(\omega),
\qquad i\geq1.
\]
Then
\[
\left\|
f_A-\frac{1}{n}\sum_{i=1}^n k(\cdot,x_i)
\right\|_{\cH_k}
\longrightarrow0
\qquad\text{as }n\to\infty.
\]
and the result follows.
\end{proof}

\section{Proof of Theorem \ref{theo:gen-Driscol-theorem}}

\label{sec:Proof-Driscol}

To prove Theorem \ref{theo:gen-Driscol-theorem}, we need a few preliminaries.
We first introduce the {\em dominance operator}, the existence of which is shown by \citet[Theorem 1.1]{LukBed01}.
\begin{theorem}[Dominance operator] \label{theo:dominance}
Let $k$ and $r$ be positive definite kernels on a set $\cX$ with respective RKHSs  $\cH_k$ and $\cH_r$  satisfying $\cH_k \subset \cH_r$.
Then the inclusion operator  $I_{kr}: \cH_k \to \cH_r$ (i.e., $I_{kr}g := g$ for $g \in \cH_k$) is continuous.
Moreover, there exists a unique linear operator $L: \cH_r \to \cH_k$, called {\em dominance operator}, such that
\begin{equation} \label{eq:dominance-op}
 \left< Lf, g \right>_{\cH_k} = \left< f, g \right>_{\cH_r}, \quad \forall f \in \cH_r,\ \forall g \in \cH_k \subset \cH_r.
\end{equation}
Furthermore, $I_{kr}L: \cH_r \to \cH_r$ is bounded, positive, and self-adjoint.
\end{theorem}

The dominance operator $L$ maps the canonical feature map in $\cH_r$ to that in $\cH_k$.
Indeed, \eqref{eq:dominance-op} with $f = r(\cdot, x)$ implies that, for all $g \in \cH_k \subset \cH_r$,
$$
\left< L r(\cdot, x), g \right>_{\cH_k} = \left< r(\cdot,x), g \right>_{\cH_r} = g(x) .
$$
Therefore $L r(\cdot, x)$ satisfies the reproducing property of $\cH_k$, and thus $$Lr(\cdot,x) = k(\cdot,x).
$$

As pointed out by \citet[Section~2]{Ste17},
\footnote{In this appendix, we cite the 2017 arXiv version (v3)
because it contains the auxiliary results and proofs used below;
elsewhere, we cite the 2019 journal version
\citep{steinwart2019convergence} for the published results.}
the dominance operator $L$ is nothing but the adjoint $I_{kr}^*$
of the inclusion operator $I_{kr}$, as summarized in the following
lemma.

\begin{lemma} \label{lemma:dominance-op-inclusion}
Consider the same notation as in Theorem \ref{theo:dominance}.
Then we have $L = I_{kr}^*$.
\end{lemma}
\begin{proof}
For all $f \in \cH_r$ and $g \in \cH_k \subset \cH_r$ we have
$$
\left< g, f \right>_{\cH_r} = \left<I_{kr}g, f \right>_{\cH_r} = \left<g, I_{kr}^* f \right>_{\cH_k}.
$$
Therefore $I_{kr}^*$ satisfies the defining property \eqref{eq:dominance-op} of the dominance operator.
Since the dominance operator is unique by Theorem \ref{theo:dominance}, we have $L = I_{kr}^*$.
\end{proof}

Next, we define the {\em nuclear dominance} as follows.

\begin{definition}[Nuclear dominance]
Consider the same notation as in Theorem \ref{theo:dominance}.
We say that the nuclear dominance holds and write  $r \gg k$ if the operator $I_{kr}L :\cH_r \to \cH_r$ is nuclear (or of trace class), i.e.,
$$
{\rm Tr}(I_{kr}L) = \sum_{i \in I} \left<I_{kr}L \psi_i, \psi_i \right>_{\cH_r} < \infty,
$$
where $(\psi_i)_{i \in I} \subset \cH_r$ is an ONB of $\cH_r$.
\end{definition}

The following result shows that the nuclear dominance is equivalent to the inclusion operator $I_{kr}$ being Hilbert--Schmidt.
The result is essentially given by the proof of \citet[Lemma 7.4, equivalence of (i) and (iii)]{Ste17}.
\begin{lemma} \label{lemma:equiv-nuclear-HS}
Under the same notation as in Theorem \ref{theo:dominance}, assume $\cH_k \subset \cH_r$.
Then the following statements are equivalent:
\begin{enumerate}
\item The nuclear dominance holds: $r \gg k$, i.e.,  $I_{kr}L:\cH_r \to \cH_r$ is nuclear.
\item The inclusion operator $I_{kr}: \cH_k \to \cH_r$ is Hilbert--Schmidt.
\end{enumerate}
\end{lemma}
\begin{proof}

From Lemma \ref{lemma:dominance-op-inclusion}, we have
\begin{align*}
   & {\rm Tr}(I_{kr}L) = {\rm Tr}(I_{kr}I_{kr}^*) := \sum_{i \in I} \left<I_{kr} I_{kr}^* \psi_i, \psi_i \right>_{\cH_r} \\
   & = \sum_{i \in I}  \left<I_{kr}^* \psi_i, I_{kr}^*\psi_i \right>_{\cH_k}  =:   \| I_{kr}^* \|_{\rm HS}^2,
\end{align*}
where $\| \cdot \|_{\rm HS}$ denotes the Hilbert--Schmidt norm and ${\rm Tr}(\cdot)$ the trace, and $(\psi_i)_{i \in I} \subset \cH_r$ is an ONB of $\cH_r$.
Since we have $\| I_{kr}^* \|_{\rm HS} = \| I_{kr} \|_{\rm HS}$ (see, e.g., \citealp[p.~506]{Steinwart2008}), the assertion immediately follows.
\end{proof}

Theorem \ref{theo:gen-Driscol-theorem}  follows from  \citet[Theorem 7.5]{LukBed01}, which shows that nuclear dominance is the necessary and sufficient condition for Driscoll's zero-one law, and from Lemma \ref{lemma:equiv-nuclear-HS}.

\section{Proof of Corollary \ref{coro:gauss-sample-path}} \label{sec:proof-gp-sample-path-se}
For Banach spaces $A$ and $B$, we denote by $A \hrarrow B$ that $A \subset B$ and that the inclusion is continuous.
We first need to define the notion of {\em interpolation spaces}; for details, see, e.g., \citet[Section 7.6]{AdaFou03}, \citet[Section 5.6]{Steinwart2008}, \citet[Section 4.5]{CuckerZhou2007} and references therein.
\begin{definition}
Let $E$ and $F$ be Banach spaces such that $F \hrarrow E$.
Let $K: E \times \Re_+  \to \Re_+$ be the $K$-functional defined by
$$
K(x,t) := K(x,t,E,F) :=  \inf_{y \in F} \left( \| x - y \|_E + t \| y \|_F \right), \quad x \in E,\ t > 0.
$$
Then for $0 < \theta < 1$, the interpolation space $[E, F]_{\theta,2}$ is a Banach space defined by
$$
[E, F]_{\theta,2} := \left\{ x \in E:\ \| x \|_{[E, F]_{\theta,2}} < \infty \right\},
$$
where the norm is defined by
$$
\| x \|_{[E, F]_{\theta,2}}^2 := \int_0^\infty \left( t^{-\theta} K(x,t) \right)^2 \frac{dt}{t}.
$$
\end{definition}
We will need the following lemma.
\begin{lemma} \label{lemma:interp-space-embed}
Let $E$, $F$ and $G$ be Banach spaces such that $G \hrarrow F \hrarrow E$.
Then we have $[E,G]_{\theta,2} \hrarrow [E,F]_{\theta,2}$ for all $0 < \theta < 1$.
\end{lemma}
\begin{proof}
First note that, since $G \hrarrow F$, there exists a constant $c > 0$ such that $ \| z \|_{F} \leq c \| z \|_G$ holds for all $z \in G$.
Therefore, for all $x \in E$ and $t > 0$, we have
\begin{eqnarray*}
K(x, t,E,F)
&=&  \inf_{y \in F} \left( \| x - y \|_E +  t \| y \|_F \right) \\
&\leq& \inf_{z \in G} \left( \| x - z \|_E +  t \| z \|_F \right) \\
&\leq& \inf_{z \in G} \left( \| x - z \|_E +  c t \| z \|_G \right) = K(x,ct,E,G).
\end{eqnarray*}
Thus, we have
\begin{eqnarray*}
\| x \|_{[E, F]_{\theta,2}}^2 &=&  \int_0^\infty \left( t^{-\theta} K(x,t,E,F) \right)^2 \frac{dt}{t} \\
&\leq&  \int_0^\infty \left( t^{-\theta} K(x,ct,E,G) \right)^2 \frac{dt}{t} \\
&=& c^{2\theta} \int_0^\infty \left( s^{-\theta} K(x,s,E,G) \right)^2 \frac{ds}{s} \quad (s := c t) \\
&=& c^{2\theta} \| x \|_{[E, G]_{\theta,2}}^2, \quad \quad x \in [E, G]_{\theta,2}.
\end{eqnarray*}
which implies the assertion.
\end{proof}
We are now ready to prove Corollary \ref{coro:gauss-sample-path}.
\begin{proof}
Fix $\theta \in (0,1)$.
First we show that $\sum_{i = 1}^\infty \lambda_i^\theta \phi_i^2(x) < \infty$ holds for all $x \in \cX$, which implies that the power of RKHS $\cH_{k_\gamma}^\theta$ is well defined.
To this end, by \citet[Theorem 2.5]{Ste17}, it is sufficient to show that
\begin{equation} \label{eq:interp-sup}
[L_2(\nu), \cH_{k_\gamma} ]_{\theta,2} \hrarrow L_\infty(\nu),
\end{equation}
where $L_2(\nu)$ and $L_\infty(\nu)$ are to be understood as quotient spaces with respect to $\nu$, and $\cH_{k_\gamma}$ as the embedding in $L_2(\nu)$; see \citet[Section 2]{Ste17} for precise definition.

Let $m \in \mathbb{N}$ be such that $\theta m > d /2$, and $W_2^m(\cX)$ be the Sobolev space of order $m$ on $\cX$.
By \citet[Theorem 4.48]{Steinwart2008}, we have $\cH_{k_\gamma} \hrarrow W_2^m(\cX)$.
Therefore Lemma \ref{lemma:interp-space-embed} implies that  $[ L_2(\nu), \cH_{k_\gamma}]_{\theta, 2} \hrarrow [ L_2(\nu), W_2^m(\cX)]_{\theta, 2}$.
Note that, since $\nu$ is the Lebesgue measure, $[ L_2(\nu), W_2^m(\cX)]_{\theta, 2}$ is the Besov space $B_{22}^{\theta m}(\cX)$ of order $\theta m$ \citep[Section 7.32]{AdaFou03}.
Since $\cX$ is a bounded Lipschitz domain and $\theta m > d/2$, we have $B_{22}^{\theta m}(\cX) \hrarrow L_\infty(\nu)$ \citep[Proposition 4.6]{Tri06}; see also \citet[Theorem 7.34]{AdaFou03}.
Combining these embeddings implies \eqref{eq:interp-sup}.

We next show that $\sum_{i = 1}^\infty \lambda_i^{1-\theta} < \infty$, which implies the assertion by Theorem \ref{theo:gp-path-power-RKHS}.
To this end, for $i \in \mathbb{N}$,  define the $i$-th (dyadic) entropy number of the embedding ${\rm id}: \cH_{k_\gamma} \to L_2(\nu)$ by
\begin{align*}
 & \varepsilon_i({\rm id}: \cH_{k_\gamma} \to L_2(\nu))  \\
 := & \inf \left\{ \varepsilon > 0: \exists\ (h_j)^{2^{i-1}}_{j=1} \subset L_2(\nu)\ {\rm s.t.}\ B_{\cH_{k_\gamma}} \subset \bigcup_{j=1}^{2^{i-1}} (h_j + \varepsilon B_{L_2(\nu)})  \right\},
\end{align*}
where $B_{\cH_{k_\gamma}}$ and $B_{L_2(\nu)}$ denote the centered unit balls in $\cH_{k_\gamma}$ and $L_2(\nu)$, respectively.
Note that since $\cX$ is bounded, there exists a ball of radius $r \geq \gamma$ that contains $\cX$.
From \citet[Theorem 12]{MeiSte17}, for all $p \in (0,1)$, we have
$$
\varepsilon_i({\rm id}: \cH_{k_\gamma}(\cX) \to L_2(\nu)) \leq c_{p,d,r,\gamma}\ i^{-1/2p}, \quad i \geq 1.
$$
$c_{p,d,r,\gamma} > 0$ is a constant depending only on $p$, $d$, $r$ and $\gamma$.
Using this inequality, the $i$-th largest eigenvalue $\lambda_i$ is upper-bounded as
$$
\lambda_i \leq 4 \varepsilon_i^2 ({\rm id}: \cH_{k_\gamma} \to L_2(\nu)) \leq 4 c_{p,d,r,\gamma}^2 i^{-1/p}, \quad i \geq 1,
$$
where the first inequality follows from \citet[Lemma 2.6 Eq.~23]{Ste17}; this lemma is applicable since we have $\int_\cX k_\gamma(x,x)d\nu(x) < \infty$ and thus the embedding ${\rm id}: \cH_{k_\gamma} \to L_2(\nu)$ is compact \citep[Lemma 2.3]{SteSco12}.
Therefore we have
$$
\sum_{i=1}^\infty \lambda_{i}^{1-\theta} < (4 c_{p,d,r,\gamma}^2)^{1-\theta} \sum_{i=1}^\infty  i^{-(1-\theta)/p}.
$$
The right-hand side is bounded if we take $p \in (0,1)$ such that $1-\theta > p$.
This implies $\sum_{i=1}^\infty \lambda_{i}^{1-\theta} < \infty$.
\end{proof}

\section{Proof of Theorem~\ref{theo:upper-bound-post-var} }

\label{sec:proof-post-var-contraction}

We first need the following fundamental result on {\em local polynomial reproduction}, which is a special case of \citet[Lemma 5.1]{WuSch93}; see also \citet[Chapters 3 and 11]{Wen05} for other versions.

\begin{lemma} \label{lemma:local-polynomial-rep}
For any $\rho > 0$ and $q \in \mathbb{N}$, there exist constants $h_0 > 0$, $C_1 > 0$, and  $C_2 > 0$ such that the following holds.
For any $n \in \mathbb{N}$,  $X^n = \{x_1, \dots, x_n \} \subset \mathbb{R}^d$, and $x \in \mathbb{R}^d$ such that $h_{\rho,X^n}(x) \leq h_0$,  there exist $u_1, \dots, u_n \in \mathbb{R}$ such that
\begin{enumerate}
    \item For any polynomial $p$ of degree less than $q$, we have
    $p(x) = \sum_{i=1}^n u_i p(x_i)$.
    \item $\sum_{i=1}^n | u_i | \leq C_1$,
    \item For all $i=1,\dots,n$ with $u_i \not= 0$, we have $\| x - x_i \| \leq C_2 h_{\rho,X^n}(x)$.
\end{enumerate}

\end{lemma}
\begin{proof}
Set the multi-indices $\mu$ in \citet[Lemma 5.1]{WuSch93} as $\mu = 0$.
\end{proof}

Intuitively, Lemma~\ref{lemma:local-polynomial-rep} shows that the value $p(x)$ of any polynomial $p$ of degree less than $q$ at $x$ can be reproduced by a ``local weighted average'' of polynomial values $p(x_i)$ around $x$.
Note that the constants $C_1$ and $C_2$ are independent of the data size $n$.

\begin{remark} \label{rem:local-poly-repro}
Under the same condition of Lemma~\ref{lemma:local-polynomial-rep}, the local polynomial reproduction property also holds when $p$ is any {\em complex-valued} polynomial of degree less than $q$.
This is because in this case $p$ can be written as
$p(x) = p_a(x) + \iu p_b(x)$, where $p_a$ and $p_b$ are real-valued polynomials of degree less than $q$ and $\iu = \sqrt{-1}$. Thus
\begin{align*}
p(x) &= p_a(x) + \iu p_b(x) = \sum_{i=1}^n u_i p_a(x_i) + \iu \sum_{i=1}^n u_i p_b(x_i) \\
& =  \sum_{i=1}^n u_i \left(  p_a(x_i) + \iu p_b(x_i) \right)
= \sum_{i=1}^n u_i p(x_i).
\end{align*}

\end{remark}

Now, we are ready to prove  Theorem~\ref{theo:upper-bound-post-var}.
\begin{proof}
For simplicity, we write $h_n := h_{\rho, X^n}(x)$ below.
Let $q \in \mathbb{N}$ be the smallest integer such that
\begin{equation} \label{eq:ineq-q-5433}
     2q  > 2s - d.
\end{equation}
For this $q$, and the given $\rho > 0$, let $h_0, C_1, C_2 > 0$ and $u_1, \dots, u_n \in \mathbb{R}$ be those in Lemma~\ref{lemma:local-polynomial-rep}, so that $h_n \leq h_0$ and $u_1, \dots, u_n$ satisfy the properties in Lemma~\ref{lemma:local-polynomial-rep}.

As the posterior variance is the projection residual of $k(\cdot,x)$ onto the linear span of $k(\cdot,x_1), \dots, k(\cdot, x_n)$ (see \eqref{eq:GP-post-cov-residual}), we have
\begin{align*}
& k(x, x) - {\bm k}_n(x)^\top {\bm K}_n^{-1} {\bm k}_n(x) \\
 & = \min_{ c_1, \dots, c_n \in \mathbb{R}} \left\| k(\cdot,x) - \sum_{i=1}^n c_i k(\cdot,x_i) \right\|_{\cH_k}^2
 \leq \left\| k(\cdot,x) - \sum_{i=1}^n u_i k(\cdot,x_i) \right\|_{\cH_k}^2.
\end{align*}
Thus, we focus on bounding this upper bound.

Recall from Theorem~\ref{theo:RKHS-shift-invariant} that the RKHS norm of any $g \in \cH_k$ can be written as
\begin{equation*}
  \| g \|_{\cH_k}^2 = C_d \int  |\cF[g](\omega)|^2  \cF[\Phi](\omega)^{-1} d\omega,
\end{equation*}
where $C_d > 0$ is a constant.
We consider the case $g := k(\cdot,x) - \sum_{i=1}^n u_i k(\cdot,x_i)$.
Since for any $v \in \mathbb{R}^d$ it holds that
$
\cF[k(\cdot,v)](\omega) = \cF[\Phi(\cdot - v)](\omega)  =  e^{ - \iu v^\top \omega } \mathcal{F}[\Phi] (\omega),
$
we have
\begin{align*}
   &  \cF\left[ k(\cdot,x) - \sum_{i=1}^n u_i k(\cdot,x_i)\right](\omega) \\
   & =  \cF[\Phi](\omega) \left( e^{ - \iu x^\top \omega }  - \sum_{i=1}^n u_i e^{ - \iu x_i^\top \omega } \right) \quad (\omega \in \mathbb{R}^d).
\end{align*}
Therefore,
\begin{align}
& \left\| k(\cdot,x) - \sum_{i=1}^n u_i k(\cdot,x_i) \right\|_{\cH_k}^2  \nonumber  \\
& = C_d \int  \left|  e^{ - \iu x^\top \omega }  - \sum_{i=1}^n u_i e^{ - \iu x_i^\top \omega }
\right|^2  \cF[\Phi](\omega) d\omega \nonumber \\
& = C_d \int  \left| 1 - \sum_{i=1}^n u_i e^{ - \iu (x_i - x)^\top \omega }
\right|^2  \cF[\Phi](\omega) d\omega \nonumber \\
& = C_d \int_{\| \omega \| \leq h_n^{-1}}  \left| 1 - \sum_{i=1}^n u_i e^{ - \iu (x_i - x)^\top \omega }   \right|^2  \cF[\Phi](\omega) d\omega \label{eq:upper-post-var-5515} \\
& +  C_d \int_{\| \omega \| > h_n^{-1} }  \left| 1 - \sum_{i=1}^n u_i e^{ - \iu (x_i - x)^\top \omega }
\right|^2  \cF[\Phi](\omega) d\omega. \label{eq:upper-post-var-5517}
\end{align}

To bound \eqref{eq:upper-post-var-5515}, let us rewrite the exponential function using the Taylor expansion~as
\begin{align}
 & e^z = p_{q-1}(z) + z^q r_q(z), ~~  \text{where} ~~  p_{q-1}(z) :=  1 + \sum_{\ell = 1}^{q-1} \frac{z^\ell}{ \ell! }, \nonumber \\
 & r_q(z) := \sum_{\ell=0}^\infty \frac{z^\ell}{(q+\ell)!}    ~~ \text{so that}~~  |r_q(z)| \leq e^{|z|}  \quad (z \in \mathbb{C}).  \label{eq:exp-residual-5515}
\end{align}
This implies that
\begin{align*}
& \sum_{i=1}^n u_i e^{ - \iu (x_i - x)^\top \omega } \\
 & = \sum_{i=1}^n u_i p_{q-1}( - \iu (x_i - x)^\top \omega )
  + \sum_{i=1}^n u_i (  - \iu (x_i - x)^\top \omega )^q r_q(  - \iu (x_i - x)^\top \omega )\\
& =  p_{q-1}(0) + \sum_{i=1}^n u_i (  - \iu (x_i - x)^\top \omega )^q r_q(  - \iu (x_i - x)^\top \omega ) \\
& = 1 + \sum_{i=1}^n u_i (  - \iu (x_i - x)^\top \omega )^q r_q(  - \iu (x_i - x)^\top \omega ),
 \end{align*}
where the second identity follows from Lemma~\ref{lemma:local-polynomial-rep} (see also Remark~\ref{rem:local-poly-repro}).
Therefore,
\begin{align}
 \left| 1 - \sum_{i=1}^n u_i e^{ - \iu (x_i - x)^\top \omega }   \right|
& =  \left| \sum_{i=1}^n u_i (  - \iu (x_i - x)^\top \omega )^q r_q(  - \iu (x_i - x)^\top \omega )  \right| \nonumber \\
& \leq \sum_{i=1}^n |u_i| \left|  (x_i - x)^\top \omega  \right|^q \left| r_q(  - \iu (x_i - x)^\top \omega )  \right| \nonumber \\
& \leq \sum_{i=1}^n |u_i|  \| x_i - x \|^q \| \omega \|^q e^{ \| x_i - x\| \| \omega \| } \nonumber \\
& \leq \sum_{i=1}^n |u_i| C_2^q h_n^q \| \omega \|^q e^{ C_2 h_n \| \omega \| } \nonumber \\
& \leq C_1 C_2^q h_n^q \| \omega \|^q e^{ C_2 h_n \| \omega \| }   \label{eq:upper-fourier-5506},
\end{align}
where we used \eqref{eq:exp-residual-5515} in the second inequality and Lemma~\ref{lemma:local-polynomial-rep} in the third and fourth inequalities.

Now, using \eqref{eq:upper-fourier-5506} and \eqref{eq:kernel-Fourier-5410}, the term \eqref{eq:upper-post-var-5515} can be upper-bounded as
\begin{align}
\eqref{eq:upper-post-var-5515} & \leq C_3 h_{\rho, X^n}^{2q}(x)  \int_{\| \omega \| \leq h_n^{-1}}  \| \omega \|^{2q} e^{ 2 C_2 h_{\rho, X^n}(x) \| \omega \| }  \cF[\Phi](\omega) d\omega  \nonumber  \\
& \leq C_4 h_n^{2q} \int_{\| \omega \| \leq h_n^{-1}} \| \omega \|^{2q} (1 + \| \omega \|)^{-2s} d\omega, \label{eq:upper-bound-5525}
\end{align}
where $C_3, C_4 > 0$ are constants.
Since the integral in the last expression depends only on the norm $\| \omega \|$, it can be written as
\begin{align*}
& \int_{\| \omega \| \leq h_n^{-1}} \| \omega \|^{2q} (1 + \| \omega \|)^{-2s} d\omega
 = C_5 \int_{ 0 }^{ h_n^{-1}} r^{d-1} r^{2q} (1 + r)^{-2s} dr  \\
& = C_5 \int_0^{h_0^{-1}}   r^{d-1 + 2q}   (1 + r)^{-2s} dr
+ C_5 \int_{h_0^{-1}}^{h_n^{-1}}   r^{d-1 + 2q} (1 + r)^{-2s} dr
\end{align*}
\begin{align*}
& \leq C_6 + C_7 \int_{h_0^{-1}}^{h_n^{-1}}   r^{d-1 + 2q - 2s} dr
= C_6 + C_7 \frac{1}{d + 2q - 2s} \left[  r^{d + 2q - 2s}  \right]_{h_0^{-1}}^{ h_n^{-1} }  \\
& \stackrel{(*)}{=} C_6 + C_8 \left( h_n^{- (d + 2q - 2s)  } - h_0^{- (d + 2q - 2s)  }    \right) \\
& = C_9 + C_{10} h_n^{- (d + 2q - 2s)},
\end{align*}
where $C_5, C_6, C_7, C_8, C_9, C_{10} > 0$ are constants, and $(*)$ uses $d + 2q - 2s > 0$, which follows from \eqref{eq:ineq-q-5433}.
Using this to upper-bound \eqref{eq:upper-bound-5525}, we obtain
\begin{align} \label{eq:upper-5542}
\eqref{eq:upper-post-var-5515}
& \leq C_{11} h_n^{2q} + C_{12} h_n^{2s -d} \leq C_{13} h_n^{2s -d},
\end{align}
where $C_{11}, C_{12}, C_{13} > 0$ are constants and the last inequality follows from \eqref{eq:ineq-q-5433} and $h_n \leq 1$.

On the other hand, to bound \eqref{eq:upper-post-var-5517}, note that
\begin{align*}
& \left| 1 - \sum_{i=1}^n u_i e^{ - \iu (x_i - x)^\top \omega }   \right|
\leq 1 + \sum_{i=1}^n |u_i| \leq 1 + C_1.
\end{align*}
Using this and \eqref{eq:kernel-Fourier-5410}, we can bound \eqref{eq:upper-post-var-5517} as
\begin{align}
 \eqref{eq:upper-post-var-5517} & =   C_d \int_{\| \omega \| > h_n^{-1} }  \left| 1 - \sum_{i=1}^n u_i e^{ - \iu (x_i - x)^\top \omega }
\right|^2  \cF[\Phi](\omega) d\omega \nonumber \\
& \leq D_1 \int_{\| \omega \| > h_n^{-1} }   (1 + \| \omega \|)^{-2s}  d\omega
= D_2 \int_{ h_n^{-1} }^\infty  r^{d-1} (1 + r )^{-2s}  dr \nonumber \\
& \leq D_3  \int_{ h_n^{-1} }^\infty  r^{d-1 - 2s}  dr
= D_3 \frac{1}{d - 2s} \left[ r^{d - 2s} \right]_{h_n^{-1}}^\infty
= D_4 h_n^{2s - d}, \label{eq:upper-5562}
\end{align}
where $D_1, D_2, D_3, D_4 > 0$ are constants and we used $2s > d$ in the last identity.

Lastly, using \eqref{eq:upper-5542} and \eqref{eq:upper-5562} to bound \eqref{eq:upper-post-var-5515} and \eqref{eq:upper-post-var-5517}, respectively, we obtain
\begin{align*}
&   \left\| k(\cdot,x) - \sum_{i=1}^n u_i k(\cdot,x_i) \right\|_{\cH_k}^2   \leq (C_{13} + D_4) h_n^{2s - d},
\end{align*}
which completes the proof.
\end{proof}

\section{Lemmas for Section~\ref{sec:KRR-as-KI}}

\label{sec:proof-KRR-as-interpolation}

We first need the following lemma.
\begin{lemma} \label{lemma:min-min}
Let $k$ and $\ell$ be kernels on $\cX$, and  $\cH_{k}$ and $\cH_\ell$ be their respective RKHSs.
Let $(x_i, y_i)_{i=1}^n \subset \cX \times \mathbb{R}$ be  such that there exist some $f \in \cH_k$ and $h \in \cH_\ell$ such that
$$
{\bm f}_n + {\bm h}_n = {\bm y}_n,$$
where ${\bm f}_n := (f(x_1), \dots, f(x_n))^\top  \in \mathbb{R}^n$, ${\bm h}_n := (h(x_1), \dots, h(x_n))^\top  \in \mathbb{R}^n$, and ${\bm y}_n = (y_1,\dots,y_n)^\top \in \mathbb{R}^n$.
Then we have
\begin{align}  \label{eq:lemma-int-reg}
   & \min_{ \substack{ f \in \cH_k, h \in \cH_\ell, \\ {\bm f}_n + {\bm h}_n = {\bm y}_n}  } \| f \|_{\cH_k}^2 + \| h \|_{\cH_\ell }^2 \\
  = & \min_{ \substack{f' \in \cH_k,  h' \in \cH_\ell, \\ {\bm f}'_n + {\bm h}'_n = {\bm y}_n} } \min_{\substack{ f \in \cH_k, h \in \cH_\ell, \\ f + h = f' + h'}} \| f \|_{\cH_k}^2 + \| h \|_{\cH_\ell }^2, \label{eq:8106}
\end{align}
where ${\bm f}'_n := (f'(x_1), \dots, f'(x_n))^\top, {\bm h}'_n := (h'(x_1), \dots, h'(x_n))^\top \in \mathbb{R}^n$.
\end{lemma}

\begin{proof}

For any $f' \in \cH_k$ and $h' \in \cH_\ell$ such that ${\bm f}'_n + {\bm h}'_n = {\bm y}_n$,
$$
\min_{\substack{ f \in \cH_k, h \in \cH_\ell, \\ f + h = f' + h'}} \| f \|_{\cH_k}^2 + \| h \|_{\cH_\ell}^2 \geq    \min_{ \substack{ f \in \cH_k, h \in \cH_\ell, \\ {\bm f}_n + {\bm h}_n = {\bm y}_n}  } \| f \|_{\cH_k}^2 + \| h \|_{\cH_\ell }^2,
$$
where the inequality holds because $f + h = f'+h'$ implies ${\bm f}_n+ {\bm h}_n = {\bm y}_n$.
This implies $\eqref{eq:8106} \geq \eqref{eq:lemma-int-reg}$.

On the other hand, for any $f' \in \cH_k$ and $h' \in \cH_\ell$, we have
$$
\min_{\substack{ f \in \cH_k, h \in \cH_\ell \\ f + h = f' + h'}} \| f \|_{\cH_k}^2 + \| h \|_{\cH_\ell}^2 \leq  \| f' \|_{\cH_k}^2 + \| h' \|_{\cH_\ell}^2,
$$
which implies $\eqref{eq:8106} \leq \eqref{eq:lemma-int-reg}$
\end{proof}

\begin{lemma} \label{lemma:sum-kernel-min-norm}
Let $k$, $\ell$ and $v := k+\ell$ be kernels on $\cX$, and  $\cH_{k}$, $\cH_\ell$ and $\cH_v$ be their respective RKHSs.
Let $(x_i, y_i)_{i=1}^n \subset \cX \times \mathbb{R}$ be such that there exist some $f \in \cH_k$, $h \in \cH_\ell$ and $g \in \cH_v$ such that
\begin{equation} \label{eq:cond-lemma-8146}
    {\bm g}_n = {\bm f}_n + {\bm h}_n = {\bm y}_n,
\end{equation}
where ${\bm f}_n := (f(x_1), \dots, f(x_n))^\top  \in \mathbb{R}^n$, ${\bm h}_n := (h(x_1), \dots, h(x_n))^\top  \in \mathbb{R}^n$, ${\bm g}_n := (g(x_1), \dots, g(x_n))^\top  \in \mathbb{R}^n$, and ${\bm y}_n = (y_1,\dots,y_n)^\top \in \mathbb{R}^n$.
Then we have
    \begin{align*}
\min_{\substack{g \in \cH_{v},\\  {\bm g}_n = {\bm y}_n}} \|  g \|_{\cH_{v}}^2
&=   \min_{\substack{f \in \cH_k, h \in \cH_{\ell}, \\  {\bm f}_n + {\bm h}_n = {\bm y}_n}} \| f \|_{\cH_k}^2 + \| h \|_{\cH_{\ell}}^2,
\end{align*}
\end{lemma}

\begin{proof}

Since $v = k +  \ell$, the RKHS $\cH_v$ can be written as (see, e.g., \citealp[Section~6]{Aronszajn1950})
\begin{align*}
\cH_v & = \left\{g = f + h\ \mid \ f \in \cH_k,\ h \in \cH_\ell \right\},
\end{align*}
with its RKHS norm
$$
 \| g \|_{\cH_v}^2 = \min_{\substack{f \in \cH_k, h \in \cH_{\ell}, \\ g = f + h}} \| f \|_{\cH_k}^2 + \| h \|_{\cH_{\ell}}^2 .
$$
Then, we obtain
\begin{align*}
\min_{\substack{g \in \cH_{v},\\  {\bm g}_n = {\bm y}_n}} \|  g \|_{\cH_{v}}^2
&=  \min_{\substack{g \in \cH_{v},\\ {\bm g}_n = {\bm y}_n}}\ \min_{\substack{f \in \cH_k, h \in \cH_{\ell}, \\ g = f + h}} \| f \|_{\cH_k}^2 + \| h \|_{\cH_{\ell}}^2 \nonumber  \\
&= \min_{ \substack{f' \in \cH_k,  h' \in \cH_{\ell}, \\ {\bm f}'_n + {\bm h}'_n = {\bm y}_n}  }\ \min_{\substack{f \in \cH_k, h \in \cH_{\ell}, \\f'+h' = f + h}} \| f \|_{\cH_k}^2 + \| h \|_{\cH_{\ell}}^2, \nonumber \\
&=   \min_{\substack{f \in \cH_k, h \in \cH_{\ell}, \\  {\bm f}_n + {\bm h}_n = {\bm y}_n}} \| f \|_{\cH_k}^2 + \| h \|_{\cH_{\ell}}^2,
\end{align*}
where the last identity follows from Lemma \ref{lemma:min-min}.
\end{proof}

\subsection{Proof of Lemma~\ref{lemma:noise-represent}}

\label{sec:proof-noise-represent}

\begin{proof}

By \eqref{eq:RKHS-noise-6369}, $\hat{\eta}$ can be written as, for some $(b_x)_{x \in \cX} \subset \mathbb{R}$,
$$
\hat{\eta} = \sum_{x \in \cX} b_x \sigma^{-1} \ell_{\sigma^2}(\cdot, x) ~~\text{with}~~ \|  \hat{\eta} \|^2_{\cH_{\sigma^2}} =  \sum_{x \in \cX} b_x^2 < \infty.
$$
Suppose that there is $x \in \cX \backslash \{ x_1, \dots, x_n \}$ with $b_x \not= 0$ (we shall show contradiction).
Define
$$
\tilde{\eta} :=  \sum_{x \in \{ x_1, \dots, x_n \}} b_x \sigma^{-1} \ell_{\sigma^2}(\cdot, x).
$$
Then $\tilde{\eta
}(x_i) = \hat{\eta}(x_i)$ for all $i = 1,\dots, n$ from \eqref{eq:delta-kernel-sigma-6354}, and $\tilde{\eta}$ satisfies the constraints $f(x_i) + \tilde{\eta}(x_i) = y_i$ in \eqref{eq:regression-decomposition}.
On the other hand,
\begin{align*}
    \| \tilde{\eta} \|_{\cH_{\sigma^2 }}^2
    &= \sum_{x \in \{ x_1, \dots, x_n \}} b_x^2 \\
    &< \sum_{x \in \{ x_1, \dots, x_n \}} b_x^2 +  \sum_{x \in \cX \backslash \{ x_1, \dots, x_n \}} b_x^2
    = \| \hat{\eta} \|_{\cH_{\sigma^2 }}^2.
\end{align*}
Thus, the objective value~\eqref{eq:regression-decomposition} for $\tilde{\eta}$ is smaller than $\hat{\eta}$, which contradicts the optimality of $\hat{\eta}$.
Therefore, there is no  $x \in \cX \backslash \{ x_1, \dots, x_n \}$ with $b_x \not= 0$ and $\hat{\eta}$ must be given in the form~\eqref{eq:noise-func-form-6423}.
\end{proof}

\section{Proof of Proposition~\ref{prop:GPIC-kernel-expression}}

\label{sec:proof-GPIC-kernel-expression}

\begin{proof}
We can rewrite GPIC~\eqref{eq:GPIC} as
\begin{align}
  &  {\rm GPIC}_{k,\ell}(X,Y) \nonumber \\
=\ & \mathbb{E}_{F, G} \left[ \bE_{X,Y} \left[ \left( F(X) - \bE_{X'}[F(X')] \right) \left( G(Y) - \bE_{Y'}[G(Y')] \right) \right]^2 \right] \nonumber \\
=\ & \mathbb{E}_{F, G} \left[ \bE_{X,Y}[ F(X) G(Y) ] \bE_{X',Y'}[ F(X') G(Y') ] \right] \nonumber \\
& +  \mathbb{E}_{F, G} \left[ \bE_{X}[F(X)] \bE_{Y}[ G(Y) ] \bE_{X'} [F(X')] \bE_{Y'}[ G(Y') ] \right] \nonumber \\
& - 2  \mathbb{E}_{F, G} \left[ \bE_{X, Y}[F(X)G(Y) ] \bE_{X'} [F(X')] \bE_{Y'}[ G(Y') ] \right] \nonumber \\
=\ & \mathbb{E}_{F, G}  \bE_{X,Y} \bE_{X',Y'} [ F(X)  F(X') G(Y)  G(Y') ]   \nonumber \\
& +  \mathbb{E}_{F, G} \bE_{X} \bE_{X'} \bE_{Y}  \bE_{Y'} [F(X) F(X') G(Y) G(Y') ]  \nonumber \\
& - 2  \mathbb{E}_{F, G}   \bE_{X, Y} \bE_{X'} \bE_{Y'} [F(X) F(X') G(Y) G(Y') ]  \nonumber \\
=\ &   \bE_{X,Y} \bE_{X',Y'} [ k(X,X') \ell(Y,Y') ] \nonumber \\
& +  \bE_{X}  \bE_{X'}   \bE_{Y} \bE_{Y'}[ k(X,X') \ell(Y, Y') ] \nonumber \\
&  -  2 \bE_{X, Y}  \bE_{X'}  \bE_{Y'} [ k(X,X') \ell(Y,Y') ], \label{eq:GPIC-kernel-express-4278}
\end{align}
where the last identity follows from Fubini's theorem, which is applicable because of the assumptions $\bE_X[k(X,X)] < \infty$ and $\bE_Y[\ell(Y,Y)] < \infty$.
Let us demonstrate the applicability of Fubini's theorem for the first term in \eqref{eq:GPIC-kernel-express-4278}  by showing that
\begin{align*}
    \bE_{X,Y} \bE_{X',Y'} \mathbb{E}_{F, G}   \left[
 \left| F(X)  F(X') G(Y)  G(Y') \right| \right] < \infty.
\end{align*}
First, by the independence between $F$ and $G$ and the Cauchy--Schwarz inequality, we have
\begin{align*}
 & \mathbb{E}_{F, G}  \left[
 \left| F(X)  F(X') G(Y)  G(Y') \right| \right] \\
 & \leq \mathbb{E}_F \left[ \left| F(X)  F(X') \right| \right] \mathbb{E}_G  \left[
 \left|  G(Y)  G(Y') \right| \right]  \\
 & \leq \sqrt{\mathbb{E}_F   [F(X)^2] } \sqrt{\mathbb{E}_F [F(X')^2] } \sqrt{  \mathbb{E}_G [G(Y)^2] } \sqrt{ \mathbb{E}_G [G(Y')^2]  } \\
 & =  \sqrt{k(X,X) } \sqrt{ k(X',X') } \sqrt{  \ell(Y,Y) } \sqrt{ \ell(Y',Y') }.
\end{align*}
Therefore, again by the Cauchy--Schwarz, we have
\begin{align*}
  &  \bE_{X,Y} \bE_{X',Y'} \mathbb{E}_{F, G}   \left[
 \left| F(X)  F(X') G(Y)  G(Y') \right| \right] \\
 & \leq \bE_{X,Y} \left[ \sqrt{k(X,X) } \sqrt{  \ell(Y,Y) } \right] \bE_{X',Y'} \left[  \sqrt{ k(X',X') }  \sqrt{ \ell(Y',Y') } \right] \\
 &  \leq \sqrt{\bE_{X}[ k(X,X)] } \sqrt{ \bE_{Y} [ \ell(Y,Y) ] } \sqrt{\bE_{X'}[ k(X',X')] } \sqrt{ \bE_{Y'} [ \ell(Y',Y') ] }  \\
 & =   \bE_{X}[ k(X,X)]   \bE_{Y} [ \ell(Y,Y) ]  < \infty.
\end{align*}
One can also show the applicability of Fubini's theorem for the second and third terms in \eqref{eq:GPIC-kernel-express-4278} in the same manner.
\end{proof}

\backmatter

\bibliographystyle{cambridgeauthordate}
\bibliography{bibfile}

\printindex



\end{document}